  \documentclass[twoside,11pt]{article}  \usepackage[arxiv]{optional}

    \RequirePackage{amsthm,amsmath,amsfonts,amssymb}
    \RequirePackage[numbers]{natbib}
    \usepackage{graphicx,url}
    \usepackage[linesnumbered,ruled,algo2e]{algorithm2e}%
    
    \usepackage{etoolbox}
    \usepackage{xcolor}
    \usepackage{bold-extra}

    \usepackage{subfigure}

\opt{jmlr}{\usepackage{jmlr2e}}
\opt{arxiv}{\usepackage{arxiv-jmlr2e}}
\usepackage{hyperref}

\opt{jmlr}{

\newtheorem{proposition}{Proposition}

\newtheorem{remark}{Remark}

\newtheorem{corollary}{Corollary}

\newtheorem{definition}{Definition}

\newtheorem{conjecture}{Conjecture}

\newtheorem{axiom}{Axiom}

\newtheorem{assumption}{Assumption}
}

\hypersetup{
  colorlinks = true,
  urlcolor = blue, 
  linkcolor = blue,
  citecolor = blue,
  pdfborder={0 0 0},
  pagebackref=true
}

\usepackage{algpseudocode}

\usepackage{booktabs} %
\usepackage{caption}
\usepackage[usestackEOL]{stackengine}
\usepackage{enumitem}

\usepackage{etoc}

\usepackage{float}

\usepackage{mathrsfs}
\usepackage[capitalize]{cleveref}  

\usepackage{autonum}

\crefname{appendix}{App.}{Apps.}
\crefname{subsubsubappendix}{App.}{Apps.}
\crefname{supplement}{Supp.}{Supps.}
\crefname{subsupplement}{Supp.}{Supps.}
\crefname{subsubsupplement}{Supp.}{Supps.}
\crefname{subsubsubsupplement}{Supp.}{Supps.}
\crefname{equation}{}{}
\crefname{lemma}{Lem.}{Lems.}
\crefname{theorem}{Thm.}{Thms.}
\crefname{corollary}{Cor.}{Cors.}
\crefname{algorithm}{Alg.}{Algs.}
\crefname{section}{Sec.}{Secs.}
\crefname{table}{Tab.}{Tabs.}
\crefname{remark}{Rem.}{Rems.}
\crefname{definition}{Def.}{Defs.}
\crefname{proposition}{Prop.}{Props.}
\crefname{assumption}{Assump.}{Assumps.}
\crefname{myremark}{Rem.}{Rems.}
\crefname{mylemma}{Lem.}{Lems.}
\crefname{mydefinition}{Def.}{Defs.}
\crefname{myproposition}{Prop.}{Props.}
\crefname{mycorollary}{Cor.}{Cors.}
\crefname{enumi}{}{}
\crefname{name}{}{} %
\SetAlCapHSkip{0pt} %
\LinesNotNumbered %


\DeclareMathAlphabet{\mathbsf}{OT1}{bch}{bx}{ui}%
\DeclareMathAlphabet{\mathssf}{OT1}{bch}{m}{sl}%

\newcommand{\rminnew}[1][]{\rmin_{#1}} 

\newcommand{\wendmatern}[1][\matone]{\Phi_{#1, \mattwo}}
\newcommand{\matscale}[1][\matone]{\phi_{d, #1, \mattwo}}
\newcommand{\matrtconst}{A_{\matone,\mattwo, d}}

\newcommand{\subgauss}{\textsc{SubGauss}\xspace}
\newcommand{\subexp}{\textsc{SubExp}\xspace}
\newcommand{\heavytailnorho}{\textsc{HeavyTail}\xspace}
\newcommand{\heavytail}[1][\rho]{\ensuremath{\textsc{HeavyTail}(#1)}\xspace}
\newcommand{\subpoly}{\ensuremath{\textsc{SubPoly}}\xspace}
\newcommand{\htailparam}{\rho}

\newcommand{\order}{\mc{O}}

\newcommand{\bndcase}{\textsc{Compact}\xspace}

\DeclareMathOperator{\sech}{sech}
\newcommand{\mfk}{\mathfrak}

\def\set#1{\mc{#1}}
\newcommand{\ball}{\set{B}}
\newcommand{\X}{\set{X}}
\newcommand{\indicator}{\mbb I}

\newcommand{\dims}{d}

\newcommand{\Rd}{\R^{\dims}}

\newcommand{\snorm}[1]{\Vert #1 \Vert}
\newcommand{\sinfnorm}[1]{\snorm{#1}_\infty}

\newcommand{\m}{m}

\newcommand{\lone}{{L}^1}
\newcommand{\ltwo}{{L}^2}
\newcommand{\ltwoinf}{{L}^{2,\infty}}

\newcommand{\lp}{{L}^p}

\newcommand{\fun}{f}

\newcommand{\sumn}[1][i]{\sum_{#1=1}^n}
\newcommand{\sless}[1]{\stackrel{#1}{\leq}}
\newcommand{\sgrt}[1]{\stackrel{#1}{\geq}}
\newcommand{\seq}[1]{\stackrel{#1}{=}}

\newcommand{\x}{x}
\newcommand{\y}{y}
\newcommand{\z}{z}

\newcommand{\axi}[1][i]{\x_{#1}}

\newcommand{\dirac}{\mbi{\delta}}

\newcommand{\kernel}{{\mbf{k}_{\star}}}
\newcommand{\gkernel}{\mbf k}

\newcommand{\rkhs}{\mc{H}}
\newcommand{\krkhs}{\mc{H}_{\kernel}}

\newcommand{\hnorm}[1]{\Vert{#1}\Vert_{\rkhs}}

\newcommand{\knorm}[1]{\Vert{#1}\Vert_{\kernel}}

\newcommand{\dotrkhs}[2]{\angles{#1}_{#2}}
\newcommand{\doth}[1]{\dotrkhs{#1}{\rkhs}}
\newcommand{\dotk}[1]{\dotrkhs{#1}{\kernel}}

\DeclareMathOperator{\mmd}{MMD}

\renewcommand{\l}{\ell}
\def\j{j}

\newcommand{\eps}{\epsilon}

\def\Matern{Mat\'ern\xspace}

\newcommand{\vareps}{\varepsilon}

\newcommand{\Linf}{L^\infty}

\newcommand{\cover}{\mc C}

\newcommand{\pcover}[1][n]{\cover^{#1}}
\newcommand{\tail}[1][\gkernel]{\tau_{#1}}
\newcommand{\klip}[1][\gkernel]{L_{#1}}

\newcommand{\rttag}{\mrm{rt}}
\newcommand{\ksqrt}[1][\gkernel]{{#1}_{\rttag}}
\newcommand{\ksqrtdom}[1][\gkernel]{\widetilde{#1}_{\rttag}}
\newcommand{\kdom}[1][\gkernel]{\widetilde{#1}}

\newcommand{\kappasqrt}[1][\kappa]{{#1}_{\rttag}}
\newcommand{\hatkappasqrt}[1][\widehat{\kappa}]{{#1}_{\rttag}}
\newcommand{\hatkappa}{\widehat{\kappa}}
\newcommand{\kappasqrtdom}[1][\kappa]{\widetilde{#1}_{\rttag}}

\newcommand{\funtwo}{g}
\newcommand{\radius}{r}
\newcommand{\term}{T}
\newcommand{\ltwonorm}[1]{\norm{#1}_{L^2}}
\newcommand{\lpnorm}[1]{\norm{#1}_{L^p}}
\newcommand{\ltwoinfnorm}[1]{\norm{#1}_{L^{2,\infty}}}

\newcommand{\sgparam}[1][i]{\sigma_{#1}}
\newcommand{\vmax}[1][i]{\mathfrak{b}_{#1}}

\newcommand{\invec}[1][i]{f_{#1}}
\newcommand{\outvec}[1][i]{\psi_{#1}}
\newcommand{\eventnotag}{\mc{E}}
\newcommand{\event}[1][]{\eventnotag_{#1}}
\newcommand{\genvec}{u}
\newcommand{\wtil}[1]{\widetilde{#1}}

\renewcommand{\natural}{\mbb{N}}

\newcommand{\gaussparam}{\sigma}
\newcommand{\matone}{\nu}
\newcommand{\mattwo}{\gamma}
\newcommand{\splineparam}{\beta}

\newcommand{\bessel}[1][\matone]{{K}_{#1}}

\newcommand{\cnew}[1][i]{\mfk{a}_{#1}}

\newcommand{\esuccess}[1][n]{\mc{E}_{#1}}

\newcommand{\wvprod}[1][i]{\alpha_{#1}}

\newcommand{\epssmall}{\varepsilon}

\newcommand{\pseqxn}[1][n]{(\axi[i])_{i\geq 1}} %
\newcommand{\pseqxnn}[1][n]{(\axi[i])_{i=1}^n} %

\newcommand{\fourier}{\mc F}

\newcommand{\sechparam}{a}
\newcommand{\coreset}[1][j]{\mathcal{S}^{(#1)}}
\newcommand{\ksplitcoresets}{(\coreset[\m, \l])_{\l=1}^{2^{\m}}}

\SetKwInput{KwInit}{Init}

\newcommand{\brackets}[1]{\left[ #1 \right]}

\newcommand{\bigbrackets}[1]{\big[ #1 \big]}
\newcommand{\parenth}[1]{\left( #1 \right)}
\newcommand{\sparenth}[1]{( #1 )}
\newcommand{\bigparenth}[1]{\big( #1 \big)}

\newcommand{\braces}[1]{\left\{ #1 \right \}}
\newcommand{\sbraces}[1]{\{ #1  \}}
\newcommand{\abss}[1]{\left| #1 \right |}
\newcommand{\sabss}[1]{| #1 |}
\newcommand{\angles}[1]{\left\langle #1 \right \rangle}

\newcommand{\tp}{^\top}
\newcommand{\inv}{^{-1}}
\newcommand{\real}{\ensuremath{\mathbb{R}}}
\newcommand{\Exs}{\ensuremath{{\mathbb{E}}}}
\newcommand{\Prob}{\ensuremath{\mrm{Pr}}}
\renewcommand{\Pr}{\Prob}

\def\balign#1\ealign{\begin{align}#1\end{align}}
\def\baligns#1\ealigns{\begin{align*}#1\end{align*}}
\def\balignat#1\ealign{\begin{alignat}#1\end{alignat}}
\def\balignats#1\ealigns{\begin{alignat*}#1\end{alignat*}}
\def\bitemize#1\eitemize{\begin{itemize}#1\end{itemize}}
\def\benumerate#1\eenumerate{\begin{enumerate}#1\end{enumerate}}

\newenvironment{talign*}
 {\let\displaystyle\textstyle\csname align*\endcsname}
 {\endalign}
\newenvironment{talign}
 {\let\displaystyle\textstyle\csname align\endcsname}
 {\endalign}

\def\balignst#1\ealignst{\begin{talign*}#1\end{talign*}}
\def\balignt#1\ealignt{\begin{talign}#1\end{talign}}
\newcommand{\qtext}[1]{\quad\text{#1}\quad} 
\newcommand{\stext}[1]{\ \text{#1}\ } 
\let\originalleft\left
\let\originalright\right
\renewcommand{\left}{\mathopen{}\mathclose\bgroup\originalleft}
\renewcommand{\right}{\aftergroup\egroup\originalright}

\def\Holder{H\"older\xspace}

\def\Matern{Mat\'ern\xspace}

\def\tinycitep*#1{{\tiny\citep*{#1}}}
\def\tinycitealt*#1{{\tiny\citealt*{#1}}}
\def\tinycite*#1{{\tiny\cite*{#1}}}
\def\smallcitep*#1{{\scriptsize\citep*{#1}}}
\def\smallcitealt*#1{{\scriptsize\citealt*{#1}}}
\def\smallcite*#1{{\scriptsize\cite*{#1}}}

\def\mbi#1{\boldsymbol{#1}} %
\def\mbf#1{\mathbf{#1}}
\def\mbb#1{\mathbb{#1}}
\def\mc#1{\mathcal{#1}}
\def\mrm#1{\mathrm{#1}}
\def\trm#1{\textrm{#1}}
\def\tbf#1{\textbf{#1}}

\newcommand{\boldzero}{{\boldsymbol{0}}}

\def\textsum{{\textstyle\sum}} %
\def\reals{\mathbb{R}} %
\def\R{\mathbb{R}}
\def\Q{\mathbb{Q}}
\def\naturals{\mathbb{N}} %
\def\N{\mathbb{N}}
\def\complex{\mathbb{C}} %

\def\<{\left\langle} %
\def\>{\right\rangle}

\def\iff{\Leftrightarrow}
\def\implies{\quad\Longrightarrow\quad}

\def\defeq{\triangleq} %
\def\half{\frac{1}{2}}
\def\quarter{\frac{1}{4}}

\newcommand{\textfrac}[2]{{\textstyle\frac{#1}{#2}}}
\newcommand{\floor}[1]{\lfloor{#1}\rfloor}
\newcommand{\ceil}[1]{\lceil{#1}\rceil}
\def\norm#1{\left\|{#1}\right\|} %
\newcommand{\twonorm}[1]{\norm{#1}_2} %
\newcommand{\infnorm}[1]{\norm{#1}_{\infty}} %
\def\staticnorm#1{\|{#1}\|} %
\newcommand{\inner}[2]{\langle{#1},{#2}\rangle} %
\def\what#1{\widehat{#1}}

\def\indic#1{\indicator\left[{#1}\right]} %
\def\E{\mbb{E}} %

\def\P{\mbb{P}} %
\def\Pstar{\P} %

\def\Var{\mrm{Var}} %

\newcommand{\iid}{\textrm{i.i.d.}\@\xspace}

\providecommand{\esssup}{\mathop\mathrm{ess\,sup}} 
\providecommand{\argmin}{\mathop\mathrm{arg min}}

\ifdefined\nonewproofenvironments\else
\ifdefined\ispres\else
\newenvironment{proof-sketch}{\noindent\textbf{Proof Sketch}
  \hspace*{1em}}{\qed\bigskip\\}
\newenvironment{proof-idea}{\noindent\textbf{Proof Idea}
  \hspace*{1em}}{\qed\bigskip\\}
\newenvironment{proof-of-lemma}[1][{}]{\noindent\textbf{Proof of Lemma {#1}}
  \hspace*{1em}}{\qed\\}
\newenvironment{proof-of-theorem}[1][{}]{\noindent\textbf{Proof of Theorem {#1}}
  \hspace*{1em}}{\qed\\}
\newenvironment{proof-attempt}{\noindent\textbf{Proof Attempt}
  \hspace*{1em}}{\qed\bigskip\\}

\newcommand{\cset}{\mc{S}}
\newcommand{\inputcoreset}{\cset_{n}}
\newcommand{\sn}{\inputcoreset}
\newcommand{\inputcoresetfull}{\cset_{\infty}}
\newcommand{\outputcoreset}{\cset_{\mrm{out}}}
\newcommand{\nout}{n_{\mrm{out}}}
\newcommand{\nin}{n_{\mrm{in}}}
\newcommand{\ktcoreset}{\cset_{\mrm{KT}}}

\newcommand{\khcoreset}{\cset_{\mrm{KH}}}
\newcommand{\basecoreset}{\cset_{\mrm{base}}}

\newcommand{\rmin}{R}
\newcommand{\rminpn}[1][\inputcoreset]{\rmin_{#1}}

\newcommand{\err}{\mathfrak{M}}%

\newcommand{\splineconst}[1][\splineparam+1]{S_{#1}}

\newcommand{\ktsplit}{\hyperref[algo:ktsplit]{\color{black}{\textsc{kt-split}}}\xspace}
\newcommand{\ktsplitlink}{\hyperref[algo:ktsplit]{\textsc{kt-split}}\xspace}
\newcommand{\ktswap}{\hyperref[algo:ktswap]{\color{black}{\textsc{kt-swap}}}\xspace}
\newcommand{\ktswaplink}{\hyperref[algo:ktswap]{\textsc{kt-swap}}\xspace}

\newcommand{\Wendland}[1][\wendparam]{{\bf Wendland$(#1)$}\xspace}

\newcommand{\ncref}[1]{\cref{#1}: \nameref*{#1}} %

\newcommand{\pcref}[1]{Proof of \ncref{#1}} %

\newcommand{\pnew}{\mu}
\newcommand{\qnew}{\nu}
\newcommand{\wendparam}{s}

\newcommand{\temprv}[1][i]{\vareps_{#1}}

\newcommand{\gknorm}[1]{\norm{#1}_{\gkernel}}
\newcommand{\gkrkhs}{\rkhs_{\gkernel}}

\graphicspath{{code/figs/},{figs/}}

\newcommand{\abstracttext}{
We introduce kernel thinning, a new procedure for compressing a distribution $\mathbb{P}$ more effectively than i.i.d.\ sampling or standard thinning. Given a suitable reproducing kernel $\mathbf{k}_{\star}$ and $\mathcal{O}(n^2)$ time, kernel thinning compresses an $n$-point approximation to $\mathbb{P}$ into a $\sqrt{n}$-point approximation with comparable worst-case integration error across the associated reproducing kernel Hilbert space. 
The maximum discrepancy in integration error is $\mathcal{O}_d(n^{-1/2}\sqrt{\log n})$ 
in probability  for compactly supported $\mathbb{P}$ and $\mathcal{O}_d(n^{-\frac{1}{2}} (\log n)^{(d+1)/2}\sqrt{\log\log n})$ 
for sub-exponential $\mathbb{P}$ on $\mathbb{R}^d$. In contrast, an equal-sized i.i.d.\ sample from $\mathbb{P}$ suffers  $\Omega(n^{-1/4})$ integration error. Our sub-exponential guarantees resemble the classical quasi-Monte Carlo error rates for uniform $\mathbb{P}$ on $[0,1]^d$ but apply to general distributions on $\mathbb{R}^d$ and a wide range of common kernels. Moreover, the same construction delivers near-optimal $L^\infty$ coresets in $\mathcal O(n^2)$ time.  We use our results to derive explicit non-asymptotic maximum mean discrepancy bounds for Gaussian, Mat\'ern, and B-spline kernels and present two vignettes illustrating the practical benefits of kernel thinning over i.i.d.\ sampling and standard Markov chain Monte Carlo thinning, in dimensions $d=2$ through $100$.
}
\newcommand{\keywordslist}{coresets, distribution compression, Markov chain Monte Carlo, maximum mean discrepancy, reproducing kernel Hilbert space, thinning}

\usepackage{lastpage}
\jmlrheading{25}{2024}{1-\pageref{LastPage}}{11/21; Revised
3/24}{4/24}{21-1334}{Raaz Dwivedi and Lester Mackey}
\ShortHeadings{Kernel Thinning}{Dwivedi and Mackey}

\begin{document}

\etoctocstyle{1}{Table of contents}
\etocdepthtag.toc{mtchapter}
\etocsettagdepth{mtchapter}{section}

\makeatletter
\patchcmd{\@algocf@start}%
  {-1.5em}%
  {0pt}%
  {}{}%
\makeatother

\title{Kernel Thinning}

\author{\name Raaz Dwivedi \email dwivedi@cornell.edu \\
       \addr Cornell Tech
       \AND
       \name Lester Mackey \email lmackey@microsoft.com \\
       \addr Microsoft Research New England}

\editor{Ingo Steinwart}

\maketitle
 \begin{abstract}%
    \abstracttext
    \end{abstract}
    
    \begin{keywords}%
    \keywordslist\opt{arxiv}{\footnote{Accepted for presentation as an extended abstract at the Conference on Learning Theory (COLT) 2021.}}
    \end{keywords}

\section{Introduction}
\label{sec:introduction}
Monte Carlo and Markov chain Monte Carlo (MCMC) methods \citep{brooks2011handbook} are commonly used to approximate intractable target expectations $\Pstar\fun\defeq\Exs_{X\sim\Pstar}[\fun(X)]$ of $\P$-integrable functions $f$ with asymptotically exact averages $\P_n\fun
\!\defeq\!\frac{1}{n}\sumn \fun(\axi)$
based on points $(x_i)_{i=1}^n$ generated from a Markov chain. 
A standard practice, to minimize the expense of downstream function evaluation, is to \emph{thin} the Markov chain output down to a smaller size $\nout$ by keeping only  every $({n/}{\nout})$-th sample point \citep{owen2017statistically}. %
We call this approach \emph{standard thinning}, and such sample compression is critical in fields like computational cardiology in which each function evaluation triggers an organ or tissue simulation consuming thousands of CPU hours \citep{niederer2011simulating,augustin2016anatomically,strocchi2020simulating}.
Unfortunately, 
standard thinning also leads to a significant reduction in accuracy. For example, thinning one's chain down to $\nout = \sqrt n$ sample points increases integration error from 
$\order(n^{-\frac12})$ in probability to $\Omega(n^{-\frac14})$ by the Markov chain central limit theorem \citep[Prop.~29]{roberts2004general}. 
Our primary contribution is a more effective thinning strategy, which provides $o_p(n^{-\quarter})$-integration error when $n^\half$ points are returned.

\subsection{Thinned MMD coresets}%
We focus on integration error in a reproducing kernel Hilbert space \citep[RKHS,][Def.~4.18]{steinwart2008support} 
of bounded, measurable functions with a target kernel $\kernel:\real^d\times\real^d \to \real$ for $d\in\naturals$ and RKHS norm $\knorm{\cdot}$. 
\begin{assumption}[RKHS of bounded, measurable functions]%
\label{asmp:bounded_measurable}
The RKHS $\rkhs_{\gkernel}$ of a kernel $\gkernel:\real^d\times\real^d \to \real$ contains only bounded measurable functions.
Equivalently, $\gkernel$ is bounded with $\gkernel(x,\cdot)$ measurable for all $x\in\Rd$ \citep[Lems.~4.23, 4.24]{steinwart2008support}.\footnote{Throughout, we use $\gkernel$ for statements involving a generic kernel that is potentially distinct from the target kernel $\kernel$.}
\end{assumption}
The worst-case integration error 
over the RKHS unit ball is given by the kernel \emph{maximum mean discrepancy} \citep[MMD,][]{JMLR:v13:gretton12a}. 
\begin{definition}[{Maximum mean discrepancy} \citep{JMLR:v13:gretton12a}]
For a kernel $\gkernel$ satisfying \cref{asmp:bounded_measurable}, 
we define the kernel \emph{maximum mean discrepancy},
\begin{align}
	\mmd_{\gkernel}(\pnew,\qnew)&\defeq \sup_{f \in \gkrkhs : \gknorm{\fun}\leq 1}\abss{\pnew\fun-\qnew\fun}
    \qtext{for all probability measures $\pnew, \qnew$ on $\reals^d$.}
\label{eq:kernel_mmd_distance}
\end{align}
For sequences of points $\mc{S}$ and $\mc{S}'$ in $\Rd$ with empirical distributions $\Q$ and $\Q'$, we  overload this notation to write $\mmd_{\gkernel}(\P, \mc{S}) \defeq \mmd_{\gkernel}(\P, \Q)$ and  $\mmd_{\gkernel}(\mc{S}, \mc{S}') \defeq \mmd_{\gkernel}(\Q, \Q')$. 
\end{definition}
Given $\kernel$ satisfying \cref{asmp:bounded_measurable}, a target distribution $\Pstar$ on $\reals^d$, and %
a sequence of 
$\reals^d$-valued points $\pseqxnn$
generated to approximate $\P$, our aim is to identify a \emph{thinned MMD coreset}, a shorter subsequence that continues to approximate $\P$ well in $\mmd_\kernel$.
\begin{definition}[MMD coreset]
\label{def:mmd}
We call a sequence of $\nout$ points in $\Rd$ with empirical measure $\Q$ an \emph{$(\nout, \epssmall)$-MMD coreset} for $(\kernel, \Pstar)$ if $\mmd_{\kernel}(\Pstar, \Q) \leq \epssmall$. 
\end{definition}
Notably, when the initial sequence is drawn \iid or from a fast-mixing Markov chain targeting $\P$, standard thinning down to size $\nout=n^{\frac{1}{2}}$ yields an order $(n^{\frac{1}{2}}, n^{-\frac{1}{4}})$-MMD coreset in probability (see \cref{mcmc_mmd}).
A benchmark for improvement is provided by the online Haar strategy of \citet{dwivedi2019power}, which generates an $(n^{\frac12}, \order_d(n^{-\frac{1}{2}}\log^{2d} n))$-MMD coreset in probability from $2n^{\frac12}$ \iid sample points when $\Pstar$ is specifically the uniform distribution on the unit cube $[0,1]^d$.\footnote{\citet{dwivedi2019power} specifically control the \emph{star discrepancy}, a quantity which in turn upper bounds a Sobolev space MMD called the \emph{$L^2$ discrepancy} \citep{hickernell1998generalized,novak2010tractability}.}
Our goal is to develop thinned coresets of improved quality for any target $\P$ with sufficiently fast tail decay. 

\subsection{Our contributions}
To this end, we introduce \emph{kernel thinning} (\cref{algo:kernel_thinning}), a new, practical solution to the thinned MMD coreset problem that takes as input an $(n, \order_p(n^{-\half}))$-MMD coreset and outputs an $(n^{\frac12}, o_p(n^{-\quarter}))$-MMD coreset for a wide-range of $(\kernel, \Pstar)$. 
Kernel thinning uses non-uniform randomness and evaluations of a less smooth \emph{square-root kernel} $\ksqrt$ (see \cref{def:square_root_kernel}) to partition the input into subsets of comparable quality and then greedily refines the best of these subsets using $\kernel$. 
Our primary contributions include:

\begin{enumerate}[leftmargin=*]
\item \tbf{Better-than-\iid MMD coresets:} 
Given $n$ input points sampled \iid or from a fast-mixing Markov chain, 
kernel thinning yields, 
in probability, 
an
$(n^{\frac12}, \order_d(n^{-\frac12}\sqrt{\log n }))$-MMD coreset for $\Pstar$ and $\ksqrt$ with bounded support, an $(n^{\frac12}, \order_d(n^{-\frac12}\sqrt{\log^{d+1} n \log\log n}))$-MMD coreset for $\Pstar$ and $\ksqrt$ with light tails, and an $(n^{\frac12}, \order_d(n^{-\frac12+\frac{d}{2\htailparam}}\sqrt{\log n \log \log n}))$-MMD coreset for $\Pstar$ and $\ksqrt^2$ with $\htailparam > 2d$ moments (\cref{theorem:main_result_all_in_one,table:mmd_rates}).
For compactly supported or light-tailed $\Pstar$ and $\ksqrt$, these results compare favorably with 
known $\Omega_d(n^{-\frac12})$ lower bounds (see \cref{sec:related}).
Our guarantees extend to more general input point sequences, including deterministic sequences based on quadrature or kernel herding \citep{chen2012super}, and give rise to explicit, non-asymptotic error bounds for a wide variety of popular kernels including Gaussian, \Matern, and B-spline kernels. 
While $(n^\half, \order_d(n^{-\frac{1}{2}}\log^{\frac{d-1}{2}} n))$-MMD coresets have been developed for specific $(\kernel, \Pstar)$ pairings like the uniform distribution on $[0, 1]^d$ and an $L^2$ discrepancy kernel $\kernel$ (see \cref{sec:related}), to the best of our knowledge, no prior $(n^\half, o_p(n^{-\quarter}))$-MMD coreset constructions were known for the range of $\Pstar$ and $\kernel$ studied in this work. 

\item \textbf{MMD error from square-root $\Linf$ error:} 
To derive our MMD guarantees for kernel thinning, we first establish an important link between MMD coresets for $\kernel$ and \emph{$\Linf$ coresets} for $\ksqrt$. 
\begin{definition}[$\Linf$ coreset]
\label{def:Linf_error}
For any kernel $\gkernel$ satisfying \cref{asmp:bounded_measurable}, probability measure $\mu$ on $\Rd$, and $z\in\reals^d$, let $\mu\gkernel(z) \defeq \E_{X\sim\mu}[\gkernel(X,z)]$. 
We call a sequence of 
$\nout$ points in $\Rd$ with empirical measure $\Q$ 
an \emph{$(\nout, \epssmall)$-$\Linf$ coreset} for  $(\gkernel, \Pstar)$ if $ \sinfnorm{\Pstar\gkernel-\Q\gkernel} \leq \epssmall$. 
\end{definition}
\cref{theorem:coreset_to_mmd,cor:mmdlinf} %
show that \emph{any} $\Linf$ coreset for $(\ksqrt,\Pstar)$ is also an MMD coreset for $(\kernel, \Pstar)$ with quality depending on the tail decay of $\ksqrt$ and $\Pstar$.
\item \textbf{Online vector balancing in Hilbert spaces:} As a building block for constructing high-quality coresets, we introduce and analyze a Hilbert space generalization of the self-balancing walk of \citet{alweiss2021discrepancy} to partition a sequence of functions (like $(\ksqrt(\x_i,\cdot))_{i=1}^n$) into nearly equal halves. Our analysis of this \emph{self-balancing Hilbert walk} (SBHW, \cref{algo:self_balancing_walk}) in \cref{sbhw_properties} may be of independent interest for solving the online vector balancing problem of \citet{spencer1977balancing} in Hilbert spaces (\cref{cor:kernel_balancing}).
\item \tbf{Efficient, near-optimal $\Linf$ coresets:} 
We then design a symmetrized version of SBHW for RKHSes---kernel halving---that delivers $2$-thinned coresets with small $\Linf$ error (\cref{algo:kernel_halving,kernel_halving_results}). 
The first stage of kernel thinning, \ktsplitlink,  recursively applies kernel halving to $\ksqrt$ to obtain near-minimax-optimal $\Linf$ coresets in $\order(n^2)$ time with $\order(n\min(d,n))$ space (\cref{corollary:kernel_thinning_coreset_bound,corollary:ktsplit_rates}).
\end{enumerate}

After describing our kernel and input point requirements in \cref{sec:setup}, we detail the kernel thinning and kernel halving algorithms in \cref{sec:kernel_thinning}.
\cref{sec:mmd} houses our main MMD guarantees, both for kernel thinning and for generic $\Linf$ square-root kernel coresets.
We introduce and analyze the self-balancing Hilbert walk in \cref{sec:from_self_balancing_walk_to_kernel_thinning} and present our main $\Linf$ guarantees for kernel halving and \ktsplit in  \cref{sub:kernel_halving}. 
\cref{sec:vignettes} complements our theoretical contributions with two vignettes illustrating the practical benefits of kernel thinning over (a) \iid sampling in dimensions $d=2$ through $100$ and (b) standard MCMC thinning across twelve  experiments targeting challenging differential equation posterior distributions. 
We conclude with a discussion of our results, related work, and future directions in \cref{sec:discussion} and defer all proofs to the appendices.

\paragraph*{Notation}
We define the shorthand $[n] \defeq \{1,\dots,n\}$ for $n\in\naturals$, $a\wedge b \defeq \min(a, b)$ for $a,b\in\reals$,  $\real_{+} \defeq \braces{x\in \real: x\geq 0}$, and
$\ball(\x; \radius)\defeq \braces{\y \in\Rd \mid \twonorm{\x-\y}< \radius}$ for $\radius\in\reals$. We use $\mrm{Vol}(\mathcal B)$ to denote the volume of a compact set $\mathcal B \subset \real^d$.
We let $\set{A}^c$ denote the complement of a set $\set{A} \subset \Rd$ and $\indicator_{\set{A}}(\x) = 1$ if $\x \in \set{A}$ and 0 otherwise. 
We use $\Pr(\mc E)$ to denote the probability of an event $\mc E$.
For real-valued kernels $\gkernel$ and functions $f$ on $\Rd$, we make frequent use of the norms 
$\infnorm{\gkernel} = \sup_{x,y\in\Rd} |\gkernel(x,y)|$ and $\infnorm{f} = \sup_{x\in\Rd} |f(x)|$. For $x>0$, we use $\Gamma(x) = \int_{0}^{\infty }t^{x-1}e^{-t}\,dt$ to denote the Gamma function (with $\Gamma(n) = (n-1)!$ for $n\in\N$).
 For two sequences of real numbers  $(a_n)_{n\in\N}$ and $(b_n)_{n\in\N}$, we say that $a_n$ is of order $b_n$ and write $a_n = \order(b_n)$ or $a_n\precsim b_n$ to denote that $a_n \leq cb_n$ for all $n\in\N$ and some constant $c > 0$. We write $a_n =\Omega(b_n)$ if $b_n = \order(a_n)$ and $a_n = \Theta(b_n)$ when $a=\Omega(b_n)$ and $a=\order(b_n)$. Moreover, we use $a_n = \order_d(b_n), a_n\precsim_d b_n, a_n= \Omega_d(b_n),  a_n\succsim_d b_n$ to indicate dependency of underlying universal constant on $d$. We say $a_n = o(b_n)$ if $\lim_{n\to\infty}a_n/b_n=0$.  
For a sequence of real-valued random variables $(X_n)_{n\in\naturals}$, we write $X_n=\order_{P}(a_n)$ or $X_n = \order(a_n)$ in probability, when $\frac{X_n}{a_n}$ is stochastically bounded, i.e., for all $\delta>0$, there exists finite $c_{\delta}$ and $n_{\delta}$ such that $\Pr(|\frac{X_n}{a_n}|>c_{\delta})<\delta$, for all $n>n_{\delta}$. We write $X_n=\Omega_{P}(a_n)$ if $1/X_n = \order_P(1/b_n)$ and $X_{n} = o_p(a_{n})$ when $\frac{X_{n}}{a_{n}}\to 0$ in probability, i.e., for all $\vareps>0$, $\Pr(|\frac{X_{n}}{a_{n}}| \geq \vareps) \to 0$.
We write \emph{order $(n, \epssmall)$-MMD (or $\Linf$) coreset} to mean an $(n, \order(\eps))$-MMD (or $\Linf$) coreset and append \emph{in probability} to mean an $(n, \order_p(\eps))$-MMD (or $\Linf$) coreset.  

\newcommand{\kup}[1][\ksqrt]{\underline{\tau}_{#1}}
\newcommand{\ktail}[1][\ksqrt]{\overline{\tau}_{#1}}
\newcommand{\ktailinv}[1][\ksqrt]{\overline{\lambda}_{#1}}
\newcommand{\klipinv}[1][\ksqrt]{\underline{\lambda}_{#1}}
\newcommand{\kcond}[1][\ksqrt]{\lambda_{#1}}
\newcommand{\rk}[1][\gkernel]{\rmin_{#1,n}}
\newcommand{\rktau}[1][\gkernel]{\rmin_{#1,n}'}
\newcommand{\rkmax}[1][\gkernel]{\rmin_{#1,n}^\dagger}
\newcommand{\rktaugen}[1][\gkernel,n]{\rmin_{#1}'}
\section{Input Point and Kernel Requirements}
\label{sec:setup}

Given a target distribution $\Pstar$ on $\Rd$, a kernel $\kernel$ satisfying \cref{asmp:bounded_measurable}, %
and a sequence of $\Rd$-valued input points $\inputcoreset \!=\! \pseqxnn$ generated either randomly or deterministically, our goal is to identify a better-than-\iid thinned MMD coreset, that is, a subsequence $\outputcoreset$ of size $n^{\half}$ satisfying $\mmd_\kernel(\Pstar,\outputcoreset) \!=\! o_p(n^{-\quarter})$. 
When drawing asymptotic conclusions, we will view $d$ as fixed and $\inputcoreset$ as a prefix of an infinite sequence of points $\inputcoresetfull \defeq (\x_i)_{i=1}^\infty$. 
\subsection{Input point requirements}
Our algorithms are designed to return high quality MMD coresets for the \emph{input} $\inputcoreset$.
To translate these into high quality coresets for the target $\P$, it suffices, by the triangle inequality, for the input points to have quality $\mmd_{\kernel}(\Pstar, \inputcoreset) = \order_p(n^{-\half})$. 
As we discuss in \cref{sec:related}, input sequences generated by \iid sampling, kernel herding \citep{chen2012super}, Stein Point MCMC \citep{chen2019stein}, and greedy sign selection \citep{karnin2019discrepancy} all satisfy this property. Moreover, we prove in \cref{proof_of_mcmc_mmd} that an analogous guarantee holds for the iterates of a fast-mixing Markov chain.

\newcommand{\mcmcmmdresultname}{MMD guarantee for MCMC}
\begin{proposition}[\mcmcmmdresultname]
\label{mcmc_mmd}
Consider a homogeneous $\phi$-irreducible geometrically ergodic Markov chain 
\citep[Thm.~1xi]{gallegosherrada2023equivalences} %
with initial state $\x_0$, subsequent iterates $\inputcoresetfull$, and stationary distribution $\P$.
If $\kernel$ satisfies \cref{asmp:bounded_measurable}, then 
there exists a $\P$-almost everywhere finite function $c : \reals^d \to (0,\infty]$ such that, for any given $n \in \naturals$ and $\delta \in (0,1)$, 
$\mmd_{\kernel}(\P, \inputcoreset) 
    \leq \sqrt{\frac{c(\x_0)\infnorm{\kernel}\log(e/\delta)}{n}}
    $ 
with probability $1-\delta$ given $\x_0$.
\end{proposition}

The \emph{input radius}, %
\begin{talign}
\label{def:inputradii}
\rminpn[\inputcoreset]\! \defeq\! \max_{\x\in \inputcoreset}\twonorm{\x},
\end{talign}
 will also play an important role in our results. In particular, the growth rate of this radius as a function of $n$ impacts the growth rate of our MMD bounds. Our next definition assigns familiar names to the most commonly encountered growth rates. 
 
\begin{definition}[Input radius growth rates]
\label{def:tailinputrate}
  We say the point sequence $\inputcoresetfull$ with prefixes $\inputcoreset$ for $n\in\N$ is 
  \bndcase if $\rminpn[\inputcoreset] = \order_d(1)$, \subgauss if $\rminpn[\inputcoreset] = \order_d(\sqrt{\log n})$, \subexp if $\rminpn[\inputcoreset] = \order_d(\log n)$, and \heavytail with $\rho>0$ if $\rminpn[\inputcoreset] = \order_d(n^{1/\rho})$.
\end{definition}
These growth rates are exactly those which arise with  probability $1$ when an input sequence is generated identically  from $\P$ with corresponding tail behavior or from a fast-mixing Markov chain targeting $\P$.
Our proof of this result is given in \cref{proof_of_prop:radii_growth}.

\begin{proposition}[Almost sure radius growth]
\label{prop:radii_growth}
Consider either 
(i) points $\inputcoresetfull$ sampled identically (but not necessarily independently) from $\P$ with $\x_0$ independent 
or 
(ii) a homogeneous $\phi$-irreducible geometrically ergodic Markov chain 
with initial state $\x_0$, subsequent iterates $\inputcoresetfull$, and stationary distribution $\P$.
Then the following statements hold true for any nonnegative $c$ and $\rho$ and $\P$-almost every $\x_0$.
\begin{enumerate}[label=(\alph*),leftmargin=*]
\itemsep0em
    \item\label{item:compactiid} %
    If $\P$ is compactly supported, then, with probability $1$ conditional on $\x_0$, $\inputcoresetfull$ is \bndcase.
    \item\label{item:sgiid} %
    If $\E_{X\sim\P}[e^{c\twonorm{X}^2}] < \infty$, then, with probability $1$ conditional on $\x_0$, $\inputcoresetfull$ is \subgauss.
    \item\label{item:seiid} %
    If $\E_{X\sim\P}[e^{c\twonorm{X}}] < \infty$, then, with probability $1$ conditional on $\x_0$, $\inputcoresetfull$ is \subexp.
    \item\label{item:htiid} %
    If $\E_{X\sim\P}[\twonorm{X}^{\rho}] < \infty$, then, with probability $1$ conditional on $\x_0$, $\inputcoresetfull$ is \heavytail.
\end{enumerate}
\end{proposition}

Finally, we will also require $\inputcoresetfull$ to be  \emph{oblivious}, that is, generated independently of any randomness in the thinning algorithm. To capture this assumption, we treat $\inputcoresetfull$ as fixed and deterministic hereafter. This treatment is without loss of generality since our results hold conditional on the observed values of $(\x_i)_{i=1}^{\infty}$  when the points are random and oblivious.

\subsection{Kernel requirements} 
We use the terms \emph{reproducing kernel} and \emph{kernel} interchangeably to indicate that $\kernel$ 
is symmetric and positive definite, i.e., that the kernel matrix $(\kernel(\z_i,\z_j))_{i,j=1}^l$ is symmetric and positive semidefinite for any evaluation points $(z_i)_{i=1}^l$ in $\reals^d$. 
In addition to $\kernel$, our algorithm takes as input a \emph{square-root kernel} for $\kernel$.
\begin{definition}[Square-root kernel]
\label{def:square_root_kernel}
	We say a kernel $\ksqrt:\Rd\times \Rd\to\R$ is a \emph{square-root
	kernel} for $\kernel:\Rd \times \Rd\to\R$ if $\ksqrt(x,\cdot)$ is square integrable for all $x\in\Rd$ with  
	\begin{talign}
		\label{eq:kernelsqrt}
		\kernel(\x, \y) = \int_{\Rd}\ksqrt(\x, \z)\ksqrt(\y, \z) d\z \qtext{for all} x, y \in \real^d.
	\end{talign}
\end{definition}
We highlight that a square-root kernel need not be unique and that its existence is an indication of a certain degree of smoothness in the target kernel $\kernel$. One convenient tool for deriving square-root kernels is the notion of a \emph{spectral density}.

\begin{definition}[Shift invariance and spectral density]
\label{def:spectral_density}
We call a kernel of the form $\gkernel(x, y) = \kappa(x-y)$ for $\kappa: \Rd \to \reals$ 
\emph{shift-invariant} and say $\gkernel$ has \emph{spectral density} $\what{\kappa}$ if $\kappa$ is the Fourier transform of 
a finite measure with Lebesgue density $\what{\kappa}$, i.e., $\kappa(z) =  \frac{1}{(2\pi)^{d/2}}\int e^{-i\inner{\omega}{z}} \what{\kappa}(\omega) d\omega$.
\end{definition}
\newcommand{\sqrttablename}{Square-root kernels $\ksqrt$ for common target kernels $\kernel$}
\newcommand{\sqrttablecaption}{
\noindent\caption{
    \tbf{\sqrttablename}.
    Each $\kernel$ satisfies $\infnorm{\kernel} = 1$,
    and the parameter range ensures the existence of $\ksqrt$.
    Above, $\circledast^{\l}$ denotes recursive convolution with $\l$ function copies, $\bessel[a]$ denotes the modified Bessel function of the third kind %
    \citep[Def.~5.10]{wendland2004scattered},
    $c_{b} \defeq \frac{2^{1-b}}{\Gamma(b)}$,
     $\matscale[\matone]=\frac{c_{\matone-d/2}}{c_{\matone}} \mattwo^{2\matone-d}$,
     $A_{\matone,\mattwo,d}\defeq\parenth{\frac{1}{4\pi}\mattwo^2}^{d/4} \sqrt{\frac{\Gamma(\matone)}{\Gamma(\matone-d/2)}} \cdot \frac{\Gamma((\matone-d)/2)}{\Gamma(\matone/2)}$,
     $\splineconst[2\splineparam+2, d]'\defeq\splineconst[2\splineparam+2, d] \cdot 
     (\frac{4^{\splineparam+1}}{\sqrt{2\pi}})^d$, and
     $\wtil{S}_{\splineparam, d}\defeq\frac{\sqrt{\splineconst[2\splineparam+2, d]}}{\splineconst[\splineparam+1, d]}$
     where $\splineconst[\splineparam,d]$ is defined in \cref{eq:spline_all_constants}.
	 See \cref{sec:proof_for_tables} for our derivation.
    } }

\begin{table}[t]
    \centering
  \resizebox{\textwidth}{!}
  {
    {
    \renewcommand{\arraystretch}{1.5}
    \begin{tabular}{cccc}
        \toprule
        \Centerstack{\bf Name of kernel\\ \bf  $\kernel(\x,\y)=\kappa(\x\!-\!\y)$ } &
        \Centerstack{\bf Expression for \\ $\kappa(\z)$
        } 

        & \Centerstack{
        \bf Fourier transform
        \\ $\widehat{\kappa}(\omega)$ %
        } 
        
        & \Centerstack{\bf Square-root kernel \\ $\ksqrt$
        } 
        
        \\[2mm]
        \midrule 
        \Centerstack{
          \textbf{Gaussian}$(\gaussparam):$\\$\gaussparam>0$}
          & $\exp\parenth{-\frac{\twonorm{\z}^2}{2\gaussparam^2}}$

      	& \Centerstack{$\gaussparam^d\exp\parenth{-\frac{\gaussparam^2
          \twonorm{\omega}^2}{2}}$
          }
      	& 
          \Centerstack{
      		$\parenth{\frac{2}{\pi\gaussparam^2}}^
      		{\frac{d}{4}}\textbf{Gaussian}\parenth{\frac{\gaussparam}
      		{\sqrt 2}}$
      		}	

          \\[4mm]

         \Centerstack{
         \textbf{Mat\'ern}$(\matone, \mattwo) :$\\ $\matone>d, \mattwo>0$} & $
         c_{\matone-\frac{d}{2}}(\mattwo\twonorm{\z})^{\matone-\frac{d}{2}}
         \bessel
         [\matone-\frac{d}{2}](\mattwo\twonorm{\z})$ 
        & \Centerstack{${\matscale[\matone]\,}{(\mattwo^2+\twonorm{\omega}^2)^{-\matone}}$}
         & \Centerstack{$A_{\matone,\mattwo,d}$\textbf{Mat\'ern}$(\frac{\matone}{2}, \mattwo)$}
        
         \\[4mm]

         \Centerstack{
         $\textbf{B-spline}(2\splineparam+1):$\\$ \splineparam\in 2\natural+1$}
         
         &$\splineconst[2\splineparam+2, d]\displaystyle\prod_{\j=1}^d\circledast^{2\splineparam+2}\indicator_
         {[-\frac12, \frac12]}(\z_{\j}) $
         & \Centerstack{$\splineconst[2\splineparam+2, d]'
         \displaystyle\prod_{\j=1}^d\textfrac{\sin^{2\splineparam+2}(\frac{\omega_{\j}}{2})}{\omega_{\j}^{2\splineparam+2}}$}
         &  \Centerstack{
         $\wtil{S}_{\splineparam, d}\textbf{B-spline}(\splineparam)$}

        \\[1ex] \bottomrule \hline
    \end{tabular}
    }
    }
 \opt{arxiv}{\sqrttablecaption}
 \opt{jmlr}{\sqrttablecaption}
 \label{table:kernel_sqrt_pair}
\end{table}

As we show in \cref{sec:proof_for_tables,sub:guarantees_with_approximate_square_root_kernels}, many familiar kernels admit spectral densities, including Gaussian, \Matern, B-spline, inverse multiquadric, sech, and  Wendland’s compactly supported kernels.
Moreover, by Bochner's theorem \citep[Thm.~6.6]{bochner1933monotone,wendland2004scattered} and the Fourier inversion theorem \citep[Cor.~5.24]{wendland2004scattered}, any continuous $\kernel(x, y) = \kappa(x-y)$ with absolutely integrable $\kappa$ has a spectral density equal to the Fourier transform of $\kappa$.
Our next result (proved in \cref{sec:proof_of_sqrt_translation_invariant}) derives a square-root kernel for any shift-invariant~$\kernel$ with a square-root integrable spectral density.
\newcommand{\sqrttranslationinvariantname}{Shift-invariant square-root kernels}
\begin{proposition}[\sqrttranslationinvariantname]
\label{sqrt_translation_invariant}
If a kernel $\kernel(\x,\y)\! =\! \kappa(\x\!-\!\y)$ admits a spectral density (\cref{def:spectral_density})  $\widehat{\kappa}$ 
with $\int \sqrt{\widehat{\kappa}(\omega)}d\omega \!<\! \infty$, then $\ksqrt(x,y)\! =\! \frac{\kappasqrt(\x-\y)}{(2\pi)^{d/4}}$
is a square-root kernel of $\kernel$ for $\kappasqrt$ the \mbox{Fourier transform of $\sqrt{\widehat{\kappa}}$.}
\end{proposition}

\cref{table:kernel_sqrt_pair} gives several examples of common kernels satisfying the conditions of \cref{sqrt_translation_invariant} along with their associated square-root kernels. %
For example, 
if $\kernel$ is Gaussian with bandwidth $\gaussparam$, then a rescaled Gaussian kernel with bandwidth $\frac{\gaussparam}{\sqrt{2}}$ is a valid choice for $\ksqrt$.
For simplicity, our results in the sequel assume the use of an exact square-root kernel $\ksqrt$, but, as we detail in \cref{sub:guarantees_with_approximate_square_root_kernels}, it suffices to use the square-root of any kernel that dominates $\kernel$ in the positive-definite order (see \cref{def:square_root_dom_kernel}).
For example, we show in \cref{matern_sqrt_dom} of \cref{sub:guarantees_with_approximate_square_root_kernels} that a standard \Matern kernel is a suitable {square-root dominating kernel} for any sufficiently-smooth shift-invariant $\kernel$ with absolutely integrable $\kappa$. 
In \cref{table:sqrt_dom_pair} of \cref{sub:guarantees_with_approximate_square_root_kernels}, we also derive convenient tailored {square-root dominating kernels} for inverse multiquadric, sech, and Wendland's compactly supported kernels.

Finally, we define several kernel growth and decay properties that will be explicitly assumed in some of our results.

\begin{assumption}[Lipschitz kernel]
\label{assum:lipkernel}
The kernel $\gkernel : \reals^d \times \reals^d \to \reals$ admits a  Lipschitz constant
\begin{talign}
  \klip[\gkernel] \defeq \sup_{\x,\y, \z} \frac{\abss{\gkernel(\x, \y)-\gkernel(\x, \z)}}{\twonorm{\y-\z}} <\infty.
  \label{eq:klip}
\end{talign}
\end{assumption}

\begin{assumption}[Kernel tail decay]
\label{assum:tailkernel}
The kernel $\gkernel$ satisfies \cref{asmp:bounded_measurable} and, for each $\epssmall>0$, 
\begin{talign}
\begin{split}
  &\max\sbraces{\ \inf\{r:\sup_{\substack{\x, \y: \\ \twonorm{\x-\y}\geq r}} \abss{\gkernel(\x,\y)}\leq \epssmall\},  \quad 
  \inf\{r:\tail[\gkernel](r) \leq \epssmall\} \ } < \infty, \\ 
  &\qtext{where}
  \tail[\gkernel](r) \defeq (\sup_{x} \int_{\twonorm{y}\geq r} \gkernel^2(\x, \x-\y)d\y)^{\frac{1}{2}}
  \qtext{for} r\geq0.
  \label{eq:tail_k_p}
  \end{split}
\end{talign}
\end{assumption}

The following definition gives  
familiar names to commonly encountered tail decay rates.
\begin{definition}[Kernel tail decay rate]
\label{def:tailkernelrate}
 For a kernel $\gkernel$ satisfying \cref{assum:tailkernel}, define 
\begin{talign}
\begin{split}
  \rk   &\defeq \inf\{r:\sup_{\substack{\x, \y: \\ \twonorm{\x-\y}\geq r}} \sabss{\gkernel(\x,\y)}\leq \frac{\infnorm{\gkernel}}{n}\},  
    \quad
     \rktau \defeq \inf\{r:\tail[\gkernel](r) \leq \frac{\infnorm{\gkernel}}{\sqrt{n}}\},\\ 
      \qtext{and}
     \rkmax&\defeq \max\sbraces{\rk, \rktau}
     \label{eq:rmin_k}
\end{split}
\end{talign}
for $\tail[\gkernel]$ defined in \cref{eq:tail_k_p}. 
We say $\gkernel$ is (a) \bndcase if $\rkmax = \order_d(1)$, (b) \subgauss if $\rkmax = \order_d(\sqrt{\log n})$, (c) \subexp if $\rkmax = \order_d(\log n)$, (d) \heavytail for $\rho>0$ if $\rkmax = \order_d(n^{1/\rho})$, and (e) \subpoly if $\log \rkmax = \order_d(\log n)$.
Notably, any \bndcase, \subgauss, \subexp, or \heavytail $\gkernel$ is also \subpoly. 
\end{definition}
\begin{remark}[$\ksqrt$ tail decay implies $\kernel$ boundedness]
If $\ksqrt$ satisfying \cref{assum:tailkernel} is a square-root kernel of $\kernel$, 
then there exists a finite $r$ for which $\kernel(x, x) = \int \ksqrt^2(x, x-y)dy \leq  \sinfnorm{\ksqrt}^2 \mrm{Vol}(\ball(0; r)) + \tail[\ksqrt]^2(r) < \infty$.
\end{remark}
Popular examples of \bndcase, \subgauss, and \subexp $\gkernel$ are B-spline, Gaussian, and \Matern kernels respectively (see \cref{table:sqrtk_details}).
Moreover, one can directly verify that an inverse multiquadric $\gkernel(\x,\y) = (\gamma^2+\twonorm{x\!-\!y}^2)^{-\nu}$ with $\gamma>0$ and $\nu>\frac{d}{4}$ is \heavytailnorho\!($2\nu \wedge (4\nu\!-\!d)$).

\newcommand{\rksmingen}[1]{\rminnew[#1]}
\newcommand{\rksmaxgen}[1]{\rminnew[#1]'}

\newcommand{\rksmin}[1][\gkernel]{\rksmingen{\inputcoreset, #1, n}}
\newcommand{\rksmax}[1][\gkernel]{\rksmaxgen{\inputcoreset, #1, n}}
\SetKwFunction{proctwo}{\texttt{get\_swap\_params}}
\section{Kernel Thinning} %
\label{sec:kernel_thinning}
Our solution to the thinned coreset problem is \emph{kernel thinning}, described in \cref{algo:kernel_thinning}.
Given a thinning parameter $m \in \natural$, kernel thinning proceeds in two stages: \ktsplitlink and \ktswaplink.
\begin{algorithm2e}[H]
\caption{Kernel Thinning\ --\ Return coreset of size $\floor{n/2^m}$ with small $\mmd_{\kernel}$} 
  \label{algo:kernel_thinning}
  \SetAlgoLined
  \DontPrintSemicolon
  \SetKwFunction{ksplit}{\textsc{KT-SPLIT}}
  \SetKwFunction{kswap}{{KT-SWAP}}
  \small
  {
\KwIn{\textup{kernels ($\kernel$,$\ksqrt$), input points $\inputcoreset\!=\!(\axi[i])_{i = 1}^n$, thinning parameter $\m \in \natural$, probabilities $(\delta_i)_{i = 1}^{\floor{\frac{n}{2}}}$}}
  \BlankLine
    $\ksplitcoresets \gets$ {\normalsize\ktsplitlink}\,$(\ksqrt, \inputcoreset, \m, (\delta_i)_{i = 1}^{\floor{\frac{n}{2}}})$ \ \,\,/\!/ \textup{Split $\inputcoreset$ into $2^\m$ candidate coresets of size $\floor{\frac{n}{2^\m}}$}\\[2pt]
    \BlankLine
    $\ \ktcoreset \quad\quad\ \ \gets$ {\normalsize\ktswaplink}\,$(\kernel, \inputcoreset, \ksplitcoresets)$  \ /\!/ \textup{Select best coreset and iteratively refine} \\
    \BlankLine
  \KwRet{\textup{coreset $\ktcoreset$ of size $\floor{n/2^\m}$}}
}
\end{algorithm2e}

\begin{figure}
\resizebox{\textwidth}{!}{
    \centering
    \begin{tabular}{cc}
         \includegraphics[width=0.5\linewidth]{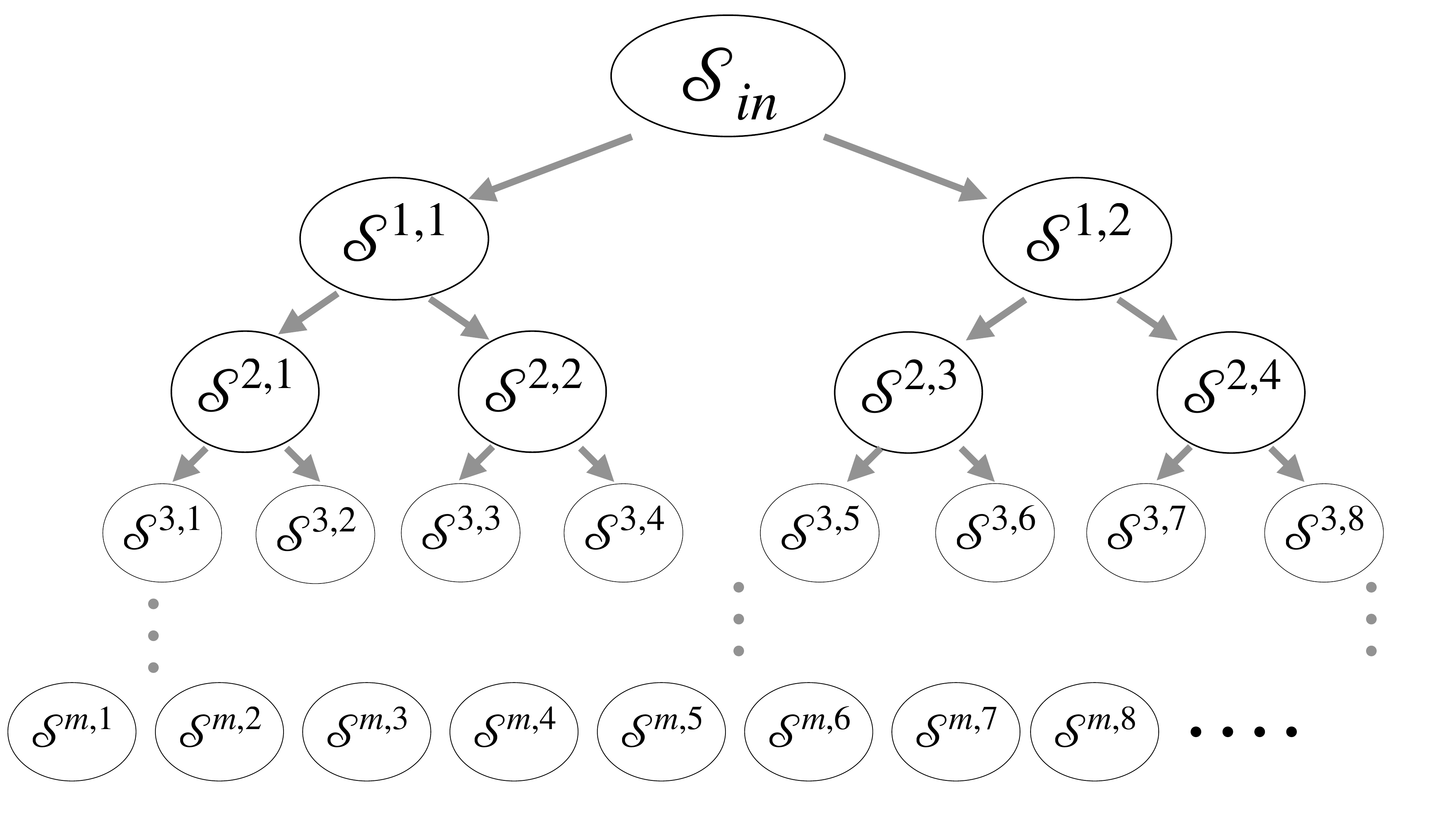} 
         &\includegraphics[width=0.5\linewidth]{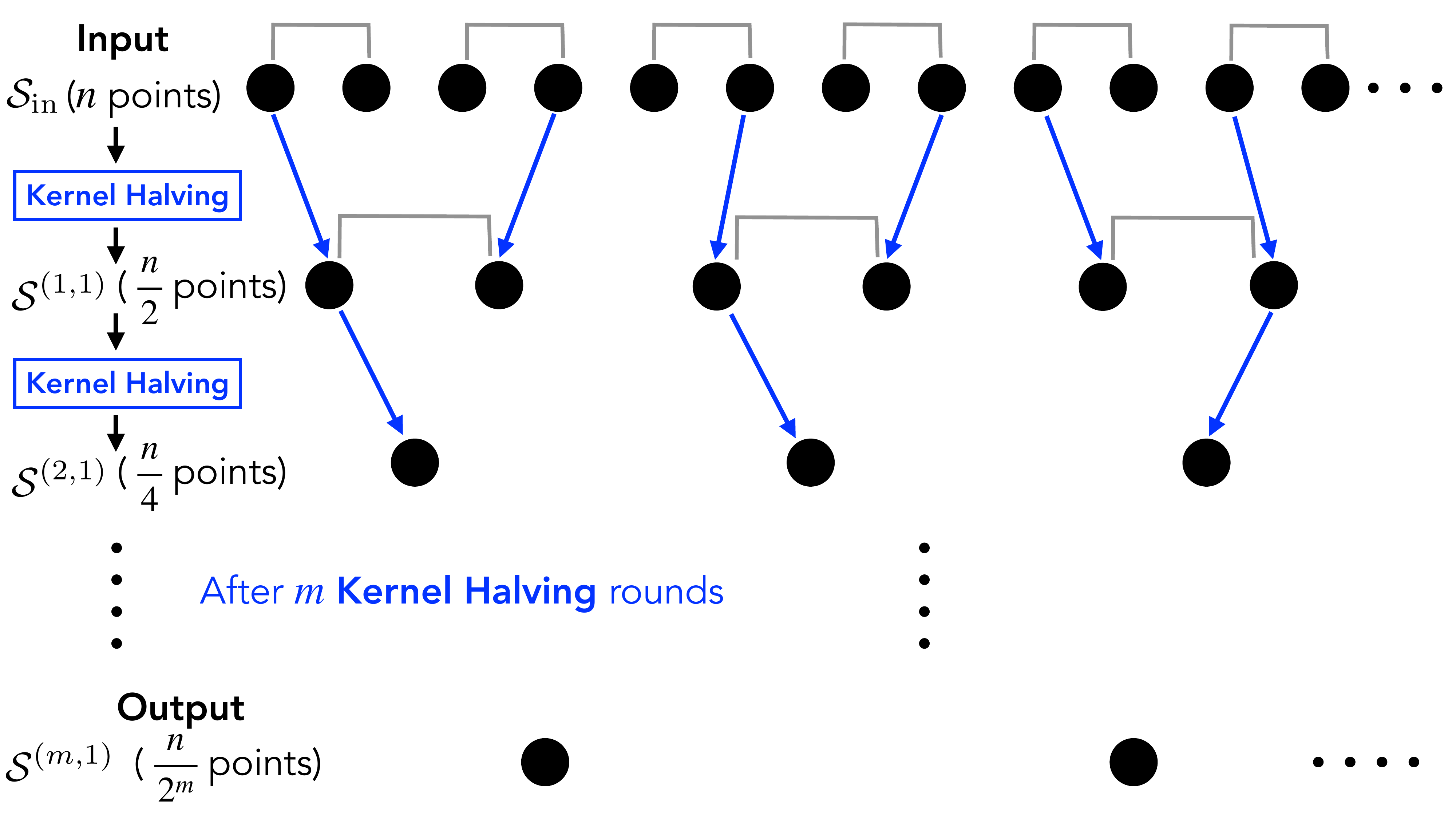} \\ 
    \end{tabular}
    }
    \caption{
    \tbf{Overview of \ktsplit}. (Left) \ktsplit recursively partitions its input $\inputcoreset$ into $2^m$ balanced coresets $\coreset[m, \l]$ of size $\floor{\frac{n}{2^m}}$.
    (Right) In \cref{sub:kernel_halving}, we interpret 
    each coreset $\coreset[m, \l]$ as the output of repeated \emph{kernel halving}: on each halving round, remaining points are paired, and one point from each pair is selected using non-uniform randomness.
    }
    \label{fig:kt_algo_diagram}
\end{figure}
\setcounter{algocf}{0}
\renewcommand{\thealgocf}{\arabic{algocf}a}
\begin{algorithm2e}[ht!]
\caption{{\large\textsc{kt-split}\ --\ } Divide points into candidate coresets of size $\floor{n/2^\m}$} 
  \label{algo:ktsplit}
  \SetAlgoLined\DontPrintSemicolon
  \small
  {
  \KwIn{kernel $\ksqrt$, point sequence $\inputcoreset = (\axi[i])_{i = 1}^n$, thinning parameter $\m \in \natural$, probabilities $(\delta_{i})_{i=1}^{ \floor{n/2}}$ 
  }
  \BlankLine
  {$\coreset[j,\ell] \gets \braces{}$ 
  for $0\leq j\leq \m$ and $1\leq \ell \leq 2^j$} 
  // Empty coresets: $\coreset[j,\ell]$ 
   has size $\floor{\frac{i}{2^{j-1}}}$ 
  after $i$ rounds\\ 
  \BlankLine
  {$\sgparam[j,\ell] \gets 0$ 
  for $1\leq j\leq \m$ and $1\leq \ell \leq 2^{j-1}$}   
  // Swapping parameters
\\
  \BlankLine
  \For{$i=1, \ldots, \floor{n/2}$}
      {
        $\coreset[0,1]\texttt{.append}(\x_{2i-1}); \coreset[0,1]\texttt{.append}(\x_{2i})$ \\[2pt]
        // Every $2^{j-1}$ rounds add point from parent coreset
$\coreset[j\!-\!1,\ell]$
            to each child $\coreset[j,2\ell\!-\!1]$,  $\coreset[j,2\ell]$\\[1pt]
        \For{\textup{($j = 1;
        \ j \leq m\ \textbf{and}\ i / 2^{j-1} \in \natural;
        \ j = j + 1$)}}
            {
            \For{$\ell=1, \ldots, 2^{j-1}$}
            {
            $(\mc{S},\mc{S}') \gets (\coreset[j-1,\ell], \coreset[j,2\ell-1])$;
            \quad
            $(\x, \x') 
                \gets
               \texttt{get\_last\_two\_points}(\mc{S})$\\[2pt]
            // Compute swapping threshold $\cnew[]$\\[1pt] %
            $ \cnew[], \sigma_{j, \l} \gets $\proctwo{$\sigma_{j, \l}, \vmax[], \delta_{|{\mc{S}}|/2} \cdot \frac{2^{j-1}}{\m}$}
     \!for $\vmax[]^2 
      \!=\! \ksqrt(\x,\x)\!+\!\ksqrt(\x',\x')\!-\!2\ksqrt(x,x')$\\[2pt]
            // Assign one point to each child after probabilistic swapping\\[1pt]
$\wvprod[]\gets \ksqrt(\x',\x')\!-\!\ksqrt(\x, \x)
          +\Sigma_{\y\in\mc{S}}(\ksqrt
                (\y, \x)-\ksqrt(\y,\x')) 
                 \!-\! 2\Sigma_{\z\in\mc{S}'}(\ksqrt(\z, \x)\!-\!\ksqrt(\z,\x'))$
             \\[2pt]
            $(x, x') \gets (x', x)$ \textit{ with probability }
            $\min(1, \half (1-\frac{\wvprod[]}{\cnew[]})_+)$
            \\[2pt]
          $\coreset[j,2\ell-1]\texttt{.append}(\x); 
                \quad \coreset[j,2\ell]\texttt{.append}(\x')$
            }
            }
      }
    \KwRet{$\ksplitcoresets$\textup{, candidate coresets of size $\floor{n/2^\m}$}}\\
    \hrulefill\\
    \SetKwProg{myproc}{function}{}{}
     \myproc{\proctwo{$\sigma, \vmax[], \delta$}:}{
     $
            \cnew[] 
                \gets \max(\vmax[] \sigma\sqrt{\smash[b]{2\log(2/\delta)}}, \vmax[]^2)$ \\
     $\sigma^2 \gets \sigma^2
            \!+\! \vmax[]^2(1 \!+\! ({\vmax[]^2}{}\! - \!2\cnew[]){\sigma^2}{/\cnew[]^2})_+$\\
     }
     \KwRet{$(\cnew[], \sigma)$}\;
  }
\end{algorithm2e} 

\paragraph{KT-SPLIT}
The first stage, \ktsplitlink, is an initialization stage that partitions the input sequence  $\inputcoreset = \pseqxnn$ into $2^m$ balanced candidate coresets, each of size $\floor{\frac{n}{2^m}}$.\footnote{\label{footnote:outputsize}When $2^m$ does not evenly divide $n$, the final $n - 2^m\floor{\frac{n}{2^m}}$ points are discarded.}
As depicted in \cref{fig:kt_algo_diagram}, this partitioning is carried out recursively in $m$ rounds, first dividing the input sequence in half, then halving those halves into quarters, and so on until coresets of size $\floor{\frac{n}{2^m}}$ are produced. %
The details of \ktsplitlink can appear a bit complicated as, in practice, all $m$ halving rounds are carried out concurrently in an online manner.
However, under the hood, each candidate coreset is generated by recursively applying 
a new simple subroutine called \emph{kernel halving}.

\paragraph{Kernel halving}
Kernel halving (KH, \cref{algo:kernel_halving}) is a simple randomized procedure for dividing an input sequence $\inputcoreset$ into two balanced, equal-sized coresets $\coreset[1]$ and $\coreset[2]$ using a kernel $\gkernel$.
KH begins with two empty coresets $\coreset[1]$ and $\coreset[2]$ 
and adds points $(x, x')$ from the input sequence two at a time, assigning one point from each pair to each coreset.
To encourage balance between the coresets during generation, KH effectively computes which assignment of $(x, x')$ leads to a smaller $\mmd_{\gkernel}(\coreset[1], \coreset[2])$ and then favors that assignment using non-uniform randomness.
More precisely, on the $i$-th step with $(x, x') = (x_{2i-1}, x_{2i})$, KH computes 
the imbalance contrast %
\[
\wvprod[i] = i^2 \big(\mmd_{\gkernel}^2(\coreset[1] \cup \{x\}, \coreset[2] \cup \{x'\}) - \mmd_{\gkernel}^2(\coreset[1] \cup \{x'\}, \coreset[2] \cup \{x\})\big)
\]
and then adds $x$ to $\coreset[1]$ and $x'$ to $\coreset[2]$ or $x'$ to $\coreset[1]$ and $x$ to $\coreset[2]$ with probability 
biased toward the more balanced outcome.
We refer to this step as \emph{probabilistic swapping} in the algorithm statements. 
The exact value of this probability depends on a \emph{swapping threshold} $\cnew$ that is produced automatically each round based on the user-supplied inputs $(\delta_{i})_{i=1}^{ \floor{n/2}}$.
In \cref{sub:mmd_bounds}, we will learn 
how to set these inputs to achieve better balance than standard thinning or uniform subsampling, and in \cref{sub:kh} we will discuss the role of $\cnew$ and its generation procedure \proctwo in achieving this balance. 

Notably, in the context of \ktsplitlink, KH is run specifically with the square-root kernel $\ksqrt$ rather than the target kernel $\kernel$.
This choice enables us to take advantage of the strong $\Linf$ balance properties established for KH in \cref{sub:kernel_halving} and the close connection between square-root $\Linf$ error and target MMD error revealed in \cref{sec:a_general_recipe_for_bounding_kernel_mmd}.

\renewcommand{\thealgocf}{\arabic{algocf}}
\begin{algorithm2e}[ht!]
\caption{Kernel Halving}
\label{algo:kernel_halving}
\small{
  \KwIn{kernel $\gkernel$, point sequence $\inputcoreset=(\axi[i])_{i = 1}^n$, probability sequence $(\delta_i)_{i = 1}^{\floor{n/2}}$ }
  \BlankLine
  {$\coreset[1], \coreset[2] \gets \braces{}$};\quad $\outvec[0]\gets \boldzero \in \rkhs$\quad /\!/ Initialize empty coresets: $\coreset[1],\coreset[2]$ have size $i$ after round $i$ \\ 
  {$\sgparam[0] \gets 0$}\qquad\qquad\qquad\qquad\qquad\quad /\!/ Swapping parameter \\
  \For{$i=1, 2, \ldots, \floor{n/2}$}
    {%
    /\!/ Construct kernel difference function using next two points \\
    $(\x, \x') \gets (\x_{2i-1}, \x_{2i})$;\quad
    $\invec[i] \gets \gkernel(\x_{2i-1}, \cdot)-\gkernel(\x_{2i}, \cdot)$; \quad $\eta_i \gets -1$ \\
	 \BlankLine
     /\!/ Compute swapping threshold $\cnew[i]$ \\ %
     
     $ \cnew[i], \sgparam[i] \gets $\proctwo{$\sgparam[i-1], \vmax[], \delta_i$}
     with\ \ $\vmax[]^2 \!=\! \norm{\invec[i]}_{\gkernel}^2 
      \!=\! \gkernel(\x,\x)\!+\!\gkernel(\x',\x')\!-\!2\gkernel(x,x')$
    \BlankLine
    /\!/ Compute RKHS inner product $\angles{\outvec[i-1], \invec[i]}_{\gkernel}$, which has a simple form \\
    $\wvprod[i]\gets  \sum_{j=1}^{2i-2}(\gkernel
	 (\x_j, \x)-\gkernel(\x_j,\x')) 
	 - 2\sum_{\z\in\coreset[1]}(\gkernel(\z, \x)-\gkernel(\z,\x'))$ \\
    \BlankLine
			 /\!/ Assign one point to each coreset after probabilistic swapping \\[2pt]
		     $(x, x') \gets (x', x)$ \text{ and } $\eta_i \gets 1$ \qtext{\textit{with probability}} $\min(1, \half (1-\frac{\wvprod[i]}{\cnew[i]})_+)$ \\ 
           $\coreset[1]\texttt{.append}(\x); 
		        \ \ \  \coreset[2]\texttt{.append}(\x'); \ \ \ 
		     \outvec[i]\gets \outvec[i-1] + \eta_i \invec[i] $ \ /\!/ $ \outvec[i]=\sum_{x'\in\coreset[2]}\!\gkernel(x', \cdot)\!-\!\sum_{x\in\coreset[1]}\!\gkernel(x, \cdot)$
  }
  \KwRet{\textup{$\coreset[1]$, coreset of size} $\floor{n/2}$}{} 
  } 
\end{algorithm2e}

\paragraph{KT-SWAP} The second stage, \ktswaplink, refines the candidate coresets produced by \ktsplit in three steps. First, \ktswap adds a baseline coreset of size $\floor{\frac{n}{2^m}}$ to the candidate list (for example, one produced by standard thinning or uniform subsampling) to ensure that the \ktswap output is never worse than that of the baseline. Next, it selects the candidate coreset closest to $\inputcoreset$ in terms of $\mmd_\kernel$. Finally, it refines the selected coreset by replacing each coreset point in turn with the best alternative in $\inputcoreset$,  as measured by $\mmd_\kernel(\inputcoreset, \cdot)$. 
This stage serves to greedily improve upon the MMD of the initial \ktsplit candidates, and, when computable, $\mmd_\kernel$ to the target distribution $\Pstar$ can be substituted for the surrogate  $\mmd_\kernel(\inputcoreset, \cdot)$ throughout. 

\setcounter{algocf}{0}
\renewcommand{\thealgocf}{\arabic{algocf}b}
\begin{algorithm2e}[ht!]
\caption{{\large\textsc{kt-swap}\ --\ } Identify and refine the best candidate coreset} 
  \label{algo:ktswap}
 \SetAlgoLined\DontPrintSemicolon
\small
{
    \KwIn{kernel $\kernel$, point sequence $\inputcoreset = (\axi[i])_{i = 1}^n$, candidate coresets $\ksplitcoresets$}
        \BlankLine
    $\coreset[\m,0] 
        \!\gets\! \texttt{baseline\_coreset}(\inputcoreset, \texttt{size}\!=\!\floor{n/2^\m})
     $ /\!/ Compare to baseline (e.g., standard thinning)
     \BlankLine
     $\ktcoreset \!\gets\! \coreset[\m, \ell^\star]
     \text{ for }
     \ell^\star 
        \!\gets\! \argmin_{\ell \in \braces{0, 1, \ldots, 2^\m}} \mmd_{\kernel}(\inputcoreset, \coreset[\m,\ell])$ 
        \ \ /\!/ \textup{Select best coreset} \\
        \BlankLine
        /\!/ Swap out each point in $\ktcoreset$ for best alternative in $\inputcoreset$ \\[1pt]
       \For{$i=1, \ldots, \floor{n/2^\m}$}{
       \BlankLine
        $\ktcoreset[i] \gets \argmin_{z\in\inputcoreset}\mmd_{\kernel}(\inputcoreset, \ktcoreset \text{ with } \ktcoreset[i] = z)$
       }
    \KwRet{
     $\ktcoreset$\textup{, refined coreset of size $\floor{n/2^\m}$}
    }
}
\end{algorithm2e}

\paragraph{Complexity}
For any $\m$, the time complexity of kernel thinning is dominated by $\order(n^2)$ kernel evaluations, 
while the space complexity is $\order(n\min(d,n))$, achieved by storing the smaller of the input sequence $(\axi)_{i=1}^n$ and the kernel matrix $(\ksqrt(\x_i,\x_j))_{i,j=1}^n$. %
In addition, scaling either $\kernel$ or $\ksqrt$ by a positive multiplier has no impact on  \cref{algo:kernel_thinning}, so the kernels need only be specified up to arbitrary rescalings.

\section{MMD Guarantees}\label{sec:mmd}
We are now prepared to present our main MMD guarantees.

\subsection{MMD guarantees for kernel thinning} %
\label{sub:mmd_bounds}
Our first main result, proved in \cref{sub:proof_of_theorem:main_result_all_in_one}, bounds the MMD of a kernel thinning coreset in terms of the input \cref{def:inputradii} and kernel \cref{eq:rmin_k} radii, the combined radii
\begin{talign}
     \rksmin \defeq \min\big(\rminpn[\inputcoreset], n^{1+\frac1d}\rmin_{\gkernel, n}+ n^{\frac1d} \frac{\sinfnorm{\gkernel}}{\klip[\gkernel]} \big)
     \qtext{and}
     \rksmax \defeq \max\big(\rminpn[\inputcoreset],\rktau\big),
     \label{eq:rmin_P}
\end{talign}
and the kernel thinning inflation factor 
\begin{talign}
 \label{eq:err_simple_defn}  
  \!\err_{\gkernel}(n,\! m,\! d,\! \delta,\! \delta'\!,\! R) 
     \!\defeq\!
     2\,\mathbb{I}(\frac{n}{2^m}\!\not\in\!\N)\!+\!
     37\sqrt{\! \log(\frac{6m}{2^m\delta})} \!\brackets{ \!\sqrt{\!\log({\frac{4}{\delta'}})} \!+\! 5 \sqrt{{\!d\log \!\big(2\!+\!\!\frac{2\klip[\gkernel]}{\sinfnorm{\gkernel}}\!(\rmin_{\gkernel, n} \!+\! R) \big) }} \!},
\end{talign}
defined for any kernel $\gkernel$ satisfying \cref{assum:lipkernel,assum:tailkernel}, $n, m,d \in\N$,  $\delta \in (0, \frac{6m}{2^m}]$, $\delta' \in(0, 1]$, and $R\geq 0$.
\newcommand{\mainresultallinonename}{MMD guarantee for kernel thinning}
\begin{theorem}[\mainresultallinonename]%
\label{theorem:main_result_all_in_one}
Consider kernel thinning (\cref{algo:kernel_thinning}) with $\kernel$ satisfying \cref{asmp:bounded_measurable}, 
$\ksqrt$ a square-root kernel of $\kernel$,  
$\delta^\star \defeq \min_{i} \delta_{i}$, and  
$\nout\defeq\floor{\textfrac{n}{2^{\m}}}$ for $m\leq \floor{\log_2 n}$. If $\ksqrt$ satisfies \cref{assum:lipkernel,assum:tailkernel}, then, for any fixed $\delta'\! \in\! (0, 1)$, we have
 \begin{talign} 
	\!\!\!\!\mmd_{\kernel}(\inputcoreset, \ktcoreset)
	\!\leq\!
	\frac{\sinfnorm{\ksqrt}}{n_{\mrm{out}}}
	\!\brackets{2
\!+ \!
\sqrt{\frac{(4\pi)^{d/2}}{\Gamma(\frac{d}{2}\!+\!1)}}
 (\rksmaxgen{\inputcoreset, \ksqrt, n_{\mrm{out}}})^{\frac{d}{2}}   \err_{\ksqrt}(n,m,\! d,\! \delta^\star, \!\delta',\!\rminnew[\inputcoreset, \ksqrt, n])}, 
\label{eq:mmd_thinning_bound_finite}
\end{talign}
with probability at least 
$1\!-\!\delta'\!-\!\sum_{j=1}^{\m} \frac{2^{j-1}}{\m} \sum_{i=1}^{2^{m-j}\nout}\delta_{i}$.
\end{theorem}
\begin{remark}[Guarantee for target $\Pstar$]
A guarantee for any target distribution $\Pstar$  follows directly from the triangle inequality, $\mmd_{\kernel}(\Pstar,\ktcoreset) \leq \mmd_{\kernel}(\Pstar, \inputcoreset) + \mmd_{\kernel}(\inputcoreset,\ktcoreset)$.
\end{remark}

\begin{remark}[Comparison with baseline thinning]
\label{rem:baseline}
The \ktswap step ensures that, deterministically,  $\mmd_{\kernel}(\inputcoreset, \ktcoreset)
	\leq 
	\mmd_{\kernel}(\inputcoreset, \basecoreset)$
for $\basecoreset$ a baseline thinned coreset of size $n_{\mrm{out}}$. %
Therefore, we additionally have
$%
\mmd_{\kernel}(\Pstar, \ktcoreset)
	\leq 
	2\mmd_{\kernel}(\Pstar, \inputcoreset)
	+ \mmd_{\kernel}(\Pstar, \basecoreset).
$%
\end{remark}

\newcommand{\failureremarkname}{Finite-time and anytime guarantees}
\begin{remark}[\failureremarkname]\label{finite_anytimee}
To obtain a success probability of at least $1-\delta$ with $\delta'=\frac{\delta}{2}$, it suffices to choose $\delta_i = \frac{\delta}{n}$ when the input size $n$ is known in advance and $\delta_i = \frac{m\delta}{2^{m+2}(i+1)\log^2(i+1)}$ when the input size $n$ is not known in advance (but is chosen independently of the randomness used in kernel thinning). %
In either case, $\delta^\star \leq \frac{6m}{2^m}$ is a valid argument to $\err_{\ksqrt}$  \cref{eq:err_simple_defn}. 
See \cref{sec:failure_prob_proof} for our proof. 
\end{remark}

\renewcommand{\thealgocf}{\arabic{algocf}}
Our next corollary, proved in \cref{sec:derivation_of_mmd_rates}, translates \cref{theorem:main_result_all_in_one} 
into specific rates of MMD decay depending on the radius growth of $\inputcoresetfull$ and the tail decay of $\ksqrt$.
\newcommand{\mmdcorollaryname}{MMD rates for kernel thinning} %
\begin{corollary}[\mmdcorollaryname]
\label{table:mmd_rates}
Under the notation and assumptions of \cref{theorem:main_result_all_in_one}, consider a sequence of kernel thinning runs (\cref{algo:kernel_thinning}), indexed by $n\in\naturals$, with $m = \floor{\half\log_2n}$, $\log(1/\delta^\star) = \order(\log n)$, and $\sum_{j=1}^{\m} \frac{2^{j-1}}{\m} \sum_{i=1}^{\floor{n/2^j}}\delta_{i} = o(1)$ as $n\to\infty$.
If $\inputcoresetfull$ and $\ksqrt$ respectively satisfy one of the radius growth (\cref{def:tailinputrate}) and tail decay (\cref{def:tailkernelrate}) conditions in the table below, then $\mmd_{\kernel}(\inputcoreset, \ktcoreset)=\order_d(\vareps_{\mmd, n})$ in probability where $\vareps_{\mmd, n}$ is the corresponding table entry.
\begin{table}[H]
    \centering
  \resizebox{\textwidth}{!}{
  {
    {
    \renewcommand{\arraystretch}{1}
    \begin{tabular}{ccccc}
        \toprule
        \Centerstack{ $\mbi{\vareps_{\mmd, n}}$ %
        }
        
        & \Centerstack{\bndcase\ $\inputcoresetfull$\\ $\rminpn[\inputcoreset] \precsim_d1$ 
        } 
        
        & \Centerstack{ \subgauss\ $\inputcoresetfull$ \\ $\rminpn[\inputcoreset] \precsim_d\sqrt{\log n}$ 
        } 
        
        & \Centerstack{ \subexp\ $\inputcoresetfull$ \\ $\rminpn[\inputcoreset] \precsim_d\log n$
        } 
       
        & \Centerstack{ \heavytail[\rho]\ $\inputcoresetfull$ \\ $\rminpn[\inputcoreset] \precsim_dn^{1/\htailparam}$ 
        }   
        \\[2mm]
        \midrule 
        
         \Centerstack{ \bndcase\ $\ksqrt$
         \\
        $\rkmax[\ksqrt] \precsim_d 1$}
          & $\sqrt{\frac{\log n}{n}}$
          & $(\log n)^{\frac{d+2}{4}}\sqrt{\frac{\log \log n}{n}}$
          & $(\log n)^{\frac{d+1}{2}}\sqrt{\frac{\log \log n}{n}}$
          & $\frac{\log n}{\sqrt{n^{1-d/\htailparam}}}$
          \\[6mm]

        \Centerstack{\subgauss\ $\ksqrt$ \\
        $\rkmax[\ksqrt] \precsim_d \sqrt{\log n}$
        }
          &  $(\log n)^{\frac{d+2}{4}}\sqrt{\frac{\log \log n}{n}}$
          &  $(\log n)^{\frac{d+2}{4}}\sqrt{\frac{\log \log n}{n}}$
          & $(\log n)^{\frac{d+1}{2}}\sqrt{\frac{\log \log n}{n}}$
          & $\frac{\log n}{\sqrt{n^{1-d/\htailparam}}}$
          \\[6mm]

         \Centerstack{ \subexp\ $\ksqrt$
         \\
        $\rkmax[\ksqrt] \precsim_d \log n$}
          & $(\log n)^{\frac{d+1}{2}}\sqrt{\frac{\log \log n}{n}}$
          & $(\log n)^{\frac{d+1}{2}}\sqrt{\frac{\log \log n}{n}}$
          & $(\log n)^{\frac{d+1}{2}}\sqrt{\frac{\log \log n}{n}}$
          & $\frac{\log n}{\sqrt{n^{1-d/\htailparam}}}$
          \\[6mm]

          \Centerstack{\heavytail[\rho']\ $\ksqrt$
         \\
        $\rkmax[\ksqrt] \precsim_d n^{1/\htailparam'}$}
          & $\frac{\log n}{\sqrt{n^{1-d/\htailparam'}}}$
          & $\frac{\log n}{\sqrt{n^{1-d/\htailparam'}}}$
          & $\frac{\log n}{\sqrt{n^{1-d/\htailparam'}}}$
          & $\frac{\log n}{\sqrt{n^{1-d/(\htailparam\wedge\htailparam')}}}$

        \\[1ex] \bottomrule
    \end{tabular}
    }
    }
    }
\end{table}
\end{corollary}
\begin{remark}[Probability parameters]
\label{rem:delta_sufficient}
    The condition $(\star) \ \sum_{j=1}^{\m} \frac{2^{j-1}}{\m}\! \sum_{i=1}^{\floor{n/2^j}}\delta_{i} \!=\! o(1)$ is satisfied when $\delta_1\!=\!\cdots \!=\! \delta_{\floor{\frac{n}{2}}} \!=\! o(\frac1n)$.
    Hence, both $\log(1/\delta^\star) \!=\! \order(\log n)$ and $(\star)$ are satisfied when, for example,  $\delta_1\!=\!\cdots \!=\!\delta_{\floor{\frac{n}{2}}} = \frac{1}{n\log\log n}$.
\end{remark}

\cref{table:mmd_rates} shows that kernel thinning returns an $(n^\half, \order_d(n^{-\half}\sqrt{\log n}))$-MMD coreset in probability
when $\inputcoresetfull$ and $\ksqrt$ are compactly supported. 
For fixed $d$, this guarantee significantly improves upon the baseline $\Omega_p(n^{-\quarter})$  rates of \iid sampling and standard MCMC thinning and matches the minimax lower bounds of \cref{sec:related} up to a $\sqrt{\log n}$ term and constants depending on $d$. 
For example, when $\inputcoresetfull$ is drawn \iid from $\Pstar$, kernel thinning is nearly minimax optimal amongst \emph{all} distributional approximations (even weighted coresets and non-coreset approximations) that depend on $\P$ only through $n$ \iid input points \citep[Thms.~1 and 6]{tolstikhin2017minimax}.

More generally, when 
$\inputcoresetfull$ and $\ksqrt$ are \subgauss, \subexp, or 
$\heavytail$ with $\htailparam >2d$, 
\cref{table:mmd_rates} shows that the kernel thinning provides an MMD error of $\order_d(n^{-\half} \sqrt{(\log n)^{d/2+1} \log \log n})$, $\order_d(n^{-\half}\sqrt{(\log n)^{d+1} \log \log n}))$, and $\order_d(n^{-\half} n^{\frac{d}{2\htailparam}}\log n)$ in probability with output coresets of size $\sqrt{n}$. In each case, we find that kernel thinning significantly improves upon an $\Omega_p(n^{-\quarter})$ baseline when $n$ is sufficiently large relative to $d$ and, by \cref{rem:baseline}, is never significantly worse than the baseline when $n$ is small.
Our \subexp guarantees also resemble the classical quasi-Monte Carlo guarantees for the uniform distribution on $[0,1]^d$ (see \cref{sec:related}) but allow for non-uniform and unbounded target distributions $\P$.

\newcommand{\mmdtablename}{Kernel thinning MMD guarantee under $\P$ and $\ksqrt$ tail decay}

\newcommand{\mmdtablecaption}{
\caption{
    \tbf{\mmdtablename.} 
    For $n$ input points and $\sqrt{n}$ thinned points, we report the $\mmd_{\kernel}(\inputcoreset,\ktcoreset)$ bound of \cref{theorem:main_result_all_in_one} up to constants depending on $d$, $\delta$, $\delta'$, $\sinfnorm{\ksqrt}$, and ${\klip[\ksqrt]}/{\sinfnorm{\ksqrt}}$.
    Here, $\rkmax[\ksqrt]\defeq \max(\rk[\ksqrt], \rktaugen[\ksqrt, \sqrt{n}])$ and the radii $(\rminpn[\inputcoreset], \rk[\ksqrt], \rktaugen[\ksqrt, \sqrt{n}])$ are defined in \cref{eq:rmin_P,,eq:rmin_k}.
    See \cref{sec:derivation_of_mmd_rates} for our derivation.
    } 
}

\cref{theorem:main_result_all_in_one} also allows us to derive more precise, explicit error bounds for specific kernels. 
For example, for the popular Gaussian, \Matern, and B-spline kernels, \cref{table:sqrtk_details} provides explicit bounds on each kernel-dependent quantity in \cref{theorem:main_result_all_in_one}: $\sinfnorm{\ksqrt}$, the kernel radii $(\rk[\ksqrt], \rktaugen[\ksqrt, \sqrt{n}])$, and the inflation factor $\err_{\ksqrt}$. 

\newcommand{\sqrtdetailstablename}[1]{Explicit bounds on {#1} quantities for common kernels}
\newcommand{\sqrtdetailstablecaption}{
{\noindent\caption{
    \tbf{\sqrtdetailstablename{\cref{theorem:main_result_all_in_one}}.}
    Here, $A_{\matone,\mattwo, d},$ and $\wtil{S}_{\splineparam, d}$ are as in \cref{table:kernel_sqrt_pair}, $ a\defeq\frac12(\matone\!-\!d) (>1),B\!\defeq\! a\log(1\!+a), E \defeq d\log(\frac{\sqrt{2e\pi}}{\mattwo}) \!+\!  \log(\frac{(\matone-2)^{\matone-\frac{3}{2}}}{
    (2(a-1))^{2a-1}d^{\frac{d}{2}+1}})$, $c_1 = \frac{2}{\sqrt 3}$, and $c_\splineparam<1$ for $\splineparam>1$ (see~\cref{eq:spline_krt_inf}).
    See \cref{sec:proof_of_table_sqrtk_details} for our derivation.\label{table:sqrtk_details}}}}
\begin{table}[t!]
    \centering
  \resizebox{\textwidth}{!}{
\small
  {
    \renewcommand{\arraystretch}{1}
    \begin{tabular}{ccccc}
        \toprule
        \Centerstack{\bf Square-root  \\  \bf kernel $\ksqrt$} 
        
        & \Centerstack{$\sinfnorm{\ksqrt}$ 
        } 
        
        & \Centerstack{ $(\rk[\ksqrt],\rktau[\ksqrt])$ \\
        $\precsim$
        } 
        
        & \Centerstack{
        $\err_{\ksqrt}(n, \half \log_2 n, d, \frac{\delta}{n},  \delta', R)$ \\
        $\precsim$
        } 
        
        \\[2mm]
        \midrule 
        
         \Centerstack{
      		$\parenth{\frac{2}{\pi\gaussparam^2}}^
      		{\frac{d}{4}}\textbf{Gaussian}\parenth{\frac{\gaussparam}
      		{\sqrt 2}}$
      		}
      	& $\parenth{\frac{2}{\pi\gaussparam^2}}^{\frac{d}{4}}$
      	&$(\gaussparam\sqrt{\log n}, \gaussparam\sqrt{d+\log n})$
    	& \Centerstack{
             $\sqrt{\log(\frac{n}{\delta}) [\log(\frac{1}{\delta'})\! +\! d \log \! \big(
            \!\sqrt{\log n}\! +\! \frac{R}{\sigma} \big)] \!  }$
      		}	
        
          \\[6mm]

        \Centerstack{$A_{\matone,\mattwo,d}$\textbf{Mat\'ern}$(\frac{\matone}{2}, \mattwo)$} 
        & $ \matone (\frac{\mattwo^2}{2\pi(a-1)})^{\frac{d}{4}}$

        &\Centerstack{$\big(\mattwo\inv({\log n\!+\!a\log(1\!+\!a)}),$
        \\
        $\mattwo\inv(a\!+\!\log n\!+\!E))$
        }
          & \Centerstack{
            $\sqrt{\log(\frac{n}{\delta}) [\log(\frac{1}{\delta'}) \! + \! d  \log \! \big(\!\log n\! + \! B\! + \! \mattwo R\big)]\!   }$
      		}
    
          \\[6mm]

          \Centerstack{
         $\wtil{S}_{\splineparam,d}\textbf{B-spline}(\splineparam)$}
         & $c_\splineparam^d$

       & $(\frac12{\sqrt{d}(\splineparam\!+\!1)}, \frac12{\sqrt{d}(\splineparam\!+\!1)})$
       
         & \Centerstack{
      		$\sqrt{ \log(\frac{n}{\delta}) [ \log(\frac{1}{\delta'})
            \! +\!  d \log(d\splineparam\! +\! \sqrt{d}R) ] \! }$
      		}

        \\[1ex] \bottomrule \hline
    \end{tabular}
    }
    }
    \opt{arxiv}{\sqrtdetailstablecaption}
    \opt{jmlr}{\sqrtdetailstablecaption}
\end{table}

\newcommand{\cdim}{e_d}
\newcommand{\linfmmdcoresetresultname}{%
MMD error from square-root $\Linf$ error}
\newcommand{\linfcoresetresultname}{MMD guarantee for square-root $\Linf$ approximations}%
\subsection{%
MMD coresets from square-root $\Linf$ coresets} %
\label{sec:a_general_recipe_for_bounding_kernel_mmd}
\cref{theorem:main_result_all_in_one} builds on a second key result, proved in \cref{sub:proof_of_theorem:coreset_to_mmd}, that bounds MMD error for $\kernel$ in terms of $\Linf$ error for the \emph{square-root} kernel $\ksqrt$. 
\begin{theorem}[\linfcoresetresultname]
\label{theorem:coreset_to_mmd}
Suppose $\kernel$ satisfying \cref{asmp:bounded_measurable} has a square-root kernel $\ksqrt$ satisfying \cref{assum:tailkernel}. 
Then for any
	 distributions $\pnew$ and $\qnew$ on $\reals^d$ and scalars $r, a,b \geq 0$ with
	$a+b=1$, %
	\begin{align}
	\label{eq:coreset_to_mmd_bound}
		\mmd_{\kernel}(\pnew,\qnew)
		\!\leq \! \cdim  r^{\frac{d}{2}}\!\cdot\! \infnorm{\pnew\ksqrt\!-\!\qnew\ksqrt}\!
		+ 2\tail[\ksqrt](a r)\!
		+ 2\infnorm{\kernel}^{\frac12}\! \cdot\! \max\braces{\tail[\pnew](br),\! \tail
		[\qnew]
		(br)},
	\end{align}
	where $\cdim \defeq (\pi^{d/2}/\Gamma(d/2+1))^{\frac12}$ decreases super-exponentially in $d$ and $\tail[\pnew](r) \defeq \pnew(\ball^c(0,r))$.
\end{theorem}
Importantly, \cref{theorem:coreset_to_mmd} implies that \emph{any} $\Linf$ coreset for the square-root kernel $\ksqrt$, even one not produced by kernel thinning, is also an MMD coreset for the target kernel $\kernel$ with MMD error depending on the tail decay of $(\ksqrt, \pnew, \qnew)$. 
\cref{cor:mmdlinf} summarizes the implications of \cref{theorem:coreset_to_mmd} for common classes of tail decay. See \cref{proof_of_cor:mmdlinf} for the proof with explicit constants. 
\begin{corollary}[\linfmmdcoresetresultname]
	\label{cor:mmdlinf}
    Under the setting and assumptions of \cref{theorem:coreset_to_mmd}, 
    define the $\Linf$ error $\vareps \defeq \infnorm{\pnew\ksqrt-\qnew\ksqrt}$ and the tail decay function $\wtil{\tail[]}(r)\!\defeq\! \tail[\ksqrt](r) \!+\! \infnorm{\kernel}^{1/2} \max\sbraces{\tail[\mu](r), \tail[\nu](r)}$ for $r\!\geq\! 0$. Then the following implications hold for any nonnegative $c$ and $\rho$.
	\begin{table}[H]
	\centering
	\begin{tabular}{ccccc}
	\toprule
		\textsc{Tail Decay} & 
            \bndcase &  \subgauss & \subexp & \heavytail\\
		$\wtil{\tail[]}(r) \precsim $ &
		$\indicator(r\leq c)$
		& $e^{-cr^2}$
		& $e^{-cr}$
		& $r^{-\rho}$
		\\
		\midrule
		 $\Rightarrow \mmd_{\kernel}(\pnew,\qnew)\precsim_{d}$ &
		$\vareps$
		& $ \vareps (\log\frac{1}{\vareps})^{\frac d4}$ 
		& $ \vareps (\log\frac{1}{\vareps})^{\frac d2}$ 
		& $\vareps^{\frac{2\rho}{d+2\rho} }$ 
		\\ 
		\bottomrule 
	\end{tabular}	
	\end{table}

	\end{corollary}
    \begin{remark}[Tail decay from finite moments]
    \label{tail_moments}
	By Markov's inequality \citep[Thm.~1.6.4]{durrett2019probability}, $\wtil{\tail[]}$ has (i) \bndcase decay when $\ksqrt$ is \bndcase and $\pnew$ and $\qnew$ have compact support; (ii) \subgauss decay when $\ksqrt$ is \subgauss and $\E_{X\sim\pnew}[e^{c\twonorm{X}^2}], \E_{X\sim\qnew}[e^{c\twonorm{X}^2}]<\infty$; (iii) \subexp decay when $\ksqrt$ is \subexp and $\E_{X\sim\pnew}[e^{c\twonorm{X}}], \E_{X\sim\qnew}[e^{c\twonorm{X}}]<\infty$; and (iv) \heavytail decay when $\ksqrt$ is \heavytail and $\E_{X\sim\pnew}[\twonorm{X}^\rho], \E_{X\sim\qnew}[\twonorm{X}^\rho]<\infty$.
    \end{remark}
    \cref{cor:mmdlinf} highlights that MMD quality for $(\kernel, \pnew)$ is of the same order as $\Linf$ quality for $(\ksqrt,\pnew)$ when $\ksqrt, \pnew,$ and the approximation $\qnew$ have compact support. MMD quality then degrades naturally as the tail behavior worsens.
	In \cref{sec:from_self_balancing_walk_to_kernel_thinning,sub:kernel_halving}, we show that, with high probability, \ktsplit provides a high-quality $\Linf$ coreset for $(\ksqrt, \inputcoreset)$ and hence, by \cref{cor:mmdlinf}, also provides a high-quality MMD coreset for $(\kernel, \inputcoreset)$. 

\newcommand{\vseq}[1][n]{\mc F_{#1}}
\newcommand{\vseqnum}{\invec[1], \invec[2], \ldots}
\newcommand{\vseqnumn}[1][n]{\invec[1], \invec[2], \ldots, \invec[#1]}
\newcommand{\pvseqnum}{(\invec[i])_{i\geq1}}
\newcommand{\pvseqnumn}[1][n]{(\invec[i])_{i=1}^{#1}}
\section{Self-balancing Hilbert Walk}
\label{sec:from_self_balancing_walk_to_kernel_thinning}
To exploit the $\Linf$-MMD connection revealed in \cref{theorem:coreset_to_mmd}, we now turn our attention to constructing high-quality thinned $\Linf$ coresets.
Our strategy relies on a new Hilbert space generalization of the self-balancing walk of \citet{alweiss2021discrepancy}.  
We dedicate this section to defining and analyzing this \emph{self-balancing Hilbert walk}, and we detail its connection to kernel thinning in \cref{sub:kernel_halving}.
\setcounter{algocf}{2}
\begin{algorithm2e}[ht!]
\small{
  \KwIn{sequence of functions $(\invec[i])_{i=1}^n$ in a  Hilbert space $\rkhs$, threshold sequence $(\cnew[i])_{i=1}^n$}
  {$\outvec[0]\gets \boldzero \in \rkhs$}\\
  \For{$i=1, 2, \ldots, n$}
    {%
   $\wvprod[i] \gets \dotrkhs{\outvec[i-1], \invec[i]}{\rkhs}$ \quad /\!/ Compute Hilbert space inner product \\
    \textbf{if} $\abss{\wvprod[i]} > \cnew[i]\tbf{:}$ \\
    \qquad
    $\outvec[i] \gets \outvec[i-1]-\invec[i]  \cdot \wvprod[i]/\cnew[i]$
	\\
    \textbf{else:} \\
     \qquad 
     $\eta_i \gets 1 \qtext{with probability}
     	 \half \parenth{1-{\wvprod[i]}/{\cnew[i]}}
     	 \qtext{and} \eta_i \gets -1 \qtext{otherwise}$ \\
     \qquad 
     $\displaystyle \outvec[i] \gets
     \outvec[i-1] + \eta_i \invec[i]$
  }{}
\KwRet{$\outvec[n]$, \textup{combination of signed input functions}}{}
  }
  \caption{Self-balancing Hilbert Walk}
  \label{algo:self_balancing_walk}
\end{algorithm2e}

\citet[Thm.~1.2]{alweiss2021discrepancy} introduced
a randomized algorithm called the \emph{self-balancing walk} that takes as
input a streaming sequence of Euclidean vectors $\axi\in
\Rd$ with $\twonorm{\axi}\leq 1$ and outputs a online sequence of random  assignments $\eta_i \in \braces{-1, 1}$ satisfying
\begin{talign}
\label{eq:alweiss_vector_balancing}
	\infnorm{\sumn \eta_i\axi} \precsim \sqrt{\log(d/\delta)\log(n/\delta)}
	\qtext{with probability at least $1-\delta$.}
\end{talign}
Since our ultimate aim is to combine kernel functions, we define a suitable Hilbert space generalization in \cref{algo:self_balancing_walk}.

Given a streaming sequence of functions $\invec[i]$ in an arbitrary Hilbert space $\rkhs$ with a norm $\hnorm{\cdot}$, this \emph{self-balancing Hilbert walk} (SBHW) outputs a streaming sequence of signed function combinations $\outvec[i]$ satisfying the following desirable properties established in \cref{sec:proof_of_sbhw_properties}.

\newcommand{\indices}{\mc{I}}
\newcommand{\sbhwpropname}{Self-balancing Hilbert walk properties}
\begin{theorem}[\sbhwpropname]
\label{sbhw_properties}
    Consider the self-balancing Hilbert walk (\cref{algo:self_balancing_walk})
	with each 
     $\invec[i] \in \rkhs$ and $\cnew[i]>0$
    and define the 
    sub-Gaussian constants
	\begin{talign} \label{eq:subgauss_const_def}
    \sgparam[0]^2
        \defeq 0
    \qtext{and}
    \sgparam[i]^2
        \defeq \sgparam[i-1]^2 + \hnorm{\invec[i]}^2\Big(1 + \frac{\sgparam[i-1]^2}{\cnew[i]^2}(\hnorm{\invec[i]}^2 - 2\cnew[i])\Big)_+
    \quad \forall i \geq 1.
    \end{talign}
    The following properties hold true.
	\begin{enumerate}[label=(\roman*),leftmargin=*]
	\itemsep0em
	\item\label{item:fun_subgauss} \tbf{Functional sub-Gaussianity:}
    For each $i \in [n]$,
	$\outvec[i]$ is $\sgparam[i]$ \emph{sub-Gaussian}:
		\begin{talign}
		\label{eq:subgaussian_bound_on_w}
			\E[\exp(\doth{\outvec[i], \genvec})]
			\leq
		\exp\parenth{\frac{\sgparam^2\norm{\genvec}^2_{\rkhs}}{2}}
			\qtext{for all}\genvec \in \rkhs.
		\end{talign}
	\item\label{item:signed_sum} \tbf{Signed sum representation:} 
	If 
	$\cnew[i]
	    \geq \sgparam[i-1]\hnorm{\invec[i]}\sqrt{2\log(2/\delta_i)}$ 
	for $\delta_i \in (0,1]$, then, with probability at least $1-\sum_{i=1}^n \delta_i$, 
	\begin{talign}\label{eq:signed_sum}
	    \abss{\wvprod[i]} \leq \cnew[i], \forall i \in [n],
	    \qtext{and}
	    \outvec[n] 
	        = \sum_{i=1}^n \eta_i \invec[i].
	\end{talign}
	\item\label{item:exact_two_thin} \tbf{Exact halving via symmetrization:} 
	If 
	$\cnew[i]
	    \geq \sgparam[i-1]\hnorm{\invec[i]}\sqrt{2\log(2/\delta_i)}$ 
	for $\delta_i \in (0,1]$ and each $\invec[i] = g_{2i-1} - g_{2i}$ for $g_1, \dots, g_{2n} \in \rkhs$, then, with prob.\ at least $1-\sum_{i=1}^n \delta_i$,
	\begin{talign}
	\abss{\wvprod[i]} \leq \cnew[i], \forall i \in [n], \ \text{ and }\ 
	\frac{1}{2n}\outvec[n] 
	    = \frac{1}{2n}\sum_{i = 1}^{2n} g_i - \frac{1}{n}\sum_{i \in \indices} g_i
	    \ \text{ for } \ %
	    \indices 
	    = \{2i+\frac{\eta_i-1}{2} : i\in [n]\}.
	\end{talign}
	\item\label{item:pointwise_subgauss} \tbf{Pointwise sub-Gaussianity in RKHS:} If $\rkhs$ is the RKHS of a kernel $\gkernel : \X \times \X \to \reals$, then, for each $i \in [n]$ and $\x\in\X$, $\outvec[i](x)$ is $\sgparam\sqrt{\gkernel(\x,\x)}$ sub-Gaussian:
		\begin{talign}
			\label{eq:subgaussian_bound_on_function_value}
			\E[\exp(\outvec[i](\x))] %
			\leq
			\exp\parenth{\frac{\sigma_{i}^2\gkernel(\x,\x)}{2}}.
		\end{talign}
	\item\label{item:subgauss_bound}
	\tbf{Sub-Gaussian constant bound:} 
	Fix any $q \in [0,1)$, and suppose $\half\hnorm{\invec[i]}^2 \leq \cnew$ for all $i\in[n]$.
	If $\frac{\hnorm{\invec[i]}^2}{1+q} 
    \leq \cnew \leq \frac{\hnorm{\invec[i]}^2}{1-q}$  
	whenever both
	$\sgparam[i-1]^2 < \frac{\cnew^2}{2\cnew-\hnorm{\invec[i]}^2}$
	and
	$\hnorm{\invec[i]} > 0$, then
    \begin{talign}
    \label{eq:sbhw_subgauss_bound}
    \sgparam[i]^2 
        \leq \frac{\max_{j\in[i]} \hnorm{\invec[j]}^2}{1-q^2}
    \qtext{for all}
    i \in [n].
    \end{talign} 
    \item\label{item:adaptive_thresh}
	\tbf{Adaptive thresholding:}
	If $\cnew = \max(c_i \sgparam[i-1] \hnorm{\invec[i]}, \hnorm{\invec[i]}^2)$ for $c_i \geq 0$,
	then 
	\begin{talign}
	\sgparam[n]^2 
        \leq \frac{\max_{i\in[n]} \hnorm{\invec[i]}^2}{4}(c^\star+1/c^\star)^2
    \qtext{for}
    c^\star \defeq \max_{i \in [n]} c_i.
	\end{talign}
	\end{enumerate}
\end{theorem}
\begin{remark}
The kernel $\gkernel$ in Property~\ref{item:pointwise_subgauss} can be arbitrary and need not be bounded.
\end{remark}

Property~\ref{item:fun_subgauss}  
ensures that the functions $\outvec[i]$ produced by \cref{algo:self_balancing_walk} are mean zero and unlikely to be large in any particular direction $\genvec$.
Property~\ref{item:signed_sum} 
builds on this functional sub-Gaussianity to ensure that $\outvec[n]$ is precisely a sum of the signed input functions $\pm \invec$ with high probability. The two properties together imply that, with high probability and an appropriate setting of $\cnew$, \cref{algo:self_balancing_walk} partitions the input functions $\invec$ into two groups such that the function sums are nearly balanced across the two groups. 
Property~\ref{item:exact_two_thin} 
uses the signed sum representation to construct a two-thinned coreset for any input function sequence $(g_i)_{i=1}^{2n}$.  This is achieved by offering the consecutive function differences $g_{2i-1} - g_{2i}$ as the inputs $\invec[i]$ to \cref{algo:self_balancing_walk}.
Property~\ref{item:pointwise_subgauss} 
highlights that functional sub-Gaussianity also implies  sub-Gaussianity of the function values $\outvec[i](x)$ whenever the Hilbert space $\rkhs$ is an RKHS. 
Finally, Properties~\ref{item:subgauss_bound} and \ref{item:adaptive_thresh} 
provide explicit bounds on the sub-Gaussian constants  $\sgparam[i]$ when adaptive settings of the thresholds $\cnew[i]$ are employed. 
In \cref{sub:kernel_halving}, we will connect the SBHW to kernel halving and use Properties~\cref{item:exact_two_thin,item:pointwise_subgauss,item:adaptive_thresh} together to show that kernel halving and hence also \ktsplit coresets have provably small $\Linf$ kernel error. 
Specifically, we will boost the pointwise sub-Gaussianity (Property~\ref{item:pointwise_subgauss}) of the output function $\outvec[n]$ into a high probability bound for $\infnorm{\outvec[n]} = \sup_x |\outvec[n](x)| = \sup_x |\inner{\outvec[n]}{\gkernel(x,\cdot)}|$ by constructing a finite cover for $(\outvec[n](x))_{x\in\Rd}$ based on the decay and smoothness of the kernel $\gkernel$.

\paragraph{Comparison with \iid signs} 
A simple alternative to \cref{algo:self_balancing_walk} is to assign signs uniformly at random to each vector, that is, to output $\outvec[n]'= \sum_{i=1}^n \eta_i' \invec[i]$ with independent Rademacher $\eta_i' \in \sbraces{\pm 1}$.
Since the minimal squared sub-Gaussian constant of a sum of independent weighted Rademachers is equal to its variance \citep[Lem.~5 \& Ex.~1]{buldygin1980subgaussian}, 
the minimal squared sub-Gaussian constant of $\outvec[n]'$ satisfies 
\begin{talign}
\sigma_{\mrm{IID}}^2 
    = \sup_{u\in\rkhs}\Var(\doth{\outvec[n]', \genvec})/\hnorm{\genvec}^2
    = \sup_{u\in\rkhs}\sum_{i=1}^n \doth{\invec[i], \genvec}^2/\hnorm{\genvec}^2.
\end{talign}
In the best case, all $\invec[i]$ are orthogonal and bounded and $\sigma_{\mrm{IID}}^2 = \max_{i\in[n]}\hnorm{\invec}^2$ does not grow with $n$; the reader can check that \cref{algo:self_balancing_walk} also reduces to \iid signing in this case.
However, in the worst case, all $\invec[i]$ are equal, and $\sigma_{\mrm{IID}}^2 = n \max_{i\in[n]}\hnorm{\invec}^2$.
In contrast, 
if we choose $\cnew[i]$ as in Property \cref{item:adaptive_thresh}  with $c_i = \sqrt{2\log(2n/\delta)}$, then the SBHW output 
$\outvec[n]= \sum_{i=1}^n \eta_i \invec[i]$  with probability $1-\delta$ by Property \cref{item:signed_sum} and has squared sub-Gaussian constant $\sigma_{\mrm{SBHW}}^2 = \order(\log(2n/\delta)\max_{i\in[n]}\hnorm{\invec[i]}^2)$ in every case by Property \cref{item:adaptive_thresh}. 
This drop from $\Omega(n)$ to $\order(\log n)$ represents an exponential improvement in worst-case balance over employing \iid signs.

We can attribute this gain to the carefully chosen updates of \cref{algo:self_balancing_walk}. 
Notice that, on round $i$, the function $\outvec[i-1]$ is updated only in the $\invec$ direction, so it suffices to examine the evolution of $\dotrkhs{\outvec[i-1],\invec}{\rkhs}$.
We show in \cref{sec:fun_subgauss} that this evolution takes the form
\begin{align}\label{eq:targeted_shrinkage}
\dotrkhs{\outvec,\invec}{\rkhs} 
    = 
(1-\hnorm{\invec[i]}^2/\cnew)\dotrkhs{\outvec[i-1],\invec}{\rkhs} 
    +
\temprv \hnorm{\invec[i]}^2
\end{align}
where $\temprv \defeq \indic{\abss{\wvprod[i]} \leq \cnew[i]}(\eta_i + \wvprod[i]/\cnew)$ is mean-zero and $1$-sub-Gaussian given $\outvec[i-1]$.
In other words, whenever $\cnew \geq \hnorm{\invec[i]}^2$ (as recommended in Property~\ref{item:adaptive_thresh}), \cref{algo:self_balancing_walk} first shrinks the magnitude of $\outvec[i-1]$ in the $\invec$ direction before adding a sub-Gaussian variable in this direction. 
This targeted shrinkage is absent in the \iid signing update, 
\begin{align}
\dotrkhs{\outvec',\invec}{\rkhs} 
    = 
\dotrkhs{\outvec[i-1]',\invec}{\rkhs} 
    +
\eta_i' \hnorm{\invec[i]}^2,
\end{align}
which simply adds a sub-Gaussian variable in the $\invec$ direction, and allows the SBHW to maintain a substantially smaller sub-Gaussian constant.

\paragraph{General recipe for exact halving}
The symmetrization construction introduced in Property~\ref{item:exact_two_thin} can be used to convert any vector balancing algorithm (i.e., any algorithm which assigns $\pm 1$ signs to a sequence of vectors) into an exact halving algorithm (i.e., one which assigns $-1$ to exactly half of the points) simply by running the algorithm on paired vector differences.
We will use this property in the sequel to painlessly construct coresets of an exact target size.

\paragraph*{Comparison with the self-balancing walk of \citet{alweiss2021discrepancy}} In the Euclidean setting with $\rkhs = \Rd$, constant thresholds $\cnew = 30\log(n/\delta)$, and $\dotrkhs{\outvec[i-1],\invec}{\rkhs}$ the usual Euclidean dot product, \cref{algo:self_balancing_walk} recovers a slight variant of the Euclidean self-balancing walk of \citet[Proof of Thm.~1.2]{alweiss2021discrepancy}. 
The original algorithm differs only superficially by terminating with failure whenever $\abss{\wvprod[i]} > \cnew[i]$.
We allow the walk to continue with the update $\outvec[i] \gets \outvec[i-1]-\invec[i]  \cdot \wvprod[i]/\cnew[i]$, as it streamlines our sub-Gaussianity analysis and avoids the reliance on distributional symmetry present in Sec.~2.1 of \citet{alweiss2021discrepancy}.
We show in \cref{proof_of_euclidean_vector_balancing} that \cref{sbhw_properties} recovers the guarantee \cref{eq:alweiss_vector_balancing} of \citet[Thm.~1.2]{alweiss2021discrepancy} with improved constants and a less conservative setting of $\cnew$.

 \section{$\Linf$ Guarantees}
 \label{sub:kernel_halving}
 In this section, we derive near-optimal $\Linf$ coreset guarantees for kernel halving (\cref{algo:kernel_halving}) and \ktsplitlink by relating the two algorithms to the self-balancing Hilbert walk (\cref{algo:self_balancing_walk}).
\subsection{$\Linf$ guarantees for kernel halving}
\label{sub:kh}
To make the connection between KH and SBHW more apparent, we have translated each line of 
\cref{algo:kernel_halving} into the notation of \cref{algo:self_balancing_walk}.
In this notation, we see that \cref{algo:kernel_halving} forms signed combinations $\outvec$ of paired kernel differences $\invec[i]=\gkernel(\axi[2i-1], \cdot)-\gkernel(\axi[2i], \cdot)$; that 
the inner product $\wvprod[i] = \angles{\outvec[i-1], \invec[i]}_{\gkernel}$ %
has a simple explicit form in terms of kernel evaluations; and 
that, under the event  $\event=\sbraces{\abss{\wvprod[i]} \leq \cnew[i] \stext{for all} i=1, \ldots, \floor{n/2}}$, the function $\outvec[\floor{n/2}]$ of \cref{algo:kernel_halving} exactly matches the output of SBHW. 
Indeed, the function \proctwo serves to compute the sub-Gaussian constants $\sigma_i$ exactly as defined in \cref{sbhw_properties} and to adaptively select the thresholds $\cnew$ exactly as recommended in \cref{sbhw_properties}\cref{item:exact_two_thin,item:adaptive_thresh}: $\cnew = \max(\hnorm{\invec[i]}\sgparam[i-1]\sqrt{2\log(2/\delta_i)}, \hnorm{\invec[i]}^2)$.

This choice of $\cnew$ simultaneously ensures that the KH-SBHW equivalence event $\event$ occurs with high probability by \cref{sbhw_properties}\cref{item:exact_two_thin} and that the SBHW sub-Gaussian constant remains small by \cref{sbhw_properties}\cref{item:adaptive_thresh}.
Hence, we can 
invoke the pointwise SBHW sub-Gaussianity revealed in  \cref{sbhw_properties}\cref{item:pointwise_subgauss}
to control the KH $\Linf$ coreset error %
with high probability.

\newcommand{\kernelhalvingresultsname}{$\Linf$ guarantees for kernel halving}
\newcommand{\kernelhalvingoneroundname}{Kernel halving yields a  $2$-thinned $\Linf$ coreset}
\newcommand{\kernelhalvingmultiroundname}{Repeated kernel halving yields a $2^\m$-thinned $\Linf$ coreset}
\begin{theorem}[\kernelhalvingresultsname]
\label{kernel_halving_results}
Let $\khcoreset(\gkernel, \cset, \Delta)$ denote the output of kernel halving (\cref{algo:kernel_halving}) with kernel $\gkernel$ satisfying \cref{assum:lipkernel,assum:tailkernel}, input point sequence $\cset$, and probability sequence $\Delta$. 
Let $\P_n \defeq \textfrac{1}{n}\textsum_{i=1}^n\dirac_{x_i}$, and
recall the definitions of $\err_{\gkernel}$ \cref{eq:err_simple_defn}
 and $\rminnew[\inputcoreset, \gkernel, n]$ \cref{eq:rmin_P}.
The following statements hold for any $m\in\naturals$ with $m \le \log_2 n$ and $\delta'\in(0,1)$.
\begin{enumerate}[label=(\alph*),leftmargin=*]
\item\label{item:kernel_halving_one_round} \emph{\textbf{\kernelhalvingoneroundname:}} The output $\coreset[1] \defeq \khcoreset(\gkernel, \inputcoreset, (\delta_i)_{i=1}^{\floor{\frac{n}{2}}})$ has size $\nout = \floor{\frac{n}{2}}$ with ${\Q}_{\mrm{KH}}^{(1)}\!\defeq\!\textfrac{1}{{\nout}}\textsum_{x\in\coreset[1]}\dirac_x$ 
satisfying
    \begin{talign}
        \label{eq:kernel_halving_bound_finite}
            \sinfnorm{\P_{n}\gkernel-{\Q}_{\mrm{KH}}^{(1)}\gkernel}
            \leq
             \frac{\sinfnorm{\gkernel}}{\nout} \cdot \err_{\gkernel}(n, 1, d, \delta^\star,  \delta', \rksmin[\gkernel])
    \end{talign}
    with probability at least $1\!-\!\delta'\!-\!\sum_{i=1}^{\nout} \delta_i$
    for $\delta^\star\defeq\min_{i\leq \nout} \delta_i$.
    \item\label{item:kernel_halving_multiple_rounds}
    \emph{\textbf{\kernelhalvingmultiroundname:}} 
    For each $j > 1$, let $\coreset[j] \defeq \khcoreset(\gkernel, \coreset[j-1], (\frac{2^{j-1}}{m}\delta_i)_{i=1}^{\floor{n/2^{j}}})$
    be the output of kernel halving recursively applied for $j$ rounds. Then 
    $\coreset[\m]$ has size $\nout = \floor{\frac{n}{2^{\m}}}$ with ${\Q}_{\mrm{KH}}^{(\m)}\!\defeq\!\textfrac{1}{\nout}\textsum_{x\in\coreset[\m]}\dirac_x$ satisfying
    \begin{talign}
        \sinfnorm{\P_{n}\gkernel\!-\!\Q^{(\m)}_{\mrm{KH}}\gkernel}
        &\leq \frac{\sinfnorm{\gkernel}}{\nout}\cdot\err_{\gkernel}(n, m, d, \delta^\star, \delta', \rksmin[\gkernel])
        \label{eq:kernel_halving_bound_m_rounds}
    \end{talign}
    with probability at least $1\!-\!\delta'\!-\!\sum_{j=1}^{\m} \frac{2^{j-1}}{m} \sum_{i=1}^{2^{m-j}\nout}\delta_i$ for $\delta^\star\defeq\min_{i\leq 2^m\nout} \delta_i$.
\end{enumerate}
\end{theorem}

\cref{kernel_halving_results}, proved in \cref{sec:proof_of_kernel_halving_results}, shows that $\Linf$ error for KH scales simply as the kernel thinning inflation factor $\err_{\gkernel}$~\cref{eq:err_simple_defn} divided by the size of the output.
Our next corollary, an immediate consequence of 
 \cref{kernel_halving_results}\cref{item:kernel_halving_multiple_rounds} and the definition of $\err_{\gkernel}$~\cref{eq:err_simple_defn}, translates these bounds into rates of decay depending on the radius growth of $\inputcoresetfull$ and the tail decay of $\gkernel$. 
\newcommand{\khcorollaryresultname}{$\Linf$ rates for kernel halving}
\begin{corollary}[\khcorollaryresultname]
	\label{cor:optimallinf}
	Under the notation and assumptions of \cref{kernel_halving_results}, 
    consider a sequence of repeated kernel halving runs (\cref{algo:kernel_halving}), indexed by $n\in\naturals$, with $m = \floor{\half\log_2n}$ rounds,
    $\log(1/\delta^\star) = \order(\log n)$, and $\sum_{j=1}^{\m} \frac{2^{j-1}}{\m} \sum_{i=1}^{\floor{n/2^j}}\delta_{i} = o(1)$ as $n\to\infty$.
    If $\inputcoresetfull$ and $\gkernel$ respectively satisfy one of the radius growth (\cref{def:tailinputrate}) and tail decay (\cref{def:tailkernelrate}) conditions in the table below, 
 then $\sinfnorm{\P_n\gkernel\!-\!\Q^{(\m)}_{\mrm{KH}}\gkernel}=\order(\vareps_{\infty, n})$ in probability where $\vareps_{\infty, n}$ is the corresponding table entry and all hidden constants are independent of the dimension $d$.
	\begin{table}[H]
    \centering
  \resizebox{\textwidth}{!}{
  {
    {
    \begin{tabular}{ccc}
        \toprule
        \Centerstack{ \textsc{Type of $\inputcoresetfull$} \\  \textsc{Type of $\gkernel$ }}
        
        & \Centerstack{\bndcase with $\rminpn[\inputcoreset] \!=\! \order(\sqrt{d})$ \\ 
        \bndcase with $\frac{\klip(\rk\wedge\rminpn[\inputcoreset]) }{\sinfnorm{\gkernel}}  \!=\! \order(1)$
        } 
        
        & \Centerstack{ \textsc{Arbitrary} \\  
        \subpoly with $\log (\frac{\klip\rk}{\sinfnorm{\gkernel}})  \!=\! \order(\log n)$
        }

        \\[2mm]
        \midrule 
        
         \Centerstack{ \Large $\mbi{\vareps_{\infty, n}}$ }

          & {\Large $\sqrt{\frac{d\log n}{n}}$}
          & {\Large $\sqrt{\frac{d}{n}}$}\! $\log n$
        \\[1ex] \bottomrule
    \end{tabular}
    }
    }
    }
\end{table}
\end{corollary}
\newcommand{\linfrateexamplename}{Example settings for KH $\Linf$ rates}
\begin{remark}[\linfrateexamplename]\label{rem:linf-rate-examples}
\cref{rem:delta_sufficient} specifies settings of $(\delta_i)_{i=1}^{\floor{n/2}}$ that meet the conditions of \cref{cor:optimallinf}. 
For any radial kernel $\gkernel(\x,\y) = \kappa(\twonorm{x-y}/\sigma)$ with $L$-Lipschitz $\kappa: \reals\to\reals$, bandwidth $\sigma > 0$, and $\sinfnorm{\gkernel} = 1$, we have $\frac{\klip}{\sinfnorm{\gkernel}} \leq \frac{L}{\sigma}$, $\rk \leq \sigma\kappa^{\dagger}({1}{/n})$, and therefore $\frac{\klip\rk}{\sinfnorm{\gkernel}} \leq L\, \kappa^{\dagger}({1}{/n})$ where $L$ and $\kappa^{\dagger}(u) \defeq \sup\{ r : \kappa(r) \geq u \}$ are independent of $d$ and $\sigma$.
Hence Gaussian, \Matern (with $\matone > \frac{d}{2}+1$),
and IMQ kernels satisfy the \subpoly requirements of \cref{cor:optimallinf} with any choice of bandwidth and satisfy the \bndcase requirements when restricted to a compact domain with $\sigma = \sqrt{d}$. See \cref{sec:proof_of_rem:linf-rate-examples} for our proof.
\end{remark}

\paragraph{Near-optimal $\Linf$ coresets}
For any bounded, {radial} $\gkernel$  satisfying mild decay and smoothness conditions,  \citet[Thm.~3.1]{phillips2020near} proved that any procedure outputting a coreset of size $n^{\half}$ must suffer $\Omega(\min(\sqrt{d}n^{-\half},n^{-\quarter}))$ $\Linf$ error with constant probability for some $\P_n$. Hence, the repeated KH quality guarantees from \cref{cor:optimallinf} are within a $\log n$ factor of minimax optimality for this kernel family which includes Gaussian, \Matern, Wendland, and IMQ kernels.

\paragraph{Online vector balancing in an RKHS}
In the online vector balancing problem of  \citet{spencer1977balancing}, one must assign signs $\eta_i$ to Euclidean vectors $\invec[i]$ in an online fashion while keeping the norm of the signed sum $\infnorm{\sum_{i=1}^{n} \eta_i \invec[i]}$ as small as possible. 
As an immediate consequence of \cref{kernel_halving_results}\cref{item:kernel_halving_one_round}, we find that kernel halving solves an RKHS generalization of the online vector balancing problem.

\begin{corollary}[Online vector balancing in an RKHS]
\label{cor:kernel_balancing}
Consider a sequence of kernel halving runs (\cref{algo:kernel_halving}), indexed by $n\in\naturals$, with 
    $\log(1/\delta^\star) = \order(\log n)$ and $\sum_{i=1}^{\floor{n/2^j}}\delta_{i} = o(1)$ as $n\to\infty$. If the kernel $\gkernel$ is \subpoly (\cref{def:tailkernelrate}) with $\log (\frac{\klip\rk}{\sinfnorm{\gkernel}})  = \order(\log n)$, 
    then $\staticnorm{\sum_{i=1}^{2n} \eps_{i}\gkernel(\x_{i},\cdot)}_\infty =  \order(\sqrt{d}\log n)$ in probability for the generated signs $\eps_i \defeq \eta_{\ceil{i/2}}(-1)^i$.
\end{corollary}
\begin{proof}
By \cref{kernel_halving_results}\cref{item:kernel_halving_one_round} and the definition of $\err_{\gkernel}$~\cref{eq:err_simple_defn}, $\staticnorm{\sum_{i=1}^{n} \eta_{i}\invec}_\infty \!=\!  \order(\sqrt{d}\log n)$ in probability. Noting that $-\eta_{i}\invec \!=\! \eta_i (\gkernel(\x_{2i},\cdot) \!-\! \gkernel(\x_{2i-1},\cdot)) \!=\! \eps_{2i-1}\gkernel(\x_{2i-1},\cdot)+\eps_{2i}\gkernel(\x_{2i},\cdot)$, the claim follows.
\end{proof}
\cref{cor:kernel_balancing} applies to \emph{any} fixed input point sequence $\inputcoresetfull$ and to a broad range of kernels including, by \cref{rem:linf-rate-examples},  Gaussian, \Matern, and IMQ kernels, as well as B-spline kernel since \cref{eq:spline_klip,eq:rmin_rdag_spline} imply that $\log (\frac{\klip\rk}{\sinfnorm{\gkernel}}) = \order(\log d) =  \order(\log n)$ for B-spline $\gkernel$.
\newcommand{\mccoreset}[1][\frac{n}{2}]{\cset_{\mrm{iid},#1}}
\subsection{$\Linf$ and MMD guarantees for KT-SPLIT}%
\label{sub:kernel_thinning_is_recursive_kernel_halving}
We finally extend our $\Linf$ and MMD guarantees to the \ktsplitlink step of kernel thinning by observing that each candidate coreset generated by \ktsplit is the product of repeated kernel halving (\cref{algo:kernel_halving}) with $\ksqrt$ as the chosen kernel $\gkernel$.
Hence, \cref{kernel_halving_results}\cref{item:kernel_halving_one_round} applies equally to the coreset $\coreset[1,1]$ returned by \ktsplit with $\m=1$, and, when $\m > 1$, \cref{kernel_halving_results}\cref{item:kernel_halving_multiple_rounds} applies to the coreset $\coreset[m,1]$ produced by \ktsplit.
Combining these $\Linf$ bounds with our $\Linf$ to MMD conversion theorem (\cref{theorem:coreset_to_mmd}) yields the following corollary proved in \cref{sec:proof_of_corollary:kernel_thinning_coreset_bound}. 

\newcommand{\kernelthinningrecursiveresultsname}{$\Linf$ and MMD guarantees for \ktsplitlink}
\begin{corollary}[\kernelthinningrecursiveresultsname]
\label{corollary:kernel_thinning_coreset_bound}
Consider \ktsplitlink with $\ksqrt$ satisfying \cref{assum:tailkernel,assum:lipkernel}, 
$\delta^\star \defeq \min_i\delta_i$, 
and $\P_n \defeq \textfrac{1}{n}\textsum_{i=1}^n\dirac_{x_i}$.
The guarantees of \cref{kernel_halving_results} with $\gkernel=\ksqrt$ hold if $\Q^{(\m, 1)}_{\mrm{KT}}\!\defeq\!\textfrac{1}{\nout}\textsum_{x\in\coreset[\m, 1]}\dirac_x$ is substituted for $\Q^{(\m)}_{\mrm{KH}}$, and the guarantees of 
\cref{theorem:main_result_all_in_one}
 hold if the output coreset $\coreset[\m, 1]$ is substituted for $\ktcoreset$, 
\end{corollary}

\cref{corollary:kernel_thinning_coreset_bound} ensures that, with high probability, at least one \ktsplit candidate is a high-quality $\Linf$ coreset for $(\ksqrt, \P_n)$ and hence also a high-quality MMD coreset for $(\kernel, \P_n)$.
The proof of \cref{theorem:main_result_all_in_one} in \cref{sub:proof_of_theorem:main_result_all_in_one} establishes the same MMD guarantee for kernel thinning by noting that the subsequent \ktswaplink step directly minimizes the $\mmd_\kernel$ to $\P_n$ and hence can only improve or maintain this MMD quality. 
Finally, exactly as in the proofs of \cref{cor:optimallinf,table:mmd_rates} we can deduce both $\Linf$ and MMD rate bounds for \ktsplit coresets.

\begin{corollary}[$\Linf$ and MMD rates for \ktsplitlink]\label{corollary:ktsplit_rates}
Under the notation and assumptions of \cref{corollary:kernel_thinning_coreset_bound}, consider a sequence of \ktsplitlink runs, indexed by $n\in\naturals$, with $m = \floor{\half\log_2n}$ rounds,
    $\log(1/\delta^\star) = \order(\log n)$, and $\sum_{j=1}^{\m} \frac{2^{j-1}}{\m} \sum_{i=1}^{\floor{n/2^j}}\delta_{i} = o(1)$ as $n\to\infty$.
The guarantees of \cref{cor:optimallinf} with $\gkernel=\ksqrt$ hold  if $\Q^{(\m, 1)}_{\mrm{KT}}$ is substituted for $\Q^{(\m)}_{\mrm{KH}}$, and the guarantees of \cref{table:mmd_rates} hold if $\coreset[\m, 1]$ is substituted for $\ktcoreset$.
\end{corollary}

\section{Vignettes}%
\label{sec:vignettes}
We complement our primary methodological and theoretical development with
two vignettes 
illustrating 
the promise of kernel thinning for improving upon (a) \iid sampling in dimensions $d=2$ through $100$ and (b) standard MCMC thinning when targeting challenging differential equation posteriors.
See \cref{sec:vignettes_supplement} for 
supplementary details and 
\begin{center}
\url{https://github.com/microsoft/goodpoints}
\end{center}
for a Python implementation of kernel thinning and 
code replicating each vignette.

\subsection{Common settings}
Throughout, we adopt a $\textbf{Gaussian}(\sigma)$ target kernel $\kernel(\x, \y) =  \exp(-\frac{1}{2\sigma^2}\twonorm{\x-\y}^2)$ 
and the corresponding square-root kernel $\ksqrt$ from \cref{table:kernel_sqrt_pair}. %
To output a coreset of size $n^{\half}$ with $n$ input points, we (a) take every $n^{\half}$-th point for standard thinning and (b) run kernel thinning (KT) with $\m=\half \log_2 n$ using a standard thinning coreset as the base coreset in \ktswap. For each input sample size $n \in \braces{2^{4}, 2^{6}, \ldots, 2^{14}, 2^{16}}$ with  $\delta_i = \frac{1}{2n}$, we report the mean coreset error $\mmd_\kernel(\Pstar, \cset)$ 
  $\pm 1$ standard error across $10$ independent replications of the experiment (the standard errors are too small to be visible in all experiments). 
We additionally regress the log mean MMD onto the log input size using ordinary least squares and display both the best linear fit and an empirical decay rate based on the slope of that fit, e.g., for a slope of $-0.25$, we report an empirical decay rate of $n^{-0.25}$ for the mean~MMD.

\begin{figure}[ht!]
    \centering
    \resizebox{\linewidth}{!}{
    \begin{tabular}{cc}
    \includegraphics[width=0.5\linewidth]{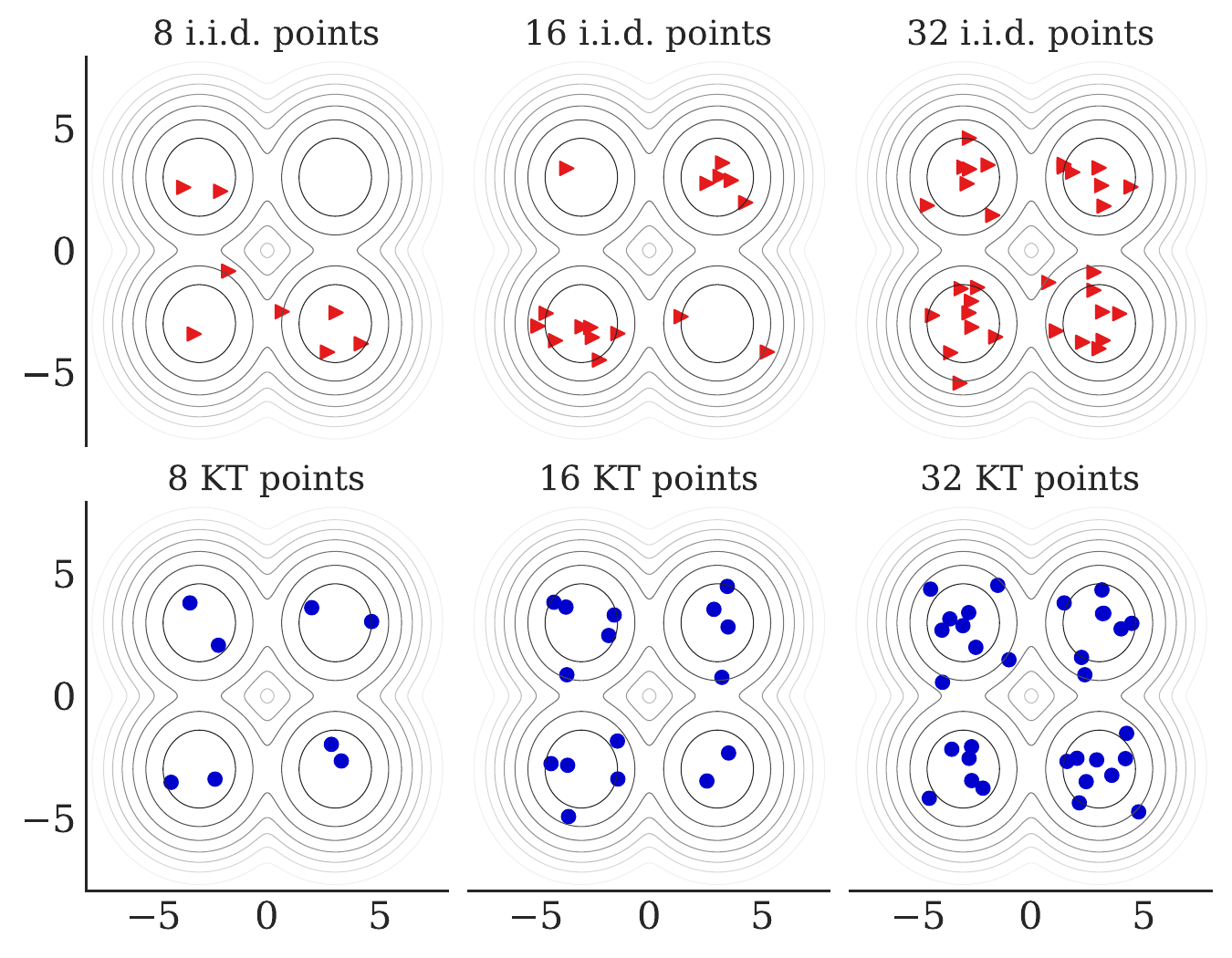} &
    \includegraphics[width=0.5\linewidth]{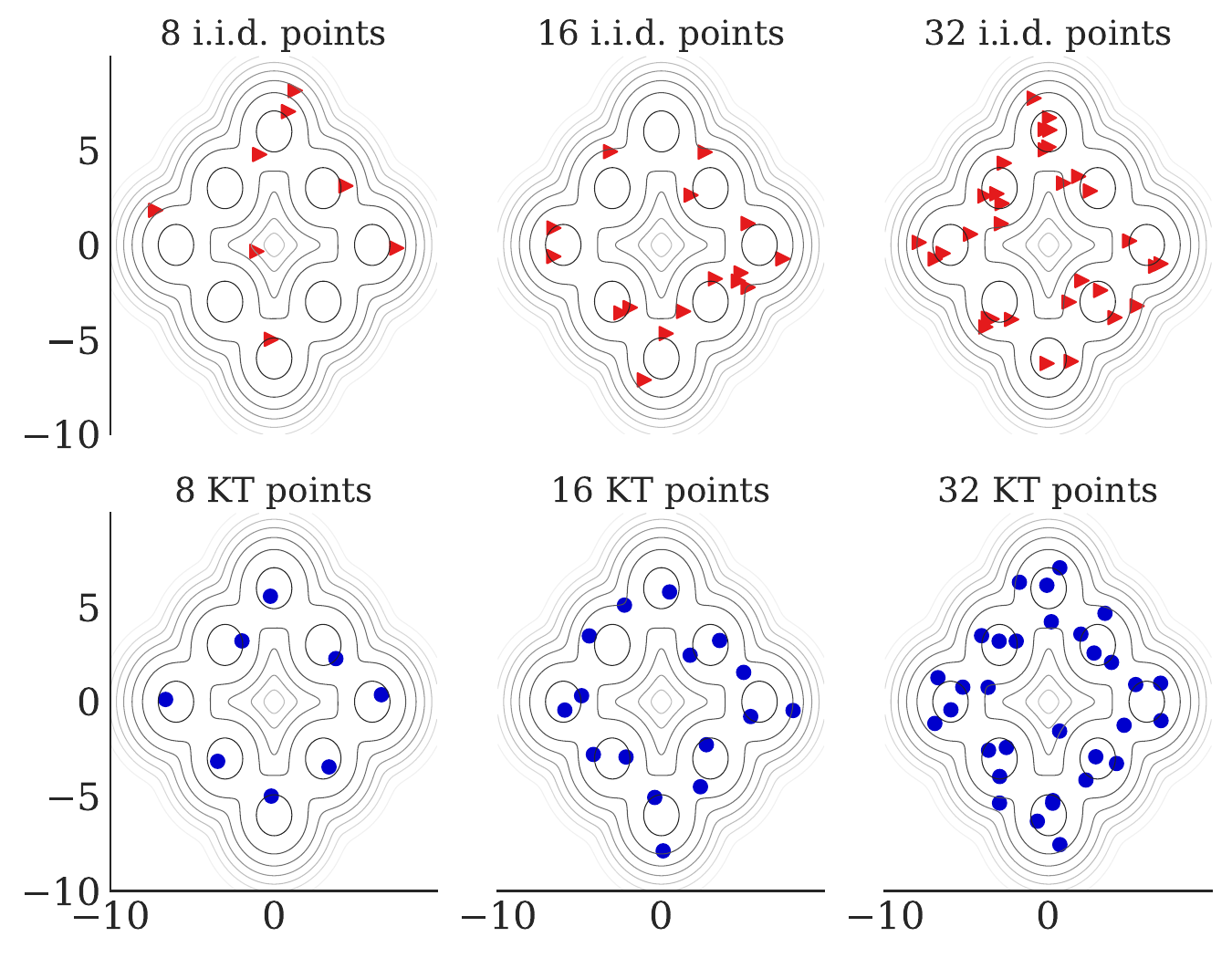}
    \end{tabular}
    }
    \caption{Kernel thinning (KT) and \iid coresets for 4- and 8-component mixture of Gaussian targets with equidensity contours of the target underlaid.
    See \cref{sec:iid_vignette} for more details.}
    \label{fig:mog_scatter}
\end{figure}
\subsection{Kernel thinning versus \iid sampling}
\label{sec:iid_vignette}
We first illustrate the benefits of kernel thinning over \iid sampling from a target $\Pstar$.  
We generate each input sequence $\inputcoreset$ \iid from $\P$, use squared kernel bandwidth $\sigma^2=2d$, and consider both Gaussian targets $\Pstar = \mc N(0, \mathbf{I}_d)$ for $d \in \braces{2, 4, 10, 100}$ and
mixture of Gaussians (MoG) targets $\Pstar = \frac{1}{M}\sum_{j=1}^{M}\mc{N}(\mu_j, \mbf{I}_2)$ with $M \in \braces{4, 6, 8}$ component locations $\mu_j\in\reals^2$ defined in \cref{sec:mog_supplement}. 

 \cref{fig:mog_scatter} highlights the visible differences between the KT and \iid coresets for the MoG targets.
Even for small sample sizes, the KT coresets achieves better stratification across components with less clumping and fewer gaps within components, suggestive of a better approximation to $\Pstar$.
Indeed, when we examine MMD error as a function of coreset size in 
\cref{fig:mmd_gaussp_plots}, we observe that kernel thinning provides a significant improvement across all settings.
For the Gaussian $d=2$ target and each MoG target, the KT MMD error scales as $n^{-\half}$, a quadratic improvement over the $\Omega_p(n^{-\frac14})$ MMD error of \iid sampling.
Notably, we would not expect to see an exact empirical rate of $n^{-\half}$ for larger $d$ and small $n$ due to the logarithmic factors in our MMD bounds. 
However, even for small sample sizes and high dimensions, we observe in \cref{fig:gauss_mmd} that KT significantly improves both the magnitude and the decay rate %
of MMD relative to \iid sampling.

\begin{figure}[th!]
    \centering
    \subfigure[Mixture of Gaussians $\mbb{P}$]{\label{fig:mog_mmd}%
    \resizebox{\linewidth}{!}{
    \begin{tabular}{c}
    \includegraphics[width=\linewidth]{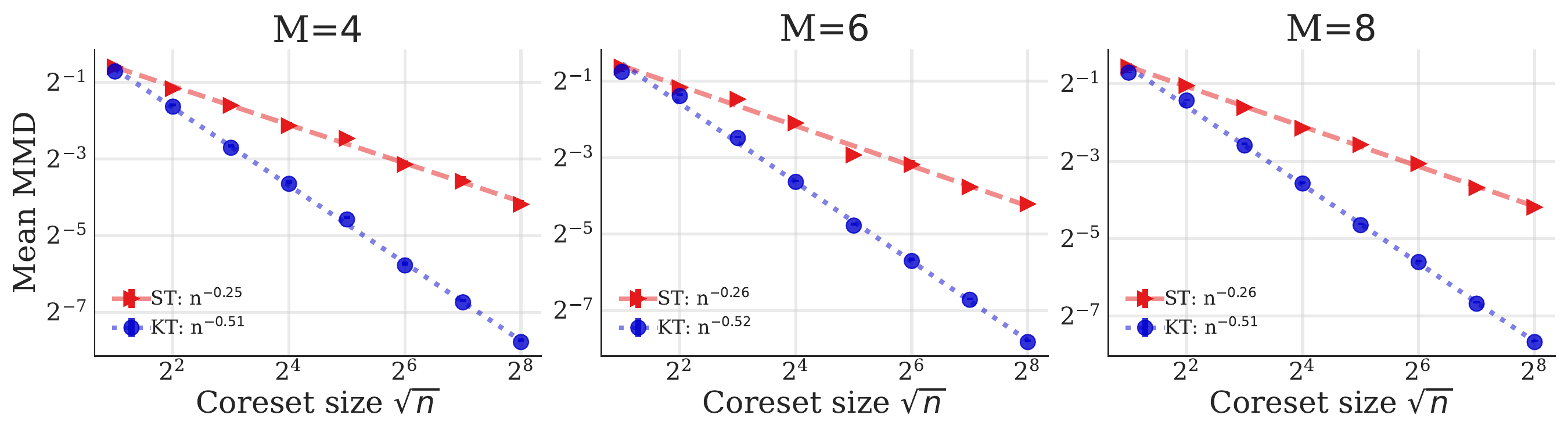}
    \end{tabular}
    }
    }
    \vspace{1mm}
    \hrule
    \vspace{3mm}
    \subfigure[ Gaussian $\P$ ]{\label{fig:gauss_mmd}%
      \resizebox{\linewidth}{!}{
    \begin{tabular}{c}
    \includegraphics[width=\textwidth]{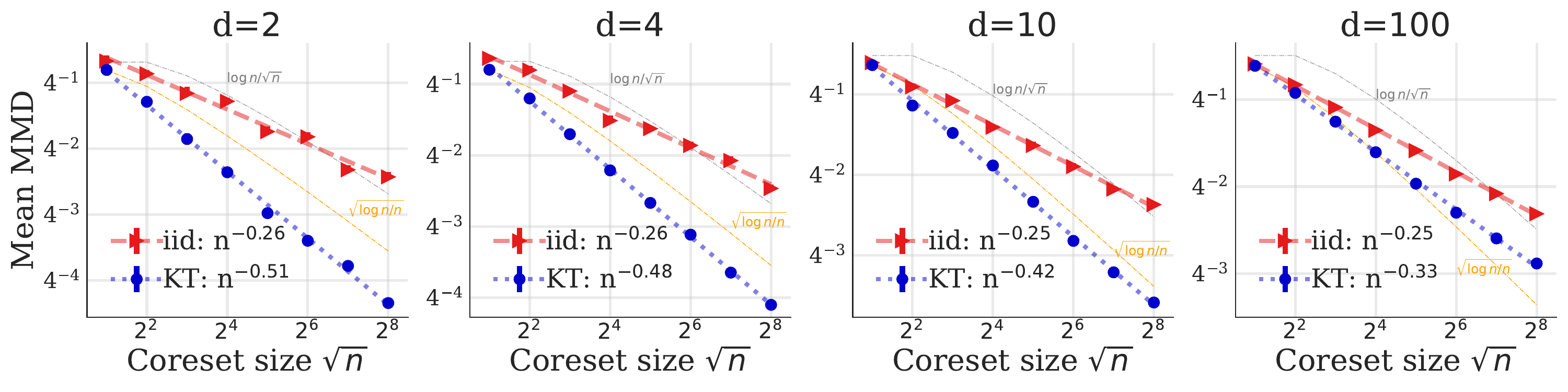}
    \end{tabular}
    }
    }
     \caption{\tbf{Kernel thinning versus \iid sampling.}
    For (a) mixture of Gaussian targets with $M \in\braces{4, 6, 8}$ components and (b) standard multivariate Gaussian targets in dimension $d \in \braces{2, 4, 10, 100}$,  kernel thinning (KT) reduces $\mmd_{\kernel}(\Pstar, \mc{S})$ significantly more quickly than the standard $n^{-\frac14}$ rate for $n^{\half}$ \iid points, even in high dimensions. For reference, decay rates of $\sqrt{\log n/n}$  and $\log n/\sqrt{n}$ are plotted in each of the figures in panel (b).
    }
    \label{fig:mmd_gaussp_plots}
\end{figure}

\subsection{Kernel thinning versus standard MCMC thinning}
\label{sub:mcmc_vignettes}
Next, we illustrate the benefits of kernel thinning over standard Markov chain Monte Carlo (MCMC) thinning on twelve %
posterior inference experiments conducted by \citet{riabiz2021optimal}. 
We briefly describe each experiment here and refer the reader to \citet[Sec.~4]{riabiz2021optimal} for more details.

\paragraph*{Goodwin and Lotka-Volterra experiments} From \citet{DVN/MDKNWM_2020}, we obtain the output of four distinct MCMC procedures targeting each of two $d=4$-dimensional posterior distributions $\Pstar$: (1) a posterior over the parameters of the \emph{Goodwin model} of oscillatory enzymatic control \citep{goodwin1965oscillatory} and (2) a posterior over the parameters of the \emph{Lotka-Volterra model} of oscillatory predator-prey evolution \citep{lotka1925elements,volterra1926variazioni}.
For each target, \citet{DVN/MDKNWM_2020} provide $2\times 10^6$ sample points from each of four MCMC algorithms: Gaussian random walk (RW), adaptive Gaussian random walk \citep[adaRW,][]{haario1999adaptive}, Metropolis-adjusted Langevin algorithm \citep[MALA,][]{roberts1996exponential}, and pre-conditioned MALA \citep[pMALA,][]{girolami2011riemann}. 

\paragraph*{Hinch experiments}
From \citet{DVN/MDKNWM_2020}, we also obtain the output of two independent Gaussian random walk MCMC chains for each of two $d=38$-dimensional posterior distributions $\P$: (1) a posterior over the parameters of the \emph{Hinch model} of calcium signalling in cardiac cells \citep{hinch2004simplified} and (2) a tempered version of the same posterior, as defined by \citet[App.~S5.4]{riabiz2021optimal}.
In computational cardiology, the calcium signalling model represents one component of a heart simulator, and one aims to propagate uncertainty in the signalling model through the whole heart simulation, an operation which requires thousands of CPU hours per sample point \citep{riabiz2021optimal}.
In this setting, the costs of running kernel thinning are dwarfed by the time required to generate the input sample (two weeks) and more than offset by the cost savings in the downstream uncertainty propagation task.

\paragraph*{Preprocessing and kernel settings}
We discard the initial points of each chain as burn-in using the maximum burn-in period reported in \citet[Tabs. S4 \& S6, App.~S5.4]{riabiz2021optimal}. 
and normalize each Hinch chain by subtracting the post-burn-in sample mean and dividing each coordinate by its post-burn-in sample standard deviation.
To form an input sequence $\inputcoreset$ of length $n$ for coreset construction, we downsample the remaining points using standard thinning.
Since exact computation of $\mmd_\kernel(\Pstar, \mc{S})$ is intractable for these posterior targets, we report $\mmd_\kernel(\inputcoreset, \mc{S})$---the error that is controlled directly in our theoretical results---for these experiments.
We select the kernel bandwidth $\gaussparam$ using the popular median heuristic \citep[see, e.g.,][]{garreau2017large}.  %
Additional details can be found in \cref{sec:mcmc_supplement}.

\begin{figure}[th!]
    \centering
    \includegraphics[width=\linewidth]{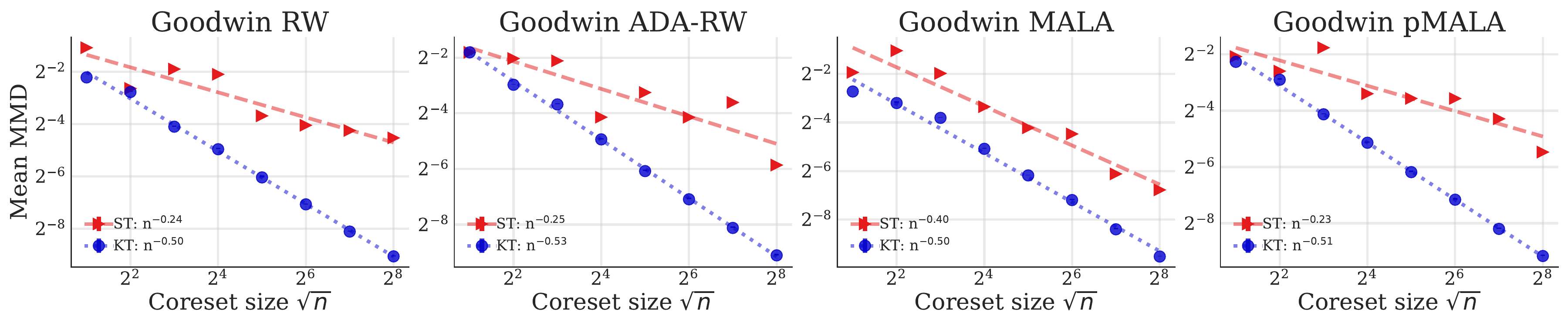}
    \hrule
    \vspace{4mm}
    \includegraphics[width=\linewidth]{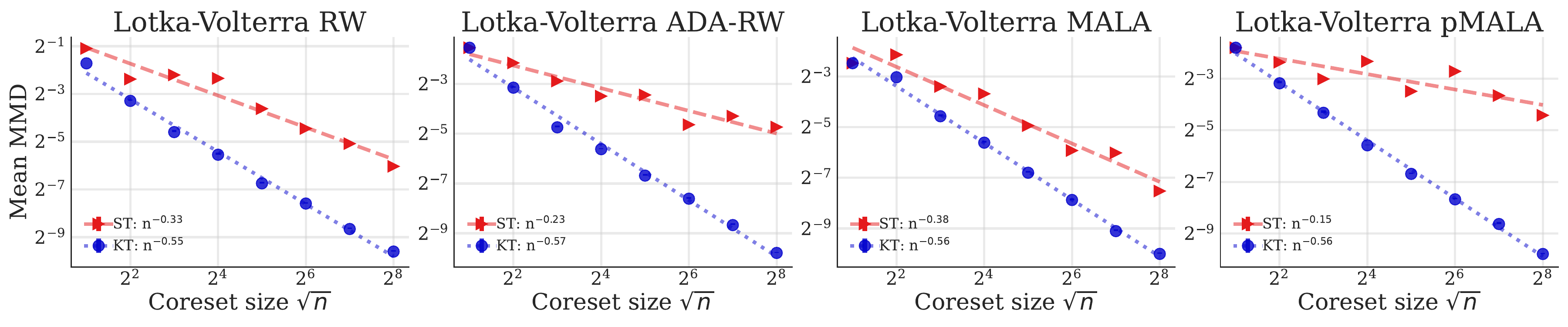} 
    \hrule
    \vspace{4mm}
    \includegraphics[width=\linewidth]{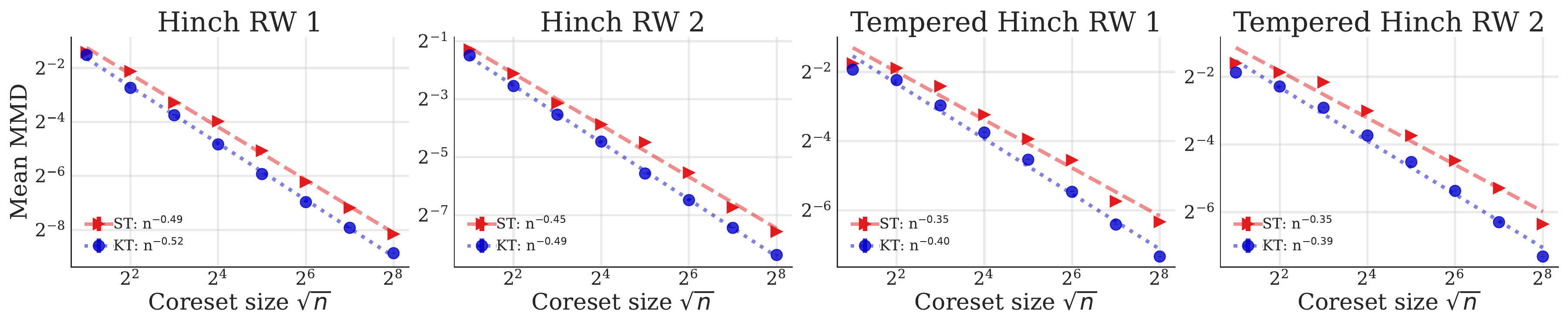} 
    \caption{\tbf{Kernel thinning versus standard MCMC thinning.}
    Kernel thinning (KT) significantly improves both the rate of decay and the order of magnitude of mean $\mmd_{\kernel}(\inputcoreset, \mc{S})$ in each posterior inference task, including eight tasks with 4-dimensional targets (Goodwin and Lotka-Volterra) and four tasks with 38-dimensional targets (Hinch). 
    See \cref{sub:mcmc_vignettes} for more details.
    }
    \label{fig:mmd_mcmc_plot}
\end{figure}

\paragraph*{Results}
\cref{fig:mmd_mcmc_plot} compares the mean  $\mmd_\kernel(\inputcoreset, \mc{S})$ error of the generated kernel thinning and standard thinning coresets.
In each of the twelve experiments, KT significantly improves both the rate of decay and the order of magnitude of mean MMD, in line with the guarantees of \cref{theorem:main_result_all_in_one}. 
Notably, in the $d=38$-dimensional Hinch experiments, standard thinning already improves upon the $n^{-\frac14}$ rate of \iid subsampling but is outpaced by KT which consistently provides further improvements.

\section{Discussion}\label{sec:discussion}
We introduced kernel thinning (\cref{algo:kernel_thinning}), a new, practical solution to the thinned MMD coreset problem that, given $\order(n^2)$ time and $\order(n\min(d,n))$ storage, improves upon the integration error of \iid sampling and standard MCMC thinning. 
To achieve this we first showed that any $\Linf$ coreset for a square-root kernel $\ksqrt$ also provides an MMD coreset for its associated target kernel $\kernel$ (\cref{theorem:coreset_to_mmd}).
We next introduced and analyzed a self-balancing Hilbert walk for solving the online vector balancing problem in Hilbert spaces (\cref{algo:self_balancing_walk,sbhw_properties}). 
We then designed a symmetrized version of SBHW for RKHSes---kernel halving---that delivers $2$-thinned coresets with small $\Linf$ error (\cref{algo:kernel_halving,kernel_halving_results}). 
Our online algorithm, \ktsplitlink,  recursively applies kernel halving to a square-root kernel to obtain near-optimal $\Linf$ coresets in $\order(n^2)$ time with $\order(n\min(d,n))$ space (\cref{cor:optimallinf,corollary:kernel_thinning_coreset_bound}).
Kernel thinning then combines \ktsplit with a greedy refinement step (\ktswaplink) to yield coresets with better-than-\iid MMD for a broad range of kernels and target distributions (\cref{theorem:main_result_all_in_one,table:mmd_rates}).
While our analysis is restricted to kernels that admit square-root dominating kernels, \cite{dwivedi2022generalized} recently generalized the KT algorithm and analysis to support arbitrary kernels. 
Separately, \cite{shetty2022distribution} have developed a  distribution compression meta-algorithm, Compress++, which reduces the runtime of KT to near-linear $\order(n \log^{3}n)$ time with MMD error that is worse by at most a factor of $4$. Hence, KT-Compress++ can be practically deployed even for very large input sizes. 
\subsection{Related work on MMD coresets}\label{sec:related} 
While $(n^\half, o_p(n^{-\quarter}))$-MMD coresets have been developed for specific $(\Pstar,\kernel)$ pairings like the uniform distribution on the unit cube paired with a Sobolev kernel $\kernel$, to the best of our knowledge, no prior $(n^\half, o_p(n^{-\quarter}))$-MMD coreset constructions were known for the range of $\Pstar$ and $\kernel$ studied in this work. 
For comparison, we review here both lower bounds and prior strategies for generating 
coresets %
with small MMD.

\paragraph*{Lower bounds}
For any bounded and {radial} (i.e., $\kernel(x,y) \!=\! \kappa(\twonorm{\x\!-\!\y}^2)$) kernel
satisfying mild decay and smoothness conditions, 
\citet[Thm.~3.1]{phillips2020near} showed that any procedure outputting coresets of size $n^{\half}$ must suffer $\Omega(\min(\sqrt{d}n^{-\half},n^{-\quarter}))$ $\mmd_{\kernel}$ for some (discrete) target distribution $\P$.
This lower bound applies, for example, to \Matern kernels and to infinitely smooth Gaussian kernels.
For any continuous and shift-invariant (i.e., $\kernel(x,y) \!=\! \kappa(\x\!-\!\y)$) kernel taking on at least two values, 
\citet[Thm.~1]{tolstikhin2017minimax} 
showed that \emph{any} estimator (including non-coreset estimators) based only on $n$ \iid draws from $\Pstar$ must suffer $C_{\kernel} n^{-\half}$ $\mmd_{\kernel}$ with probability at least $1/4$ for some discrete target $\Pstar$ and a constant $C_{\kernel}$ depending only $\kernel$.
If, in addition, $\kernel$ is {characteristic} (i.e., $\mmd_\kernel(\pnew, \qnew) \neq 0$ when $\pnew\!\neq\! \qnew$), then \citet[Thm.~6]{tolstikhin2017minimax} establish the same lower bound for some  continuous target $\Pstar$ with infinitely differentiable density.
These last two lower bounds hold, for example, for Gaussian, \Matern, and B-spline kernels and apply in particular to any thinning algorithm that compresses $n$ \iid sample points without additional knowledge of $\Pstar$.
For light-tailed $\Pstar$ and $\ksqrt$, the kernel thinning guarantees of \cref{theorem:main_result_all_in_one} match each of these lower bounds up to factors of $\sqrt{\log n}$ and constants depending on $d$.

\paragraph*{Order $(n^\half, n^{-\quarter})$-MMD coresets for general target $\Pstar$}
By Prop.~A.1 of \citet{tolstikhin2017minimax}, an \iid sample from $\Pstar$ yields an order $(n^\half, n^{-\quarter})$-MMD coreset in probability. %
\citet{chen2012super} showed that \emph{kernel herding} with a finite-dimensional kernel (like the linear $\kernel(\x,\y)=\inner{\x}{\y}$) finds an  $(n^\half, (C_{\Pstar,\kernel,d}n)^{-\half})$-MMD coreset for an inexplicit parameter $C_{\Pstar,\kernel,d}$. However, \citet{bach2012equivalence} showed that their analysis does not apply to any infinite-dimensional kernel (like the Gaussian, \Matern, and B-spline kernels studied in this work), as $C_{\Pstar,\kernel,d}$ would necessarily equal $0$. The best known rate for kernel herding with bounded infinite-dimensional kernels \citep[Thm. G.1]{lacoste2015sequential} guarantees an order $(n^{\frac12}, n^{-\frac14})$-MMD coreset, matching the \iid guarantee. 
For bounded kernels, the same guarantee is available for \emph{Stein Point MCMC} \citep[Thm.~1]{chen2019stein} which greedily minimizes MMD\footnote{To bound $\mmd_{\kernel}$ using \citet[Thm.~1]{chen2019stein}, choose $\kernel_0(\x,\y) = \kernel(\x,\y) - \P\kernel(\x) - \P\kernel(\y) + \P\P\kernel$.} over random draws from $\P$ and for a variant of the \emph{greedy sign selection} algorithm described in \citet[Sec.~3.1]{karnin2019discrepancy}.\footnote{The statement of \citet[Thm.~24]{karnin2019discrepancy} bounds $\infnorm{\cdot}$, but the proof bounds $\mmd_{\kernel}$.}
Slightly inferior guarantees were established for Stein points~\citep[Thm.~1]{Chen2018SteinPoints} and Stein thinning~\citep[Thm.~1]{riabiz2021optimal}, both of which accommodate unbounded kernels as well.

\paragraph*{Finite-dimensional kernels}
\citet{harvey2014near} construct $(n^\half, \sqrt{d}n^{-\half}\log^{2.5}n)$-MMD coresets for finite-dimensional linear kernels on $\R^d$ but do not address infinite-dimensional kernels.

\paragraph*{Uniform distribution on $[0,1]^d$}
The explicit low discrepancy \emph{quasi-Monte Carlo (QMC)} construction of 
 \citet{chen2002explicit}   
provides a $(n^{\frac12}, \order_d(n^{-\frac{1}{2}}\log^\frac{d-1}{2} n))$-MMD coreset for an $L^2$ discrepancy kernel when $\Pstar$ is the uniform distribution on the unit cube $[0,1]^d$. %
For the same target, the \emph{online Haar strategy} of \citet{dwivedi2019power} yields an
$(n^{\frac12}, \order_d(n^{-\frac{1}{2}}\log^{2d} n))$-MMD coreset in probability. \citet{dwivedi2019power} also conjecture that a greedy variant of their Haar strategy would provide an improved $(n^{\frac12}, \order_d(n^{-\frac{1}{2}}\log^{d} n))$-MMD coreset.
These constructions satisfy our quality criteria but are tailored specifically to the uniform distribution on the unit cube.
\paragraph*{Unknown coreset quality}
On compact manifolds, optimal coresets of size $n^{\half}$ minimize the weighted Riesz energy (a form of relative MMD with a weighted Riesz kernel) at known rates \citep{borodachov2014low}; however,   practical \emph{minimum Riesz energy} \citep{borodachov2014low} and \emph{minimum energy design} \citep{joseph2015sequential,joseph2019deterministic} constructions have not been analyzed. 
When $\kernel$ is nonnegative and the kernel matrix $(\kernel(x_i,x_j))_{i,j=1}^n$ satisfies a strong diagonal dominance condition, \citet[Cor.~3, Thm.~6]{kim2016examples} show that greedy optimization of $\mmd_{\kernel}$ yields an \emph{MMD-critic} coreset $\hat{\cset}$ of size $n^{\half}$ satisfying \begin{talign}
\mmd_{\kernel}^2(\P_n, \hat{\cset}) \leq (1-\frac1e)  \mmd_\star^2+ \frac1e\P_n\P_n\kernel
\qtext{for} \mmd_\star = \min_{|\cset| = \sqrt{n}} \mmd_{\kernel}(\P_n, \cset).
\end{talign}
In the usual case when $\P_n\P_n\kernel = \Omega(1)$, this error bound does not decay to $0$ with $n$.
\citet{paige2016super} analyze the impact of approximating a kernel in \emph{super-sampling with a reservoir} but do not analyze the quality of the constructed MMD coreset.
For the conditionally positive definite energy distance kernel, \citet{mak2018support} establish that an optimal  coreset of size $n^{\frac12}$ has $o(n^{-\frac14})$ MMD but do not provide a construction; in addition, \citet{mak2018support} propose two \emph{support points convex-concave procedures} for constructing MMD coresets but do not establish their optimality and do not analyze their quality.

\subsection{Related work on weighted MMD coresets}\label{sec:related_weighted} 
While coresets satisfy a number of valuable constraints that are critical for some downstream applications---exact approximation of constants, automatic preservation of convex integrand constraints, compatibility with unweighted downstream tasks, easy visualization, straightforward sampling, and increased numerical stability against errors in integral evaluations \citep{karvonen2019positivity}---some applications also support  \emph{weighted coreset} approximations of $\Pstar$ of the form $\sum_{i=1}^{\sqrt{n}} w_i \dirac_{\x_i}$ for weights $w_i\in\reals$ that need not be equal, need not be nonnegative, or need not sum to $1$. 
Notably, weighted coresets that depend on $\P$ only through an \iid sample of size $n$ are subject to the same $\Omega(n^{-\half})$ MMD lower bounds of \citet{tolstikhin2017minimax} described in \cref{sec:related}.
Any constructions that violate these bounds do so only by exploiting additional information about $\Pstar$ (for example, exact knowledge of $\Pstar \kernel$) that is not generally available and not required for our kernel thinning guarantees.
Moreover, while weighted coresets need not provide satisfactory solutions to the unweighted coreset problem studied in this work, kernel thinning coreset points can be converted into an optimally weighted coreset of no worse quality by explicitly minimizing $\mmd_{\kernel}(\P_n, \sum_{i=1}^{\sqrt{n}} w_i \dirac_{\x_i})$ or, if computable, $\mmd_{\kernel}(\Pstar, \sum_{i=1}^{\sqrt{n}} w_i \dirac_{\x_i})$ over the weights $w_i$ in $\order(n^{3/2})$ time.

With this context, we now review known weighted MMD coreset guarantees.
We highlight that 
only one of the weighted $(n^\half, o(n^{-\quarter}))$-MMD guarantees covers the unbounded distributions addressed in this work and that the single unbounded guarantee relies on a restrictive uniformly bounded eigenfunction assumption that is typically not satisfied.
In other words, our analysis establishes MMD improvements for practical $(\kernel, \Pstar)$ pairings not covered by prior weighted analyses.

\paragraph*{$\Pstar$ with bounded support}
If the target $\Pstar$ has bounded density and bounded, regular support and $\kernel$ is a Gaussian or \Matern kernel, then \emph{Bayesian quadrature}~\citep{o1991bayes} and \emph{Bayes-Sard cubature}~\citep{karvonen2018bayes} with quasi-uniform unisolvent point sets yield weighted $(n^\half, o(n^{-\quarter}))$-MMD coresets 
by \citet[Thm.~11.22 and Cor.~11.33]{wendland2004scattered}.
If $\Pstar$ has bounded support, and $\kernel$ has more than $d$ continuous derivatives, then the \emph{P-greedy} algorithm~\citep{de2005near} also yields weighted $(n^\half, o(n^{-\quarter}))$-MMD coresets by \citet[Thm.~4.1]{santin2017convergence}.
For $(\kernel, \Pstar)$ pairs with compact support and sufficiently rapid eigenvalue decay, approximate \emph{continuous volume sampling kernel quadrature} \citep{belhadji2020kernel} using the Gibbs sampler of \citet{rezaei2019polynomial} yields weighted coresets with $o(n^{-\quarter})$ root mean squared MMD. %

\paragraph*{Finite-dimensional kernels with compactly supported $\Pstar$} 
For compactly supported $\Pstar$, 
\citet[Thm.~1]{briol2015frank} and
\citet[Prop.~1]{bach2012equivalence} 
 show that \emph{Frank-Wolfe Bayesian quadrature} and weighted variants of kernel herding respectively yield weighted $(n^{\half}, o(n^{-\quarter}))$-MMD coresets for continuous finite-dimensional kernels,  %
but, by \citet[Prop. 2]{bach2012equivalence}, these analyses do not extend to infinite-dimensional kernels, like the Gaussian, \Matern, and B-spline kernels studied in this work.

\paragraph*{Eigenfunction restrictions}
For $(\kernel, \Pstar)$ pairs with known Mercer eigenfunctions, \citet{belhadji2019kernel} bound the expected squared MMD of \emph{determinantal point process (DPP) kernel quadrature} in terms of kernel eigenvalue decay and provide explicit rates for univariate Gaussian $\Pstar$ and uniform $\Pstar$ on $[0,1]$.  
Their construction makes explicit use of the kernel eigenfunctions which are not available for most $(\kernel, \Pstar)$ pairings.
For $(\kernel, \Pstar)$ pairs with $\Pstar \kernel = \boldzero$, %
uniformly bounded eigenfunctions, and rapidly decaying eigenvalues, \citet[App.~B.2]{liu2016black} prove that \emph{black-box importance sampling} generates probability-weighted coresets with $o(n^{-\quarter})$ root mean squared MMD 
but do not provide any examples verifying their assumptions. 
The uniformly bounded eigenfunction condition is considered particularly difficult to check \citep{steinwart2012mercer}, does not hold for Gaussian kernels with Gaussian $\P$ \citep[Thm.~1]{minh2010some}, and need not hold even for infinitely univariate smooth kernels on $[0,1]$ \citep[Ex.~1]{zhou2002covering}.

\paragraph*{Unknown coreset quality}
\citet[Prop.~2]{Huszr2012OptimallyWeightedHI} bound the MMD error of weighted \emph{sequential Bayesian quadrature} coresets using weak submodularity, but this bound does not decay to zero with $n$. 
\citet[Thm.~2]{khanna2019linear} prove that \emph{weighted kernel herding} yields a weighted $(n^\half, \exp(-n^\half/\kappa_n ))$-MMD coreset. 
However, the $\kappa_n$ term in \citet[Thm.3, Assum. 2]{khanna2019linear} is at least as large as the condition number of an $\sqrt{n} \times\sqrt{n}$ kernel matrix, which for typical kernels (including the Gaussian and \Matern kernels) is $\Omega(\sqrt{n})$~\citep{koltchinskii2000random,el2010spectrum}; the resulting MMD error bound therefore does not decay with $n$. The \emph{ProtoGreedy} and \emph{ProtoDash} algorithms of \citet[Thm.~IV.3, IV.5]{gurumoorthy2019efficient} yield nonnegative weighted coresets $\hat\cset$ of size $n^{\half}$ satisfying 
$
\mmd_{\kernel}^2(\P_n, \hat{\cset}) \leq \mmd_\star^2 + (\P_n\P_n\kernel-\mmd_\star^2) e^{-\lambda_{\sqrt{n}}}
$
where $\mmd_\star$ is the optimal MMD error to $\P_n$ for a nonnegatively weighted coreset of size $n^\half$.
However, careful inspection reveals that $\lambda_{\sqrt n}\leq 1$ for any kernel and any $n$.
Hence, in the usual case in which $\P_n\P_n\kernel = \Omega(1)$, this error bound does not decay to $0$ with $n$.
\citet[Thm.~4.4]{campbell2019automated} prove that \emph{Hilbert coresets via Frank-Wolfe} with $n$ input points yield weighted order $(n^\half, \nu_n^{\sqrt{n}})$-MMD coresets for some $\nu_n < 1$ but do not analyze the dependence of $\nu_n$ on $n$. 

\paragraph*{Non-MMD guarantees}
For $\Pstar$ with continuously differentiable Lebesgue density and $\kernel$ a bounded Langevin Stein kernel with $\Pstar\kernel = \boldzero$, Thm.~2 of \citet{oates2017control} does not bound MMD but does prove that a randomized \emph{control functionals} weighted coreset 
satisfies $\sqrt{\E[(\E f - \sum_{i=1}^{\sqrt{n}} w_i f(x_i))^2]} \leq {C_{\Pstar,\kernel,d,f}}{/n^{\frac{7}{24}}}$ for each $f$ in the RKHS of $\kernel$ and an unspecified $C_{\Pstar,\kernel,d,f}$.
This bound is asymptotically better than the $\Omega(n^{-\quarter})$ guarantee for unweighted \iid coresets  but worse than the unweighted kernel thinning guarantees of \cref{theorem:main_result_all_in_one}.
On compact domains, Thm.~1 of \citet{oates2019convergence} 
establishes improved rates for the same weighted coreset when both $\Pstar$ and $\kernel$ are sufficiently smooth.
\citet{bardenet2020monte} establish an $n^{-\quarter - \frac{1}{4d}}$ asymptotic decay of $\E f - \sum_{i=1}^{\sqrt{n}} w_i f(x_i)$ for DPP kernel quadrature with $\Pstar$ on $[-1,1]^d$ and each $f$ in the RKHS of a particular kernel. %
\subsection{Related work on $\Linf$ coresets} %
\label{sub:prior_work_on_linf_coresets}
A number of alternative strategies are available for constructing coresets with $\Linf$ guarantees. 
For example, for any bounded $\ksqrt$, Cauchy-Schwarz and the reproducing property imply that
\begin{talign}
\label{eq:linf_mmd}
   \infnorm{(\P-\P_n)\kernel} 
   = \sup_{\z\in\Rd}\abss{\dotk{\kernel(z, \cdot),\P\kernel-\P_n\kernel}}
	\leq \mmd_{\kernel}(\P, \P_n) \cdot \infnorm{\kernel}^{\frac12},
\end{talign}
so that all of the order $(n^\half, n^{-\quarter})$-MMD coreset constructions discussed in \cref{sec:related} also yield order $(n^\half, n^{-\quarter})$-$\Linf$ coresets.
However, none of those constructions is known to provide a  $(n^\half, o(n^{-\quarter}))$-$\Linf$ coreset.

A series of breakthroughs due to \citet{joshi2011comparing,phillips2013varepsilon,phillips2018improved,phillips2020near,tai2020new} has led to a sequence of increasingly compressed $(n^\half, o(n^{-\quarter}))$-$\Linf$ coreset constructions, with the best known guarantees currently due to \citet{phillips2020near} and \citet{tai2020new}.
Given $n$ input points, 
\citet{phillips2020near} developed an offline, polynomial-time construction to find an $(n^{\half}, \order_p(\sqrt{d}n^{-\half}\sqrt{\log n}))$-$\Linf$ coreset 
for Lipschitz kernels exhibiting suitable decay, while \citet{tai2020new} developed an offline construction for Gaussian kernels that runs in $\Omega(d^{5d})$ time and yields an %
$(n^{\half}, \order_p(2^d n^{-\frac12}\sqrt{\log(d\log n)}))$-$\Linf$ coreset.  %
More details on these constructions based on the Gram-Schmidt walk of \citet{bansal2018gram} can be found in \cref{sub:pt_coresets}.
Notably, the Phillips and Tai (hereafter, PT) guarantee is tighter than that of \cref{kernel_halving_results} by a factor of $\sqrt{\log \log n}$ for sub-Gaussian kernels and input points and $\sqrt{\log n}$ for heavy-tailed kernels and input points.
Similarly, the Tai guarantee provides an improvement when $n$ is doubly-exponential in the dimension, that is, when $\sqrt{d \log n} = \Omega(2^d)$.

Moreover, by \cref{theorem:coreset_to_mmd}, we may apply the PT and Tai constructions to a square-root kernel $\ksqrt$ to obtain comparable MMD guarantees for the target kernel $\kernel$ with high probability.
However, kernel thinning has a number of practical advantages that lead us to recommend it.
First with $n$ input points, using standard matrix multiplication, the PT and Tai constructions have $\Omega(n^4)$ computational complexity and $\Omega(n^2)$ storage costs, a substantial increase over the $\order(n^2)$ running time and $\order(n\min(d,n))$ storage of kernel thinning.
Second, \ktsplit is an online algorithm while the PT and Tai constructions require the entire set of input points to be available a priori.
Finally, each halving round of \ktsplit splits the sample size exactly in half, allowing the user to run all $m$ halving rounds simultaneously; the PT and Tai constructions require a rebalancing step after each round forcing the halving rounds to be conducted sequentially.
\subsection{Future directions}
Several other opportunities for future development recommend themselves.
First, since our results cover any target $\P$ with at least $2d$ moments---even discrete and other non-smooth targets---a natural question is whether tighter error bounds with better sample complexities are available when $\P$ is also known to have a smooth Lebesgue density. 
Second, the MMD to $\Linf$ reduction in \cref{theorem:coreset_to_mmd} applies also to weighted $\Linf$ coresets, and, in applications in which weighted point sets are supported, we would expect either quality or compression improvements from employing non-uniform weights \citep[see, e.g.][]{turner2021statistical}.
\newcommand{\acknowledgments}{%
RD acknowledges support by National Science Foundation under Grant No. DMS2023528 for the Foundations of Data Science Institute (FODSI). The authors thank Fran\c{c}ois-Xavier Briol, Lucas Janson, Lingxiao Li, Chris Oates, Art Owen, and the anonymous reviewers for their valuable feedback on this work and Jeffrey Rosenthal for helpful discussions surrounding geometric ergodicity. 
Part of this work was done when RD was interning at Microsoft Research New England. 
}
\begin{appendix}

\vspace{5mm}\noindent \tbf{\large {Appendix}}
    {\small\tableofcontents}
    \section{Appendix Notation}
    For each $p\geq 1$, we define $\lp$ as the set of measurable $g : \Rd \to \reals$ with  $\lpnorm{g}\defeq (\int |g(x)^p| dx)^{1/p} < \infty$ and $C^p$ as the set of $g : \Rd \to \reals$ for which all partial derivatives of order $p$ exist and are continuous.
    For a kernel $\gkernel : \Rd\times \Rd \to \reals$, we also write 
    $\gkernel\in\ltwoinf$ to indicate that $\gkernel$ is measurable with finite
    \begin{talign}
    \label{eq:kernel_ltwo_norm}
        \ltwoinfnorm{\gkernel} \defeq \sup_{\x\in\Rd}\parenth{\int\gkernel^2(\x, y)d\y}^{\frac{1}{2}}
        = \sup_{\x\in\Rd} \ltwonorm{\gkernel(\x, \cdot)}.
    \end{talign}
    Throughout, we follow
the unitary angular frequency convention of \citet[Def.~5.15]{wendland2004scattered} and define the Fourier transform $\fourier(\fun)$ of an integrable complex function $\fun: \reals^d \to \complex$ via
\begin{talign}
    \label{defn:fourier_transform}
	\fourier({\fun})(\omega) \defeq \frac{1}{(2\pi)^{{\dims}/{2}}}
	\int_{\Rd} \fun(\x) e^{-i\inner{\x}{\omega}} d\x
    \qtext{for all}
    \omega\in\Rd.
\end{talign}
\section{Proof of \lowercase{\Cref{mcmc_mmd}}: \mcmcmmdresultname}
\label{proof_of_mcmc_mmd}
By \citet[Lem.~9.3.9, Cor.~9.2.16]{douc2018markov}, a homogeneous 
 $\phi$-irreducible geometrically ergodic Markov chain with stationary distribution $\P$ is also aperiodic with a unique stationary distribution.\footnote{In \citet[Def.~9.2.1]{havet2020quantitative,douc2018markov} the term \emph{irreducible} is synonymous with $\phi$-irreducible as  defined by \citet[Sec.~2]{gallegosherrada2023equivalences}.}
Since $(\x_i)_{i=0}^\infty$ are the iterates of such a chain, 
there exist, 
by \citet[Thm.~1xi]{gallegosherrada2023equivalences},
constants $\rho \in (0,1)$ and $\tau < \infty$ and a measurable $\P$-almost everywhere finite function $V : \reals^d \to [1,\infty]$ 
satisfying $\P V < \infty$ and 
\begin{align}
\label{eq:vgeom_ergo_condition}
\!\sup_{\text{measurable } h : \frac{|h(x)|}{V(x)} \leq 1, \forall x\in\reals^d}
|\E[h(\x_{i})\mid \x_0 = \x] - \P h| \leq \tau V(\x) \rho^{i}, \text{ for all } \x \in \reals^d \stext{and} i \in\naturals.
\end{align}
Since $V$ is finite $\P$-almost everywhere, we will choose $c(\x) = \infty \iff V(\x) = \infty$ to ensure that our claim is (vacuously) true whenever $V(\x_0) = \infty$. 

Hereafter, suppose $V(\x_0) < \infty$. 
Since the Markov chain is irreducible and aperiodic with a unique stationary distribution $\P$, Assump.~H1 of \citet{havet2020quantitative} is satisfied. 
Hence, by an application of  \citet[Prop.~2.1]{havet2020quantitative} with 
$V' = V \rho^g$ and $\zeta = 2/\rho^g$ for sufficiently large $g\in\naturals$, 
there exists a 
set $C(\x_0)\subseteq \Rd$ that contains $\x_0$ and satisfies Assumps.~H2 and H3 of \citet{havet2020quantitative}.

Now fix any $y_1,\dots, y_n,z_1,\dots,z_n\in\reals^d$. 
We invoke the definition of MMD \cref{eq:kernel_mmd_distance}, the triangle inequality, the reproducing property of an RKHS \citep[Def.~4.18]{steinwart2008support}, and Cauchy-Schwarz in turn to deduce a bounded differences property for MMD:
\begin{talign}
&\mmd_{\kernel}(\P, \frac{1}{n}\sum_{i=1}^n\dirac_{y_i})
-
\mmd_{\kernel}(\P, \frac{1}{n}\sum_{i=1}^n\dirac_{z_i}) \\
    &=
\sup_{\knorm{f} \leq 1} |\P f -\frac{1}{n}\sum_{i=1}^n f(y_i)|
-
\sup_{\knorm{f} \leq 1} |\P f -\frac{1}{n}\sum_{i=1}^n f(z_i)| \\
    &\leq
\sup_{\knorm{f} \leq 1} |\frac{1}{n}\sum_{i=1}^n f(y_i) - f(z_i)|
    =
\sup_{\knorm{f} \leq 1}\frac{1}{n}\sum_{i=1}^n |\dotk{\kernel(y_i,\cdot) - \kernel(z_i,\cdot),f}| \\
    &\leq
\sup_{\knorm{f} \leq 1} \frac{1}{n}\sum_{i=1}^n \knorm{\kernel(y_i,\cdot) - \kernel(z_i,\cdot)}\knorm{f} \\
    &= 
\sup_{\knorm{f} \leq 1} \frac{1}{n}\sum_{i=1}^n \sqrt{\kernel(y_i,y_i)+\kernel(z_i,z_i)-2\kernel(y_i,z_i)}\knorm{f}
    \leq
\frac{2}{n}\infnorm{\kernel}^{\half}\sum_{i=1}^n \indic{y_i\neq z_i}.
\end{talign}
Since $\x_0$ belongs to a set $C(\x_0)$ satisfying Assumps.~H2 and H3 of \citet{havet2020quantitative}, McDiarmid's inequality for geometrically ergodic Markov chains
\citep[Thm.~3.1]{havet2020quantitative} 
implies that, 
with probability at least 
$1-\delta$ conditional on $\x_0$,
\newcommand{\tmix}{t_{\textup{mix}}}
\begin{talign}
\mmd_{\kernel}(\P, \P_n)
    \leq 
        \E[\mmd_{\kernel}(\P, \P_n)\mid \x_0] + \sqrt{c_1(\x_0)\infnorm{\kernel}\log(1/\delta)/n}
\end{talign}
where $c_1(\x_0)$ is a finite value depending only on the transition probabilities of the chain and the set $C(\x_0)$. 

Now, define the $\P$ centered kernel $\kernel_{\P}(\x,\y) = \kernel(\x,\y) - \P\kernel(\x) -\P\kernel(\y) + \P\P\kernel$.
To bound the expectation, we will use a slight modification of Lem. 3 of \citet{riabiz2021optimal}.  
The original lemma used the assumption of $V$-uniform ergodicity \citep[Defn. (16.0.1)]{meyn2012markov} and the assumption $V(\x) \geq \sqrt{\kernel_{\P}(\x,\x)}$ solely to argue that, for some $R > 0$, 
\begin{talign}
|\E[f(x_{i})\mid \x_0 = \x] - \P f|
    \leq
R V(x) \rho^{i}
    \qtext{for all} \x \in \Rd
    \qtext{and}
    f \in\rkhs_{\kernel_{\P}}
    \stext{with} \norm{f}_{\kernel_{\P}} = 1.
\end{talign}
In our case, since $\kernel_{\P}$ is bounded and any $f\in\rkhs_{\kernel_{\P}}$ with $\norm{f}_{\kernel_{\P}} = 1$ satisfies
\begin{talign}
|f(\x)| 
    = 
|\inner{\kernel_{\P}(\x,\cdot)}{f}_{\kernel_{\P}}|
    \leq
\norm{\kernel_{\P}(\x,\cdot)}_{\kernel_{\P}} \norm{f}_{\kernel_{\P}}
    = 
\sqrt{\kernel_{\P}(\x,\x)}
    \leq 
\sqrt{\infnorm{\kernel_{\P}}}
    \stext{for all} \x \in \Rd
\end{talign}
by the reproducing property and Cauchy-Schwarz, the geometric ergodicity property \cref{eq:vgeom_ergo_condition} implies the analogous bound
\begin{talign}
|\E[f(\x_i)\mid \x_0 = x] - \P f|
    \leq
\tau \sqrt{\infnorm{\kernel_{\P}}} V(\x) \rho^{i}
    \stext{for all} \x \in \Rd
    \stext{and}
    f \in\rkhs_{\kernel_{\P}} 
    \stext{with} \norm{f}_{\kernel_{\P}} = 1.
\end{talign}
Hence, the conclusions of \citet[Lem.~3]{riabiz2021optimal} with $R = \tau\sqrt{\infnorm{\kernel_{\P}}}$ hold under our assumptions.
Jensen's inequality and the conclusion of Lem. 3 of \citet{riabiz2021optimal} now yield the sure bound
\begin{talign}
\E[&\mmd_{\kernel}(\P, \P_n)\mid \x_0]^2
    \leq
\E[\mmd_{\kernel}(\P, \P_n)^2 \mid \x_0] \\
    &= 
\E[\frac{1}{n^2}\sum_{i=1}^n \kernel_{\P}(\x_i,\x_i)
+ \frac{1}{n^2}\sum_{i=1}^n \sum_{j \neq i} \kernel_{\P}(\x_i,\x_j)\mid \x_0] \\
    &\leq
\frac{1}{n} \infnorm{\kernel_{\P}}
(1+\frac{2\tau\rho}{1-\rho} \frac{1}{n}\sum_{i=1}^{n-1}\E[V(\x_i)\mid \x_0])
    \leq
\frac{4}{n} \infnorm{\kernel}
(1+\frac{2\tau\rho}{1-\rho} \frac{1}{n}\sum_{i=1}^{n-1}\E[V(\x_i)\mid \x_0]).
\end{talign}
Now, define $c_2(\x_0) \defeq \sup_{n\in\naturals}\frac{1}{n}\sum_{i=1}^{n-1}\E[V(\x_i)\mid \x_0]$.
Since $V(\x_0) < \infty$, the geometric ergodicity property \cref{eq:vgeom_ergo_condition} and the fact that $\P V < \infty$ imply
\begin{talign}
c_2(\x_0)
\leq 
    \P V + V(\x_0) \sup_{n\in\naturals} \textfrac{1}{n}\textsum_{i=1}^{n-1}\rho^i
\leq
    \P V + V(\x_0)\rho
< \infty.
\end{talign}
Taking $c(x_0) = \sqrt{2} \max(c_1(x_0), 4 c_2(x_0) (1+\frac{2\tau\rho}{1-\rho}))$ completes the proof.
\section{\pcref{prop:radii_growth}}
\label{proof_of_prop:radii_growth}
We prove this result for the identically distributed case in \cref{proof_of_prop:radii_growth_identical} and for the Markov chain case in \cref{proof_of_prop:radii_growth_mcmc}.

\subsection{Radius growth for identically distributed sequences}
\label{proof_of_prop:radii_growth_identical}
Suppose $\x_0$ and $\inputcoresetfull$ are drawn identically from $\P$.
Claim~\cref{item:compactiid} is true by definition.
To establish the remaining claims, 
we use the following more general result proved in \cref{proof_of_lemma:identical_growth_rate}.
\begin{lemma}[Growth rate for identically distributed sequence]
 \label{lemma:identical_growth_rate}
   Consider a sequence of identically distributed random variables $(Y_i)_{i=1}^\infty$ on $\R$ and a measurable function $\psi:\real \to \real$ with an increasing inverse function $\psi\inv$. If $\psi(Y_1)\geq 0$ almost surely, then
   \begin{talign}
   \label{eq:max_prob}
     \Pr(\max_{i\leq n}Y_i > \psi\inv(n) \stext{for some} n \in \naturals) \leq \E(\psi(Y_1)).
   \end{talign}
   Consequently, if $\E(\psi(Y_1))<\infty$, then, for any $\delta\in (0,1]$, 
   \begin{talign}
   \label{eq:max_prob_2}
     \Pr(\max_{i\leq n} Y_i \leq \psi\inv(\frac{n\E[\psi(Y_1)]}{\delta}) \stext{for all} n \in \naturals) \geq 1-\delta. %
   \end{talign}
 \end{lemma}

Fix any $\delta \in (0,1]$.  
\cref{lemma:identical_growth_rate}  with $Y_i = \twonorm{x_i}$
implies that, with probability $1-\delta$,
\begin{talign}
\label{eq:tailinput_generic}
    \rminpn[\sn] \leq \psi\inv(\frac{n\E[\psi(\twonorm{\x_1})]}{\delta}) \stext{for all} n \in \mbb N
\end{talign}
for $\psi\inv(r) = \frac{\sqrt{\log r}}{\sqrt{c}}$ in case~\cref{item:sgiid}, $\psi\inv(r) = \frac{\log r}{c}$ in case~\cref{item:seiid}, and $\psi\inv(r) = r^{1/\rho}$ in case~\cref{item:htiid}.
As a result we have, with probability $1-\delta$, $\rminpn[\sn] = \order_d(\sqrt{\log n})$ in case~\cref{item:sgiid}, $\rminpn[\sn] = \order_d(\log n)$ in case~\cref{item:seiid}, and $\rminpn[\sn] =  \order_d(n^{1/\rho})$ in case~\cref{item:htiid}.
Since $\delta$ is arbitrary, these orders hold with probability $1$ as claimed.

\subsubsection{\pcref{lemma:identical_growth_rate}}
\label{proof_of_lemma:identical_growth_rate}
We make use of three lemmas.
The first rewrites maximum exceedance events in terms of individual variable exceedance events when thresholds are nondecreasing.
\begin{lemma}[Exceedance equivalence]
 \label{lemma:max_equivalence}
  For any real-valued $(a_i)_{i=1}^{\infty}$ and nondecreasing $(b_i)_{i=1}^{\infty}$,  
 \begin{talign}\label{eq:max_equivalence}
   \max_{i \leq n} a_i > b_n \stext{for some} n \in \naturals \quad \Longleftrightarrow \quad a_i > b_i \stext{for some} i \in \naturals.
 \end{talign}
 \end{lemma}
 \begin{proof}
    The $\Leftarrow$ part follows immediately. To prove the $\Rightarrow$ part, suppose $\max_{i \leq n^\star} a_i>b_{n^\star}$ for some $n^\star$. Then there exists an $i \leq n^\star$ with $a_i > b_{n^\star}\geq b_i$ since $(b_i)_{i=1}^{\infty}$ is nondecreasing.
 \end{proof}

 The second bounds the probability of growth rate violation for \emph{any} sequence of random variables in terms of a sum of exceedance probabilities.

\begin{lemma}[Growth rate for arbitrary sequence]
 \label{lemma:arbitrary_growth_rate}
   For any sequence of random variables $(Y_i)_{i=1}^\infty$ on $\R$ and 
   a nondecreasing real-valued sequence $(b_i)_{i=1}^{\infty}$, we have
   \begin{talign}
    \Pr(\max_{i\leq n}Y_i > b_n \stext{for some} n \in \naturals) 
        = 
    \Pr(Y_i > b_i \stext{for some} i \in \naturals) 
        \leq 
    \sum_{i=1}^{\infty} \Pr(Y_i > b_i).
   \end{talign}
 \end{lemma}
 \begin{proof}
 The result follows from immediately from \cref{lemma:max_equivalence} and the union bound.
 \end{proof}

 The third lemma bounds a sum of exceedance probabilities whenever the random variables are identically distributed and nonnegative.

 \begin{lemma}[Bounding exceedances with expectations]
 \label{lemma:tailsum_bound_expectation}
   If the random variables $(Z_i)_{i=1}^\infty$ are identically distributed and almost surely nonnegative, then
   \begin{talign}
   \label{eq:tailsum_bound_expectation}
     \sum_{i=1}^{\infty} \Pr(Z_i > i)
     \leq 
     \E(Z_1).
   \end{talign}
 \end{lemma}
 \begin{proof}
   Since $Z_1$ is almost surely nonnegative, we have $Z_1 = \int_0^{Z_1} dt = \int_{0}^{\infty} \indicator(Z_1>t) dt$ almost surely.
   Tonelli's theorem~\citep[Thm.~1]{mukherjea1972remark} therefore implies that 
   \begin{talign}
     \E(Z_1) 
     &= \E(\int_{0}^{\infty} \indicator(Z_1>t) dt)
     = \int_{0}^{\infty} \Pr(Z_1> t) dt \\
     &\geq  \int_{0}^{\infty} \Pr(Z_1> \ceil{t}) dt 
     = \sum_{i=1}^{\infty} \Pr(Z_1>i)
     = \sum_{i=1}^{\infty} \Pr(Z_i>i)
   \end{talign}
 where the final inequality uses the identically distributed assumption.
 \end{proof}

 Since $\psi(Y_1) \geq 0$ almost surely and $\psi\inv$  
 is increasing, we invoke  
 \cref{lemma:tailsum_bound_expectation} with $Z_i = \psi(Y_i)$,
 the invertibility of $\psi$, and 
 \cref{lemma:arbitrary_growth_rate} with $b_i =\psi\inv(i)$ in turn to conclude
 \begin{talign}
    \E(\psi(Y_1)) &\sgrt{\cref{eq:tailsum_bound_expectation}} \sum_{i=1}^{\infty} \Pr(\psi(Y_i)>i)
    \seq{} \sum_{i=1}^{\infty} \Pr(Y_i>\psi\inv(i)) \\ 
    &\sgrt{\cref{eq:max_equivalence}} \Pr( \max_{i\leq n}Y_i>\psi\inv(n) \stext{for some} n \in \N).
 \end{talign}

\subsection{Radius growth for MCMC}
\label{proof_of_prop:radii_growth_mcmc}
Now suppose 
$\x_0$ and $\inputcoresetfull$ are the iterates of a homogeneous $\phi$-irreducible geometrically ergodic Markov chain 
with initial state $\x_0$, subsequent iterates $\inputcoresetfull$, and stationary distribution $\P$.
Our claims will follow from the following more detailed result proved in \cref{proof_of_lemma:mcmc_growth_rate}.

\begin{lemma}[Growth rate for MCMC]
 \label{lemma:mcmc_growth_rate}
 Consider a homogeneous $\phi$-irreducible geometrically ergodic Markov chain 
with initial state $\x_0$, subsequent iterates $(\x_i)_{i=1}^\infty$, and stationary distribution $\P$.
There exist constants $\rho \in (0,1)$ and $\tau < \infty$ and a measurable $\P$-almost everywhere finite function $V : \reals^d \to [1,\infty]$ such that, for any index $j\in\naturals$, 
measurable function $g: \reals^d \to \reals$, measurable nonnegative function $\psi$ on $\reals$ with increasing inverse function $\psi\inv$,
and $X\sim\P$, 
   \begin{talign}
   \label{eq:mcmc_max_prob}
     &\Pr(\max_{i\leq n}g(x_i) > \psi\inv(n) \stext{for some} n \in \naturals \mid x_0) \leq \\ &\E(\psi(g(X))) 
     + \textfrac{\tau\rho^{j+1}}{1-\rho} V(\x_0)
     + \sum_{i=1}^j \Pr(g(x_i) > \psi\inv(i) \mid x_0).
   \end{talign}
   Now suppose $V(x_0)<\infty$, and fix any measurable nonnegative 
 $\psi$ on $\reals$ with increasing $\psi\inv$ and any measurable $g: \reals^d \to \reals$.
   If $\E(\psi(g(X)))<\infty$, then, for any $\delta\in(0,1]$, there exists a constant $c_{\delta,\psi\circ g}(x_0) \in (0,\infty)$ such that
   \begin{talign}
   \label{eq:mcmc_max_prob_2}
     \Pr(\max_{i\leq n} g(x_i) \leq \psi\inv(c_{\delta,\psi\circ g}(x_0) n) \stext{for all} n \in \naturals \mid x_0) \geq 1-\delta.
   \end{talign}
   Moreover, if $\P$ is compactly supported, then for any $\delta\in(0,1]$, there exists a constant $c_\delta(x_0) \in (0,\infty)$ such that 
   \begin{talign}
   \label{eq:mcmc_max_prob_3}
     \Pr(\sup_{i \in \N} \twonorm{x_i} \leq c_{\delta}(x_0) \mid x_0) \geq 1-\delta.
   \end{talign}
 \end{lemma}
Instantiate the function $V$ from \cref{lemma:mcmc_growth_rate}, and suppose that $V(x_0) < \infty$, an event that holds for $\P$-almost every $x_0$.
Claims \cref{item:sgiid,item:seiid,item:htiid} then follow by invoking the time-uniform tail bound \cref{eq:mcmc_max_prob_2} with $g = \twonorm{\cdot}$ and proceeding as in \cref{proof_of_prop:radii_growth_identical}. Finally, claim~\cref{item:compactiid} follows from the bound \cref{eq:mcmc_max_prob_3}, which establishes $\twonorm{x_i} = \order_d(1)$ with probability $1$ conditional on $\x_0$.

\subsubsection{\pcref{lemma:mcmc_growth_rate}}
\label{proof_of_lemma:mcmc_growth_rate}
The proof closely parallels that of \cref{lemma:identical_growth_rate} except that we substitute the following estimate for \cref{lemma:tailsum_bound_expectation}.
\begin{lemma}[Bounding MCMC exceedances with expectations]
 \label{lemma:tailsum_bound_expectation_mcmc}
Consider a homogeneous $\phi$-irreducible geometrically ergodic Markov chain 
with initial state $\x_0$, subsequent iterates $(\x_i)_{i=1}^\infty$, and stationary distribution $\P$.
There exist constants $\rho \in (0,1)$ and $\tau < \infty$ and a measurable $\P$-almost everywhere finite function $V : \reals^d \to [1,\infty]$ such that, for any index $j\in\naturals$, measurable nonnegative function $f$ on $\Rd$, and $X\sim\P$, 
   \begin{talign}
   \label{eq:tailsum_bound_expectation}
     \sum_{i=j}^{\infty} \P(f(\x_i) > i \mid \x_0)
     \leq 
     \E(f(X)) + \frac{\tau\rho^j}{1-\rho} V(\x_0).
   \end{talign}
 \end{lemma}
 \begin{proof}
 By \citet[Lem.~9.3.9]{douc2018markov}, a homogeneous 
 $\phi$-irreducible geometrically ergodic Markov chain with stationary distribution $\P$ is also aperiodic. 
 By \citet[Thm.~1xi]{gallegosherrada2023equivalences}, there exist constants $\rho \in (0,1)$ and $\tau < \infty$  and a measurable $\P$-almost everywhere finite function $V : \reals^d \to [1,\infty]$ satisfying  
\begin{align}
\label{eq:vgeom_ergo_condition}
\sup_{h : |h(x)| \leq V(x), \forall x\in\reals^d}
|\E[h(\x_{i})\mid \x_0 = \x] - \P h| \leq \tau V(\x) \rho^{i}
\qtext{for all} \x \in \reals^d.
\end{align}
Applying this result to the functions $h_i(x) \defeq \indic{f(x) > i}$, we find that
\begin{talign}
\sum_{i=j}^{\infty} \P(f(\x_i) > i \mid \x_0)
    \leq
\sum_{i=j}^{\infty} \P(f(X) > i)
    +
\sum_{i=j}^{\infty} \tau V(\x_0) \rho^i
    \leq 
\E[f(X)] + \frac{\tau\rho^j}{1-\rho} V(\x_0)
\end{talign}
where the final inequality uses \cref{lemma:tailsum_bound_expectation}.
 \end{proof}

Fix any $V$, $\rho$, and $\tau$ satisfying the conclusions of  \cref{lemma:tailsum_bound_expectation_mcmc}, any measurable $g: \reals^d \to \reals$, and any measurable nonnegative 
 $\psi$ on $\reals$ with increasing $\psi\inv$.
The first claim \cref{eq:mcmc_max_prob} follows by applying \cref{lemma:arbitrary_growth_rate} with $b_i =\psi\inv(i)$, the assumed invertiblity and strict monotonicity of $\psi^{-1}$, and  \cref{lemma:tailsum_bound_expectation_mcmc} with $f = \psi \circ g$ in turn to find that
 \begin{talign}
 &\P(\max_{i\leq n}g(x_i) > \psi\inv(n) \stext{for some} n \mid x_0) 
    \leq
 \sum_{i=1}^\infty \P(g(x_i) > \psi\inv(i) \mid x_0) \\
    &= 
    \sum_{i=1}^\infty \P(\psi(g(x_i)) > i \mid x_0)   
    \leq 
\sum_{i=1}^j \P(\psi(g(x_i)) > i \mid x_0)+ \E(\psi(g(X))) + \textfrac{\tau \rho^{j+1}}{1-\rho} V(\x_0).
\end{talign}

Now suppose $V(x_0)<\infty$, fix any $\delta\in(0,1]$, and let $j$ be the smallest positive index with $\textfrac{\tau \rho^{j+1}}{1-\rho} V(\x_0) < \frac{\delta}{3}$.
Since each $\psi(g(x_i))$ is a tight random variable given $x_0$, we can additionally choose a constant $c_{\delta,\psi\circ g}(x_0) \in (0, \infty)$ satisfying  $\sum_{i=1}^{j} \P(\psi(g(x_i))/c_{\delta,\psi\circ g}(x_0) > i \mid x_0) < \frac{\delta}{3}$ and $\E(\psi(g(X)))/c_{\delta,\psi\circ g}(x_0) < \frac{\delta}{3}$.  The claim \cref{eq:mcmc_max_prob_2} now follows by applying the initial result \cref{eq:mcmc_max_prob} to the function $\psi / c_{\delta,\psi\circ g}(x_0)$.

To establish the final claim \cref{eq:mcmc_max_prob_3}, 
suppose that $\P$ is compactly supported. 
Since $\P$ has compact support and each $\twonorm{x_i}$ is a tight random variable, there exists a constant $c_{\delta}\in (0,\infty)$ satisfying $\Pr_{X\sim \P}(\twonorm{X}>c_{\delta})=0$ and $\sum_{i=1}^{j} \P(\psi(\twonorm{x_i}) > c_{\delta}  \mid x_0) < \frac{\delta}{2}$.
The union bound and the geometric ergodicity property \cref{eq:vgeom_ergo_condition} applied to the function $h(x) = \indicator(\twonorm{x} > c_{\delta})$ with $\P h = 0$ now imply
\begin{talign}
\Pr(\sup_{i\in\N} \twonorm{x_i} > c_{\delta}\mid x_0)
&\leq \sum_{i=1}^{\infty} \Pr(\twonorm{x_i} > c_{\delta} \mid x_0)\\
&\leq \sum_{i=1}^{j} \Pr(\twonorm{x_i} > c_{\delta} \mid x_0) + \tau V(x_0) \sum_{i=j+1}\rho^i\\
&= \sum_{i=1}^{j} \Pr(\twonorm{x_i} > c_{\delta} \mid x_0) + \frac{\tau\rho^{j+1}}{1-\rho} V(x_0)
< \frac{\delta}{2} + \frac{\delta}{3} < \delta.
\end{talign}

\section{Proof of \lowercase{\Cref{sqrt_translation_invariant}}: \sqrttranslationinvariantname} 
\label{sec:proof_of_sqrt_translation_invariant}

Bochner's theorem \citep[Thm.~6.6]{bochner1933monotone,wendland2004scattered} implies that
$\ksqrt$ is a kernel since $\kappasqrt$ is the Fourier transform of a finite Borel measure with Lebesgue density $\sqrt{\hatkappa}$.
Moreover, as  
$\ksqrt(\x, \cdot) = \frac{1}{(2\pi)^{d/4}}\fourier(e^{-i\inner{\cdot}{x}} \sqrt{\hatkappa})$ and $e^{-i\inner{\cdot}{x}}\sqrt{\hatkappa}$ is  integrable and square integrable, 
the Plancherel-Parseval identity \citep[Proof of Thm.~5.23]{wendland2004scattered} implies that
\balignt
\int_{\Rd}\ksqrt(\x, \z)\ksqrt(\y, \z) d\z
    &= 
\int_{\Rd} 
    \frac{1}{(2\pi)^{d/4}}e^{-i\inner{\omega}{\x}} \sqrt{\hatkappa(\omega)}
    \ \frac{1}{(2\pi)^{d/4}}e^{i\inner{\omega}{\y}} \sqrt{\hatkappa(\omega)} d\omega \\
    &= 
\frac{1}{(2\pi)^{d/2}}\int_{\Rd} 
    e^{-i\inner{\omega}{\x-\y}} \hatkappa(\omega) d\omega
    =
\kernel(x,y)
\ealignt
confirming that $\ksqrt$ is a square-root kernel of $\kernel$.

\section{Proof of \lowercase{\Cref{theorem:main_result_all_in_one}}: \mainresultallinonename} %
\label{sub:proof_of_theorem:main_result_all_in_one}
By design, \ktswap ensures
\begin{talign}
     \mmd_{\kernel}(\inputcoreset, \ktcoreset) \leq 
     \mmd_{\kernel}(\inputcoreset, \coreset[\m, 1]),
     \label{eq:thm1_proof_step1}
\end{talign}
where $\coreset[\m, 1]$ denotes the first coreset returned by \ktsplit. Next, applying \cref{corollary:kernel_thinning_coreset_bound}, in particular, the bound \cref{eq:mmd_ktsplit_bound} yields the desired claim.

    \section{Proof of \cref{finite_anytimee}: \failureremarkname}
\label{sec:failure_prob_proof}
We prove the three claims one by one.

\paragraph{Finite time guarantee}
For the case with known $n$, the claim follows simply by noting that
\begin{talign}
\sum_{j=1}^{m} \frac{2^{j-1}}{m}\sum_{i=1}^{2^{m-j}\floor{n/2^{m}}} \frac{\delta}{n} =\sum_{j=1}^{m}\frac{2^m}{m} \floor{\frac{n}{2^m}} \frac{\delta}{2n}
\leq \sum_{j=1}^{m}\frac{2^m}{m} \frac{n}{2^m} \frac{\delta}{2n}= \frac{\delta}{2}.
\end{talign}

\paragraph{Any time guarantee}
When the input size $n$ is not known in advance but is chosen independently of the randomness in kernel thinning, 
we first note that
\begin{talign}
\sum_{i= 1}^\infty\frac{1}{(i+1)\log^2(i+1)} \sless{(i)} 2,
\qtext{and}
\sum_{j=1}^{m} 2^{j} = 2^{m+1}-2 \leq 2^{m+1}.
\label{log_sum}
\end{talign}
 where step~(i) can be verified using  mathematical programming software.
Therefore, for any $n\in \natural$, with $\delta_i = \frac{m\delta}{ 2^{m+2}(i+1)\log^2(i+1)}$, we have
\begin{talign}
    \sum_{j=1}^m \frac{2^{j-1}}{m} \sum_{i=1}^{2^{m-j}\floor{n/2^m}} \delta_i 
    \leq \sum_{j=1}^m \frac{2^{j-1}}{m} \sum_{i=1}^{\infty} \delta_i 
    &=\sum_{j=1}^m \frac{2^{j-1}}{m} \sum_{i=1}^{\infty} \frac{m\delta}{ 2^{m+2}(i+1)\log^2(i+1)}
    \\
    &= \frac{\delta}{2^{m+3}} \parenth{\sum_{j=1}^m 2^{j}} \parenth{\sum_{i=1}^{\infty} \frac{1}{(i+1)\log^2(i+1)}}\\
    &\sless{\cref{log_sum}} \frac{\delta}{2^{m+3}}
    \cdot 
    2^{m+1} \cdot 2
    \leq \frac{\delta}{2}.
\end{talign}

\paragraph{Upper bound on $\delta^\star$} The probability lower bound in \cref{theorem:main_result_all_in_one} is
$1\!-\!\delta'\!-\!\sum_{j=1}^{\m} \frac{2^{j-1}}{\m} \sum_{i=1}^{2^{m-j}\floor{n/2^m}}\delta_{i}$, which is non-negative only if
\begin{talign}
1\geq \sum_{j=1}^{\m} \frac{2^{j-1}}{\m} \sum_{i=1}^{2^{m-j}\floor{n/2^m}}\delta_{i}
\geq \frac{2^m}{2} \floor{\frac{n}{2^m}} \delta^\star, 
\label{eq:delta_star_condn}
\end{talign}
which holds only if $\delta^\star \leq \frac{2}{2^m\floor{n/2^m}} \leq \frac{6m}{2^m} $ since $m\in [1, \log_2 n]$. The claim follows.

\section{Proof of \lowercase{\Cref{table:mmd_rates}}: \mmdcorollaryname} %
\label{sec:derivation_of_mmd_rates}
 Repeating arguments similar to those deriving  \cref{eq:simplified_mmd_bound_general} in  \cref{sec:proof_for_tables}, we find that for the advertised choices of $m$ and for any fixed $\delta$ such that $\delta'=\frac{\delta}{2}$ and $\log(1/\delta^\star) = \order(\log(n/\delta))$, the RHS of the bound~\cref{eq:mmd_thinning_bound_finite} on MMD from \cref{theorem:main_result_all_in_one} can be simplified as follows:
\begin{talign}
&\mmd_{\kernel}(\inputcoreset, \ktcoreset)\\
&\leq c \sinfnorm{\ksqrt} \parenth{c'\frac{\max(\rktau[\ksqrt], \rminpn[\inputcoreset])^2}{d}}^{\frac{d}{4}} d^{\frac14}  \sqrt{ \frac{\log (n/\delta)}{n} \brackets{ \log(\frac8\delta)+ \log \parenth{2+\frac{\klip[\ksqrt](\rk[\ksqrt]+ \rminpn[\inputcoreset])}{\sinfnorm{\ksqrt}}}  }}, \\
&= c_{\delta, d} \sinfnorm{\ksqrt} \parenth{\max(\rktau[\ksqrt], \rminpn[\inputcoreset])}^{\frac{d}{2}} \sqrt{\frac{\log n}{n}} 
 \sqrt{\log(1+\max(\rk[\ksqrt], \rminpn[\inputcoreset])) + 
\log(1+\frac{\klip[\ksqrt]}{\sinfnorm{\ksqrt}})
},
\label{eq:mmd_bound_generic_table_1}
\end{talign}
for some universal constants $c, c'$ where to simplify the expressions, we have used the fact that $\rminnew[\inputcoreset,\ksqrt,n] \leq \rminpn[\inputcoreset]$~\cref{eq:rmin_P}. 
 Noting that \cref{eq:mmd_bound_generic_table_1} holds with probability at least $1-\delta$, \cref{table:mmd_rates} now follows from plugging the assumed growth rate bounds into the estimate~\cref{eq:mmd_bound_generic_table_1}, and treating $\sinfnorm{\ksqrt}$ and $\frac{\klip[\ksqrt]}{\sinfnorm{\ksqrt}}$ as some constant while $n$ grows.

\newcommand{\saitohresultname}{Square-root representation of MMD}
\newcommand{\errordecompresultname}{$\Linf$ bound on $\ltwo$ kernel error}
\section{Proof of \lowercase{\Cref{theorem:coreset_to_mmd}}: \linfcoresetresultname} %
\label{sub:proof_of_theorem:coreset_to_mmd}
Our proof will use the following two lemmas proved in \cref{ssub:proof_of_lemma:saitoh_equivalence,ssub:proof_of_lemma:mmd_coreset_decomposition} respectively.
\begin{lemma}[\saitohresultname]
	\label{lemma:saitoh_equivalence}
	For $\kernel$ satisfying \cref{asmp:bounded_measurable} with square-root kernel $\ksqrt \in \ltwoinf$
 we have, for any  distributions $\pnew$ and $\qnew$ on $\Rd$, 
	\begin{talign}
	\label{eq:saitoh_equivalence}
		\mmd_{\kernel}(\pnew, \qnew)
		= \sup_{\funtwo\in\ltwo:\ltwonorm{\funtwo} \leq 1}\abss{\int\funtwo(\y)(\pnew\ksqrt(\y)-\qnew\ksqrt(\y))d\y}.
	\end{talign}
\end{lemma}
\begin{lemma}[\errordecompresultname]
	\label{lemma:mmd_coreset_decomposition}
	Consider any kernel $\gkernel \in \ltwoinf$ satisfying \cref{asmp:bounded_measurable}, distributions $\pnew, \qnew$ on $\reals^d$, and function $g\in\ltwo$ with $\ltwonorm{g}\leq1$.
    For any $r, a, b \geq 0$ with $a+b=1$, 
	\begin{talign}
	\label{eq:mmd_coreset_decomposition}
	\abss{\int\funtwo(\y)(\pnew\gkernel(\y)-\qnew\gkernel(\y))d\y}
	&\leq
	 \infnorm{\pnew\gkernel-\qnew\gkernel}\mrm{Vol}^{\half}(r)
		 + 2\tail[\gkernel](ar)  
		+ 2\ltwoinfnorm{\gkernel} \max\{\tail[\pnew](br),\tail[\qnew](br)\},
	\end{talign}
	where $\mrm{Vol}(r) \defeq \pi^d/\Gamma(d/2+1) r^d$ denotes the volume of the Euclidean ball $\ball(0; r)$.
\end{lemma}
We first note that, by the square-root kernel definition (\cref{def:square_root_kernel}), $\ltwonorm{\ksqrt(\x, \cdot)} = \sqrt{\kernel(\x, \x)}$ for each $x \in \real^d$.
Since $\kernel$ is bounded, we therefore have $\ksqrt\in\ltwoinf$ with $\ltwoinfnorm{\ksqrt} = \sqrt{\infnorm{\kernel}}$.
The result now follows by invoking \cref{lemma:saitoh_equivalence,lemma:mmd_coreset_decomposition} with $\gkernel = \ksqrt$.

\subsection{Proof of \lowercase{\Cref{lemma:saitoh_equivalence}}: \saitohresultname} %
\label{ssub:proof_of_lemma:saitoh_equivalence}
    Let $\krkhs$ represent the RKHS of $\kernel$.
	By \citet[Thms.~1 and 2]{saitoh1999applications} 
	and the definition~\eqref{eq:kernelsqrt} of $\ksqrt$, for any $\fun\in\krkhs$, there exists
	a function $\funtwo \in \ltwo$ such that
	\begin{talign}
	\label{eq:saitoh_equivalence_two}
 \norm{\fun}_{\kernel} = \ltwonorm{\funtwo}
 \qtext{and}
		\fun(\x) = \int \funtwo(\y) \ksqrt(\x, \y)d\y
		\qtext{for all} x\in\Rd
		,
	\end{talign}
 and, for any $g \in \ltwo$, there exists an $\fun \in \krkhs$
	such that \cref{eq:saitoh_equivalence_two} holds.
Note that the integral in  \cref{eq:saitoh_equivalence_two} is well defined for each $x$ since $g\in \ltwo$ and $\ksqrt \in \ltwoinf$.
	Hence, we have
	\begin{talign}
		\mmd_{\kernel}(\pnew, \qnew)
		&= \sup_{\fun\in\krkhs:\norm{\fun}_{\kernel} \leq 1}\abss{\pnew \fun - \qnew \fun}
		\\
		&\seq{(i)} \sup_{\funtwo\in\ltwo:\norm{\funtwo}_{\ltwo} \leq 1}\abss{\int\int\funtwo(\y)\ksqrt(\y, \x)d\y d\pnew(\x)
		- \int\int\funtwo(\y)\ksqrt(\y, \x)d\y d\qnew
		(\x)}\\
		&\seq{(ii)} \sup_{\funtwo\in\ltwo:\norm{\funtwo}_{\ltwo} \leq 1}\abss{\int\funtwo(\y)\pnew\ksqrt(\y)d\y
		- \int\funtwo(\y)\qnew\ksqrt(\y)d\y}.
		\label{eq:proof_coreset_step_1}
	\end{talign}
where step~(i) follows from \cref{eq:saitoh_equivalence_two}, and  we can swap the order of integration to obtain step~(ii) using Fubini's
theorem along with the following fact justified by \Holder's inequality:
\begin{talign}
\label{eq:g_ksqrt_integral}
	\int\int\abss{\funtwo(\y)\ksqrt(\y, \x)}d\y d\tilde{\pnew}(\x)
\leq \ltwonorm{\funtwo} \norm{\ksqrt}_{\ltwoinf} \int d\tilde{\pnew}(\x) <\infty
\stext{for any distribution.} \tilde{\pnew} 
\end{talign}

\subsection{Proof of \lowercase{\Cref{lemma:mmd_coreset_decomposition}}: \errordecompresultname} 
\label{ssub:proof_of_lemma:mmd_coreset_decomposition}
Fix any $r\ge0$, introduce the shorthand $\ball(r) = \ball(0; r)$, and define the restrictions
\begin{talign}
	\funtwo_{\radius}(\x)= \funtwo(\x) \indicator_{\ball(r)}(\x), \quad
	\gkernel_{\radius}(\x, \z) \defeq \gkernel(\x, \z) \cdot \indicator_{\ball
	(\radius)} (\z),
	\qtext{ and }
	\gkernel^{(c)}_{\radius} \defeq \gkernel - \gkernel_{\radius}, 
\end{talign}
so that
$\pnew\gkernel =\pnew\gkernel_{\radius} + \pnew\gkernel^{(c)}_{\radius}.$ 
We first note that, by Cauchy-Schwarz, $g_r\in\lone\cap\ltwo$ with 
\begin{talign}
\norm{\funtwo_{\radius}}_{\lone} \leq \norm{\funtwo_\radius}_{\ltwo}
	\cdot \sqrt{\trm{Vol}(\radius)} \leq \norm{\funtwo}_{\ltwo}
	\cdot \sqrt{\trm{Vol}(\radius)} \leq \sqrt{\trm{Vol}(\radius)}
	\label{eq:proof_coreset_claim_1_step_1a}
\end{talign}
and that, exactly as in \cref{eq:g_ksqrt_integral}, 
$\int\int\abss{\funtwo(\y)\gkernel(\y, \x)}d\y d\tilde{\pnew}(\x) < \infty$ for any distribution $\tilde{\pnew}$ so that each of the integrals to follow is well defined.
We now apply the triangle inequality and \Holder's inequality to obtain
\begin{talign}
	\abss{\int \funtwo(\y) (\pnew\gkernel(\y)-\qnew\gkernel(\y)) d\y}
	&= \abss{\int \funtwo(\y) (\pnew\gkernel_{\radius}(\y)-\qnew\gkernel_{\radius}
	(\y))
	d\y
	+  \int \funtwo(\y) (\pnew\gkernel^{(c)}_{\radius}(\y)-\qnew\gkernel^
	{(c)}_{\radius}(\y))d\y} \\
	&\le \abss{\int \funtwo_{\radius}(\y) (\pnew\gkernel_{\radius}(\y)-\qnew\gkernel_
	{\radius}(\y))
	d\y}+  
	\abss{\int_{\twonorm{\y}\geq \radius} \funtwo(\y) (\pnew\gkernel^{(c)}_{\radius}
	(\y)-\qnew\gkernel^{(c)}_{\radius}(\y))d\y} \\
	&\leq \norm{\funtwo_{\radius}}_{\lone} \cdot \infnorm{\pnew\gkernel\!-\!\qnew\gkernel} 
	\!+\! \abss{\int_{\twonorm{\y}\geq \radius} \funtwo(\y) (\pnew\gkernel(\y)\!-\!\qnew\gkernel
	(\y))d\y}.\label{eq:proof_coreset_claim_1_step_1}
\end{talign}
Next, we bound the second term in \cref{eq:proof_coreset_claim_1_step_1}.
For any $\x,\y\in\Rd$ with $\twonorm{\y} \geq \radius$ and scalars
$a, b \in [0, 1]$ such that $a+b=1$, 
either $\twonorm{\x-\y}\geq a \radius$ or $\twonorm{\x} \geq b \radius$.
Hence,
\begin{talign}
	\abss{\int_{\twonorm{\y}\geq \radius} \funtwo(\y) (\pnew\gkernel(\y)-\qnew\gkernel
	(\y))d\y}
	&=\abss{\int_{\twonorm{\y}\geq \radius} \funtwo(\y) \int_{\x\in\Rd}\gkernel
	(\x,
	\y)(d\pnew(\x)-d\qnew(\x))d\y}\\
	&\leq\abss{\int_{\twonorm{\y}\geq \radius}\int_{\twonorm{\x-\y}\geq a\radius}
	\funtwo
	(\y) \gkernel(\x, \y) (d\pnew(\x)-d\qnew(\x))d\y} \\
	&\qquad+\abss{\int_{\twonorm{\y}\geq \radius}\int_{\twonorm{\x}\geq b\radius}
	\funtwo (\y) \gkernel(\x, \y) (d\pnew(\x)-d\qnew(\x))d\y}\\
	&=:\term_1 + \term_2.
	\label{eq:proof_coreset_claim_1_step_1b}
\end{talign}
Note that both $T_1, T_2 <\infty$ since $g \in \ltwo$, $\ksqrt \in \ltwoinf$ and $\mu,\nu$ are probability measures.
We now bound the terms $\term_1$ and $\term_2$ separately in \cref{eq:proof_coreset_claim_1_step_1ba,eq:proof_coreset_claim_1_step_1bb}
below. These bounds, together with the estimates  \cref{eq:proof_coreset_claim_1_step_1,eq:proof_coreset_claim_1_step_1a}, 
yield our claim.

\begin{subequations}
\paragraph*{Bounding $\term_1$}
Substituting $\x-\y=\z$, we have
\begin{talign}
	\term_1 
	&\leq \int_{\twonorm{\x-\z}\geq \radius}\int_{\twonorm{\z}\geq a\radius}
	\abss{\funtwo(\x-\z) \gkernel(\x, \x-\z)} \abss{d\pnew(\x)-d\qnew(\x)  d\z}\\
	&\leq \int\int_{\twonorm{\z}\geq a\radius}
	\abss{\funtwo(\x-\z) \gkernel(\x, \x-\z)} d\z \abss{d\pnew(\x)-d\qnew(\x)}\\
	&\sless{(i)} \int
	\ltwonorm{\funtwo(\x-\cdot)} \sup_{\x'}\parenth{\int_{\twonorm{\z}\geq
	a\radius}
	\gkernel^2(\x',\x'-\z)d\z}^{1/2} \abss{d\pnew(\x)-d\qnew(\x)} \\
	&\seq{(ii)}\int \ltwonorm{\funtwo} \tail[\gkernel](a \radius) \abss{d\pnew(\x)-d\qnew(\x)}
	\sless{(iii)} 2\ltwonorm{\funtwo} \tail[\gkernel](a \radius),
	\label{eq:proof_coreset_claim_1_step_1ba}
\end{talign}
where step~(i) follows from Cauchy-Schwarz, step~(ii) from
the definition~\eqref{eq:tail_k_p} of $\tail[\gkernel]$, and step~(iii)
from the fact \mbox{$\int\abss{d\pnew(\x)-d\qnew(\x)} \leq 2$}.

\paragraph*{Bounding $\term_2$}
We have
\begin{talign}
	\term_2 
	&\leq \int_{\twonorm{\x}\geq b\radius}\parenth{\int_{\twonorm{\y}\geq \radius}
		\abss{\funtwo (\y) \gkernel(\x, \y)} d\y} \abss{d\pnew(\x)-d\qnew
	(\x)} \\
	&\leq \int_{\twonorm{\x}\geq b\radius}\ltwonorm{\funtwo} \sup_{\x'}\ltwonorm{\gkernel(\x', \cdot)}
	\abss{d\pnew(\x)-d\qnew(\x)} \\
	&\sless{(iv)}2\ltwonorm{\funtwo}\cdot \ltwoinfnorm{\gkernel} \max\braces{\tail[\pnew](b\radius),\tail[\qnew](b\radius)}.
	\label{eq:proof_coreset_claim_1_step_1bb}
\end{talign}
where step~(iv) follows from the definitions~\cref{eq:tail_k_p,eq:kernel_ltwo_norm} of $(\tail[\pnew],\tail[\qnew])$ and $\ltwoinfnorm{\gkernel}$.
\end{subequations}
\section{Proof of \cref{cor:mmdlinf}: \linfmmdcoresetresultname}
\label{proof_of_cor:mmdlinf}
Let $\cdim'\defeq2^{\frac{d}{2}}\cdim$ for $\cdim$ defined in \cref{theorem:coreset_to_mmd}. Notice that the bound~\cref{eq:coreset_to_mmd_bound} for the choice of $a=b=\half$ can be rewritten as
\begin{talign}
\label{eq:mmd_linf_simple}
	\mmd_{\kernel}(\P,\Q) \leq \inf_{r} \sparenth{ \cdim'  r^{\frac{d}{2}} \vareps + 2\wtil{\tail[]}(r)}.
\end{talign}
Then the claims of \cref{cor:mmdlinf} follow by optimizing the RHS of \cref{eq:mmd_linf_simple} over the choice of $r$ depending on the tail decay of $\wtil{\tail[]}$. Throughout the proofs $c, c', \rho$ denote the (exactly same) constants underlying the assumed tail decay of $\wtil{\tail[]}$.

\paragraph{Proof for \bndcase part}
Choosing $r\downarrow c'$, we obtain that
\begin{talign}
	\mmd_{\kernel}(\P,\Q) \leq \inf_{r} \sparenth{ \cdim'  r^{\frac{d}{2}} \vareps + 2c\indicator(r\leq c')}
	\leq \vareps \cdot \cdim' (c')^{\frac d2}.
\end{talign}

\paragraph{Proof for \subgauss part}
Choosing $r = \sqrt{\frac1{c'} \log(1\vee \frac{8cc'}{d\cdim'\vareps})}$,  we obtain that
\begin{talign}
	\mmd_{\kernel}(\P,\Q) \leq \inf_{r} \bigparenth{ \cdim'  r^{\frac{d}{2}} \vareps + 2c e^{-c'r^2}} \leq \vareps \cdot \cdim' \bigbrackets{\bigparenth{\frac{\log(1\vee \frac{8cc'}{d\cdim'\vareps})}{c'} }^{\frac{d}{4}} + \frac{d}{4c'}}.
\end{talign}

\paragraph{Proof for \subexp part}
Choosing $r = \frac1{c'} \log(1\vee \frac{4cc'}{d\cdim'\vareps})$, we obtain that
\begin{talign}
	\mmd_{\kernel}(\P,\Q) \leq \inf_{r} \bigparenth{ \cdim' r^{\frac{d}{2}} \vareps + 2c e^{-c'r}} \leq \vareps \cdot \cdim' \bigbrackets{\bigparenth{\frac{\log(1\vee \frac{4cc'}{d\cdim'\vareps})}{c'} }^{\frac{d}{2}} + \frac{d}{2c'}}.
\end{talign}

\paragraph{Proof for \heavytail part}
Choosing $r = (\frac{4c\rho}{d\cdim'\vareps})^\frac{2}{d+2\rho}$, we obtain that
\begin{talign}
	\mmd_{\kernel}(\P,\Q) \leq \inf_{r} \bigparenth{\cdim' r^{\frac{d}{2}} \vareps + 2c r^{-\rho}} \leq (\vareps \cdot \cdim')^{\frac{2\rho}{d+2\rho}} \cdot (\frac{4c\rho}{d})^{\frac{d}{d+2\rho}} (1+\frac{2\rho}{d}).
\end{talign}

\newcommand{\vn}[1][n]{\braces{\invec[j]}_{j=1}^{#1}}
\newcommand{\sbonestepresultname}{Alternate representation of $\outvec[i]$}
\newcommand{\highprobsuccessresultname}{Self-balancing Hilbert walk success probability}
\section{Proof of \lowercase{\Cref{sbhw_properties}}: \sbhwpropname} %
\label{sec:proof_of_sbhw_properties}
We prove each property from \cref{sbhw_properties} one by one.
\subsection{Property~\ref{item:fun_subgauss}: Functional sub-Gaussianity}\label{sec:fun_subgauss}
We  prove the functional sub-Gaussianity claim~\cref{eq:subgaussian_bound_on_w}
 by induction on the iteration $i \in \{0,\dots,n\}$.
Our proof uses 
the following lemma proved in \cref{sub:proof_of_cref_lemma_sb_one_step}, which supplies a convenient decomposition for the self-balancing Hilbert walk iterates.
\begin{lemma}[\sbonestepresultname]
	\label{lemma:sb_one_step}
    For each $i\in[n]$, the iterate $\outvec[i]$ of the self-balancing Hilbert walk (\cref{algo:self_balancing_walk}) satisfies 
	\begin{talign}
	\label{eq:temprv_defn}
		\doth{\outvec[i], \genvec}
			=  \doth{\outvec[i-1], \genvec - \invec[i] \textfrac{\doth{\invec[i], \genvec}
			}{\cnew[i]}}
		+ \temprv
		[i]\doth{\invec[i], \genvec}
		\qtext{for all} \genvec \in \rkhs
	\end{talign}
	for the random variable $\temprv \defeq \indic{\abss{\wvprod[i]} \leq \cnew[i]}(\eta_i + \wvprod[i]/\cnew)$ which satisfies
	\begin{talign}
	\label{eq:temprv_constraints}
		\E[\temprv[i]\vert\outvec[i-1]] = 0,
		\quad
		\temprv[i] \in [-2, 2], 
		\qtext{and}
		\E[e^{t\vareps_i}\vert \outvec[i-1]] \leq e^{t^2/2} \qtext{for all} t \in \R.
	\end{talign}
\end{lemma}
Now we proceed with our induction argument.

\paragraph*{Base case} For $i=1$, noting that  $\outvec[0]=0$, we have
  \begin{talign}
      \E[\exp\parenth{\doth{\outvec[1], \genvec}} ] \seq{\cref{eq:temprv_defn}} \E[\exp\parenth{\vareps_1 \doth{\invec[1], \genvec}}\vert]
      \sless{\cref{eq:temprv_constraints}}
       \exp\parenth{\half\doth{\invec[1], \genvec}^2} 
       \leq \exp\parenth{\half\hnorm{\invec[1]}^2\hnorm{\genvec}^2},
  \end{talign}
  where the last step follows from Cauchy-Schwarz, and thus $\outvec[1]$ is sub-Gaussian with parameter $\sgparam[1]=\hnorm{\invec[1]}$ as desired.

\paragraph*{Inductive step} 
Fix any $i \in [n]$ with $i \geq 2$ and assume that the functional sub-Gaussianity claim~\cref{eq:subgaussian_bound_on_w} holds
for $\outvec[i-1]$ with $\sgparam[i-1]$. We have
\begin{talign}
    \E[\exp\parenth{\doth{\outvec, \genvec}} ] &= 
    \E[\E[\exp\parenth{\doth{\outvec, \genvec}}\vert \outvec[i-1]] ] \\
    &\seq{\cref{eq:temprv_defn}}\E\brackets{
    \exp\parenth{\doth{\outvec[i-1], \genvec-\invec\frac{\doth{\invec, \genvec}}
	{\cnew[i]}}} \cdot
    \E[\exp\parenth{\vareps_i \doth{\invec, \genvec}}\vert \outvec[i-1]]} \\
    &\sless{\cref{eq:temprv_constraints}}
    \E\brackets{
    \exp\parenth{\doth{\outvec[i-1], \genvec-\invec\frac{\doth{\invec, \genvec}}
	{\cnew[i]}}} \cdot
    \exp\parenth{\half\doth{\invec, \genvec}^2} }\\
    &=
    \exp\parenth{\half\doth{\invec, \genvec}^2}  \cdot
    \E\brackets{
    \exp\parenth{\doth{\outvec[i-1], \genvec-\invec\frac{\doth{\invec, \genvec}}
	{\cnew[i]}}}}
	\\
	&\sless{(i)} 
    \exp\parenth{\half\doth{\invec, \genvec}^2+ \frac{\sgparam[i-1]^2}{2} \norm{\genvec-\invec\frac{
    \doth{\invec, \genvec}}{\cnew[i]}}_{\rkhs}^2},
    \label{eq:subgauss_proof_step}
\end{talign}
where step~(i) follows from 
 the induction hypothesis.
Simplifying the exponent in the display~\cref{eq:subgauss_proof_step} using Cauchy-Schwarz and the definition~\cref{eq:subgauss_const_def} of  $\sgparam$, we have
\begin{talign}
\half\doth{\invec, \genvec}^2 + 
\frac{\sgparam[i-1]^2}{2} \norm{\genvec-\invec\frac{\doth{\invec, \genvec}}
	{\cnew[i]}}_{\rkhs}^2 
&=
    \half\doth{\invec, \genvec}^2 + 
    \frac{\sgparam[i-1]^2}{2}\parenth{\norm{\genvec}_{\rkhs}^2
	+ \doth{\invec, \genvec}^2 \parenth{\frac{\norm{\invec}^2_{\rkhs}}{\cnew[i]^2}
	-\frac{2}{\cnew[i]}}} \\
&=
    \frac{\sgparam[i-1]^2}{2} \norm{\genvec}_{\rkhs}^2
    + \doth{\invec, \genvec}^2 \cdot 
    \parenth{\half+\frac{\sgparam[i-1]^2\norm{\invec}_{\rkhs}^2}{2\cnew[i]^2}
	-\frac{\sgparam[i-1]^2}{\cnew[i]}}\\
&\leq
    \frac{\sgparam[i-1]^2}{2} \norm{\genvec}_{\rkhs}^2
    + \doth{\invec, \genvec}^2 \cdot 
    \parenth{\half+\frac{\sgparam[i-1]^2\norm{\invec}_{\rkhs}^2}{2\cnew[i]^2}
	-\frac{\sgparam[i-1]^2}{\cnew[i]}}_+\\
&\leq
    \frac{\hnorm{\genvec}^2}{2}\left( \sgparam[i-1]^2 + \hnorm{\invec[i]}^2\Big(1 + \frac{\sgparam[i-1]^2}{\cnew[i]^2}(\hnorm{\invec[i]}^2 - 2\cnew[i])\Big)_+\right)\\
&=
    \frac{\sgparam[i]^2}{2}\norm{\genvec}_{\rkhs}^2.
\end{talign}

\subsection{Property~\ref{item:signed_sum}: Signed sum representation}
Since \cref{algo:self_balancing_walk} adds $\pm \invec$ to $\outvec[i-1]$ whenever 
$|\wvprod| = |\dotrkhs{\outvec[i-1], \invec[i]}{\rkhs}| 
\leq \cnew$, by the union bound, it suffices to lower bound the probability of this event by $1-\delta_i$ for each $i$.
The following lemma 
establishes this bound using the functional sub-Gaussianity \cref{eq:subgaussian_bound_on_w} of each $\outvec[i-1]$.
\begin{lemma}[\highprobsuccessresultname]
	\label{lemma:high_prob_event}
	The self-balancing Hilbert walk (\cref{algo:self_balancing_walk}) with
	threshold 
    $\cnew[i]
        \geq \sgparam[i-1]\hnorm{\invec[i]}\sqrt{2\log(2/\delta_i)}$
    for  $\delta_i \in (0,1]$ 
    satisfies 
	\begin{talign}
	\label{eq:prob_event_algo}
		\Prob(\esuccess[i]) \geq 1-\delta_i
		\qtext{for}
	\esuccess[i] = \braces{\abss{\dotrkhs{\outvec[i-1], \invec[i]}{\rkhs}}  \leq \cnew}.%
	\end{talign}
\end{lemma}
\begin{proof}
Instantiate the notation of \cref{sbhw_properties}. 
The sub-Gaussian Hoeffding inequality \citep[Prop.~2.5]{wainwright2019high}, the functional  sub-Gaussianity of $\outvec[i-1]$ \cref{eq:subgaussian_bound_on_w}, and the choice of $\cnew$ imply that
\begin{talign}
\Pr(\esuccess[i]^c)
    =
\Pr(\abss{\dotrkhs{\outvec[i-1],\invec[i]}{\rkhs}} > \cnew )
    &\leq 2\exp(-\cnew^2/(2\sgparam[i-1]^2\hnorm{\invec[i]}^2)) 
    \leq 2\exp(-\log(2/\delta_i)) 
    = \delta_i.
\end{talign}
\end{proof}

\subsection{Property~\ref{item:exact_two_thin}: Exact halving via symmetrization}
Whenever the signed sum representation \cref{eq:signed_sum} holds, we have
\begin{talign}
    \outvec[n] 
        = \sum_{i = 1}^{n} \eta_i f_i 
        = \sum_{i = 1}^{n} (\eta_i g_{2i-1} - \eta_i g_{2i})
        = \sum_{i = 1}^{2n} g_{i} - 2 \sum_{i \in \indices} g_i
\end{talign}
where the last step follows from the definition of $\indices$.

\subsection{Property~\ref{item:pointwise_subgauss}: Pointwise sub-Gaussianity in RKHS}
The reproducing property of the kernel $\kernel$ and the established functional sub-Gaussianity \cref{eq:subgaussian_bound_on_w} yield
\begin{talign}
	\E[\exp(\outvec[i](x)) ]
	    = \E[\exp(\doth{\outvec[i], \kernel(\x,\cdot)}) ]
			\leq
			\exp\parenth{\frac{\sigma_{i}^2\hnorm{\kernel(\x,\cdot)}^2}{2}}
			=
			\exp\parenth{\frac{\sigma_{i}^2\kernel(\x,\x)}{2}},
			\quad\forall\x \in \X.
\end{talign}
\subsection{Property~\ref{item:subgauss_bound}: Sub-Gaussian constant bound}
We establish the bound  \cref{eq:sbhw_subgauss_bound}
for all $i \in \{0,\dots,n\}$ by induction on the iteration $i$.

\paragraph*{Base case} The claim \cref{eq:sbhw_subgauss_bound} holds for the base case, $i=0$, since $\sgparam[0]=0$.

\paragraph*{Inductive step} 
Fix any $i \in [n]$ and assume that the claim~\cref{eq:sbhw_subgauss_bound} holds
for $\sgparam[i-1]$. 

If either $\hnorm{\invec[i]} = 0$ or $\sgparam[i-1]^2 \geq \frac{\cnew^2}{2\cnew-\hnorm{\invec[i]}^2}$,
then $\sgparam[i]^2 = \sgparam[i-1]^2$ by the definition \cref{eq:subgauss_const_def} of $\sgparam$ and the assumption that $\frac{\hnorm{\invec[i]}^2}{2} \leq \cnew$, completing the inductive step.

If, alternatively, $\sgparam[i-1]^2 < \frac{\cnew^2}{2\cnew-\hnorm{\invec[i]}^2}$ and $\hnorm{\invec[i]} > 0$, then our assumptions $\frac{\hnorm{\invec[i]}^2}{1+q} \leq \cnew \leq \frac{\hnorm{\invec[i]}^2}{1-q}$ 
imply that $(\frac{\hnorm{\invec[i]}^2}{\cnew} - 1)^2 \leq q^2$.
Hence, by the definition \cref{eq:subgauss_const_def} of $\sgparam$
and the inductive hypothesis, 
\begin{talign}
\sgparam[i]^2
    &= 
        \sgparam[i-1]^2 + \hnorm{\invec[i]}^2\Big(1 + \frac{\sgparam[i-1]^2}{\cnew[i]^2}(\hnorm{\invec[i]}^2 - 2\cnew[i])\Big) \\
    &= 
        \hnorm{\invec[i]}^2
    + \sgparam[i-1]^2
    (\frac{\hnorm{\invec[i]}^2}{\cnew} - 1)^2 \\
    &\leq 
        (1-q^2)\frac{\hnorm{\invec[i]}^2}{1-q^2}
    + q^2\frac{\max_{j\in[i-1]} \hnorm{\invec[j]}^2}{1-q^2}
    \leq
        \frac{\max_{j\in[i]} \hnorm{\invec[j]}^2}{1-q^2},
\end{talign}
completing the inductive step.

\subsection{Property~\ref{item:adaptive_thresh}: Adaptive thresholding}
\newcommand{\cstar}{c^{\star}}
Define $\cstar_1 = \max(\cstar,1)$ and
let 
$q = \frac{(\cstar_1)^2 - 1}{(\cstar_1)^2 + 1} \in [0,1)$
so that
\begin{talign}
\frac{1}{1-q} 
    = \frac{1+(\cstar_1)^2}{2}
\qtext{and}
\frac{1}{1-q^2} 
    = \frac{(\cstar_1+1/\cstar_1)^2}{4}
    \leq \frac{(\cstar+1/\cstar)^2}{4},
\end{talign}
since $1 = \argmin_{c \geq 0} c + 1/c$.
By assumption, 
$\cnew 
    \geq \hnorm{\invec[i]}^2 
    \geq \frac{\hnorm{\invec[i]}^2}{1+q}$
for all $i\in[n]$.

Now suppose that  	
$\sgparam[i-1]^2 
    < \frac{\cnew^2}{2\cnew-\hnorm{\invec[i]}^2}$
and
$\hnorm{\invec[i]} > 0$.
If $\cnew \leq \hnorm{\invec[i]}^2$, then $\cnew \leq \frac{\hnorm{\invec[i]}^2}{1-q}$. 
If, alternatively, $\cnew \leq c_i \sgparam[i-1] \hnorm{\invec[i]}$, then 
\begin{talign}
\cnew < \half\hnorm{\invec[i]}^2(1 + c_i^2)
    \leq \frac{\hnorm{\invec[i]}^2}{1-q}.
\end{talign}
The conclusion now follows from the sub-Gaussian constant bound \cref{eq:sbhw_subgauss_bound}.

\subsection{Proof of \lowercase{\Cref{lemma:sb_one_step}}: \sbonestepresultname} %
\label{sub:proof_of_cref_lemma_sb_one_step}
\cref{algo:self_balancing_walk} and our definition $\temprv \defeq \indic{\abss{\wvprod[i]} \leq \cnew[i]}(\eta_i + \wvprod[i]/\cnew)$ give
\begin{talign}
\outvec = \outvec[i-1] - \invec\, \wvprod[i]/\cnew + \indic{\abss{\wvprod[i]} \leq \cnew[i]}( \invec \,\wvprod[i]/\cnew + \eta_i \invec)
    = \outvec[i-1] - \invec\,\frac{\dotrkhs{\invec, \outvec[i-1]}{\rkhs}}{\cnew} + \temprv\invec.
\end{talign}
Taking an inner product with $u\in\rkhs$ now yields the equality \cref{eq:temprv_defn}.
By construction, $\temprv[i]  \in [c_{min},c_{max}] \subseteq [-2,2]$ for
\begin{talign}
c_{min} 
    = \max(-2, \min(0, -1+\wvprod[i]/\cnew[i]))
\qtext{and}
c_{max}
    = \min(2, \max(0, 1+\wvprod[i]/\cnew[i]))
\end{talign}
by construction.
Moreover, 
\begin{talign}
    \E[\temprv[i]\mid\outvec[i-1],\abss{\wvprod[i]} > \cnew[i]] 
        &= 0 \\
	\E[\temprv[i]\mid\outvec[i-1],\abss{\wvprod[i]} \leq \cnew[i]] 
	    &= 
    \parenth{1+\frac{\wvprod[i]}{ \cnew[i]}} \cdot \frac{1}{2}\parenth{1-\frac{\wvprod[i]}{ \cnew[i]}} + 
    \parenth{-1+\frac{\wvprod[i]}{ \cnew[i]}} \cdot \frac{1}{2}\parenth{1+
    \frac{\wvprod[i]}{ \cnew[i]}} = 0,
\end{talign}
so that $\E[\temprv[i]\mid\outvec[i-1]] = 0$ as claimed.
The conditional sub-Gaussianity claim 
\begin{talign}
\E[e^{t\vareps_i}\mid \outvec[i-1]] \leq e^{t^2/2} \qtext{for all} t \in \R,
\end{talign}
now follows from Hoeffding's lemma~\citep[(4.16)]{hoeffding1963probability} since $\temprv$ is bounded with 
$c_{max} - c_{min} \leq 2$
and mean-zero conditional on $\outvec[i-1]$.

\newcommand{\xseq}[1][n]{\mc X^{#1}}
\newcommand{\xseqnum}{\axi[1], \axi[2], \ldots}
\newcommand{\xseqnumn}[1][n]{\axi[1], \axi[2], \ldots, \axi[n]}
\newcommand{\dothrt}[1]{\angles{#1}_{\ksqrt}}

\newcommand{\orlicz}{\Psi_2}
\newcommand{\ometric}{\rho}
\newcommand{\diam}{D}
\newcommand{\tempone}{\alpha}
\newcommand{\temptwo}{\beta}
\newcommand{\einf}{\event[\infty]}
\newcommand{\einftwo}{\widetilde\eventnotag_{\infty}}
\newcommand{\esup}{\event[\textup{sup}]}
\newcommand{\ehalf}{\event[\textup{half}]}

\newcommand{\einfj}[1][\m]{\event[\infty]^{(#1)}}
\newcommand{\einfjtwo}[1][\m]{\widetilde\eventnotag_{\infty}^{(#1)}}
\newcommand{\esupj}[1][\m]{\event[\textup{sup}]^{(#1)}}
\newcommand{\ehalfj}[1][j]{\event[\textup{half}]^{(#1)}}
\newcommand{\sumpsi}[1][\m]{\mc{W}_{#1}}

\newcommand{\boundonsigmaresultname}{Bound on $\sigma_{n/2}$}
\newcommand{\outputbasicboundresultname}{A basic bound on $\sinfnorm{\outvec[n/2]}$}
\newcommand{\orliczboundresultname}{A high probability bound on supremum of $\outvec[n/2]$ differences}
\newcommand{\orliczprocessresultname}{$\outvec[n/2]$ is an Orlicz $\orlicz$-process}
\newcommand{\boundsoncoveringandjresultname}{Bounds on $\mc N_{\mbb{T},\ometric}$ and $\mc J_{\mbb{T},\ometric}$}
\newcommand{\multioutputbasicboundresultname}{A basic bound on $\sinfnorm{\sumpsi}$}
\newcommand{\multiorliczboundresultname}{A high probability bound on supremum of $\sumpsi$ differences}
\newcommand{\directcoverboundresultname}{A direct covering bound on $\sinfnorm{\outvec[n/2]}$}
\section{Proof of \lowercase{\Cref{kernel_halving_results}}:  \kernelhalvingresultsname} %
\label{sec:proof_of_kernel_halving_results}
We start by showing that after $m$ rounds of kernel halving the output size is $\nout = \floor{\frac{n}{2^\m}}$ (also see \cref{footnote:outputsize}). We prove this by induction. The base case of $m=1$ can be directly verified. Let $n_j\defeq |\coreset[j]|$ denote the output size after $j$ rounds of kernel halving and assume that the hypothesis is true for round $j$, i.e., $n_j = \floor{\frac{n}{2^j}}$. Then for round $j+1$, we have $n_{j+1} = \floor{n_j/2} = \floor{\frac{\floor{n/2^j}}{2}} \seq{(i)}  \floor{\frac{n}{2^{j+1}}}$, where step~(i) follows from the fact \citep[Eqn.~(3.11)]{graham94} that $\floor{\frac{\floor{\l}}{2}}=\floor{\frac{\l}{2}}$ for any integer $\l$.
Our desired claim follows.

Next, define $\nin=2^m\nout=2^m\floor{\frac n{2^m}}$. 
In \cref{sub:proof_one_round_kernel_halving,sub:proof_multiple_rounds_kernel_halving} we will show that the respective claims~\cref{eq:kernel_halving_bound_finite,eq:kernel_halving_bound_m_rounds} 
hold whenever $n=\nin$, that is, whenever $2^m$ divides $n$ evenly.
Now suppose that $n\neq\nin$.
Since the output of $m$ rounds of kernel halving depends only on the first $\nin$ points, the evenly divisible case of \cref{sub:proof_one_round_kernel_halving} implies that, 
for part \cref{item:kernel_halving_one_round}, 
    \begin{talign}
        \label{eq:kernel_halving_bound_finite_nin}
            \sinfnorm{\P_{\nin}\gkernel-{\Q}_{\mrm{KH}}^{(1)}\gkernel}
            \leq
             \frac{\sinfnorm{\gkernel}}{\nout} \cdot \err_{\gkernel}(\nin,\! 1, \!d, \!\delta^\star,  \!\delta', \!\rksmingen{\cset_{\nin}, \gkernel, \nin})
             \stext{for} \nin = 2\floor{\frac{n}{2}} = 2 \nout,
    \end{talign}
    with probability at least $1\!-\!\delta'\!-\!\sum_{i=1}^{\nin/2} \delta_i$, and \cref{sub:proof_one_round_kernel_halving} implies that, for  part  \cref{item:kernel_halving_multiple_rounds}, 
    \begin{talign}
        \sinfnorm{\P_{\nin}\gkernel\!-\!\Q^{(\m)}_{\mrm{KH}}\gkernel}
        &\leq \frac{\sinfnorm{\gkernel}}{\nout}\cdot\err_{\gkernel}(\nin, \!m, \!d, \!\delta^\star,\! \delta', \!\rksmingen{\cset_{\nin}, \gkernel, \nin})
        \stext{for}
        \nin \!=\! 2^m\floor{\frac{n}{2^m}} \!=\! 2^m \nout,
        \label{eq:kernel_halving_bound_m_rounds_nin}
    \end{talign}
    with probability at least $1-\delta'\!-\!\sum_{j=1}^{\m} \frac{2^{j-1}}{m} \sum_{i=1}^{\nin/2^j}\delta_i $. %
    Next, to recover the  bounds~\cref{eq:kernel_halving_bound_finite,eq:kernel_halving_bound_m_rounds}, we use the following deterministic inequalities:
    \begin{talign}
    \infnorm{\P_n\gkernel - \P_{\nin}\gkernel} &\leq \frac{2(n-\nin)\infnorm{\gkernel}}{n} \leq \frac{2(2^m-1)\infnorm{\gkernel}}{n} \leq \frac{2\infnorm{\gkernel}}{\nout},  
    \label{eq:perr_nin}
    \\
    \rminnew[\cset_{\nin}, \gkernel, \nin]&\leq\rminnew[\inputcoreset, \gkernel, n],
    \qtext{and} \label{eq:r_nin}\\
    \err_{\gkernel}(\nin, \!m, \!d, \!\delta, \!\delta', \!\rminnew[\cset_{\nin}, \gkernel, \nin]) \!+\! 2 
    &\sless{(i)} 
    \err_{\gkernel}(\nin, \!m, d,\! \delta,\!\delta', \!\rminnew[\inputcoreset, \gkernel, n]) \!+\! 2 \\
    &\sless{(ii)} \err_{\gkernel}(n, \!m, \!d,\! \delta,\!\delta', \!\rminnew[\inputcoreset, \gkernel, n]),
    \label{eq:err_err_nin}
    \end{talign}
    where \cref{eq:r_nin} follows directly from the definition~\cref{eq:rmin_P}, and step~(i) follows from \cref{eq:r_nin} and the fact that $\err_{\gkernel}(n,\! m,\! d,\! \delta,\! \delta',\! R)$~\cref{eq:err_simple_defn} is non-decreasing in $R$, and step~(ii) follows since $\err_{\gkernel}(n,\! m,\! d,\! \delta,\! \delta',\! R)$ is non-decreasing in $n$ and $\nin\neq n$.
    For part~\cref{item:kernel_halving_one_round} we conclude, by \cref{eq:kernel_halving_bound_finite_nin},
    \begin{talign}
        \sinfnorm{\P_{n}\gkernel\!-\!\Q^{(1)}_{\mrm{KH}}\gkernel}
        &\leq \infnorm{\P_n\gkernel - \P_{\nin}\gkernel}+ \sinfnorm{\P_{\nin}\gkernel\!-\!\Q^{(1)}_{\mrm{KH}}\gkernel} \\
        &\sless{\cref{eq:perr_nin}} \frac{2\infnorm{\gkernel}}{\nout} +  \frac{\sinfnorm{\gkernel}}{\nout} \cdot \err_{\gkernel}(\nin, 1, d, \delta^\star,  \delta', \rksmingen{\cset_{\nin}, \gkernel, \nin})  \\ 
        &\sless{\cref{eq:err_err_nin}} \frac{\sinfnorm{\gkernel}}{\nout} \cdot \err_{\gkernel}(n, 1, d, \delta, \delta', \rminnew[\inputcoreset, \gkernel, n]).
    \end{talign}
    For part~\cref{item:kernel_halving_multiple_rounds}, we conclude, by   \cref{eq:kernel_halving_bound_m_rounds_nin},
    \begin{talign}
        \sinfnorm{\P_{n}\gkernel\!-\!\Q^{(\m)}_{\mrm{KH}}\gkernel}
        &\leq \infnorm{\P_n\gkernel - \P_{\nin}\gkernel}+ \sinfnorm{\P_{\nin}\gkernel\!-\!\Q^{(\m)}_{\mrm{KH}}\gkernel} \\
        &\sless{\cref{eq:perr_nin}} \frac{2\infnorm{\gkernel}}{\nout} + \frac{\sinfnorm{\gkernel}}{\nout} \cdot \err_{\gkernel}(\nin, 1, d, \delta^\star,  \delta', \rksmingen{\cset_{\nin}, \gkernel, \nin})  \\ 
        &\sless{\cref{eq:err_err_nin}} \frac{\sinfnorm{\gkernel}}{\nout} \cdot \err_{\gkernel}(n, m, d, \delta,\delta', \rminnew[\inputcoreset, \gkernel, n]).
    \end{talign}

\subsection{Proof of part~\lowercase{\Cref{item:kernel_halving_one_round}}: \kernelhalvingoneroundname}
\label{sub:proof_one_round_kernel_halving}

As noted earlier, we prove this part assuming $n$ is even.
Consider a self-balancing Hilbert walk (\cref{algo:self_balancing_walk}) with inputs $(\invec)_{i=1}^{n/2}$ and $(\cnew[i])_{i=1}^{n/2}$, where $\invec = \gkernel(\axi[2i-1], \cdot) - \gkernel(\axi[2i], \cdot)$ and $\cnew[i] = \max(\knorm{\fun_i}\sgparam[i-1] \sqrt{2\log(2/\delta_i)}, \knorm{\fun_i}^2)$, and $\sgparam[i]$ was defined in \cref{eq:subgauss_const_def}. 
Property~\ref{item:exact_two_thin} of \cref{sbhw_properties} implies that for a self-balancing walk with the choices summarized above, the event
$\ehalf = \sbraces{\abss{\wvprod[i]} \leq \cnew[i] \stext{for all} i \in [{n/2}]}$ satisfies
\begin{talign}
\label{eq:ehalf}
    \Pr(\ehalf) \geq 1-\sum_{i=1}^{{n/2}} \delta_i.
\end{talign}
Consider kernel halving coupled with the above instantiation of self-balancing Hilbert walk.
Due to the equivalence with kernel halving on the event $\ehalf$, we conclude that the output $\outvec[n/2]$ of self-balancing Hilbert walk matches with that of the kernel halving, and satisfies
\begin{talign}
\label{outvec_as_kernel_diff}
\frac{1}{n}\outvec[n/2] 
    = \frac{1}{n}\sum_{i=1}^n \gkernel(x_i,\cdot)
    - \frac{1}{{n/2}}\sum_{x\in\coreset[1]}\gkernel(x,\cdot)
    = \P_n\gkernel-{\Q}_{\mrm{KH}}^{(1)}\gkernel,
\end{talign}
on the event $\ehalf$.
Furthermore, on the event $\ehalf$, we can also write that the kernel halving coreset satisfies $\coreset[1] = (x_i)_{i\in\indices}$ for $\indices = \{2i-\frac{\eta_i-1}{2} : i\in [n/2]\}$ and $\eta_i$ defined in \cref{algo:self_balancing_walk}.
Finally, applying property~\ref{item:adaptive_thresh} of \cref{sbhw_properties} for $\sigma_{n/2}$ with $c_i = \sqrt{2\log(2/\delta_i)}$, we obtain that
\begin{talign}
    	\label{eq:sigma_bound}
    \sigma^2_{n/2} \leq \frac{\max_{i\leq n/2} \knorm{\invec[i]}^2}{4}\cdot  2\log(\frac{2}{\delta^\star}) (1+\frac{1}{2\log(2/\delta^\star)})^2  \sless{(i)}4\infnorm{\gkernel} \log(\frac{2}{\delta^\star})
\end{talign}
where in step~(i), we use the fact that $\fun_i = \gkernel(\axi[2i-1], \cdot) - \gkernel(\axi[2i], \cdot)$ satisfies
\begin{talign}
\label{eq:bi_kernel_havling}
\knorm{\invec[i]}^2 
    = \gkernel(x_{2i-1},x_{2i-1}) + \gkernel(x_{2i},x_{2i}) - 2\gkernel(x_{2i-1},x_{2i})
    \leq 4\infnorm{\gkernel}.
\end{talign}

Next, we split the proof in two parts: \textbf{Case (I)} When $\rminnew[\inputcoreset, \gkernel, n] < \rminpn$, and \textbf{Case (II)} when $\rminnew[\inputcoreset, \gkernel, n] = \rminpn$, where $\rminnew[\inputcoreset,\gkernel, n]$ was defined in \cref{eq:rmin_P}.  We prove the results for these two cases in \cref{ssub:when_rmin_geq_n,ssub:when_rmin_leq_n} respectively.
In the sequel, we make use of the following tail quantity of the kernel:
\begin{talign}
\label{eq:ktailbar}
    \ktail[\gkernel](R')\defeq\sup\braces{\abss{\gkernel(\x,\y)} : \twonorm{\x-\y}\geq R'}.
\end{talign}

\newcommand{\ecover}{\event[\textup{cover}]}
\subsubsection{Proof for case (I): When $\rminnew[\inputcoreset, \gkernel, n] < \rminpn$}
\label{ssub:when_rmin_geq_n}
By definition~\cref{eq:rmin_P}, for this case, 
\begin{talign}
\label{eq:rmin_new_case_a}
    \rminnew[\inputcoreset, \gkernel, n] = 
    n^{1+\frac1d}\rmin_{\gkernel, n}+ n^{\frac1d} \frac{\sinfnorm{\gkernel}}{\klip[\gkernel]}.
\end{talign}
On the event $\ehalf$, the following lemma provides a high probability bound on  $\sinfnorm{\outvec[n/2]}$ in terms of the kernel parameters, the sub-Gaussianity parameter $\sgparam[n/2]$, and the size of the \emph{cover} \citep[Def.~5.1]{wainwright2019high} of a neighborhood of the input points $(x_i)_{i=1}^n$. 
\begin{lemma}[\directcoverboundresultname]
    \label{lemma:direct_cover_bound}
    Fix $R \geq r > 0$ and $\delta'>0$, and suppose $\pcover(r, R)$ is a set of minimum cardinality satisfying
    \begin{talign}
    \label{eq:sample_cover_set}
    \bigcup_{i=1}^n \ball(x_i, R)
        \subseteq 
        \bigcup_{z\in\pcover(r,R)} \ball(z,r).
    \end{talign}
    If $\gkernel$ satisfies \cref{asmp:bounded_measurable,assum:lipkernel}, then, conditional on the event $\ehalf$~\cref{outvec_as_kernel_diff}, the event 
    \begin{talign}
    \label{eq:einf_direct_cover}
        \einf
        \defeq
        \braces{\sinfnorm{\outvec[n/2]}
        \leq
        \max\parenth{
        n \ktail[\gkernel](R), \ \ 
        n \klip[\gkernel]r +
        \sgparam[n/2]\sqrt{2\infnorm{\gkernel} \log(2|\pcover(r, R)|/\delta')}
        }},
    \end{talign}
    occurs with probability at least $1\!-\!\delta'$, i.e., $\Pr(\einf \vert \ehalf) \geq 1\!-\!\delta'$, where $\ktail[\gkernel]$ was defined in \cref{eq:ktailbar}.
\end{lemma}
\cref{lemma:direct_cover_bound} succeeds in controlling $\outvec[n/2](x)$ for all $x\in\reals^d$ since either $x$ lies far from every input point $x_i$ so that each $\gkernel(x_i, x)$ in the expansion \cref{outvec_as_kernel_diff} is small or $x$ lies near some $x_i$, in which case $\outvec[n/2](x)$ is well approximated by $\outvec[n/2](z)$ for $z \in  \pcover(r, R)$.
The proof inspired by the covering argument of \citet[Lem.~2.1]{phillips2020near} and using the pointwise sub-Gaussianity property of \cref{sbhw_properties} over the finite cover $\pcover$ can be found in \cref{ssub:proof_of_lemma:direct_cover_bound}.

Now we put together the pieces to prove \cref{kernel_halving_results}.

First, \citep[Lem.~5.7]{wainwright2019high} implies that 
$|\pcover[1](r, R)| \leq (1+2R/r)^d$ (i.e., any ball of radius $R$ in $\Rd$ can be covered by $(1+2R/r)^d$ balls of radius $r$). Thus for an arbitrary $R$, we can conclude that
\begin{talign}
\label{eq:cover_size}
    |\pcover[n](r, R)| \leq n(1+2R/r)^d = (n^{1/d} + 2n^{1/d}R/r)^d.
\end{talign}

Second we fix $R$ and $r$ such that $n \ktail[\gkernel](R)=\infnorm{\gkernel}$ and $n \klip[\gkernel]r=\sinfnorm{\gkernel}$, so that $\frac{R}{r}\gets n\rmin_{\gkernel, n}\frac{\klip[\gkernel]}{\sinfnorm{\gkernel}}$ (c.f. \cref{eq:ktailbar,eq:rmin_k}). Substituting these choices of radii in the bound~\cref{eq:einf_direct_cover} of \cref{lemma:direct_cover_bound}, we find that conditional to $\ehalf\cap \einf$, we have
\begin{talign}
\label{eq:err_equal_to_max_defn}
\sinfnorm{\outvec[n/2]} &\leq
 \max\parenth{
        n \ktail[\gkernel](R), \ \ 
        n \klip[\gkernel]r +
        \sgparam[n/2]\sqrt{2\infnorm{\gkernel} \log(2|\pcover(r, R)|/\delta')}} \\
    &\sless{\cref{eq:sigma_bound}} \sinfnorm{\gkernel} + 
    2\sqrt{2} \sinfnorm{\gkernel} \sqrt{ \log(\frac{2}{\delta^\star}) \brackets{\log(\frac{2}{\delta'})
    + d \log\parenth{n^{\frac1d}+\frac{2\klip[\gkernel]}{\sinfnorm{\gkernel}} \cdot n^{1+\frac1d}\rmin_{\gkernel, n} } } } 
    \label{eq:direct_cover_argument_1}\\
    &\sless{\cref{eq:rmin_new_case_a}} \sinfnorm{\gkernel} + 
    2\sqrt{2} \sinfnorm{\gkernel} \sqrt{ \log(\frac{2}{\delta^\star}) \brackets{\log(\frac{2}{\delta'})
    + d \log\parenth{\frac{2\klip[\gkernel]}{\sinfnorm{\gkernel}} (\rmin_{\gkernel, n} +  \rminnew[\inputcoreset,\gkernel, n] ) } } } \\
    &\sless{\cref{eq:err_simple_defn}} \sinfnorm{\gkernel} \cdot 2\err_{\gkernel}(n, 1, d, \delta^\star, \delta', \rminnew[\inputcoreset,\gkernel, n]).
    \label{eq:direct_cover_argument_2}
\end{talign}
Putting \cref{outvec_as_kernel_diff,eq:direct_cover_argument_2} together, we conclude
\begin{align}
\Pr\big(\sinfnorm{\P_n\gkernel-{\Q}_{\mrm{KH}}^{(1)}\gkernel} > \frac{1}{n/2} \err_{\gkernel}(n,1, d,\delta^\star, \delta', \rminnew[\inputcoreset,\gkernel, n])\big)
&\leq \Pr((\ehalf \cap \einf)^c) \\
&= \Pr(\ehalf^c \cup \einf^c) \\
&=\Pr(\ehalf^c) + \Pr(\ehalf\cap \einf^c) \\
&=\Pr(\ehalf^c) + \Pr(\einf^c \vert \ehalf) \\
&\leq \delta' + \textsum_{i=1}^{{{n/2}}} \delta_i,
\end{align}
where the last step follows from \cref{eq:ehalf,lemma:direct_cover_bound}. The claim now follows.

\subsubsection{Proof for case (II): When $\rminnew[\inputcoreset,\gkernel, n] = \rmin_{\inputcoreset}$}
\label{ssub:when_rmin_leq_n}
In this case, we split the proof for bounding $\infnorm{\outvec[n/2]}$ into two lemmas. 
First, we relate the $\infnorm{\outvec[n/2]}$ in terms of the tail behavior of $\gkernel$ and the supremum of differences for $\outvec[n/2]$ between any pair points on a Euclidean ball (see \cref{sub:proof_of_lemma_output_basic_bound} for the proof): 
\begin{lemma}[\outputbasicboundresultname]
    \label{lemma:output_basic_bound}
    If $\gkernel$ satisfies \cref{asmp:bounded_measurable}, 
    then, conditional on the event $\ehalf$~\cref{outvec_as_kernel_diff},
    for any fixed $R=R'+\rminpn[\inputcoreset]$ with $R'>0$ and any fixed $\delta' \in (0, 1)$, 
    \begin{align}
    \label{eq:infnorm_basic_bound}
        \einftwo\! = \!\bigg\{\sinfnorm{\outvec[n/2]}
        \!\leq\!
        \max\!\big(
        n \ktail[\gkernel](R'),
        \sigma_{n/2}\sinfnorm{\gkernel}^{\frac12} \sqrt{2\log(\textfrac{4}{\delta'})}
        \!+\!\!\!\!
        \sup_{\x,\x'\in\ball(0, R)}\!\!\!\!\!|{\outvec[n/2](\x)\!-\!\outvec[n/2](\x')}|
        \big)\bigg\}
    \end{align}
    occurs with probability at least $1\!-\!\frac{\delta'}{2}$, i.e., $\Pr(\einftwo\vert\ehalf) \!\geq\! 1\!-\!\frac{\delta'}{2}$, where $\ktail[\gkernel]$ was defined in \cref{eq:ktailbar}.
\end{lemma}
Next, to control the supremum term on the RHS of the display~\cref{eq:infnorm_basic_bound}, we establish a high probability bound in the next lemma. Its proof in \cref{sub:proof_of_lemma_orlicz_bound} proceeds by showing that $\outvec[n/2]$ is an Orlicz process with a suitable metric and then applying standard concentration arguments for such processes.
\begin{lemma}[\orliczboundresultname]
\label{lemma:orlicz_bound}
If $\gkernel$ satisfies \cref{asmp:bounded_measurable,assum:lipkernel}, then, for any fixed $R>0, \delta' \in (0, 1)$, the event
\begin{align}
 \label{eq:orlicz_hpbound_lemma}
 \esup \defeq \bigg\{\sup_{\x, \x' \in \ball(0, R)} |{\outvec[n/2](\x)\!-\!\outvec[n/2](\x') }| \!\leq\! 
 8\diam_R \parenth{\sqrt{\log(\textfrac{4}{\delta'})}
 + 6\sqrt{d\log\big(2+\textfrac{\klip[\gkernel]R}{\infnorm{\gkernel}}\big)} }
  \bigg\}
 \end{align}
 occurs with probability at least 
 $1-\frac{\delta'}{2}$,
 where
 $\diam_R \defeq \sqrt{\frac{32}{3}} \sigma_{n/2}\infnorm{\gkernel}^{\frac12} \min\big(1,  \sqrt{\frac12\klip[\gkernel] R} \big)$.
\end{lemma}

We now turn to the rest of the proof for \cref{kernel_halving_results}. We apply both \cref{lemma:output_basic_bound,lemma:orlicz_bound} with $R = \rminpn[\inputcoreset]+\rmin_{\gkernel, n}$~\cref{eq:rmin_P,eq:rmin_k}.
For this $R$, we have $R'=\rmin_{\gkernel, n}$ in \cref{lemma:output_basic_bound} and hence $n\ktail[\gkernel](R') \leq \sinfnorm{\gkernel}$.
Now, condition on the event $\ehalf \cap \einftwo \cap \esup$. Then, we have
\begin{talign}
	&n\sinfnorm{\P_n\gkernel-{\Q}_{\trm{KH}}^{(1)}\gkernel}\\
	&\seq{\cref{outvec_as_kernel_diff}}
	\infnorm{\outvec[n/2]} \\
	&\sless{\cref{eq:infnorm_basic_bound}} \max(\infnorm{\gkernel},\  \sigma_{n/2}\sinfnorm{\gkernel}^{\frac12}  \sqrt{2\log(\frac{4}{\delta'})} +
	\sup_{\x,\x'\in\ball(0, R)}|{\outvec[n/2](\x)\!-\!\outvec[n/2](\x')}|)\\
	&\sless{\cref{eq:orlicz_hpbound_lemma}} \max\Big(\infnorm{\gkernel},\ \sigma_{n/2}\sinfnorm{\gkernel}^{\frac12}  \sqrt{2\log(\frac{4}{\delta'})}
	+ \\
	&\hspace{4\baselineskip}32\sqrt{\frac23} \sigma_{n/2}\sinfnorm{\gkernel}^{\frac12} \parenth{
	\sqrt{\log(\frac{4}{\delta'})} 
	+ 6 \sqrt{d \log(2+\frac{\klip[\gkernel](\rminpn[\inputcoreset]+\rmin_{\gkernel, n})}{\infnorm{\gkernel}})} }\Big)\\
	&\sless{(i)} \max(\infnorm{\gkernel},\ 32\sigma_{n/2}\sinfnorm{\gkernel}^{\frac12}  \parenth{ \sqrt{\log(\frac{4}{\delta'})}
	+ 5\sqrt{d\log \parenth{2+\frac{\klip[\gkernel](\rminpn[\inputcoreset]+\rmin_{\gkernel, n})}{\infnorm{\gkernel}} } }})\label{eq:kernel_halving_proof_mbound_first}
	\\
	&\sless{\cref{eq:sigma_bound}} \sinfnorm{\gkernel}  \cdot 64 \sqrt{\log\parenth{\frac{2}{\delta^\star}}} \brackets{\sqrt{\log(\frac{4}{\delta'})}
	+ 5\sqrt{d\log \parenth{2+\frac{\klip[\gkernel](\rminpn[\inputcoreset]+\rmin_{\gkernel, n})}{\infnorm{\gkernel}} } }}  \\
	&\sless{(ii)}\infnorm{\gkernel} 2\,\err_{\gkernel}(n, 1, d, \delta^\star, \delta', \rminnew[\inputcoreset,\gkernel, n]),
	\label{eq:kernel_halving_proof_mbound}
\end{talign}
where step~(i) follows from the fact that $\diam_R \leq \sqrt{\frac{32}
{3}} \sigma_{n/2}\infnorm{\gkernel}^{\frac12}$, and in step~(ii) we have
used the working assumption for this case, i.e., $\rminnew[\inputcoreset,\gkernel, n]=\rminpn[\inputcoreset]$.
As a result,
\begin{talign}
\Pr(\sinfnorm{\P_n\gkernel\!-\!{\Q}_{\trm{KH}}^{(1)}\gkernel}\! >\! \frac{2}{n}\err_{\gkernel}(n,\!1,\! d,\! \delta^\star,\! \delta',\!\rminnew[\inputcoreset,\gkernel, n])\! &\leq \Pr((\ehalf \cap \einftwo \cap \esup)^c) \\
&= \Pr(\ehalf^c \cup \einftwo^c \cup \esup^c) \\
&= \Pr(\ehalf^c \cup \esup^c) + \Pr((\ehalf^c \cup \esup^c)^c \cap \einftwo^c) \\
&= \Pr(\ehalf^c \cup \esup^c) + \Pr(\ehalf \cap \esup \cap \einftwo^c) \\
&\leq \Pr(\ehalf^c \cup \esup^c) + \Pr(\ehalf \cap \einftwo^c) \\
&\leq \Pr(\ehalf^c) + \Pr(\esup^c) + \Pr(\einftwo^c\vert \ehalf) \\ 
&\leq \textsum_{i=1}^{n/2} \delta_i + \frac{\delta'}{2} + \frac{\delta'}{2},
\end{talign}
where the last step follows from \cref{eq:ehalf,lemma:output_basic_bound,lemma:orlicz_bound}. The desired claim follows.

\subsection{Proof of part~\lowercase{\Cref{item:kernel_halving_multiple_rounds}}: \kernelhalvingmultiroundname}
\label{sub:proof_multiple_rounds_kernel_halving}

As noted earlier, we prove this part assuming $n/2^m \in \N$. The proof in this section follows by applying the arguments from the previous section, separately for each round and then invoking the sub-Gaussianity of a weighted sum of the output functions from each round. 

Note that coreset $\coreset[j-1]$ is independent of the randomness for the $j$-th round of kernel havling and thus can be treated as fixed for that round. When running kernel halving with the input $\coreset[j-1]$, let $\invec[i,j], \cnew[i, j], \outvec[i, j], \wvprod[i, j], \eta_{i, j}$ denote the analog of the quantities $\invec[i], \cnew[i], \outvec[i], \wvprod[i], \eta_{i}$ defined in \cref{algo:kernel_halving}. Like in part~\cref{item:kernel_halving_one_round}, we would couple the $j$-th round of kernel halving with an instantiation of self-balancing Hilbert walk with inputs $(\invec[i, j],\cnew[i, j])_{i=1}^{n/2^{j}}$ and final output $\outvec[n/2^{j}, j]$, where $\cnew[i, \j] = \max(\norm{\invec[i, j]}_{\gkernel}\sgparam[i-1,j] \sqrt{2\log(\frac{4m}{2^j\delta_{i}})}, \norm{\invec[i, j]}_{\gkernel}^2)$ and we define $\sgparam[i, j]$ in a recursive manner as in \cref{eq:subgauss_const_def} with $\invec[i,j], \cnew[i, j], \sgparam[i, j]$ taking the role of $\invec[i], \cnew[i], \sgparam[i]$ respectively.

With this set-up, first we apply property~\ref{item:fun_subgauss} of \cref{sbhw_properties} which implies that given $\coreset[j-1]$, the function $\outvec[n/2^{j}, j]$ is $\sgparam[n/2^j, j]$ sub-Gaussian, where
\begin{talign}
    	\label{eq:sigmaj_bound}
    \sgparam[n/2^j, j]^2 \leq 4\infnorm{\gkernel} \log(\frac{4m}{2^j\delta^\star_j })
    \qtext{with} \delta^\star_j = \min(\delta_i)_{i=1}^{n/2^j},
\end{talign}
using an argument similar to \cref{eq:sigma_bound} with  property~\ref{item:adaptive_thresh} of \cref{sbhw_properties}.
Next, we note that for $j$-th round, the event $\ehalfj = \abss{\wvprod[i, j]} \leq \cnew[i, j]$ for all $i \in [n/2^{j}]$, satisfies
\begin{talign}
\label{eq:ehalfj}
    \P(\ehalfj)\geq 1-\sum_{i=1}^{{n/2^j}} \delta_i \frac{2^{j-1}}{m};
\end{talign}
and this event serves as the analog the equivalence event $\ehalf$ from part~\cref{item:kernel_halving_one_round} for the $j$-th round of kernel halving. Thus on the event $\ehalfj$, we can write that the kernel halving coreset $\coreset[j] = (x_i)_{i\in\indices_j}$ for $\indices_j = \{2i-\frac{\eta_{i,j}-1}{2} : i\in [n/2^j]\}$ and $\eta_{i, j}$ as defined while running \cref{algo:self_balancing_walk}, so that
\begin{talign}
\label{outvec_as_kernel_diff_multi}
\frac{2^{j-1}}{n}\outvec[n/2^j, j] 
    = \frac{1}{n/2^{j-1}}\sum_{\x\in\coreset[j-1]} \gkernel(\x,\cdot)
    - \frac{1}{{n/2^{j}}}\sum_{x\in\coreset[j]}\gkernel(x,\cdot).
\end{talign}
Now conditional on the event $\cap_{j=1}^{\m}\ehalfj$, we conclude that the output of all $m$ kernel halving rounds are equivalent to the output of $m$ different self-balancing Hilbert walks (each with exact two-thinning in each round). Putting the pieces together, we conclude that on the event $\cap_{j=1}^{\m}\ehalfj$, we have
\begin{talign}
    \P_n\gkernel-\Q_{\mrm{KH}}^{(\m)}\gkernel  &= \frac{1}{n}\sum_{\x\in\inputcoreset} \gkernel(\x,\cdot)
    - \frac{1}{{n/2^{\m}}}\sum_{x\in\coreset[\m]}\gkernel(x,\cdot) \\
    &\seq{(i)}\sum_{j=1}^{\m} \frac{1}{n/2^{j-1}}\sum_{\x\in\coreset[j-1]} \gkernel(\x,\cdot)
    - \frac{1}{{n/2^{j}}}\sum_{x\in\coreset[j]}\gkernel(x,\cdot) \\
    &= \sum_{j=1}^{\m} \frac{2^{j-1}}{n} \outvec[n/2^j, j].
    \label{eq:diff_as_sum_of_psi}
\end{talign}
where we abuse notation $\coreset[0]\defeq\inputcoreset$ in step~(i) for simplicity of expressions.

Next, we use the following basic fact: 
\begin{lemma}[Sub-Gaussian additivity]
\label{lem:sub_gauss_additivity}
For a sequence of random variables $(Z_j)_{j=1}^{\m}$ such that $Z_j$ is a  $\sgparam[j]$ sub-Gaussian variable conditional on $(Z_1, \ldots, Z_{j-1})$, the random variable $\mc{Z} = \sum_{j=1}^{\m} \theta_j Z_j $ is $(\sum_{j=1}^{\m}\theta_j^2 \sgparam[j]^2)^{1/2}$ sub-Gaussian. 
\end{lemma}
\begin{proof}
We will prove the result for $\mc{Z}_s = \sum_{j=1}^{s} \theta_j Z_j $ by induction on $s \leq \m$.
The result holds for the base case of $s = 0$ as $\mc{Z}_s = 0$ is $0$ sub-Gaussian.
For the inductive case, suppose the result holds for $s$.
Then we may apply the tower property, our conditional sub-Gaussianity assumption, and our inductive hypothesis in turn to conclude 
\begin{talign}
\E[e^{\sum_{j=1}^{s+1} \theta_j Z_j}]
    = \E[e^{\sum_{j=1}^{s+1} \theta_j Z_j} \E[e^{\theta_{s+1} Z_{s+1}}\mid Z_{1:s}]]
    \leq \E[e^{\sum_{j=1}^{s} \theta_j Z_j}] e^{\frac{\theta_{s+1}^2 \sigma_{s+1}^2}{2}}
    = e^{\frac{\sum_{j=1}^{s+1}\theta_j^2 \sgparam[j]^2}{2}}.
\end{talign}
Hence, $\mc{Z}_{s+1}$ is $\sqrt{\sum_{j=1}^{s+1}\theta_j^2 \sgparam[j]^2}$ sub-Gaussian, and the proof is complete.
\end{proof}
Applying \cref{lem:sub_gauss_additivity} to the output sequence $(\outvec[n/2^j, j])_{j=1}^{\m}$ for the $m$ rounds of self-balancing Hilbert walks, we conclude that,
the random variable
\begin{talign}
\label{eq:sumpsi_def}
    \sumpsi \defeq \sum_{j=1}^{\m} \frac{2^{j-1}}{n} \outvec[n/2^j, j]
\end{talign}
is sub-Gaussian with parameter 
\begin{talign}
\label{eq:sgparam_bigpsi}
    \sgparam[\sumpsi] \!\defeq\! \frac{2}{\sqrt 3}
     \frac{2^m}{n}\sqrt{ \sinfnorm{\gkernel} \log(\frac{6m}{2^m\delta^\star})}
     \!\sgrt{(i)}\!
     \sqrt{\sinfnorm{\gkernel}\sum_{j=1}^{m} \frac{4^j}{n^2} \log(\frac{4m}{2^{j}\delta^\star})}
     \!\sgrt{(ii)}\! \sqrt{\sum_{j=1}^{\m} (\frac{2^{j-1}}{n})^2 \sgparam[n/2^j, j]^2 },
\end{talign}
conditional to the input $\inputcoreset$, 
where step~(ii) follows since
\begin{talign}
    \sinfnorm{\gkernel} \sum_{j=1}^{m} 4^j \log(\frac{4m}{2^{j}\delta^\star})
    \sgrt{(iii)} 4\sinfnorm{\gkernel} \sum_{j=1}^{m} 4^{j-1} \log(\frac{4m}{2^{j}\delta^\star})
    \sgrt{\cref{eq:sigmaj_bound}}
    \sum_{j=1}^{\m} 4^{j-1} \sgparam[n/2^j, j]^2,
\end{talign}
and step~(iii) follows from the fact that $\delta^\star = \min(\delta^\star_j)_{j=1}^{\m}$~\cref{eq:sigmaj_bound}. Now to prove step~(i), we note that 
\begin{talign}
\label{eq:g1_g2}
    \mfk{S}_1 \defeq \sum_{j=1}^m 4^j =\frac{4}{3} (4^{m}-1) \leq \frac{4}{3} 4^{m},
    \qtext{and}
    \mfk{S}_2 \defeq 
    \sum_{j=1}^m j 4^j
    =\frac{4m}{3} \cdot 4^{m}- \frac{\mfk{S}_1}{3},
\end{talign}
 which in turn implies that
\begin{talign}
    \sum_{j=1}^{m} 4^j \log(\frac{4m}{2^{j}\delta^\star})
    \!=\! \mfk{S}_1 \log(\frac{4m}{\delta^\star})
    \!-\! \log2 \cdot \mfk{S}_2
    &= \mfk{S}_1 \parenth{\log(\frac{4m}{\delta^\star})
    \!+\! \frac{\log 2}{3}}
    \!-\! \frac{4m\log2}{3}\cdot 4^{m}\\
    &\leq \mfk{S}_1 \cdot \log(\frac{6m}{\delta^\star})
    \!-\! \frac{4m\log2}{3}\cdot 4^{m}\\
    &\leq \frac{4}{3} 4^{m} \cdot  \log(\frac{6m}{\delta^\star})
    \!-\! \frac{4}{3}4^{m} \cdot \log2^{m}
    \!=\! \frac{4}{3} 4^{m} \cdot \log(\frac{6m}{2^{m}\delta^\star}),
\end{talign}
thereby establishing step~(i).

Next, analogous to the proof of part~\cref{item:kernel_halving_one_round}, we split the proof in two parts: \textbf{Case (I)} When $\rminnew[\inputcoreset,\gkernel, n]<\rminpn$, in which case, we proceed with a direct covering argument to bound $\sinfnorm{\sumpsi}$ using a lemma analogous to \cref{lemma:direct_cover_bound}; and \textbf{Case (II)} when $\rminnew[\inputcoreset,\gkernel, n]=\rminpn$, in which case, we proceed with a metric-entropy based argument to bound $\sinfnorm{\sumpsi}$ using two lemmas that are analogous \cref{lemma:output_basic_bound,lemma:orlicz_bound}.

\newcommand{\multirounddirectcoverboundresultname}{A direct covering bound on $\sinfnorm{\sumpsi}$}
\subsubsection{Proof for case (I): When $\rminnew[\inputcoreset,\gkernel, n]<\rminpn$}
Recall that for this case, $\rminnew[\inputcoreset,\gkernel, n]$ is given by \cref{eq:rmin_new_case_a}. The next lemma, a straightforward extension of \cref{lemma:direct_cover_bound}, provides a high probability control on $\sinfnorm{\sumpsi}$. Its proof is omitted for brevity.
\begin{lemma}[\multirounddirectcoverboundresultname]
    \label{lemma:multi_round_direct_cover_bound}
    Fix $R \geq r > 0$ and $\delta'>0$, and recall the definition~\cref{eq:sample_cover_set} of $\pcover(r, R)$.
    If $\gkernel$ satisfies \cref{asmp:bounded_measurable,assum:lipkernel}, then, conditional on the event $\cap_{j=1}^{\m}\ehalfj$~\cref{outvec_as_kernel_diff_multi}, the event 
    \begin{talign}
    \label{eq:multi_round_einf_direct_cover}
        \einfj
        \defeq
        \braces{\sinfnorm{\sumpsi}
        \leq
        \max\parenth{
        2^{\m} \ktail[\gkernel](R), \ \ 
        2^{\m} \klip[\gkernel]r +
        \sgparam[\sumpsi]\sqrt{2\infnorm{\gkernel} \log(2|\pcover(r, R)|/\delta')}
        }},
    \end{talign}
    occurs with probability at least $1\!-\!\delta'$, i.e., $\Pr(\einfj \vert \cap_{j=1}^{\m}\ehalfj) \geq 1\!-\!\delta'$.
\end{lemma}
Next, we repeat arguments similar to those used earlier around the display~\cref{eq:direct_cover_argument_1,eq:direct_cover_argument_2}.
Fix $R$ and $r$ such that $\ktail[\gkernel](R)=\infnorm{\gkernel}/n$ and $ \klip[\gkernel]r=\sinfnorm{\gkernel}/n$, so that $\frac{R}{r}\gets n\rmin_{\gkernel, n}\frac{\klip[\gkernel]}{\sinfnorm{\gkernel}}$ (c.f. \cref{eq:ktailbar,eq:rmin_k}). Substituting these choices of radii in the bound~\cref{eq:multi_round_einf_direct_cover} of \cref{lemma:multi_round_direct_cover_bound}, we find that conditional to $\cap_{j=1}^{\m}\ehalfj[j]\cap \einfj$, we have
\begin{talign}
\sinfnorm{\P_n\gkernel-\Q_{\mrm{KH}}^{(\m)}\gkernel} &\seq{\cref{eq:diff_as_sum_of_psi}}
\sinfnorm{\sum_{j=1}^{\m} \frac{2^{j-1}}{n} \outvec[n/2^j, j]} \\
&\seq{\cref{eq:sumpsi_def}}\sinfnorm{\sumpsi} \\
&\sless{\cref{eq:multi_round_einf_direct_cover}}
 \max\parenth{
        2^\m \ktail[\gkernel](R), \ \ 
        2^\m \klip[\gkernel]r +
        \sgparam[\sumpsi]\sqrt{2\infnorm{\gkernel} \log(2|\pcover(r, R)|/\delta')}} \\
    &\sless{(i)} \frac{2^{\m}}{n}\sinfnorm{\gkernel} (1+ 
    2\sqrt{\frac{2}{3}} \sqrt{ \log(\frac{6m}{2^m\delta^\star}) \brackets{\log(\frac{2}{\delta'})
    + d \log\parenth{n^{\frac1d}+\frac{2\klip[\gkernel]}{\sinfnorm{\gkernel}} \cdot n^{1+\frac1d}\rmin_{\gkernel, n} } } } \\
    &\sless{\cref{eq:rmin_new_case_a}} \frac{2^{\m}}{n}\sinfnorm{\gkernel}
    (1+ 2\sqrt{\frac{2}{3}} \sqrt{ \log(\frac{6m}{2^m\delta^\star}) \brackets{\log(\frac{2}{\delta'})
    + d \log\parenth{\frac{2\klip[\gkernel]}{\sinfnorm{\gkernel}} (\rmin_{\gkernel, n} +  \rminnew[\inputcoreset,\gkernel, n] ) } } }) \\
    &\sless{\cref{eq:err_simple_defn}} \sinfnorm{\gkernel} \cdot \frac{2^\m}{n} \err_{\gkernel}(n, m, d, \delta^\star, \delta', \rminnew[\inputcoreset,\gkernel, n]),
\end{talign}
where in step~(i), we have used the bounds~\cref{eq:sgparam_bigpsi,eq:cover_size}.
Putting the pieces together, we conclude
\begin{align}
&\Pr\big(\sinfnorm{\P_n\gkernel-{\Q}_{\mrm{KH}}^{(m)}\gkernel} > \sinfnorm{\gkernel}\cdot \frac{2^\m}n \err_{\gkernel}(n, m, d,\delta^\star, \delta', \rminnew[\inputcoreset,\gkernel, n])\big) \\
\qquad&\leq \Pr((\cap_{j=1}^{\m}\ehalfj[j] \cap \einfj)^c) \\ 
&= \Pr(\cup_{j=1}^{\m}(\ehalfj[j])^c \cup (\einfj)^c) \\ 
&= \Pr(\cup_{j=1}^{\m}(\ehalfj[j])^c) + \Pr( \cap_{j=1}^{\m}\ehalfj[j] \cap (\einfj)^c) \\ 
&\leq \sum_{j=1}^m\Pr((\ehalfj[j])^c) + \Pr( (\einfj)^c \vert \cap_{j=1}^{\m}(\ehalfj[j])) \\ 
&\leq \textsum_{j=1}^{\m}\textsum_{i=1}^{{n/2^j}} \delta_i \cdot \textfrac{2^{j-1}}{m} + \delta',
\end{align}
where the last step follows from \cref{eq:ehalfj,lemma:multi_round_direct_cover_bound}.
The claim follows.

\subsubsection{Proof for case (II): When $\rminnew[\inputcoreset,\gkernel, n]=\rminpn$}
In this case, we makes use of two lemmas. Their proofs (omitted for brevity) can be derived essentially by replacing $\outvec[n/2]$ and $\sgparam[n/2]$ with $\sumpsi$~\cref{eq:sumpsi_def} and $\sgparam[\sumpsi]$~\cref{eq:sgparam_bigpsi} respectively, and repeating the proof arguments from \cref{lemma:output_basic_bound,lemma:orlicz_bound}.
\begin{lemma}[\multioutputbasicboundresultname]
    \label{lemma:multi_round_output_basic_bound}
    If $\gkernel$ satisfies \cref{asmp:bounded_measurable}, then, conditional on the event $\cap_{j=1}^{\m}\ehalfj$~\cref{outvec_as_kernel_diff_multi},
    for any fixed $R=R'+\rminpn[\inputcoreset]$ with $R'>0$ and any fixed $\delta' \in (0, 1)$, 
    \begin{align}
    \label{eq:multi_round_infnorm_basic_bound}
        \einfjtwo\! = \!\bigg\{\sinfnorm{\sumpsi}
        \!\leq\!
        \max\bigg(
        2^{\m} \ktail[\gkernel](R'), \
        \sgparam[\sumpsi]\sinfnorm{\gkernel}^{\frac12} \sqrt{2\log(\frac{4}{\delta'})}
        \!+\!
        \sup_{\x,\x'\in\ball(0, R)}|{\sumpsi(\x)\!-\!\sumpsi(\x')}|
        \bigg)\bigg\},
    \end{align}
    occurs with probability at least $1\!-\!\frac{\delta'}{2}$, i.e., $\Pr(\einfjtwo \vert \cap_{j=1}^{\m}\ehalfj) \geq 1\!-\!\frac{\delta'}{2}$.
\end{lemma}
\begin{lemma}[\multiorliczboundresultname]
\label{lemma:multi_orlicz_bound}
If $\gkernel$ satisfies \cref{asmp:bounded_measurable,assum:lipkernel}, then, for any fixed $R>0, \delta' \in (0, 1)$, given $\inputcoreset$, the event
\begin{align}
 \label{eq:multi_orlicz_hpbound_lemma}
 \esupj \defeq \bigg\{\sup_{\x, \x' \in \ball(0, R)} |{\sumpsi(\x)\!-\!\sumpsi(\x') }| \!\leq\! 
 8\diam_R \parenth{\sqrt{\log(\textfrac{4}{\delta'})}
 + 6\sqrt{d\log\big(2+\textfrac{\klip[\gkernel]R}{\infnorm{\gkernel}}\big)} }
  \bigg\}
 \end{align}
 occurs with probability at least $1-\delta'/2$, where $\diam_R \defeq \sqrt{\frac{32}{3}} \sgparam[\sumpsi]\infnorm{\gkernel}^{\frac12} \min\big(1,  \sqrt{\frac12\klip[\gkernel] R} \big)$.
\end{lemma}

Mimicking the arguments like those in display~\cref{eq:kernel_halving_proof_mbound_first,eq:kernel_halving_proof_mbound}, with $R = \rminpn[\inputcoreset]+\rmin_{\gkernel, n}$, and $R'=\rmin_{\gkernel, n}$, we find that conditional on the event $\einfjtwo\cap\esupj\cap_{j=1}^{\m}\ehalfj[j]$,
\begin{talign}
\sinfnorm{\P_n\gkernel-\Q_{\mrm{KH}}^{(\m)}\gkernel} &\seq{\cref{eq:diff_as_sum_of_psi}}
\sinfnorm{\sum_{j=1}^{\m} \frac{2^{j-1}}{n} \outvec[n/2^j, j]} \\
&\seq{\cref{eq:sumpsi_def}}\sinfnorm{\sumpsi} \\
   &	\leq \max(\frac{2^{\m}}{n}\infnorm{\gkernel},\ 32\sgparam[\sumpsi]\sinfnorm{\gkernel}^{\frac12}  \parenth{ \sqrt{\log(\frac{4}{\delta'})}
	+ 5\sqrt{d\log \parenth{2+\frac{\klip[\gkernel](\rminpn[\inputcoreset]+\rmin_{\gkernel, n})}{\infnorm{\gkernel}} } }})
	\\
	&\sless{\cref{eq:sgparam_bigpsi}} \sinfnorm{\gkernel}  \cdot  \frac{2^\m}{n}  \cdot 37 \sqrt{\log\parenth{\frac{6m}{2^m\delta^\star}}} \brackets{\sqrt{\log(\frac{4}{\delta'})}
	+ 5\sqrt{d\log \parenth{2+\frac{\klip[\gkernel](\rminpn[\inputcoreset]+\rmin_{\gkernel, n})}{\infnorm{\gkernel}} } }} 
	\label{eq:err_actual_bound}\\
	&\seq{\cref{eq:err_simple_defn}}\sinfnorm{\gkernel} \cdot \frac{2^\m}{n} \err_{\gkernel}(n, m, d, \delta^\star, \delta', \rminpn[\inputcoreset]),
\end{talign}
which does not happed with probability at most
\begin{talign}
    \P((\cap_{j=1}^{\m}\ehalfj[\m] \cap \einfjtwo\cap\esupj)^c)
    & \leq  \sum_{j=1}^m\Pr((\ehalfj[j])^c)
    + \Pr((\esupj)^c)
    + \Pr( (\einfjtwo)^c \vert \cap_{j=1}^{\m}(\ehalfj[j])) \\
    &\leq\sum_{j=1}^{\m}\sum_{i=1}^{\floor{n/2^j}} \delta_i \frac{2^{j-1}}{m}
    + \frac{\delta'}{2} + \frac{\delta'}{2},
\end{talign}
where the last step follows from \cref{eq:ehalfj,lemma:multi_round_output_basic_bound,lemma:multi_orlicz_bound} as claimed. The proof is now complete.

\subsection{Proof of \lowercase{\Cref{lemma:direct_cover_bound}}: \directcoverboundresultname} %
\label{ssub:proof_of_lemma:direct_cover_bound}
We claim that conditional on the event $\ehalf$~\cref{outvec_as_kernel_diff},
we deterministically have
    \begin{talign}
    \label{eq:psi_deterministic_bound}
        \sinfnorm{\outvec[n/2]}
        \leq
        \max\braces{
        n \ktail[\gkernel](R), \ \ 
        n \klip[\gkernel]r +
        \max_{\z\in\pcover(r, R)}|\outvec[n/2](z)|
        },
    \end{talign}
and the event
\begin{talign}
       \braces{
        \max_{\z\in\pcover(r, R)}\abss{\outvec[n/2](\z)}
        \leq 
        \sgparam[n/2]\sqrt{2\infnorm{\gkernel} \log(2|\pcover(r, R)|/\delta')}
        }
        \label{eq:w_max_cover_bound}
\end{talign}
occurs with probability at least $1-\delta'$.
Putting these two claims together yields the lemma.
We now prove these two claims separately.

\subsubsection{Proof of \cref{eq:psi_deterministic_bound}}
Note that on the event $\ehalf$~\cref{outvec_as_kernel_diff} we have
\begin{talign}
\label{eq:psi_rep_ehalf}
\outvec[n/2] 
    = n(\P_n\gkernel-\wtil{\Q}_{n/2}\gkernel)
    = \sum_{i=1}^n \eta_i \gkernel(\x_i,\cdot)
\end{talign}
for some $\eta_i \in \{-1,1\}$.
Now fix any $\x\in\reals^d$, and introduce  the shorthand $\pcover = \pcover(r, R)$.
The result follows by considering two cases.

\paragraph*{(Case 1) $\mbi{x\not\in \bigcup_{i=1}^n \ball(\x_i, R)}$} 
In this case, we have $\twonorm{\x-\axi}
\geq R$ for all $i \in [n]$ and therefore, representation~\cref{eq:psi_rep_ehalf} yields that
\begin{talign}
    \abss{\outvec[n/2](\x)}
    = \abss{\sum_{i=1}^n\eta_i\gkernel(\axi, \x)}
    \leq \sum_{i=1}^n \abss{\eta_i}\abss{\gkernel(\axi, \x)}
    = \sum_{i=1}^n\abss{\gkernel(\axi, \x)}
    \leq n\ktail[\gkernel](R),
\end{talign}
by Cauchy-Schwarz's inequality and the definition~\cref{eq:ktailbar} of $\ktail[\gkernel]$.

\paragraph*{(Case 2) $\mbi{x\in \bigcup_{i=1}^n \ball(\x_i, R)}$}
By the definition \cref{eq:sample_cover_set} of our cover $\pcover$, there exists $z \in \pcover$ such that $\twonorm{\x-\z}\leq r$.
Therefore, on the event $\ehalf$, using representation~\cref{eq:psi_rep_ehalf}, we find that
\begin{talign}
    \abss{\outvec[n/2](\x)}
    \leq \abss{\sum_{i=1}^n\eta_i\gkernel(\axi, \x)}
    &= \abss{\sum_{i=1}^n\eta_i\parenth{\gkernel(\axi, \x)-\gkernel(\axi, \z)+\gkernel(\axi, \z)}} \\
    &\leq \abss{\sum_{i=1}^n\eta_i\parenth{\gkernel(\axi, \x)-\gkernel
    (\axi, \z)}}+\abss{\sum_{i=1}^n\eta_i\gkernel(\axi, \z)}
    \\
    &\leq \sum_{i=1}^n\abss{\gkernel(\axi, \x)-\gkernel(\axi, \z)}+
    \abss{\outvec[n/2](\z)} \\
    &\leq n\klip[\gkernel](r) +\sup_{\z'\in\pcover}\abss{\outvec[n/2](\z')}
\end{talign}
by Cauchy-Schwarz's inequality.

\subsubsection{Proof of \cref{eq:w_max_cover_bound}}
Introduce the shorthand $\pcover = \pcover(r, R)$. Then applying the union bound, the pointwise sub-Gaussian property~\ref{item:pointwise_subgauss} of \cref{sbhw_properties}, and the sub-Gaussian Hoeffding inequality \citep[Prop.~2.5]{wainwright2019high}, we find that
\begin{talign}
\Pr(\max_{\z\in\pcover}
|\outvec[n/2](z)| > t \mid)
    &\leq 
\sum_{\z \in \pcover} 2\exp(-t^2/(2\sgparam[n/2]^2\gkernel(\z,\z))) \\
    &\leq
2|\pcover|\exp(-t^2/(2\sgparam[n/2]^2\infnorm{\gkernel})) 
    = \delta' \\
    \qtext{for} 
t &\defeq \sgparam[n/2]\sqrt{2\infnorm{\gkernel} \log(2|\pcover|/\delta')},
\end{talign}
as claimed.

\subsection{Proof of \lowercase{\Cref{lemma:output_basic_bound}}: \outputbasicboundresultname}
\label{sub:proof_of_lemma_output_basic_bound}
For any $R>0$, we deterministically have
\begin{talign}
\label{eq:maxima_decomposition}
    \infnorm{\outvec[n/2]} \leq \max(\sup_{\x\in\ball(0, R)}\abss{\outvec[n/2](\x)}
    , \sup_{\x\in\ball^c(0, R)}\abss{\outvec[n/2](\x)})
\end{talign}
Since $R = R'+ \rminpn[\inputcoreset]$, for any $\x \in \ball^c(0, R)$, we have $\twonorm{\x-\axi}\geq R'$ for all $i \in [n]$. Thus conditional on the event $\ehalf$, applying property~\ref{item:signed_sum} from \cref{sbhw_properties}, we find that
\begin{talign}
    \abss{\outvec[n/2](\x)}
    = \abss{\sum_{i=1}^n\eta_i\gkernel(\axi, \x)}
    \leq \sum_{i=1}^n \abss{\eta_i}\abss{\gkernel(\axi, \x)}
    = \sum_{i=1}^n\abss{\gkernel(\axi, \x)}
    \leq n\ktail[\gkernel](R'),
\end{talign}
by Cauchy-Schwarz's inequality and the definition~\cref{eq:ktailbar} of $\ktail[\gkernel]$.

For the first term on the RHS of \cref{eq:maxima_decomposition}, we have
\begin{talign}
\sup_{\x\in\ball(0, R)}\abss{\outvec[n/2](\x)}&\leq \abss{\outvec[n/2](0)} + \sup_{\x\in\ball(0, R)}\abss{\outvec[n/2](\x)-\outvec[n/2](0)} \\
    &\leq \abss{\outvec[n/2](0)} +  \sup_{\x,\x'\in\ball(0, R)}\abss{\outvec[n/2](\x)-\outvec[n/2](\x')}.
\end{talign}
Now, the sub-Gaussianity of $\outvec[n/2]$ (property \ref{item:pointwise_subgauss} of \cref{sbhw_properties}) with $x=0$, and the sub-Gaussian Hoeffding inequality \citep[Prop.~2.5]{wainwright2019high} imply that
\begin{align}
    \abss{\outvec[n/2](0)} \leq \sigma_{n/2} \sqrt{2\gkernel(0, 0)\log(4/\delta)}
    \leq \sigma_{n/2} \sqrt{2\infnorm{\gkernel}\log(4/\delta')},
\end{align}
with probability at least $1-\delta'/2$. Putting the pieces together completes the proof.

\subsection{Proof of \lowercase{\Cref{lemma:orlicz_bound}}: \orliczboundresultname}
\label{sub:proof_of_lemma_orlicz_bound}

The proof proceeds by using concentration arguments for Orlicz processes \citep[Def.~5.5]{wainwright2019high}. Given a set $\mbb{T} \subseteq \Rd$, a random process $\braces{Z_{\x}, \x \in \mbb{T}}$ is called an Orlicz $\orlicz$-process with respect to the metric $\ometric$ if 
\begin{talign}
\label{eq:subgaussian_process_defn}
\norm{Z_{\x}-Z_{\x'}}_{\orlicz} \leq \ometric(\x, \x')
\qtext{for all} \x,\x'\in\mbb{T},
\end{talign}
where for any random variable $Z$, its Orlicz $\orlicz$-norm is defined as
 \begin{talign}
\label{eq:orlicz_norm}
\norm{Z}_{\orlicz} = \inf\braces{\lambda>0: \Exs[\exp(Z^2/\lambda^2)] \leq 2}.
\end{talign}
Our next result (see \cref{sub:proof_of_lemma_orlicz_process} for the proof) establishes that $\outvec[n/2]$ is an Orlicz process with respect to a suitable metric. (For clarity, we use $\ball_2$ to denote the Euclidean ball.)
\begin{lemma}[\orliczprocessresultname]
\label{lemma:orlicz_process}
If $\gkernel$ satisfies \cref{asmp:bounded_measurable,assum:lipkernel}, then, given any fixed $R>0$, the random process $\braces{\outvec[n/2](\x), \x \in \ball_2(0; R)}$ is an Orlicz $\orlicz$-process with respect to the metric $\ometric$ defined in \cref{eq:orlicz_metric}, i.e., 
\begin{talign}
\label{eq:subgaussian_process}
\norm{\outvec[n/2](\x)-\outvec[n/2](\x')}_{\orlicz} \leq \ometric(\x, \x')
\qtext{for all} \x,\x'\in\ball_2(0; R),
\end{talign}
where  the metric $\ometric$ is defined as
\begin{talign}
\label{eq:orlicz_metric}
\ometric(\x, \x')= \sqrt{\frac{8}{3}} \cdot \sigma_{n/2} \cdot \min\parenth{\sqrt{2\klip\twonorm{\x-\x'}}, 2\sqrt{\infnorm{\gkernel}}}.
\end{talign} 
\end{lemma}

Given \cref{lemma:orlicz_process}, we can invoke high probability bounds
for Orlicz processes. For the Orlicz process $\braces{\outvec[n/2](\x),
\x \in \mbb{T}}$, given any fixed $\delta'\in (0, 1]$, \citet[Thm~5.36]{wainwright2019high} implies that
 \begin{talign}
 \label{eq:orlicz_hpbound}
 \sup_{\x, \x' \in \mbb{T}} \abss{\outvec[n/2](\x)-\outvec[n/2](\x')} \leq 8\parenth{\mc J_{\mbb{T}, \ometric}(\diam) + \diam \sqrt{\log(4/\delta')}},
 \end{talign}
 with probability at least $1-\delta'/2$, where
 $\diam \defeq \sup_{\x, \x' \in \mbb{T}} \ometric(\x, \x')$ denotes the $\ometric$-diameter of $\mbb{T}$, and the quantity $\mc J_{\mbb{T}, \ometric}(\diam)$ is defined as
 \begin{talign}
 \label{eq:def_J}
 \mc J_{\mbb{T}, \ometric}(\diam) \defeq \int_{0}^{\diam} \sqrt{\log(1+\mc{N}_{\mbb{T}, \ometric}(u))} du.
 \end{talign}
 Here $\mc{N}_{\mbb{T}, \ometric}(u)$ denotes the $u$-covering number of $\mbb{T}$ with respect to metric $\ometric$, namely the cardinality of the smallest cover $\mc{C}_{\mbb{T}, \ometric}(u) \subseteq\Rd$ such that 
\begin{talign}
\label{eq:rho_ball}
\mbb{T} \subseteq \cup_{\z \in \mc{C}_{\mbb{T}, \ometric}(u)} \ball_{\ometric}(\z; u)
\qtext{where} \ball_{\ometric}(\z; u) \defeq \braces{\x\in\Rd: \ometric(\z,\x)\leq u}.
\end{talign}
To avoid confusion, let $\ball_2$ denote the Euclidean ball (metric induced by $\twonorm{\cdot}$). Then in our setting, we have $ \mbb{T} = \ball_2(0; R) $ and hence
 \begin{talign}
 \label{eq:diam_beta_def}
 D_R \defeq \sup_{\x, \x' \in \ball_2(0; R)} \ometric(\x, \x') = \min\parenth{\sqrt{\frac{\temptwo_R}{2}}, \sigma_{n/2}\sqrt{\frac{32\infnorm{\gkernel}}{3}}},
 \ \text{with}\ \temptwo_R = \frac{32}{3} \sigma^2_{n/2} \klip R.
 \end{talign}
 Some algebra establishes that this definition of $\diam_R$ is identical to that specified in \cref{lemma:orlicz_bound}.
 Next, we derive bounds for $\mc N_{\mbb{T}, \ometric}$ and $\mc J_{\mbb{T}, \ometric}$ (see \cref{sub:proof_of_lemma_covering_and_j2} for the proof).
\begin{lemma}[\boundsoncoveringandjresultname]
 \label{lemma:bound_covering_alternative}
 If $\gkernel$ satisfies \cref{asmp:bounded_measurable,assum:lipkernel}, then, for $\mbb{T} = \ball_2(0; R)$ with $R>0$ and $\ometric$ defined in \cref{eq:orlicz_metric}, we have
 \begin{talign}
 \label{eq:rho_covering_bound}
 \mc N_{\mbb{T}, \ometric}(u) &\leq \parenth{1+\temptwo_R/{u^2}}^{d}, \qtext{and} \\
 \label{eq:bound_j2}
 \mc J_{\mbb{T}, \ometric}(\diam_R) &\leq \sqrt{d} D_R \brackets{ 3 + \sqrt{2\log({\temptwo_R}/{\diam_R^2})}},
 \end{talign}
 where $\diam_R$ and $\temptwo_R$ are defined in \cref{eq:diam_beta_def}.
\end{lemma} 
Doing some algebra, we find that
\begin{talign}
\label{eq:max_by_sum}
\frac{\temptwo_R}{\diam_R^2} = \max\parenth{2, \frac{\klip[\gkernel]R}{\infnorm{\gkernel}} } \leq 2+\frac{\klip[\gkernel]R}{\infnorm{\gkernel}}.
\end{talign}
Moreover, note that $3+ \sqrt{2\log(2+a)} \leq 6\sqrt{\log (2+a)}$ for all $a \geq 0$, so that the bound~\cref{eq:bound_j2} can be  simplified to
\begin{talign}
\label{eq:j2_simple_bound}
    \mc J_{\mbb{T}, \ometric}(\diam_R) \leq 6\sqrt{d}\diam_R \sqrt{\log(2+\frac{\klip[\gkernel]R}{\infnorm{\gkernel}})}.
\end{talign}
Substituting the bound~\cref{eq:j2_simple_bound} in \cref{eq:orlicz_hpbound} yields that the event $\esup$ defined in \cref{eq:orlicz_hpbound_lemma} holds with probability at least $1-\delta'/2$, as claimed.

\subsection{Proof of \lowercase{\Cref{lemma:orlicz_process}}: \orliczprocessresultname}
\label{sub:proof_of_lemma_orlicz_process}
The proof of this lemma follows from the sub-Gaussianity of $\outvec[n/2]$ established in \cref{sbhw_properties}. Introduce the shorthand $Y\defeq\outvec[n/2](\x)-\outvec[n/2](\x')$. Applying property~\ref{item:fun_subgauss} of \cref{sbhw_properties} with $u = \gkernel(\x, \cdot)-\gkernel(\x', \cdot)$ along with the reproducing property of the kernel $\gkernel$, we find that for any $t \in \real$,
\begin{talign}
\label{eq:subgaussian_Y}
	\Exs[\exp(tY) ] = 
	\Exs[\exp(t\dotk{\outvec[n/2], \gkernel(\x, \cdot)\!-\!\gkernel(\x', \cdot)})  ] \leq \exp(\frac12t^2\sigma_{n/2}^2 \knorm{\gkernel(\x, \cdot)\!-\!\gkernel(\x',\cdot)}^2 ).
\end{talign}
That is, the random variable $Y$ is sub-Gaussian with parameter $\sigma_Y \defeq \sigma_{n/2}\knorm{\gkernel(\x, \cdot)-\gkernel(\x',\cdot)}$. Next, \citet[ Thm~2.6 (iv)]{wainwright2019high} yields that
$\Exs[\exp(\frac{3Y^2}{8\sigma_Y^2}) ]\leq 2 $, which in turn implies that $\norm{Y}_{\orlicz} \leq \sqrt{\frac{8}{3}} \sigma_Y$.
Moreover, we have
 \begin{talign}
	 \knorm{\gkernel(\x, \cdot)-\gkernel(\x',\cdot)}^2 = \gkernel(\x, \x) - \gkernel(\x, \x') + \gkernel(\x', \x') - \gkernel(\x', \x)
	 \leq \min\parenth{2 \klip\twonorm{\x-\x'}, 4 \infnorm{\gkernel}},
 \end{talign}
 using the Lipschitz continuity of $\gkernel$ and the definition of $\infnorm{\gkernel}$~\cref{eq:klip}.
 Putting the pieces together along with the definition~\eqref{eq:orlicz_metric} of $\ometric$, we find that $\norm{Y}_{\orlicz} \leq \ometric(\x, \x')$ thereby yielding the claim~\eqref{eq:subgaussian_process}.

 \subsection{Proof of \lowercase{\Cref{lemma:bound_covering_alternative}}: \boundsoncoveringandjresultname}
 \label{sub:proof_of_lemma_covering_and_j2} 

 We use the relation of $\ometric$ to the Euclidean norm $\twonorm{\cdot}$ to establish the bound~\eqref{eq:rho_covering_bound}. The definitions~\cref{eq:orlicz_metric,eq:rho_ball} imply that
 \begin{talign}
	 \label{eq:ball_inclusion}
	 \ball_{2}(\z; \frac{u^2}{\tempone}) \subseteq  \ball_{\ometric}(\z; u)
	 \qtext{for} \tempone \defeq  \frac{16}{3}\sigma^2_{n/2}\klip,
 \end{talign}
 since $\ometric(\x, \x') \leq \sqrt{\tempone\twonorm{\x-\x'}}$. Consequently, any $u^2/\tempone$-cover $\mc C$ of $\mbb{T}$ in the Euclidean ($\twonorm{\cdot}$) metric automatically yields a $u$-cover of $\mbb{T}$ in $\ometric$ metric as we note that
 \begin{talign}
	 \ball_2(0; R) \subseteq \cup_{\z\in\mc C} \ball_2(\z; \frac{u^2}{\tempone}) 
	 \subseteq \cup_{\z\in\mc C} \ball_{\ometric}(\z; u).
 \end{talign}
 Consequently, the smallest cover $\mc C_{\mbb{T}, \ometric}(u)$ would not be larger than than the smallest cover $\mc C_{\mbb{T}, \twonorm{\cdot}}(u^2/\tempone)$, or equivalently:
 \begin{talign}
	 \mc N_{\mbb{T}, \ometric}(u) \leq \mc N_{\mbb{T}, \twonorm{\cdot}}(\frac{u^2}{\tempone})) 
	 \sless{(i)} \parenth{1+{2R\tempone}/{u^2}}^{d},
 \end{talign}
 where inequality~(i) follows from \citet[Lem~5.7]{wainwright2019high} using the fact that $\mbb{T} = \ball_2(0; R)$. Noting that $\temptwo_{R}=2R\tempone$ yields the bound~\eqref{eq:rho_covering_bound}.
 
 We now use the bound~\eqref{eq:rho_covering_bound} on the covering number to establish the bound~\cref{eq:bound_j2}.

 \paragraph*{Proof of bound on $\mc J_{\mbb{T}, \ometric}$}
 Applying the definition~\cref{eq:def_J} and the bound~\cref{eq:rho_covering_bound}, we find that
 \begin{talign}
	 \mc J_{\mbb{T}, \ometric}(\diam_R) 
	 &\leq \int_{0}^{\diam_R} \sqrt{\log(1+\parenth{1+{\temptwo_R}/{u^2}}^{d})} du \\
	 &\seq{(i)} \sqrt{\frac{\temptwo_R}{2}} \int_{{\temptwo_R}/{\diam_R^2}}^{\infty} s^{-3/2}\sqrt{\log(1+(1+s)^d)} ds,
	 \label{eq:j_first_bound}
 \end{talign}
 where step~(i) follows from a change of variable $s\gets \temptwo_R/u^2$.
 Note that $\log(1+a) \leq \log a + 1/a$, and $\sqrt{a+b}\leq \sqrt{a} + \sqrt{b}$ for any $a, b\geq0$. Applying these inequalities, we find that
 \begin{talign}
	 \label{eq:log_bound}
	 &\log(1+(1+s)^d) \leq 
	 d \log (2+s) \leq d(\log s + \frac{1}{s} + \frac{1}{1+s}) 
	 \leq d(\log s + \frac{2}{s}),
 \end{talign}
 for $s \in [\temptwo_R/\diam_R^2, \infty) \stackrel{\cref{eq:max_by_sum}}{\subseteq} [2, \infty)$. Consequently,
 \begin{talign}
	 \label{eq:integral_first_bound}
	 \int_{\temptwo_R/\diam_R^2}^{\infty} s^{-3/2}\sqrt{\log(1+(1+s)^d)} ds 
	 \leq \sqrt{d} \int_{\temptwo_R/\diam_R^2}^{\infty} \parenth{\frac{\sqrt{\log s}}{s^{3/2}} + \frac{\sqrt{2}}{s^2}} ds.
 \end{talign}
 Next, we note that
 \begin{talign}
 \label{eq:ssq_integral}
    \int_{\temptwo_R/\diam_R^2}^{\infty} \frac{\sqrt{2}}{s^2} ds
    &= \sqrt{2} \frac{\diam_R^2}{\temptwo_R}, \qtext{and}\\
	  \int_{\temptwo_R/\diam_R^2}^{\infty} \frac{\sqrt{\log s}}{s^{3/2}} ds 
	  &= -2\sqrt{\frac{\log s}{s}}\! -\! \sqrt{2} \Gamma(\frac12, \frac12\log s) \big\vert_{s=\temptwo_R/\diam_R^2}^{s=\infty} \\
	  &\sless{(i)} \frac{2\diam_R}{\sqrt{\temptwo_R}} \parenth{ \sqrt{\log(\temptwo_R/\diam_R^2)} 
	  +  \frac{1}{\sqrt{\log(\temptwo_R/\diam_R^2)}}},
  \label{eq:log_s_integral}
 \end{talign}
 where step~(i) follows from the following bound on the incomplete Gamma function:
 \begin{talign}
	 \Gamma(\frac12, a) = \int_{a}^\infty\frac{1}{\sqrt{t}} e^{-t} dt \leq 
	 \frac{1}{\sqrt{a}} \int_{a}^\infty e^{-t} dt = a^{-\frac12} e^{-a}
	 \qtext{for any} a >0.
 \end{talign}
 Putting the bounds~\cref{eq:j_first_bound,eq:integral_first_bound,eq:ssq_integral,eq:log_s_integral} together, we find that
 \begin{talign}
	 \mc J_{\mbb{T}, \ometric}(\diam_R)  \leq \sqrt{d} \diam_R \parenth{ \frac{\diam_R}{\sqrt{\temptwo_R}} + \sqrt{2} \parenth{\sqrt{\log(\temptwo_R/\diam_R^2)} +  1/{\sqrt{\log(\temptwo_R/\diam_R^2)}}} }.
 \end{talign}
 Note that by definition~\cref{eq:diam_beta_def} $\frac{\diam_R}{\sqrt{\temptwo_R}} \leq \frac{1}{\sqrt{2}}$. Using this observation and the fact that $\frac1{\sqrt{2}} + \sqrt{\frac2{\log2}} \leq 3$ yields the claimed bound~\cref{eq:bound_j2}.

\section{Proof of \cref{rem:linf-rate-examples}: \linfrateexamplename}
\label{sec:proof_of_rem:linf-rate-examples}
Note that when the kernel is restricted to  a compact domain, we have $\rk \leq \order(\sigma)$ so that in this case, $\frac{\klip(\rk\wedge\rminpn[\inputcoreset]) }{\sinfnorm{\gkernel}}  \leq \order(L)$.
Next, we establish the bounds on $L$ and $\kappa^\dagger(1/n)$ for each of the kernels, which immediately imply the respective claims both when supported on $\real^d$ and on a compact domain with $\sigma  = \sqrt{d}$.
\begin{itemize}
    \item For a Gaussian kernel, we  have $\kappa(r) = e^{-r^2/2}$, which when put together with \cref{eq:gauss_klip,eq:gauss_rmin} implies that $L=\sqrt{2/e}$, and $\kappa^\dagger(1/n) = \sqrt{2\log n}$.
    \item For a \Matern kernel, we have $\kappa(r) = c_b r^b K_b(r)$ for $b = \nu-\frac{d}{2}$, for suitable constant $c_b$, and function $K_b$; see \cref{table:kernel_sqrt_pair}. Moreover, \cref{eq:collect_matern_klip,eq:collect_matern_rmin} imply that $L = \frac{C_1}{\sqrt{C_2+b}} \leq \frac{C_1}{\sqrt{C_2+1}}$ for universal constants $C_1, C_2$, and $\kappa^{\dagger}(1/n)\leq \order(\max(b\log b, \log n)) = \order(\log n)$.
    \item For a IMQ kernels, we have $\kappa(r) = \frac{1}{(\mattwo^2+r^2)^{\matone}}$, so that
    \begin{talign}
            L \leq \sup_r \kappa'(r) = \sup_r \frac{2\matone r}{(\mattwo^2+r^2)^{\matone+1}} = 
            \frac{2\matone \mattwo}{\mattwo^{2(\matone+1)}} \frac{1}{\sqrt{1+2\matone}} (\frac{2\matone+1}{2\matone+2})^{\matone+1},
    \end{talign}
    and $\kappa^{\dagger}(1/n) = \sqrt{n^{1/\matone}-\mattwo^2} \leq n^{1/2\matone}$.
\end{itemize}

\section{Proof of \cref{corollary:kernel_thinning_coreset_bound}: \kernelthinningrecursiveresultsname}
\label{sec:proof_of_corollary:kernel_thinning_coreset_bound}
First applying \cref{kernel_halving_results}\cref{item:kernel_halving_multiple_rounds} with $\gkernel=\ksqrt$ yields that
\begin{talign}
\label{eq:ksqrt_linf_guarantee}
 	\sinfnorm{\P_n\ksqrt-\Q_{\mrm{KT}}^{(\m, 1)}\ksqrt}
 	\leq \frac{\sinfnorm{\ksqrt}}{\nout}
 	\cdot\err_{\ksqrt}(n, m, d, \delta^\star, \delta', \rminnew[\inputcoreset, \ksqrt, n])
\end{talign}
with probability at least $1\!-\!\delta'\!-\!\sum_{j=1}^{\m} \frac{2^{j-1}}{m} \sum_{i=1}^{2^{m-j}\floor{n/2^m}}\delta_i$. Call this event $\mc E$.

Let $r' = \rksmaxgen{\inputcoreset, \ksqrt, n_{\mrm{out}}}+\vareps$ denote a shorthand notation, where $\vareps>0$ is an arbitrary scalar.
Applying \cref{theorem:coreset_to_mmd} with %
$r=2 r'$ and $a=b=\frac12$, we find that on event $\mc E$,
\begin{talign}
    &\mmd_{\kernel}(\inputcoreset, \coreset[\m, 1]) \\
    &=
    \mmd_{\kernel}(\P_n, \Q_{\mrm{KT}}^{(\m, 1)}) \\
     &\leq 
     	\sinfnorm{\P_n\ksqrt-\Q_{\mrm{KT}}^{(\m, 1)}\ksqrt} \cdim r^{d/2}
		+ 2\tail[\ksqrt](r')
		+ 2\sinfnorm{\kernel}^{\frac12}\! \cdot \max({\tail[\P_n](r'), \tail[\Q_{\mrm{KT}}^{(\m, 1)}](r')})
		\\
		&\seq{(i)} \sinfnorm{\P_n\ksqrt-\Q_{\mrm{KT}}^{(\m, 1)}\ksqrt} \cdot \sqrt{\frac{(4\pi)^{d/2}}{\Gamma(\frac{d}{2}+1)}} (r')^{\frac{d}{2}}+ \frac{2\sinfnorm{\ksqrt}}{\nout}, \\
		&\sless{\cref{eq:ksqrt_linf_guarantee}}
		\frac{\sinfnorm{\ksqrt}}{\nout}
 	\cdot\err_{\ksqrt}(n, m, d, \delta^\star, \delta', \rminnew[\inputcoreset, \ksqrt, n]) \cdot \sqrt{\frac{(4\pi)^{d/2}}{\Gamma(\frac{d}{2}+1)}} (r')^{\frac{d}{2}}+ \frac{2\sinfnorm{\ksqrt}}{\nout} \\
 	&= \frac{\sinfnorm{\ksqrt}}{\nout} \brackets{2+\sqrt{\frac{(4\pi)^{d/2}}{\Gamma(\frac{d}{2}+1)}} (r')^{\frac{d}{2}} \err_{\ksqrt}(n, m, d, \delta^\star, \delta', \rminnew[\inputcoreset, \ksqrt, n])}
		\label{eq:thm1_proof_step2}
\end{talign}
where step~(i) uses the following arguments: (a) $\cdim r^{d/2} = \sqrt{\frac{(4\pi)^{d/2}}{\Gamma(\frac{d}{2}+1)}} (r')^{\frac{d}{2}}$; (b) $\Q_{\mrm{KT}}^{(\m, 1)}$ is supported on a subset of points from $\inputcoreset$ and hence 
\begin{talign}
     \max\{{\tail[\P_n](r'), \tail[\Q_{\mrm{KT}}^{(\m, 1)}](r')}\} = \tail[\P_n](r') = 0
     \qtext{for any }r' > \rminpn[\inputcoreset];
\end{talign}
and (c) $\tail[\ksqrt](r') \leq \frac{\sinfnorm{\ksqrt}}{\nout}$ for any $r' > \rktaugen[\ksqrt,\nout]$.
Noting that \cref{eq:thm1_proof_step2} applies simultaneously for all $\vareps>0$ under event $\mc E$, taking the limit  $\vareps\to0$ yields that
\begin{talign}
\mmd_{\kernel}(\inputcoreset, \coreset[\m, 1])
\leq \frac{\sinfnorm{\ksqrt}}{\nout} \brackets{2+\sqrt{\frac{(4\pi)^{d/2}}{\Gamma(\frac{d}{2}+1)}} (\rksmaxgen{\inputcoreset,\ksqrt,\nout})^{\frac{d}{2}} \err_{\ksqrt}(n, m, d, \delta^\star, \delta', \rminnew[\inputcoreset, \ksqrt, n])}
\label{eq:mmd_ktsplit_bound}
\end{talign}
with probability at least $1\!-\!\delta'\!-\!\sum_{j=1}^{\m} \frac{2^{j-1}}{m} \sum_{i=1}^{2^{m-j}\floor{n/2^m}}\delta_i$, as claimed.

\section{Derivation of \lowercase{\Cref{table:kernel_sqrt_pair}}: \sqrttablename} %
\label{sec:proof_for_tables}
In this appendix, we derive the results stated in \cref{table:kernel_sqrt_pair}.

\subsection{General proof strategy}
\label{sec:proof_of_sqrt_kernels}
Let $\fourier$ denote the Fourier operator~\cref{defn:fourier_transform}.
In \cref{table:kernel_sqrt_pair_proof}, we state the continuous $\kappa$ such that
$\kernel(\x,\y) = \kappa(\x-\y)$ in the first column, its  Fourier transform~\cref{defn:fourier_transform} $\widehat{\kappa}$ in the second column, the square-root Fourier transform in the third column, and the square-root kernel in the fourth column, given by
$\ksqrt(x,y) = \frac{1}{(2\pi)^{d/4}}\kappasqrt(\x-\y)$
with $\kappasqrt =\fourier(\sqrt{\widehat{\kappa}})$.
\cref{sqrt_translation_invariant} along with expressions in \cref{table:kernel_sqrt_pair_proof} directly establishes the validity of the square-root kernels for the Gaussian and (scaled) B-spline kernels. For completeness, we also illustrate the remaining calculus for the B-spline kernels in  \cref{ssub:derive_krt_spline}. We do a similar calculation in \cref{ssub:derive_krt_matern} for the \Matern kernel for better exposition of the involved expressions.

\newcommand{\fouriertransformtablecaption}{
\noindent\caption{\label{table:kernel_sqrt_pair_proof}\tbf{Fourier transforms of kernels $\kernel(\x,\y) = \kappa(\x-\y)$ and square-root kernels $\ksqrt(\x,\y) = \kappasqrt(\x-\y)$ from  \cref{table:kernel_sqrt_pair}.}
    Here  $\fourier$ denotes the Fourier operator~\eqref{defn:fourier_transform}, and
    $\circledast^{\l}$ denotes the convolution operator applied $\l-1$ times, with the convention $\circledast^{1}\fun \defeq \fun$ and $\circledast^{2} \fun = \fun \circledast \fun$ for $(\fun \circledast \funtwo) (x) = \int \fun(y) \funtwo(x-y)dy$.
    Each Fourier transform is derived from  \citet[Tab.~2]{sriperumbudur2010hilbert}.
    } 
}

\begin{table}[ht]
  \resizebox{1. \textwidth}{!}{
    {\renewcommand{\arraystretch}{2}
    \begin{tabular}{cccccc}
        \toprule
        \Centerstack{\bf Expression \\ \bf for $\kappa(\z)$} &
        \Centerstack{\bf Expression for Fourier \\ \bf transform of $\kappa$:
        $\widehat{\kappa}(\omega)$ } 
        & \Centerstack{\bf Square-root Fourier  \\ \bf transform: $\sqrt{\widehat{\kappa}(\omega)}$} 
        & \Centerstack{\bf Expression for\\ $\kappasqrt(\z) \defeq (\frac{1}{2\pi})^{\frac{d}{4}} \fourier
        (\sqrt{\widehat{\kappa}})(\z)$} 
        \\[2mm]
        \midrule 
          $\exp\parenth{-\displaystyle\frac{\twonorm{\z}^2}{2\gaussparam^2}}$
          & $\gaussparam^d\exp\parenth{-\displaystyle\frac{\gaussparam^2
          \twonorm{\omega}^2}{2}}$
           & $\gaussparam^{\frac{d}{2}}\exp\parenth{-\displaystyle\frac{\gaussparam^2 \twonorm{\omega}^2}{4}}$
          & $\parenth{\displaystyle\frac{2}{\pi\gaussparam^2}}^{\frac{d}{4}}\exp\parenth{-\displaystyle\frac{\twonorm{\z}^2}{\gaussparam^2}}$
          \\[6mm]

         $\displaystyle\prod_{\j=1}^d\circledast^{2\splineparam+2}\indicator_
         {[-\frac12, \frac12]}(\z_{\j}) $
         & $\displaystyle\parenth{\frac{4^{\splineparam+1}}{\sqrt{2\pi}}}^d
         \prod_{\j=1}^d\frac{\sin^{2\splineparam+2}(\omega_{\j}/2)}{\omega_{\j}^{2\splineparam+2}}$
          & $\displaystyle\parenth{\frac{4^{\splineparam+1}}{\sqrt{2\pi}}}^
          {\frac{d}{2}}
          \prod_{\j=1}^d\frac{\sin^{\splineparam+1}(\omega_{\j}/2)}{\omega_{\j}^{\splineparam+1}}$
         & $\displaystyle\prod_{\j=1}^d\circledast^{\splineparam+1}\indicator_
         {[-\frac12, \frac12]}(\z_{\j}) $
        \\[2ex] \bottomrule \hline
    \end{tabular}
    }
    }
    \opt{arxiv}{\fouriertransformtablecaption}
    \opt{jmlr}{\fouriertransformtablecaption}
\end{table}

\subsection{Deriving $\ksqrt$ for the \Matern kernel}
\label{ssub:derive_krt_matern}
For $\kernel\!=\!\textbf{\Matern}(\matone,\mattwo)$ from \cref{table:kernel_sqrt_pair}, we have
\begin{talign}
\label{eq:def_wendmatern}
 \kernel(\x, \y)&=\matscale[\matone]\cdot \wendmatern(\x-\y)
\qtext{where}
\wendmatern(\z) \defeq c_{\matone}\parenth{\frac{\twonorm{\z}}{\mattwo}}^{\matone-\frac{d}{2}}\bessel[\matone-\frac{d}{2}](\mattwo\twonorm{\z}),
\label{eq:matern_scaling}\\
\matscale[\matone] &\defeq  \frac{c_{\matone-d/2}}{c_{\matone}} \mattwo^{2\matone-d}
\qtext{and} c_{\matone} \defeq \frac{2^{1-\matone}}{\Gamma(\matone)},
\label{eq:matern_scaling_constants}
\end{talign}
and $\bessel[a]$ denotes the modified Bessel function of third kind of order $a$~\citep[Def.~5.10]{wendland2004scattered}. That is, for \Matern kernel with $\kernel(\x, \y)=\kappa(\x-\y)$, the function $\kappa$ is given by  $\kappa = \matscale \wendmatern$. Now applying \citep[Thm.~is 8.15]{wendland2004scattered}, we find that
\begin{talign}
\label{eq:fourier_wendmatern}
\fourier(\wendmatern) = \frac{1}{(\mattwo^2+\twonorm{\omega}^2)^{\matone}}
\quad {\Longrightarrow}\quad \fourier(\sqrt{\fourier(\wendmatern)}) = \frac{1}{(2\pi)^{d/4}}\wendmatern[\matone/2],
\end{talign}
where in the last step we also use the facts that $\wendmatern$ is an even function and $\sqrt{\fourier(\wendmatern)}\in\lone$ for all $\matone>d/2$.
Thus, by \cref{sqrt_translation_invariant}, a valid square-root kernel $\ksqrt(\x, \y)=\kappasqrt(\x-\y)$ is defined via
\begin{talign}
\kappasqrt = \sqrt{\matscale} \frac{1}{(2\pi)^{d/4}}\wendmatern[\matone/2],
\implies
\ksqrt 
&=  \matrtconst \cdot  \textbf{\Matern}(\matone/2, \mattwo),
\label{eq:matern_ksqrt} \\
\qtext{where} \matrtconst &\defeq \parenth{\frac{1}{4\pi}\mattwo^2}^{d/4} \sqrt{\frac{\Gamma(\matone)}{\Gamma(\matone-d/2)}} \cdot \frac{\Gamma((\matone-d)/2)}{\Gamma(\matone/2)}. 
\label{eq:matern_ksqrt_const}
\end{talign}

\newcommand{\tempsplineconst}[1][\splineparam+1]{\mfk{B}_{#1}}
\newcommand{\splinefun}[1][\splineparam]{\chi_{#1}}
\subsection{Deriving $\ksqrt$ for the B-Spline kernel} 
\label{ssub:derive_krt_spline}
For positive integers $\splineparam, d$, define the constants
\begin{talign}
\label{eq:spline_all_constants}
\tempsplineconst[\splineparam] \defeq
\frac{1}{(\splineparam-1)!} \sum_{j=0}^{\lfloor\splineparam/2\rfloor} (-1)^j {\splineparam \choose j} (\frac{\splineparam}{2} - j)^{\splineparam-1},
\qquad
\splineconst[\splineparam, d]  \defeq
\tempsplineconst[\splineparam]^{-d},
\qtext{and}
\wtil{S}_{\splineparam, d} \defeq \frac{\sqrt{\splineconst[2\splineparam+2, d]}}{\splineconst[\splineparam+1,d]}.
\end{talign}
Define the function $\splinefun:\Rd \to \real$ as follows:
\begin{talign}
\label{eq:splinefun}
\splinefun(\z) \defeq \splineconst[\splineparam, d] \prod_{i=1}^d \fun_{\splineparam}(\z_i).
\end{talign}
Then for kernel $\kernel = \textbf{B-spline}(2\splineparam+1)$, we have
\begin{talign}
\label{eq:splinefun_kernel}
\kernel(\x, \y) = \splinefun[2\splineparam+2](\x-\y).
\end{talign}
The second row of \cref{table:kernel_sqrt_pair_proof} implies that
\begin{talign}
(\frac{1}{2\pi})^{\frac{d}{4}}\fourier\parenth{\sqrt{\fourier(\splinefun[{2\splineparam+2}])}}
=  \frac{\sqrt{\splineconst[2\splineparam+2, d]}}{\splineconst[\splineparam+1,d]} \cdot\splinefun[{\splineparam+1}] \seq{\cref{eq:spline_all_constants}}
 \wtil{S}_{\splineparam, d}\cdot \splinefun[{\splineparam+1}],
\end{talign}
where we also use the fact that $\splinefun$ is an even function.
Putting the pieces together with \cref{sqrt_translation_invariant}, we conclude that
\begin{talign}
\label{eq:splinefun_krt}
\ksqrt = \wtil{S}_{\splineparam, d} \cdot \textbf{B-spline}(\splineparam),
\end{talign}
(with $\kappasqrt= \wtil{S}_{\splineparam, d} \splinefun[\splineparam+1]$)
is a valid square-root kernel for $\kernel =\textbf{B-spline}(2\splineparam+1)$.

\newcommand{\rtemp}{\zeta_{\ksqrt, \inputcoreset}}
\newcommand{\Qkt}{\Q_{\sqrt{n}}^{\trm{KT}}}
\newcommand{\mmdkt}{\mmd_{\kernel}(\inputcoreset, \ktcoreset)}

\newcommand{\psubgauss}{\P_{\mbf{sg}}}
\newcommand{\psubexp}{\P_{\mbf{se}}}
\newcommand{\pcompact}{\P_{\mbf{c}}}
\section{Derivation of \lowercase{\Cref{table:sqrtk_details}}: {\sqrtdetailstablename{\lowercase{\Cref{theorem:main_result_all_in_one}}}}}
\label{sec:proof_of_table_sqrtk_details}
We start by collecting some common tools in \cref{sub:common_tools_for_derivation_sqrt_details} that we later use for our derivations for the Gaussian, \Matern and B-spline kernels in \cref{sub:gaussian_kernel_quantities_proofs,sub:matern_kernel_quantities_proofs,sub:bspline_kernel_quantities_proofs} respectively. Finally, in \cref{sub:explicit_mmd_rates}, we put the pieces together to derive explicit MMD rates, as a function of $d, n, \delta$ and kernel parameters for the three kernels, and summarize the rates in \cref{table:mmd_k_explicit}. (This table is complementary to the generic results stated in \cref{table:mmd_rates}.)

\subsection{Common tools for our derivations}
\label{sub:common_tools_for_derivation_sqrt_details}
We collect some simplified expressions, and techniques that come handy in our derivations to follow.

\paragraph*{Simplified bounds on $\err_{\ksqrt}$}
From \cref{eq:err_simple_defn}, we have
\begin{talign}
\err_{\ksqrt}(n, \half\log_2n, d, \frac{\delta}{n}, \delta',  R) 
&\precsim  \sqrt{\log \frac{n}{\delta} \cdot \brackets{ \log(\frac{1}{\delta'}) + d\log \parenth{\frac{\klip[\ksqrt]}{\infnorm{\ksqrt}} \cdot (\rmin_{\ksqrt, n}+R) } }} \\
\implies 
\err_{\ksqrt}(n, \half\log_2n, d, \frac{\delta}{n}, \delta,  R) 
&\precsim_{\delta}  \sqrt{d\log n \cdot\log \parenth{\frac{\klip[\ksqrt]}{\infnorm{\ksqrt}} \cdot (\rmin_{\ksqrt, n}+R) } }.
\label{eq:scaling_of_a_M}
\end{talign} 
Thus, given \cref{eq:scaling_of_a_M}, to get a bound on $\err_{\ksqrt}$, we need to derive bounds on $(\klip[\ksqrt], \infnorm{\ksqrt},\rmin_{\ksqrt, n})$ for various kernels to obtain the desired guarantees on MMD and $\Linf$ coresets.

\paragraph*{Bounds on Gamma function} Our proofs make use of the bounds from \citet[Thm~2.2]{batir2017bounds} on the Gamma function:
\begin{talign} 
\label{eq:stirling_approx} 
\Gamma(b+1) \geq (b/e)^{b} \sqrt{2\pi b} 
\text{ for any } b \geq 1,
\text{ and }
\Gamma(b+1)  \leq (b/e)^{b} \sqrt{e^2 b}
\text{ for any } b \geq 1.1.
\end{talign}

\paragraph*{General tools for bounding the Lipschitz constant}
To bound the Lipschitz constant $\klip[\ksqrt]$, the following two observations come in handy. For a radial kernel $\ksqrt(\x, \y) = \wtil{\kappa}_{\rttag}(\twonorm{\x-\y})$ with $\wtil{\kappa}_{\rttag}:\real_{+} \to \real$, we note
\begin{talign}
    \sup_{\substack{\x, \y, \z: \\ \twonorm{\y-\z}\leq r}} \abss{\ksqrt(\x,\y)-\ksqrt (\x,\z)} &\leq \sup_{a>0, b \leq r} \abss{\wtil{\kappa}_{\rttag}(a) - \wtil{\kappa}_{\rttag}(a+b)} \leq \infnorm{\wtil{\kappa}_{\rttag}'} r \\
    \implies \klip[\ksqrt] &\leq \infnorm{\wtil{\kappa}_{\rttag}'}.
    \label{eq:derivative_relation}
\end{talign}
For a translation-invariant kernel $\ksqrt(\x, \y) = \kappasqrt(\x-\y)$ with $\kappasqrt:\Rd \to \real$, we use the bound
\begin{talign}
    \sup_{\substack{\x, \y, \z: \\ \twonorm{\y-\z}\leq r}} \abss{\ksqrt(\x,\y)-\ksqrt (\x,\z)} &= \sup_{\z', \z'', \twonorm{\z-\z'} \leq r} \abss{\kappasqrt(\z') - \kappasqrt(\z'')} \\ 
    &\leq \sup_{\z'\in\Rd}\twonorm{\nabla\kappasqrt(\z')} r, \\
    \implies \klip[\ksqrt] &\leq  \sup_{\z'\in\Rd} \twonorm{\nabla\kappasqrt(\z')} \sless{(i)} \sqrt{d} \sup_{\z'\in\Rd}\infnorm{\nabla\kappasqrt(\z')},
    \label{eq:derivative_relation_inf}
\end{talign}
where step~(i) follows from Cauchy-Schwarz's inequality (and is handy when coordinate wise control on $\nabla\kappasqrt$ is easier to derive).
We later apply the inequality~\cref{eq:derivative_relation} for the Gaussian and \Matern kernels, and \cref{eq:derivative_relation_inf} for the B-spline kernel.

\subsection{Proofs for Gaussian kernel}
\label{sub:gaussian_kernel_quantities_proofs}
For the kernel $\kernel=\textbf{Gauss}(\gaussparam)$ and its square-root kernel $\ksqrt= \parenth{\frac{2}{\pi\gaussparam^2}}^{\frac{d}{4}}\textbf{Gauss}(\frac{\gaussparam}{\sqrt 2})$, we claim the following bounds on various quantities
\begin{talign}
\label{eq:gauss_kinf}
\infnorm{\ksqrt} &= \parenth{\frac{2}{\pi\gaussparam^2}}^{\frac{d}{4}}
\qtext{and} \infnorm{\kernel}=1, \\
\klip[\ksqrt] &\leq \parenth{\frac{2}{\pi\gaussparam^2}}^{\frac{d}{4}} \frac{\sqrt{2/e}}{\gaussparam}
\implies {\klip[\ksqrt]}/{\infnorm{\ksqrt}} \seq{\cref{eq:gauss_kinf}} \frac{1}{\gaussparam} {\sqrt{2/e}}
\label{eq:gauss_klip}
\\
\label{eq:gauss_rmin}
\rmin_{\ksqrt, n} &= \gaussparam\sqrt{\log n}, \\
\label{eq:tail_int_gauss}
\rktaugen[\ksqrt, \sqrt{n}] &=\order(\gaussparam \sqrt{d+\log n }).
\end{talign}
Substituting these expressions in \cref{eq:scaling_of_a_M}, we find that
\begin{talign}
     \err_{\ksqrt}^{\textbf{Gauss}}(n,\half\log_2n, d,  \frac{\delta}{2n},\delta', R)
    &\precsim 
    \sqrt{ \log(\frac{n}{\delta}) \brackets{\log\frac1{\delta'}+ d \log \parenth{
    \sqrt{\log n} + \frac{R}{\sigma} } }},
    \label{eq:gauss_krt_err}
\end{talign}
as claimed in \cref{table:sqrtk_details}.

\paragraph*{Proof of claims~\cref{eq:gauss_kinf,eq:gauss_klip,eq:gauss_rmin,eq:tail_int_gauss}}
The claim~\cref{eq:gauss_kinf} follows directly from the definition of $\ksqrt$. The bound~\cref{eq:gauss_klip} on $\klip[\ksqrt]$ follows from the fact $\Vert{\frac{d}{dr}e^{-r^2/\gaussparam^2}}\Vert_{\infty} = \frac{\sqrt{2/e}}{\gaussparam}$ and invoking the relation~\cref{eq:derivative_relation}. Next, recalling the definition~\cref{eq:rmin_k} and noting that
\begin{talign}
\sup_{\substack{\x, \y: \\ \twonorm{\x-\y}\geq r}}\! \abss{\ksqrt(\x,\y)} \leq \parenth{\frac{2}{\pi\gaussparam^2}}^{\frac{d}{4}} e^{-r^2/\gaussparam^2}
\end{talign}
implies the bound~\cref{eq:gauss_rmin} on $\rmin_{\ksqrt, n}$. 

Next, we have 
\begin{talign}
\tail[\ksqrt]^2(R)& = \int_{\twonorm{z} \geq R} \parenth{\frac{2}{\pi\gaussparam^2}}^{\frac{d}{2}} \exp(-2\twonorm{z}^2/\gaussparam^2) dz
\\
&= \Pr_{X\sim\mc{N}(0, \gaussparam^2/4 \cdot \mbf{I}_d)}(\twonorm{X}\geq R) 
\sless{(i)}e^{-R^2/\gaussparam^2},
\label{eq:tail_gaussian}
\end{talign} 
for $R \geq \gaussparam \sqrt{2d}$, where step~(i) follows from the standard tail bound for a chi-squared random variable $Y$ with $k$ degree of freedom \citep[Lem.~1]{laurent2000adaptive}: $\Pr(Y-k>2\sqrt{kt} + 2t) \leq e^{-t}$, wherein we substitute $k=d$, $Y=\frac{4}{\gaussparam^2}\Vert X\Vert_2^2$, $t = d\alpha$ with $\alpha \geq 2$, and $R^2 = \gaussparam^2 t$. Using the tail bound \cref{eq:tail_gaussian}, the bound~\cref{eq:gauss_kinf} on $\infnorm{\ksqrt}$ and the definition~\cref{eq:rmin_k} of $\rktaugen[\ksqrt, \sqrt{n}]$, we find that
\begin{talign}
\rktaugen[\ksqrt, \sqrt{n}] =\order(\gaussparam \max(\sqrt{(\log n-\frac{d}{2}\log(\gaussparam{\sqrt{\frac{\pi}{2}}} ))_+}, \sqrt{d}))
=\order(\gaussparam \sqrt{d+\log n }).
\end{talign}
yielding the claim~\cref{eq:tail_int_gauss}.

\subsection{Proofs for \Matern kernel}
\label{sub:matern_kernel_quantities_proofs}
Recall the notations from~\cref{eq:def_wendmatern,eq:matern_ksqrt,eq:matern_ksqrt_const} for the \Matern kernel related quantities. Let $\bessel[b]$ denote the modified
Bessel function of the third kind with order $b$ \citep[Def.~5.10]{wendland2004scattered}.  Let $a \defeq \frac{\matone-d}{2}$. Then, the kernel $\kernel=\textbf{\Matern}(\matone,\mattwo)$ and its square-root kernel $\ksqrt=\matrtconst \cdot \textbf{\Matern}(a,\mattwo)$ satisfy
\begin{talign}
\label{eq:matern_orig_reparam}
\kernel(\x, \y) &= \wtil{\kappa}_{\matone-d/2}(\mattwo \twonorm{\x-\y}),
\qtext{and}
\\
\label{eq:matern_reparam}
\ksqrt(\x, \y) &= \matrtconst \wtil{\kappa}_{a}(\mattwo \twonorm{\x-\y}),
\\
\text{ where }
\wtil{\kappa}_b(r) &\defeq c_b r^{b} \bessel[b](r), \text{ and } 
c_{b} = \frac{2^{1-b}}{\Gamma(b)} \stext{for} b>0.
\label{eq:wtil_kappa_def}
\end{talign}
We claim the following bounds on various quantities assuming $a\geq2.2$, and $d\geq 2$:\footnote{When $a \in (1, 2.2)$, and $d\in [1, 2)$, analogous bounds with slightly different constants follow from employing the Gamma function upper bound of \citet[Thm~2.3]{batir2017bounds} in place of our upper bound \cref{eq:stirling_approx}. For brevity, we omit these derivations.}
\begin{talign}
  \infnorm{\ksqrt} &= \matrtconst 
  \qtext{and} \infnorm{\kernel} = 1,
  \label{eq:collect_matern_rt_kinf} \\
\matrtconst &\leq  5 \matone (\frac{\mattwo^2}{2\pi(a-1)})^{\frac{d}{4}},
\label{eq:collect_striling_matrtconst} \\
  \klip[\ksqrt] &\leq  \matrtconst  \frac{C_1 \mattwo}{\sqrt{a+C_2}} 
  \implies \frac{\klip[\ksqrt]}{\infnorm{\ksqrt}} \sless{\cref{eq:collect_matern_rt_kinf}}
  \frac{C_1\mattwo}{\sqrt{a+C_2}},
\label{eq:collect_matern_klip} \\
  \rk[\ksqrt] &= \frac{1}{\mattwo }  \order(\max(\log n - a\log
  (1+a), C_2 a\log (1+a))),
  \label{eq:collect_matern_rmin} \\
    \rktaugen[\ksqrt, \sqrt{n}] &= 
    \frac{1}{\mattwo}\cdot\order(a\!+\!\log n\!+\! d\log(\frac{\sqrt{2e\pi}}{\mattwo}) \!+\!  \log(\frac{(\matone-2)^{\matone-\frac{3}{2}}}{
    (2(a-1))^{2a-1}d^{\frac{d}{2}+1}}))
    \label{eq:collect_rdag_ksqrt_matern}
\end{talign}
We prove these claims in \cref{ssub:set_up_matern_proofs} through \cref{ssub:proof_of_the_bound_eq:collect_matern_rmin_on_}. Putting these bounds together with \cref{eq:scaling_of_a_M} yields that
\begin{talign}
    \err_{\ksqrt}^{\textbf{\Matern}}(n, \half\log_2n, d, \frac{\delta}{2n}, \delta',  R)
    &\precsim 
    \begin{cases}
    \sqrt{\log(\frac{n}{\delta}) \brackets{\log\frac1\delta + d\log \parenth{
    \frac{1}{\sqrt{1+a}} \cdot  (\log n\! +\! \mattwo R)\! } }}
    \ \text{ if } a = o(\log n)\\
    \sqrt{\log(\frac{n}{\delta}) \brackets{\log\frac1{\delta'} + d\log \parenth{\sqrt{a}\log(1+a) \!+\! \mattwo R \! } }}
    \ \text{ if } a =\Omega(\log n)
    \end{cases} \\
    &\precsim  
    \sqrt{\log(\frac{n}{\delta}) \brackets{\log\frac1{\delta'}+ d\log (\log n\! +\! B +\! \mattwo R)\!  }}
    \label{eq:matern_krt_err}
\end{talign}
with $B=a\log (1+a)$,
as claimed in \cref{table:sqrtk_details}.

\subsubsection{Set-up for proofs of \Matern kernel quantities}
\label{ssub:set_up_matern_proofs}
Before proceeding to the proofs of the claims~\cref{eq:collect_matern_rt_kinf,eq:collect_striling_matrtconst,eq:collect_matern_klip,eq:collect_matern_rmin,eq:collect_rdag_ksqrt_matern},
we collect some handy inequalities. Applying \citet[Lem.~5.13, 5.14]{wendland2004scattered},
we have
\begin{talign}
\label{eq:matern_basic_bound}
\wtil{\kappa}_{a}(\mattwo r) &\leq \min\parenth{1, \sqrt{2\pi} c_{a} (\mattwo r)^{a-\frac{1}{2}} e^{{-\mattwo r + \frac{a^2}{2\mattwo r}}}}
 \qtext{for} r>0.
 \end{talign}
 For a large enough $r$, we also establish the following bound:
 \begin{talign}
\label{eq:matern_alternate_bound}
\wtil{\kappa}_{a}(\mattwo r) &\leq  \min\parenth{1, 4 c_{a} (\mattwo r)^{a-1}e^{-\mattwo r/2}}
\qtext{for} \mattwo r \geq 2(a - 1), a\geq 1.
\end{talign}

\paragraph*{Proof of \cref{eq:matern_alternate_bound}}
Noting the definition~\cref{eq:wtil_kappa_def} of $\kappa_a$, it suffices to show the following bound:
\begin{talign}
\label{eq:bessel_bound}
\abss{\bessel[a](r)} \leq \frac{4}{r} e^{-r/2}
\qtext{for} r/2 \geq a-1,
\end{talign}
where $\bessel[a]$ is the modified Bessel function of the third kind.
Using \citet[Def.~5.10]{wendland2004scattered}, we have
\begin{talign}
\bessel[a](r) &= \frac12 \int_{0}^{\infty} e^{-r\cosh t} \cosh(at) dt
\sless{(i)} \int_{0}^{\infty} e^{-\frac{r}{2}e^t} e^{at} dt
\seq{(ii)} \int_{r/2}^{\infty} e^{-s} (\frac{2s}{r})^a \cdot  \frac1s ds
\sless{(iii)}  \frac{4}{r} e^{-r/2}
\end{talign}
where step~(i) uses the following inequalities: $\cosh t \defeq \frac12(e^t+e^{-t}) \geq\frac12 e^{t}$ and $\cosh(at) \leq e^{at}$ for $a>0, t>0$, step~(ii) follows from a change of variable $s\gets \frac{r}{2}e^t$, and finally step~(iii) uses the following bound on the incomplete Gamma function obtained by substituting $B=2$ and $A=a$ in \citet[Eq.~1.5]{borwein2009uniform}:
\begin{talign}
\label{eq:gamma_bound}
\int_{r}^{\infty} t^{a-1} e^{-t} dt \leq 2 r^{a-1} e^{-r} 
\qtext{for} r\geq a-1.
\end{talign}
The proof is now complete.

\subsubsection{Proof of the bound~\cref{eq:collect_matern_rt_kinf} on $
\infnorm{\ksqrt}$
and $\infnorm{\kernel}$}
\label{sub:proof_of_the_bound_eq:collect_matern_rt_kinf_on_}
We claim that
 \begin{talign}
 \label{eq:infnorm_wtil_kappa}
 \infnorm{\wtil{\kappa}_b} =1 \qtext{for all} b>0.
 \end{talign}
 where was defined in \cref{eq:wtil_kappa_def}. Putting the equality~\cref{eq:infnorm_wtil_kappa}
 with \cref{eq:matern_orig_reparam,eq:matern_reparam,eq:matern_basic_bound}
 immediately implies the bounds~\cref{eq:collect_matern_rt_kinf} on the
 $\infnorm{\ksqrt}$ and $\infnorm{\kernel}$.
 To prove \cref{eq:infnorm_wtil_kappa}, we follow the steps from the proof
 of \citet[Lem.~5.14]{wendland2004scattered}.
 Using \citet[Def.~5.10]{wendland2004scattered}, for $b>0$ we have
\begin{talign}
\bessel[b](r) &= \frac12 \int_{-\infty}^{\infty} e^{-r\cosh t} e^{bt}
dt
= \frac12\int_{-\infty}^{\infty} e^{-\frac{r}{2}(e^t+e^{-t})} e^{bt}
dt
= k^{-b} \frac12 \int_{0}^{\infty} e^{-r/2(u/k+k/u)} u^{b-1} du
\end{talign}
where the last step follows by substituting $u=ke^t$. Setting $k=r/2$, we
find that
\begin{talign}
  r^b \bessel[b](r) = 2^{b-1} \int_{0}^\infty e^{-u} e^{-r^2/(4u)} u^
  {b-1}
  du \leq 2^{b-1} \Gamma(b),
\end{talign}
where we achieve equality in the last step when we take the limit  $r \to 0$. Noting that
$\wtil{\kappa}_b(r) = \frac{2^{1-b}}{\Gamma(b)} r^b\bessel[b](r)$~\cref{eq:wtil_kappa_def} yields the claim~\cref{eq:infnorm_wtil_kappa}.

\subsubsection{Proof of the bound~\cref{eq:collect_striling_matrtconst}
 on $\matrtconst$} %
\label{ssub:proof_of_the_bound_eq:collect_striling_matrtconst_on_}
Using the definition~\cref{eq:matern_ksqrt_const} we have  $\matrtconst
= \parenth{\frac{1}{4\pi}\mattwo^2}^{\frac{d}{4}} \cdot \matrtconst'$,
where
\begin{talign}
\matrtconst' &= \sqrt{\frac{\Gamma(\matone)}{\Gamma(\matone-d/2)}} \cdot \frac{\Gamma((\matone-d)/2)}{\Gamma(\matone/2)}
\sless{\cref{eq:stirling_approx}} \sqrt{\frac{e}{\sqrt{2\pi}} \frac{(\frac{\matone-1}{e})^{\matone-1}
\sqrt{\matone-1}}{(\frac{\matone-\frac{d}{2}-1}{e})^{\matone-\frac{d}{2}-1}\sqrt{\matone-\frac{d}{2}-1}} }
\cdot\frac{e}{\sqrt{2\pi}} \frac{(\frac{\matone-d-2}{2e})^{\frac{\matone-d-2}{2}} \sqrt{\frac{\matone-d-2}{2}}}{(\frac{\matone-2}{2e})^{\frac{\matone-2}{2}}\sqrt{\frac{\matone-2}{2}}}  \\
&\leq \parenth{\frac{e^2}{2\pi}}^{\frac34} (2\sqrt{e})^{\frac{d}{2}}
\cdot \parenth{\frac{\matone-1}{\matone-2}}^{\frac{\matone-2}{2}} \cdot \sqrt{\frac{\matone-1}{\matone-2}} 
\cdot (\matone-1)^{\frac14} \cdot \parenth{\frac{\matone-d-2}{\matone-\frac{d}{2}-1}}^{\frac{\matone-\frac{d}{2}-2}{2}} 
\cdot \frac{\sqrt{(\matone-d-2)/\sqrt{(\matone-\frac{d}{2}-1)}}}
{\parenth{\matone-d-2}^{\frac12(\frac{d}{2}-1)}} \\
&\leq \parenth{\frac{e^2}{2\pi}}^{\frac34} (2\sqrt{e})^{\frac{d}{2}}
\cdot \sqrt{e} \cdot \sqrt{2} \cdot (\matone-1)^{\frac14} \cdot  (\sqrt{e})^{-\frac{d}{2}+1}
\cdot 
\frac{(\matone-d-2)/(\matone-\frac{d}{2}-1)^{\frac14}}
{\parenth{\matone-d-2}^{\frac{d}{4}}} \\
&\leq \frac{e^2\sqrt{2e}}{(2\pi)^{\frac34}}
(\matone-1)^{\frac14} \cdot (\matone-\frac{d}{2}-1)^{\frac34}
\parenth{\frac{\matone-d-2}{4}}^{-\frac{d}{4}}.
\end{talign}
As a result, we have
\begin{talign}
\matrtconst = \parenth{\frac{1}{4\pi}\mattwo^2}^{d/4} \cdot \matrtconst' \leq 5 \matone (\frac{\mattwo^2}{\pi(\matone-d-2)})^{\frac{d}{4}}
=  5 \matone (\frac{\mattwo^2}{2\pi(a-1)})^{\frac{d}{4}},
\end{talign}
as claimed.

\subsubsection{Proof of the bound~\cref{eq:collect_matern_klip} on $\klip[\ksqrt]$} %
\label{ssub:proof_of_the_bound_eq:collect_matern_klip_on_}
To derive a bound on the Lipschitz constant $\klip[\ksqrt]$, we bound the derivative of $\wtil{\kappa}_a$. Using $(r^a\bessel[a](r))' = -r^a\bessel[a-1](r)$ \citep[Eq.~5.2]{wendland2004scattered}, we find that
\begin{talign}
(\wtil{\kappa}_a(\mattwo r))'=
c_{a}\frac{d}{dr}((\mattwo r)^a\bessel[a](\mattwo r)) =-\frac{c_{a}}{c_{a-1}}\mattwo^2 r \wtil{\kappa}_{a-1}(\mattwo r).
\end{talign}
Using \cref{eq:matern_basic_bound}, we have
\begin{talign}
\frac{c_{a}}{c_{a-1}}\mattwo^2 r \wtil{\kappa}_{a-1}(\mattwo r) \leq\frac{c_{a} \mattwo}{c_{a-1}} \min(\mattwo r,
\sqrt{2\pi} c_{a-1} (\mattwo r)^{a-\frac12} e^{-\mattwo r + \frac{(a-1)^2}{2\mattwo r}})
\qtext{for} r>0.
\end{talign}
And \cref{eq:matern_alternate_bound} implies that
\begin{talign}
\frac{c_{a}}{c_{a-1}}\mattwo^2 r \wtil{\kappa}_{a-1}(\mattwo r) \leq\frac{c_{a} \mattwo}{c_{a-1}} \min\parenth{\mattwo r,
4 c_{a-1} (\mattwo r)^{a-1} e^{-\mattwo r/2} }
\qtext{for} r\geq 2(a-2),a\geq 2.
\end{talign}
Next, we make use of the following observations: The function $t^{b} e^{-t/2}$ achieves its maximum at $t=2b$. Then for any function $\fun:\R_+\to\R_+$ such that $\fun(t)\leq t$ for all $t\geq 0$ and $\fun(t)\leq \min(t, C t^{b} e^{-t/2})$ for $t>t^\dagger$ with $t^\dagger < 2b$, we have
\begin{talign}
\sup_{t\geq 0}\fun(t) \leq \min(2b, C (2b)^b e^{-b}).
\end{talign}
As a result, we conclude that
\begin{talign}
\label{eq:Lk_interim_bound}
\sup_{r\geq 0}\frac{c_{a}}{c_{a-1}}\mattwo^2 r \wtil{\kappa}_{a-1}(\mattwo r)
\leq \frac{c_{a} \mattwo}{c_{a-1}} \min\parenth{2(a-1),
4 c_{a-1} (2(a-1))^{a-1} e^{-a+1} }.
\end{talign}
Substituting the value of $c_{a-1}$ and the bound~\cref{eq:stirling_approx}, we can bound the second argument inside the min on the RHS of \cref{eq:Lk_interim_bound} (for $a\geq 3$) as follows:
\begin{talign}
\label{eq:Lk_matern_interim_bound_two}
4c_{a-1} (2(a-1))^{a-1} e^{-a+1}
&\leq 4 \cdot 2^{2-a} \frac{1}{\sqrt{2\pi(a-2)}} (\frac{e}{a-2})^{a-2} 2^{a-1} (a-1)^{a-1} e^{1-a} \\
&\leq \frac{8}{e\sqrt{2\pi}}  (1+\frac{1}{a-2})^{a-2} \parenth{\sqrt{a-2} + \frac{1}{\sqrt{a-2}}} 
\leq \frac{8}{\sqrt{\pi}} \sqrt{a-2}
\end{talign}
for $a \geq 3$. When $a\in [2, 3)$, one can directly show that 
$4c_{a-1} (2(a-1))^{a-1} e^{-a+1} \leq 8/e$. 
Putting the pieces together, and noting that $\frac{c_{a}}{c_{a-1}} = \frac{1}{\max(2(a-1), 1)}$, we find that
\begin{talign}
\sup_{r\geq 0}\frac{c_{a}}{c_{a-1}}\mattwo^2 r \wtil{\kappa}_{a-1}(\mattwo r)
&\leq \frac{\mattwo}{\max(2(a-1), 1)} \min\parenth{2(a-1),
\frac{8}{e} + \frac{8}{\sqrt{\pi}}\sqrt{a-2}
} 
\leq  \frac{C_1 \mattwo}{\sqrt{a+C_2}}
\end{talign}
for any $a \geq 2$. And hence
\begin{talign}
\klip[\ksqrt] \leq \matrtconst
\sup_{r\geq 0}|\kappa_a'(\mattwo r)| 
    \leq \matrtconst  \frac{C_1 \mattwo}{\sqrt{a+C_2}} 
\implies \frac{\klip[\ksqrt]}{\infnorm{\ksqrt}} \sless{\cref{eq:collect_matern_rt_kinf}} \frac{C_1\mattwo}
{\sqrt{a+C_2}},
\end{talign}
as claimed.

\subsubsection{Proof of the bound~\cref{eq:collect_matern_rmin} on $\rmin_{\ksqrt, n}$} %
\label{ssub:proof_of_the_bound_eq:collect_matern_rmin_on_}
Using arguments similar to those used to obtain~\cref{eq:Lk_matern_interim_bound_two}, we find that
\begin{talign}
\max_{\mattwo r\geq 2(a-1)}\wtil{\kappa}_{a}(\mattwo r)\leq  4c_a (2(a-1))^{a-1} e^{-(a-1)}
\leq \sqrt{\frac{4}{a-1}} = \sqrt{\frac{8}{\matone-d-2}}.
\end{talign}
Next, note that
\begin{talign}
\label{eq:matern_tail_bound}
(\mattwo r)^{a-\frac12} e^{-\mattwo r+\frac{a^2}{2\mattwo r}} &\precsim e^{-\mattwo r/2} \qtext{for} \mattwo r \succsim a\log (1+a), \\
c_a &= \frac{2^{1-a}}{\Gamma(a)} \sless{\cref{eq:stirling_approx}} (\frac{2e}{a-1})^{a-1} \frac{1}{\sqrt{2\pi(a-1)}},
\label{eq:ca_bound}
\end{talign}
where \cref{eq:matern_tail_bound} follows from standard algebra.
Thus, we have
\begin{talign}
c_{a}\wtil{\kappa}_{a}(\mattwo r) &\precsim c_{a} \exp(-\mattwo r/2) \qtext{for} \mattwo r\succsim   a\log (1+a) \\
\quad \stackrel{\cref{eq:stirling_approx}}{\Longrightarrow} \quad  \rmin_
{\ksqrt, n} &\precsim \frac{1}{\mattwo }  \max(\log n - a\log (1+a),  C_1
a\log (1+a)), 
\end{talign}
as claimed.

\subsubsection{Proof of the bound~\cref{eq:collect_rdag_ksqrt_matern} on
 $\rktaugen[\ksqrt, \sqrt{n}]$} %
\label{ssub:proof_of_the_bound_eq:collect_rdag_ksqrt_matern_on_}
Let $V_d = \pi^{d/2}/\Gamma(d/2+1)$ denote the volume of unit Euclidean ball in $\Rd$. Using \cref{eq:matern_reparam}, we have
\begin{talign}
   \frac{1}{\matrtconst^2}\tail[\ksqrt]^2(R) = \int_{\twonorm{z}\geq R} \wtil{\kappa}_a^2(\mattwo \twonorm{\z}) d\z 
   &= dV_d \int_{r > R} r^{d-1}\wtil{\kappa}_a^2(\mattwo r) dr \\
   &\sless{\cref{eq:matern_basic_bound}} 2\pi c_{a}^2 dV_d \int_{r > R} r^{d-1} (\mattwo r)^{2a-1} \exp({-2\mattwo r + \frac{a^2}{\mattwo r}})dr \\
   &= 2\pi c_{a}^2 dV_d \mattwo^{1-d} \int_{r > R} (\mattwo r)^{\matone-2} \exp({-2\mattwo r + \frac{a^2}{\mattwo r}})dr
   \label{eq:int_tail_matern_one}
\end{talign}
where we have also used the fact that $2a + d = \matone$. Next, we note that
\begin{talign}
\label{eq:exp_bound_simplified}
\exp(-2\mattwo r + \frac{a^2}{\mattwo r}) \leq \exp(-\frac32 \mattwo r)
\qtext{for} \mattwo r \geq \sqrt{2} a
\end{talign}
For any integer $b>1$, noticing that $f(t)= \frac1{\Gamma(b)} t^{b-1}e^{-t}$ is the density function of an Erlang distribution with shape parameter $b$ and rate parameter $1$, and using the expression for its complementary cumulative distribution function, we find that
\begin{talign}
\label{eq:IGF_formula}
\int_{r}^{\infty} t^{b-1} e^{-t} dt =  \Gamma(b) \sum_{i=0}^{b-1}\frac{r^i}{i!} e^{-r}
\end{talign}
Thus, for $R>\frac{\sqrt{2}}{\mattwo}a$, we have
\begin{talign}
\int_{r > R} (\mattwo r)^{\matone-2} \exp({-2\mattwo r + \frac{a^2}{\mattwo r}})dr
&\sless{\cref{eq:exp_bound_simplified}} \int_{r > R} (\mattwo r)^{\matone-2} \exp({-\frac32\mattwo r})dr \\
&= \frac{1}{\mattwo} \parenth{\frac23}^{\matone-1} \int_{3\mattwo R/2}^{\infty} t^{\matone-2} e^{-t}dt \\
&\seq{\cref{eq:IGF_formula}} \frac{1}{\mattwo} \parenth{\frac23}^{\matone-1} \Gamma(\matone-1) e^{-\frac32\mattwo R} \sum_{i=0}^{\matone-1} \frac{(3\mattwo R/2)^i}{i!} \\
&\leq \frac{1}{\mattwo} \Gamma(\matone-1) e^{-\frac32\mattwo R} e^{\mattwo R}
\leq \frac{1}{\mattwo} \Gamma(\matone-1) e^{-\frac12\mattwo R}.
\label{eq:int_tail_matern_two}
\end{talign}
Putting the bounds~\cref{eq:int_tail_matern_one,eq:int_tail_matern_two} together for $ R \geq \frac{\matone-d}{\sqrt2 \mattwo}$, we find that
\begin{talign}
\tail[\ksqrt]^2(R) &\leq \matrtconst^2  \cdot  \mattwo^{-d} \cdot 2\pi \frac{2^{2-2a}}{\Gamma^2(a)} \cdot  \frac{\pi^{\frac{d}{2}}}{\Gamma(\frac{d}{2}+1)} \cdot \Gamma(\matone-1) \exp(-\frac12\mattwo R).
\end{talign}
Using \cref{eq:stirling_approx}, we have
\begin{talign}
2\pi \frac{2^{2-2a}}{\Gamma^2(a)} \cdot  \frac{\pi^{\frac{d}{2}}}{\Gamma(\frac{d}{2}+1)} \cdot \Gamma(\matone-1)
&\leq 2\pi \cdot \frac{2^{2-2a} e^{2(a-1)}}{2\pi(a-1) (a-1)^{2(a-1)}}
\cdot \frac{\pi^{d/2} (2e)^{d/2}}{ \sqrt{\pi d} \cdot d^{d/2}}
\cdot e\sqrt{\matone-2} (\frac{\matone-2}{e})^{\matone-2} \\
&= \frac{2}{\sqrt \pi}  e^{2a-2+d/2+1-\matone} (2\pi)^{d/2} (2a-2)^{-(2a-1)} d^{-d/2-1} (\matone-2)^{\matone-3/2} \\
&= \frac{2}{e\sqrt \pi} (2e\pi)^{d/2} (\matone-d-2)^{-(\matone-d-1)} d^{-d/2-1} (\matone-2)^{\matone-3/2}\\
&= \frac{2}{e\sqrt \pi} (2e\pi)^{d/2} (2(a-1))^{-(2a-1)} d^{-d/2-1} (\matone-2)^{\matone-3/2}.
 \label{eq:int_matern_tail_three}
\end{talign}
Putting the pieces together, and solving for $\tail[\ksqrt](R) \leq \infnorm{\ksqrt}/\sqrt{n} \seq{\cref{eq:collect_matern_rt_kinf}} \matrtconst/\sqrt{n}$, we find that $\rktaugen[\ksqrt, \sqrt{n}]$~\cref{eq:rmin_k} can be bounded as
\begin{talign}
\label{eq:rdag_ksqrt_matern}
    \rktaugen[\ksqrt, \sqrt{n}] &\leq \frac{2}{\mattwo} \cdot\max (\frac{a}{\sqrt 2}, \log n\!+\!\log(\frac{2}{e\sqrt{\pi}}) \!+\! d\log(\frac{\sqrt{2e\pi}}{\mattwo}) \!+\!  \log(\frac{(\matone-2)^{\matone-\frac{3}{2}}}{
    (2(a-1))^{2a-1}d^{\frac{d}{2}+1}}))\\
    &\precsim\frac{1}{\mattwo}\cdot(a\!+\!\log n\!+\! d\log(\frac{\sqrt{2e\pi}}{\mattwo}) \!+\!  \log(\frac{(\matone-2)^{\matone-\frac{3}{2}}}{
    (2(a-1))^{2a-1}d^{\frac{d}{2}+1}}),
\end{talign}
as claimed.

\subsection{Proofs for B-spline kernel}
\label{sub:bspline_kernel_quantities_proofs}
Recall the notation from \cref{eq:spline_all_constants,eq:splinefun,eq:splinefun_kernel,eq:splinefun_krt}. Then, the kernel $\kernel=\textbf{B-spline}(2\splineparam+1)$ and its square-root kernel $\ksqrt=\wtil{S}_{\splineparam, d} \cdot \textbf{B-spline}(\splineparam)$ satisfy
\begin{talign}
\label{eq:spline_krt_inf}
\infnorm{\ksqrt} &= \wtil{S}_{\splineparam, d}
\seq{(i)} c_\splineparam^d
\qtext{where} c_\splineparam\defeq \frac{\tempsplineconst[\splineparam+1]}{\sqrt{\tempsplineconst[2\splineparam+2]}} \begin{cases}
\seq{(ii)} \frac{2}{\sqrt 3} \text{ when }\splineparam=1 \\
\sless{(iii)}1 \text{ when }\splineparam>1
\end{cases} 
, \\
\qquad \infnorm{\kernel} &= 1,  \label{eq:spline_k_inf}\\
\frac{\klip[\ksqrt]}{\infnorm{\ksqrt}} &\leq \frac43 \sqrt{d},
\label{eq:spline_klip}\\
\label{eq:rmin_rdag_spline}
\rmin_{\ksqrt, n} &\leq  \sqrt{d} (\splineparam+1)/2,
\qtext{and}
\rktaugen[\ksqrt, \sqrt{n}] \leq  \sqrt{d} (\splineparam+1)/2.
\end{talign}
While claims~(i) and (ii) follow directly from the definitions in the display~\cref{eq:spline_all_constants}, claim~(iii) can be verified numerically, e.g., using SciPy~\citep{2020SciPy-NMeth}.
From numerical verification, we also find that the constant $c_{\splineparam}$ in \cref{eq:spline_krt_inf} is decreasing with $\splineparam$.
See \cref{ssub:spline_param_proof} for the proofs of the remaining claims.
Finally, substituting various quantities from \cref{eq:spline_krt_inf,eq:spline_klip,eq:rmin_rdag_spline} in \cref{eq:scaling_of_a_M}, we find that
\begin{talign}
\err_{\ksqrt}^{\textbf{B-spline}}(n,\half\log_2n, d,  \frac{\delta}{2n}, \delta', R)
    &\precsim
    \sqrt{ \log(\frac{n}{\delta}) \parenth{ \log(\frac{1}{\delta})
    +d \log (d\splineparam+\sqrt{d}R)} }.
\end{talign}

\subsubsection{Proofs of the bounds on B-spline kernel quantities}
\label{ssub:spline_param_proof}
We start with some basic set-up.
Consider the (unnormalized) univariate B-splines 
\begin{talign}
\label{eq:unispline}
\fun_{\splineparam}:\real \to [0, 1]
\qtext{with}
\fun_{\splineparam}(a)=\circledast^{\splineparam}\indicator_{[-\frac12,\frac12]}(a)
\seq{(i)} \frac{1}{(\splineparam-1)!}\sum_{j=0}^{\splineparam} (-1)^{j} {\splineparam \choose j} (a+\frac{\splineparam}{2}-j)_+^{\splineparam-1}
\end{talign}
where step~(i) follows from \citet[Eqn 4.46, 4.47, 4.59, p.135, 138]{schumaker2007spline}. Noting that $\fun_{\splineparam}$ is an even function with a unique global maxima at 0 (see \citet[Ch 4.]{schumaker2007spline} for more details), we find that
\begin{talign}
\infnorm{\fun_{\splineparam}} = \fun_{\splineparam}(0) = 
\frac{1}{(\splineparam-1)!}  \sum_{j=0}^{\lfloor\splineparam/2\rfloor} (-1)^j {\splineparam \choose j} (\frac{\splineparam}{2} - j)^{\splineparam-1}
\seq{\cref{eq:spline_all_constants}} \tempsplineconst[\splineparam].
\label{eq:unisplineinf}
\end{talign}

\paragraph*{Bounds on $\infnorm{\kernel}$ and $\sinfnorm{\ksqrt}$}
Recalling \cref{eq:splinefun,eq:splinefun_kernel}, we find that
\begin{talign}
\infnorm{\kernel} = \infnorm{\splinefun[{2\splineparam+2}]} =
\splineconst[2\splineparam+2, d] \infnorm{\fun_{2\splineparam+2}}^d = \splineconst[2\splineparam+2, d]  \tempsplineconst[2\splineparam+2]^d =  1,
\end{talign}
thereby establishing~\cref{eq:spline_k_inf}.

\paragraph*{Bounds on $\klip[\ksqrt]$}
We have
\begin{talign}
\fun'_{\splineparam+1}(a) = \int \fun_{\splineparam}(b)   \frac{d}{da}\indicator_{[-\half, \half]} (a-b) d\y
= \fun_{\splineparam}(a-\half) -\fun_{\splineparam}(a+\half)
\end{talign}
and hence $\sinfnorm{\fun'_{\splineparam+1}}=\sup_{a}|\fun_{\splineparam}(a-\half) -\fun_{\splineparam}(a+\half) | \leq \sinfnorm{\fun_{\splineparam}}$ since $\fun_{\splineparam}$ is non-negative.
Putting the pieces together, we have
\begin{talign}
\frac{\klip[\ksqrt]}{\sinfnorm{\ksqrt}} = \frac{1}{ \wtil{S}_{\splineparam, d} }\klip[\ksqrt] \leq\sup_{\z} \twonorm{\nabla\splinefun[{\splineparam+1}](\z)}
&\sless{\cref{eq:splinefun}}\sqrt{d} \cdot  \splineconst[\splineparam+1,d]  \sinfnorm{\fun_{\splineparam+1}'} \cdot \sinfnorm{\fun_{\splineparam+1}}^{d-1}  \\
&\sless{(\ref{eq:spline_all_constants},\ref{eq:unisplineinf})} \sqrt{d} \cdot  \tempsplineconst[\splineparam+1]^{-d}
\cdot \tempsplineconst[\splineparam] \cdot \tempsplineconst[\splineparam+1]^{d-1}\\
&= \sqrt{d} \frac{\tempsplineconst[\splineparam]}{\tempsplineconst[\splineparam+1]}
\sless{(i)}\sqrt{d} \frac{4}{3},
\label{eq:lip_inf_ratio_spline}
\end{talign}
where step~(i) can be verified numerically.

\paragraph*{Bounds on $\rmin_{\ksqrt, n}$, and $\rktaugen[\ksqrt, \sqrt{n}]$}
Using the property of convolution, we find that $\fun_{\splineparam+1}(a)=0$ if $\abss{a} \geq \frac12(\splineparam+1)$.
Hence $\kappasqrt(\z) = 0$ for $\infnorm{\z} > (\splineparam+1)/2$ and applying the definitions~\cref{eq:rmin_k}, we find that
\begin{talign}
\rmin_{\ksqrt, n} \leq  \sqrt{d} (\splineparam+1)/2
\qtext{and}
\rktaugen[\ksqrt, \sqrt{n}] \leq  \sqrt{d} (\splineparam+1)/2
\end{talign}
as claimed in \cref{eq:rmin_rdag_spline}. 

\subsection{Explicit MMD rates for common kernels}
\label{sub:explicit_mmd_rates}
Putting the quantities from \cref{table:sqrtk_details} together with \cref{theorem:main_result_all_in_one} (with $\delta'=\delta$, and $\delta$ treated as a constant) yields the MMD rates summarized in \cref{table:mmd_k_explicit}.

For completeness, we illustrate a key simplification that can readily yield the results stated in \cref{table:mmd_k_explicit}.
Define $\rtemp$ as follows:
\begin{talign}
\label{eq:A_rdag_sqrt_def}
\rtemp \defeq \frac{1}{d}\max(\rktaugen[\ksqrt, \sqrt{n}], \rminpn[\inputcoreset])^2, 
\end{talign}
where $ \rminpn$, and $\rktaugen[\ksqrt, \sqrt{n}]$ were defined in \cref{eq:rmin_P,eq:rmin_k} respectively.
Then applying \cref{theorem:main_result_all_in_one}, and substituting the simplified bound~\cref{eq:scaling_of_a_M} in \cref{eq:mmd_thinning_bound_finite}, we find that
\begin{talign}
\mmdkt 
&\!\leq\! 2\frac{\infnorm{\ksqrt}}{\sqrt n}\!+\! \frac{\infnorm{\ksqrt}}{\sqrt n} (B \rtemp)^{\frac{d}{4}} \!\cdot\!  d^{-\frac14} \sqrt{\log \frac{n}{\delta} \!\cdot\! \brackets{ \log(\frac{1}{\delta}) \!+\! d\log \parenth{\frac{\klip[\ksqrt](\rmin_{\ksqrt, n}+R)}{\infnorm{\ksqrt}} } }} \\
&\stackrel{(i)}{\precsim_{\delta}} 2\frac{\infnorm{\ksqrt}}{\sqrt n} + \infnorm{\ksqrt}  (B \rtemp)^{\frac{d}{4}} \cdot  d^{\frac14} \sqrt{ \frac{\log n}{n}  \log \parenth{\frac{\klip[\ksqrt](\rmin_{\ksqrt, n}+ \rminpn[\inputcoreset])}{\infnorm{\ksqrt}}}  },
\label{eq:simplified_mmd_bound_general}
\end{talign}
with probability at least $1-2\delta$, where $B\defeq8e\pi$ is a universal constant, and
in step~(i) we use the following bound: For any $r = \sqrt{\alpha d}$, we have
\begin{talign} 
\label{eq:cd_simplified} 
c_d r^{\frac{d}{2}} = \frac{(4\pi)^{\frac{d}{4}}}{(\Gamma(\frac{d}{2}+1))^{\frac12}} (\alpha d)^{\frac{d}{4}} \sless{(ii)}  \frac{(8e\pi \alpha )^{\frac{d}{4}}}{(\pi d)^{\frac14}} \leq (B \alpha)^{\frac{d}{4}} d^{-\frac14} \qtext{where} B \defeq 8e\pi, 
\end{talign}
and step~(ii) follows (for $d\geq 2$) from the Gamma function bounds~\cref{eq:stirling_approx}. Now the results in \cref{table:mmd_k_explicit} follow by simply substituting the various quantities from \cref{table:sqrtk_details} in \cref{eq:simplified_mmd_bound_general}.

\newcommand{\explicitguaranteetablecaption}{
{\noindent\caption{
    \tbf{Explicit kernel thinning MMD guarantees for common kernels.}
    Here, $a\defeq\frac12(\matone\!-\!d)$, $G_{\matone,d}\defeq \log(\frac{(\matone-2)^{\matone-\frac{3}{2}}}{
    (2(a-1))^{2a-1}d^{\frac{d}{2}+1}})$ and each ${C_i}$ denotes a universal constant. 
    See \cref{sub:explicit_mmd_rates} for more details on deriving these bounds from \cref{theorem:main_result_all_in_one}.     \label{table:mmd_k_explicit}
}}}

\begin{table}[t!]
    \centering
  \resizebox{\textwidth}{!}
  {
\small
  {
    \renewcommand{\arraystretch}{1}
    \begin{tabular}{ccccc}
        \toprule
        \Centerstack{\bf Kernel $\kernel$}

        & \Centerstack{
        $\mmd_{\kernel}(\inputcoreset,\ktcoreset)$ 
        $\precsim$
        } 
        
        \\
        \midrule 
        
         \Centerstack{
      		$\textbf{Gaussian}(\gaussparam)$
      		}
    
    	& \Centerstack{
             $C_1^d \cdot \sqrt{\frac{\log n}{n}  \cdot [1\!+\!\frac1d (\log n\! +\!(\frac{\rminpn}{\gaussparam})^2)]^{\frac{d}{2}}  \cdot \log(\sqrt{\log n} \!+\! \frac{\rminpn}{\gaussparam} )  }$
      		}	
        
          \\[6mm]

        \Centerstack{\textbf{\Matern}$(\matone, \mattwo)$} 
        
          & \Centerstack{
            $C_2^d \cdot \sqrt{\frac{\matone^2\log n}{n} \brackets{ \frac{1}{d(a-1)}
 (a+\log^2n\! +\! d^2\log^2(\frac{\sqrt{2e\pi}}{\mattwo})\!+\!
 G^2_{\matone,d}\!+\! 
 \mattwo^2\rminpn^2)}^{\frac{d}{2}} \cdot \log (\log n\!+\! a \!+\! \mattwo \rminpn)}$
      		}
    
          \\[6mm]

          \Centerstack{
         $\textbf{B-spline}(2\splineparam+1)$}
         
         & \Centerstack{
      		$C_3^d \cdot \sqrt{\frac{\log n}{n} \cdot 
 [\splineparam^2+\frac{\rminpn^2}{d}]^{\frac{d}{2}} \cdot  \log (\splineparam +\frac{\rminpn}{\sqrt d}) }$
      		}

        \\[1ex] \bottomrule \hline
    \end{tabular}
    }
    }
     \opt{arxiv}{\explicitguaranteetablecaption}
     \opt{jmlr}{\explicitguaranteetablecaption}
\end{table}

    \section{Supplementary Details for Vignettes of \lowercase{\Cref{sec:vignettes}}}
\label{sec:vignettes_supplement}
This section provides supplementary details for the vignettes of \cref{sec:vignettes}.

\subsection{Mixture of Gaussians target}
\label{sec:mog_supplement}
Our mixture of Gaussians target is given by 
 $\P = \frac{1}{M}\sum_{j=1}^{M}\mc{N}(\mu_j, \mbf{I}_d)$ for $M \in \braces{4, 6, 8}$ where
    \begin{talign}
    \label{eq:mog_description}
    \mu_1 &= [-3, 3]\tp,  \quad \mu_2 = [-3, 3]\tp,  \quad  \mu_3 = [-3, -3]\tp,  \quad  \mu_4 = [3, -3]\tp,\\
    \mu_5 &= [0, 6]\tp, \qquad \mu_6 = [-6, 0]\tp,  \quad  \mu_7 = [6, 0]\tp,  \qquad  \mu_8 = [0, -6]\tp.
    \end{talign}

\subsection{Empirical decay rates}
\label{sub:decay_rates}
In \cref{fig:decay_rates}, we provide results for regressing the true error rate of $\sqrt{\frac{\log n}{n}}$ and $\frac{\log n}{\sqrt n}$ to the input size $n$ (on the log-log scale as in all of mean MMD plots) for various ranges of coreset sizes. Notably, the observed empirical rates are significantly slower than $n^{-\half}$ for the smallest coreset range (the same coreset range used in the panels in \cref{fig:mmd_gaussp_plots,fig:mmd_gaussp_plots}) but approach $n^{-\half}$ as the coreset size increases.
Hence, the observed empirical rates of \cref{fig:mmd_gaussp_plots} are consistent with $\frac{\log^c n}{\sqrt{n}}$ rates of decay.

\begin{figure}[th!]
    \centering
    \subfigure[True error scaling of $\sqrt{\frac{\log n}{n}}$]{\label{fig:sqrtlog_decay}%
    \resizebox{\linewidth}{!}{
    \begin{tabular}{c}
    \includegraphics[width=\linewidth]{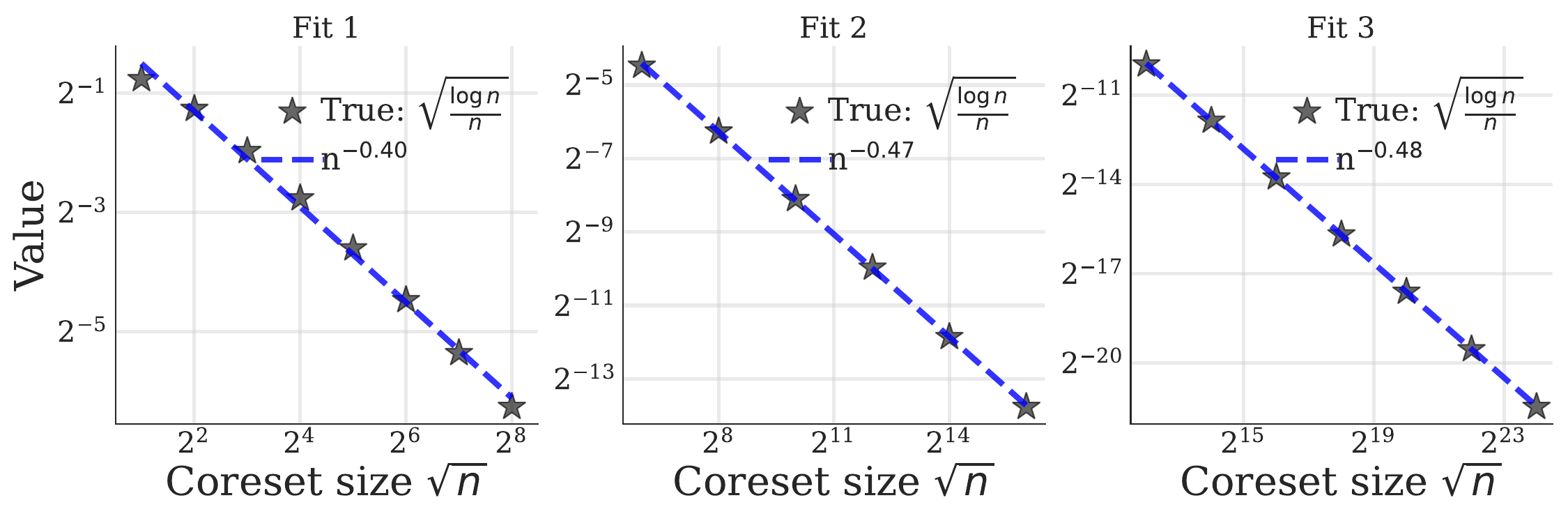}
    \end{tabular}
    }
    }
    \vspace{1mm}
    \hrule
    \vspace{3mm}
    \subfigure[True error scaling of $\frac{\log n}{\sqrt{n}}$ ]{\label{fig:log_decay}%
      \resizebox{\linewidth}{!}{
    \begin{tabular}{c}
    \includegraphics[width=\textwidth]{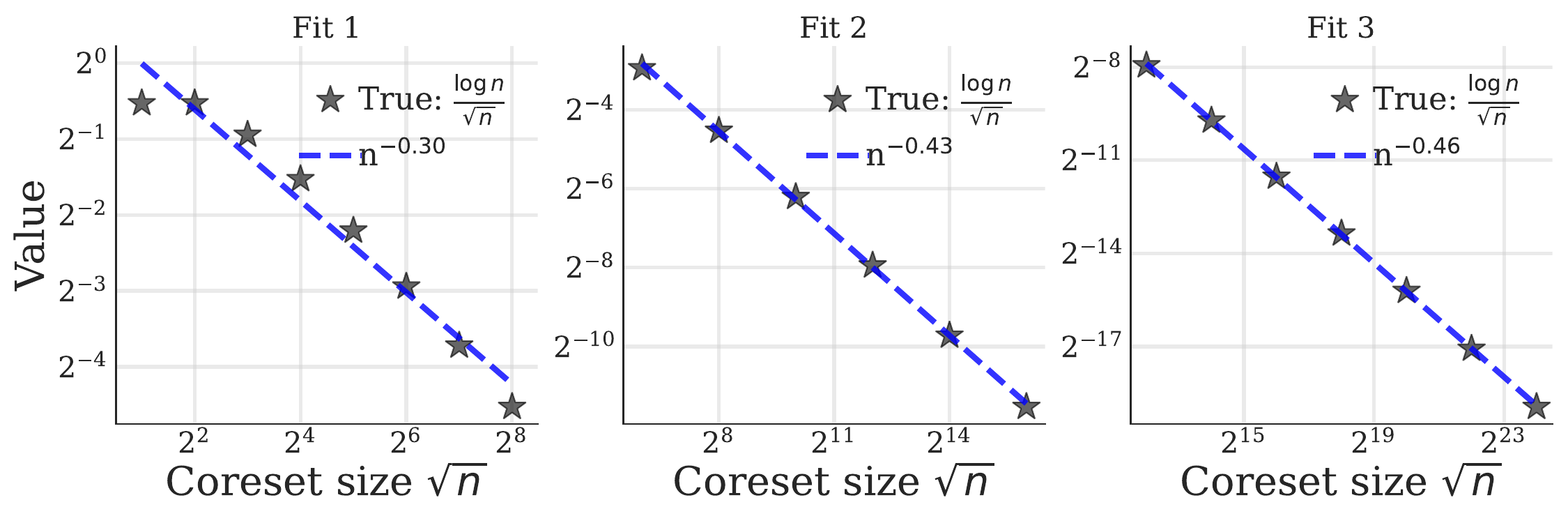}
    \end{tabular}
    }
    }
     \caption{\tbf{Empirical decay rates when true error $= \frac{\log^c n}{\sqrt n}$.} We display the ordinary least squares fits for regressing the log of true error onto the log of input size $n$ when the true error equals either (a) $\sqrt{\frac{\log n}{n}}$ or (b) $\frac{\log n}{\sqrt{n}}$. As the input range increases, the fitted decay rate tends towards $n^{-\half}$ despite appearing significantly slower for smaller input ranges.
    }
    \label{fig:decay_rates}
\end{figure}

\subsection{MCMC vignette details}
\label{sec:mcmc_supplement}
For complete details on the targets and sampling algorithms we refer the reader to \citet[Sec. 4]{riabiz2021optimal}. 
When applying standard thinning to any Markov chain output, we adopt the convention of keeping the final sample point. 
For all experiments, we used only the odd indices of the post burn-in sample points when thinning to form $\inputcoreset$.

The selected burn-in periods for the Goodwin task were 820,000 for RW; 824,000 for adaRW; 1,615,000 for MALA; and 1,475,000 for pMALA. The respective numbers for the Lotka-Volterra task were 1,512,000 for RW; 1,797,000 for adaRW; 1,573,000 for MALA; and 1,251,000 for pMALA.  
For the Hinch experiments, we discard the first $10^6$ points as burn-in following \citet[App.~S5.4]{riabiz2021optimal}. For all $n$, the parameter $\gaussparam$ for the Gaussian kernel is set to the median distance between all pairs of $4^7$ points obtained by standard thinning the post-burn-in odd indices. The resulting values for the Goodwin chains were 0.02 for RW; 0.0201 for adaRW;
0.0171 for MALA; and 0.0205 for pMALA. The respective numbers for Lotka-Volterra task were   0.0274 for RW; 0.0283 for adaRW;
0.023 for MALA; and 0.0288 for pMALA. Finally, the numbers for the Hinch task were 8.0676 for RW; 8.3189 for adaRW;
 8.621 for MALA; and 8.6649 for pMALA.

\section{Kernel Thinning with Square-root Dominating Kernels} %
\label{sub:guarantees_with_approximate_square_root_kernels}
As alluded to in \cref{sec:setup}, it is not necessary to identify an exact square-root kernel $\ksqrt$
to run kernel thinning. Rather, it suffices to identify any \emph{square-root dominating kernel} $\ksqrtdom$ defined as follows. 
\begin{definition}[Square-root dominating kernel]
\label{def:square_root_dom_kernel}
	We say a reproducing kernel $\ksqrtdom:\Rd\times \Rd\to\R$ is a \emph{square-root dominating kernel} for a reproducing kernel $\kernel:\Rd \times \Rd\to\R$ if $\ksqrtdom(x,\cdot)$ is square-integrable for all $x\in\Rd$ and either of the following equivalent conditions hold.
	\begin{enumerate}[label=(\alph*)]
	    \item\label{item:rkhs_dom_containment} The RKHS of $\kernel$ belongs
       to the RKHS of $\kdom(\x, \y) \defeq \int_{\Rd}\ksqrtdom(\x, \z)\ksqrtdom(\y, \z) d\z$.
    	\item\label{item:kernel_difference_psd} The function $\kdom - c \kernel$
       is positive definite for some $c > 0$.
	\end{enumerate}
\end{definition}
\begin{remark}[Controlling $\mmd_{\kernel}$]
\label{rem:square_root_dom_kernel}
A square-root dominating kernel $\ksqrtdom$ is a suitable surrogate for $\ksqrt$ as, by \citet[Lem.~2.2, Prop.~2.3]{zhang2013inclusion}, $\mmd_{\kernel}(\mu,\nu) \leq \sqrt{1/c} \cdot \mmd_{\kdom}(\mu,\nu)$ for $c$ the constant appearing in \cref{def:square_root_dom_kernel}\ref{item:kernel_difference_psd} and all distributions $\mu$ and $\nu$.
\end{remark}

\cref{def:square_root_dom_kernel,rem:square_root_dom_kernel} enable us to use convenient dominating surrogates whenever an exact square-root kernel is inconvenient to derive or deploy. 
For example, our next result, proved in \cref{sec:proof_of_matern_sqrt_dom}, shows that a standard \Matern kernel is a square-root dominating kernel for every sufficiently-smooth shift-invariant and absolutely integrable $\kernel$.
\newcommand{\maternsqrtdomname}{Dominating smooth kernels}
\begin{proposition}[Dominating smooth kernels]
\label{matern_sqrt_dom}
If $\kernel(\x,\y) = \kappa(\x-\y)$ and $\kappa \in \lone \cap C^{2\matone}$ for $\matone > d$, then, for any $\mattwo>0$, the \textbf{Mat\'ern}$(\frac{\matone}{2}, \mattwo)$ kernel of \cref{table:kernel_sqrt_pair} is a square-root dominating kernel for $\kernel$.
\end{proposition}

Checking the square-root dominating condition is also particularly simple for any pair of continuous shift-invariant kernels as the next result, proved in \cref{proof_of_square_root_dom_kernel}, shows.
\newcommand{\domsqrtresultname}{Dominating shift-invariant kernels}
\begin{proposition}[\domsqrtresultname]
\label{square_root_dom_kernel}
Suppose that the real-valued kernels 
$\kernel(\x,\y) = \kappa(\x-\y)$
and
$\ksqrtdom(\x,\y) = \kappasqrtdom(\x-\y)$
have respective spectral densities (\cref{def:spectral_density})
$\hatkappa$ and $\hatkappasqrt$.
If   $\hatkappasqrt\in\ltwo$,  
then
$\ksqrtdom$ 
is a square-root dominating kernel for $\kernel$ if and only if 
\begin{align}\label{eq:inf_norm_fourier_ratio}
    \esssup_{\omega\in\Rd:\,\hatkappa(\omega) > 0} \textstyle{\frac{\hatkappa(\omega)}{\hatkappasqrt(\omega)^2}} < \infty.
\end{align}
\end{proposition}
In \cref{table:sqrt_dom_pair}, we use \cref{square_root_dom_kernel} to derive convenient tailored square-root dominating kernels $\ksqrtdom$ for standard inverse multiquadric kernels, hyberbolic secant (sech) kernels, and Wendland's compactly supported kernels.
In each case, we can identify a square-root dominating kernel from the same family.

\newcommand{\sqrtdomtablename}{Square-root dominating kernels $\ksqrtdom$ for common kernels $\kernel$}

\newcommand{\sqrtdomtablecaption}{
{\noindent\caption{
    \tbf{\sqrtdomtablename} (see \cref{def:square_root_dom_kernel}). Here $\natural_0$ denotes the non-negative integers.
    See \cref{proof_of_table:sqrt_dom_pair} for our derivation.
    \label{table:sqrt_dom_pair}} 
    }}

\begin{table}[ht]
\centering
\small{
  \resizebox{0.8\textwidth}{!}
  {
    {\renewcommand{\arraystretch}{2}
    \begin{tabular}{cccccc}
    	\toprule
        \Centerstack{\bf Name of kernel\\ $\kernel(\x,\y) = \kappa(\x-\y)$} & \Centerstack
        {\bf Expression for \\ $\kappa(\z)$%
        } 
        & \Centerstack{\bf Name of square-root dominating  \\ {\bf kernel
        $\ksqrtdom(\x,\y) = \kappasqrtdom(\x-\y)$}}
         \\[3mm]
        \midrule 
        \addlinespace[2mm]
         \Centerstack{\bf InverseMultiquadric$(\matone, \mattwo)$ \\
          $\matone>0, \mattwo>0$}
          & $(\mattwo^2+\twonorm{\z}^2)^{-\matone}$
          & \Centerstack{\bf InverseMultiquadric$(\matone', \mattwo')$ \\
         $(\matone', \mattwo') \in \mc C_{\matone,\mattwo,
          d}$; see \cref{eq:cgammaset}}
         \\[6mm]
         \Centerstack{\bf Sech$(\sechparam)$ \\ $\sechparam>0$ }
         & $ \displaystyle \prod_{j=1}^d \textstyle\sech\parenth{\sqrt{
         \frac{\pi}{2}}
         \sechparam \z_{\j}}$
         & \Centerstack{\bf Sech$(2\sechparam)$}
            \\[6mm]
        \Centerstack{{\bf Wendland$(\wendparam)$} \\
          $\wendparam \in \natural,\wendparam\geq \frac12 (d+1)$}
          &  \Centerstack{$\phi_{d, \wendparam}(\twonorm{z})$; see \cref{eq:wendland_kernel}}
          & \Centerstack{\bf Wendland$(\wendparam')$\\
          $\wendparam'\in \natural_0, \wendparam' \leq \quarter(2\wendparam\!-\!1\!-d)$ 
          }
        \\[2ex] \bottomrule \hline
    \end{tabular}
    }
    }
    }
    \opt{arxiv}{\sqrtdomtablecaption}
    \opt{jmlr}{\sqrtdomtablecaption}
\end{table}

\paragraph*{Expressions for Wendland kernels}  The \Wendland kernel is a compactly-supported radial kernel $\phi_{d, \wendparam}(\twonorm{\x-\y})$ on $\Rd$ where $\phi_{d, \wendparam}:\real_{+}\to \real$ is a truncated minimal-degree polynomial with $2\wendparam$ continuous derivatives.
We collect here the expressions for $\phi_{d, \wendparam}$ for $\wendparam=0, 1, 2$ and refer the readers to \citet[Ch.~9, Thm.~9.13, Tab.~9.1]{wendland2004scattered} for more general $\wendparam$. Let $(r)_+=\max(0, r)$ and $\l \defeq \floor{d/2}+3$, then we have
\begin{talign}
\label{eq:wendland_kernel}
\phi_{d, 0}(r) &= (1-r)_+^{\floor{d/2}+1},  \quad
\phi_{d, 1}(r) = (1-r)_+^{\floor{d/2}+2} [(\floor{d/2}+2)r+1] \\
\phi_{d, 2}(r) &= (1-r)_+^{\floor{d/2}+3} [(\l^2+4\l+3)r^2+(3\l+6)r +3].
\end{talign}

\subsection{Proof of \lowercase{\Cref{matern_sqrt_dom}}: \maternsqrtdomname}
\label{sec:proof_of_matern_sqrt_dom}
Since $\kappa \in \lone$, $\fourier(\kappa)$ is bounded by the Babenko-Beckner inequality \citep{beckner1975inequalities} and nonnegative by Bochner's theorem \citep[Thm.~6.6]{bochner1933monotone,wendland2004scattered}.
Moreover, since $\kappa \in C^{2\nu}$, \citet[Thm.~4.1]{sun1993conditionally} implies that $\int \twonorm{\omega}^{2\nu} \fourier(\kappa)(\omega) d\omega < \infty$.
By \citet[Theorem 8.15]{wendland2004scattered},
the \textbf{Mat\'ern}$(\matone, \mattwo)$ kernel $\kdom(\x, \y) \propto \wendmatern(\x-\y)$ for $\wendmatern$ continuous with $\fourier(\wendmatern)(\omega) = {}{(\mattwo^2+\twonorm{\omega}^2)^{-\matone}}$.
Since we have established that
\balignt
\int \frac{\fourier(\kappa)(\omega)^2}{\fourier(\wendmatern)(\omega)} d\omega
    \!=\!
\int (\mattwo^2+\twonorm{\omega}^2)^{\matone} \fourier(\kappa)(\omega)^2 d\omega
    \!\leq \!
\infnorm{\fourier(\kappa)} 
    \int (\mattwo^2+\twonorm{\omega}^2)^{\matone} \fourier(\kappa)(\omega) d\omega
    \!<\! \infty,
\ealignt
\citet[Thm.~10.12]{wendland2004scattered} implies that $\kappa$ belongs to $\rkhs_{\kdom}$ and hence that $\rkhs_{\kernel} \subseteq \rkhs_{\kdom}$.
Finally, by \cref{sec:proof_of_sqrt_kernels}, \textbf{Mat\'ern}$(\frac{\matone}{2}, \mattwo)$ is a valid square-root dominating kernel for $\kdom$ and therefore for~$\kernel$.

\subsection{Proof of \lowercase{\Cref{square_root_dom_kernel}}: \domsqrtresultname}
\label{proof_of_square_root_dom_kernel}
\cref{def:spectral_density} implies that
\balignt
\kernel(x,y) 
    = \frac{1}{(2\pi)^{d/2}}\int e^{-i\inner{\omega}{x-y}} \hatkappa(\omega) d\omega
\qtext{and}
\ksqrtdom(x,y) 
    = \frac{1}{(2\pi)^{d/2}}\int e^{-i\inner{\omega}{x-y}} \hatkappasqrt(\omega) d\omega,
\ealignt
Moreover, since 
$\ksqrt(\x, \cdot) = \fourier(e^{-i\inner{\cdot}{x}} \hatkappasqrt)$ for 
$e^{-i\inner{\cdot}{x}}\hatkappasqrt \in \lone\cap \ltwo$, the Plancherel-Parseval identity \citep[Proof of Thm.~5.23]{wendland2004scattered} implies that
\balignt
\kdom(x,y)
    \defeq
\int_{\Rd}\ksqrtdom(\x, \z)\ksqrtdom(\y, \z) d\z
    &= 
\int_{\Rd} 
    e^{-i\inner{\omega}{\x}} {\hatkappasqrt(\omega)}
    \ e^{i\inner{\omega}{\y}} {\hatkappasqrt(\omega)} d\omega \\
    &= 
\int_{\Rd} 
    e^{-i\inner{\omega}{\x-\y}} \hatkappasqrt(\omega)^2 d\omega,
\ealignt
and Bochner's theorem \citep[Thm.~6.6]{bochner1933monotone,wendland2004scattered} implies that $\kdom$ is a kernel.
Finally, Prop. 3.1 of \citet{zhang2013inclusion} now implies that the RKHS of $\kernel$ belongs to the RKHS of $\kdom$ if and only if the ratio condition \cref{eq:inf_norm_fourier_ratio} holds.

\subsection{Derivation of \cref{table:sqrt_dom_pair}: \sqrtdomtablename}
\label{proof_of_table:sqrt_dom_pair}
Thanks to \cref{square_root_dom_kernel}, to establish the validity of the square root dominating kernels stated in \cref{table:sqrt_dom_pair}, it suffices to verify that
\begin{talign}
\label{eq:fourier_G_relation}
    \hatkappasqrt \in \ltwo
    \qtext{and}
    \hatkappa \precsim_{\,d, \kappa, \kappasqrtdom} 
    \hatkappasqrt^2,
\end{talign}
where the functions $\kappa$ and $\hatkappasqrt$ are the spectral densities of $\kappa$ and $\kappasqrtdom$.

\paragraph*{Inverse multiquadric}
\newcommand{\scalarw}{\twonorm{\omega}}
Consider the positive definite function $\wendmatern$ \cref{eq:def_wendmatern} underlying the \Matern kernel, which is continuous on $\R^d \backslash \{\boldzero\}$.
When $\kappa(\z) = (\mattwo^2+\twonorm{\z}^2)^{-\matone}$ 
and 
$\kappasqrtdom(\z) = ({\mattwo'}^2+\twonorm{\z}^2)^{-\matone'}$, 
\citet[Theorem 8.15]{wendland2004scattered} implies that
\begin{talign}
\hatkappa(\omega) 
    = 
\wendmatern(\omega) 
    \qtext{and}
\Phi_{\matone',\mattwo'}^2(\omega)
    =
\hatkappasqrt(\omega)^2.
\end{talign}
Let $a(\omega) \asymp_{d, \matone,\mattwo} b(\omega)$ denote asymptotic equivalence up to a constant depending on $d, \matone,\mattwo$.
Then, by \citep[Eqs.~10.25.3 \& 10.30.2]{NIST:DLMF},
we have
\begin{talign}
\wendmatern(\omega) &\asymp_{d, \matone,\mattwo} \scalarw^{\matone-\frac{d}2-\half} e^{-\mattwo\scalarw}  \qtext{as} \scalarw \to \infty, \label{eq:matern_infty_bound} \\
\wendmatern(\omega)  &\asymp_{d, \matone,\mattwo} \scalarw^{-(d-2\matone)_+} \qtext{as} \scalarw \to 0.
\label{eq:matern_zero_bound}
\end{talign}
Hence, $\hatkappasqrt \in\ltwo$ whenever $\matone'>\frac d4$.

Moreover, applying \cref{eq:matern_infty_bound}, we find that for $ \scalarw \to \infty $,
\begin{talign}
\frac{\wendmatern(\omega)}{\Phi_{\matone',\mattwo'}^2(\omega)} &\asymp_{d, \matone,\mattwo,\matone', \mattwo'}\scalarw^{\matone-\frac{d}2-\half} e^{-\mattwo\scalarw} 
\cdot \scalarw^{-2\matone'+d+1} e^{2\mattwo'\scalarw} \\
&= \scalarw^{\matone+\frac{d}2+\half-2\matone'} e^{(2\mattwo'-\mattwo)\scalarw}.
\end{talign}
If $2\mattwo'-\mattwo < 0$, this expression is bounded for any value of $\matone'$.
If $2\mattwo'-\mattwo = 0$, this expression is bounded when $\matone' \geq \frac{\matone}{2} + \frac{d}{4} + \frac{1}{4}$.

Applying \cref{eq:matern_zero_bound}, we find that for $ \scalarw \to 0$,
\begin{talign}
\frac{\wendmatern(\omega)}{\Phi_{\matone',\mattwo'}^2(\omega)} &\asymp_{d, \matone,\mattwo,\matone', \mattwo'} \scalarw^{-(d-2\matone)_+}
\cdot \scalarw^{2(d-2\matone')_+} 
= \scalarw^{2(d-2\matone')_+-(d-2\matone)_+},
\end{talign}
If $\matone \geq  \frac{d}2$, this expression is finite for any value of $\matone'$.
If $\matone < \frac{d}2$, this expression is finite when $\matone' \leq \frac{\matone}{2} + \frac{d}{4}$.
Hence, our condition~\cref{eq:fourier_G_relation} is verified whenever $(\matone',\mattwo')$ belongs to the set
\balignt
\label{eq:cgammaset}
\mc C_{\matone,\mattwo, d} 
    \defeq
\big\{ (\matone', \mattwo') :\  
    & (1)\ \matone' > \frac{d}{4} \qtext{and} \mattwo' \leq \frac{\mattwo}{2}, \qtext{and}\\
    & (2)\ \matone' \leq \frac{d}{4} + \frac{\matone}{2} \qtext{if} \matone < \frac{d}{2}, \qtext{and}\\
    & (3)\ \matone' \geq \frac{d}{4} + \frac{\matone}{2} + \quarter \qtext{if} \mattwo' = \frac{\mattwo}{2}
\big\}.
\ealignt

\paragraph*{Sech}
Define $\kappa_{a}(\z) \defeq \prod_{j=1}^d \sech\parenth{\sqrt{\frac{\pi}{2}} \sechparam \z_{\j}}$, and suppose $\kappa = \kappa_{a}$ and $\kappasqrtdom=\kappa_{2a}$.
\citet[Ex.~3.2]{huggins2018random} yields that
\begin{talign}
\label{eq:sech_fourier}
 \hatkappa(\omega) 
 = \fourier(\kappa_a)(\omega) 
 &= \frac{1}{\sechparam^d}\prod_{j=1}^d \sech(\sqrt{\frac{\pi}{2}} \frac{\omega_j}{\sechparam}) = a^{-d} \cdot \kappa_{1/a}(\omega), \qtext{and}\\
 \hatkappasqrt(\omega) 
 = \fourier(\kappasqrtdom)(\omega) &=  (2a)^{-d} \kappa_{1/(2a)}(\omega).
 \label{eq:sech_fourier_sqrt}
\end{talign}
Since $\sech^2(b/2) = \frac{4}{e^b + e^{-b}+2}
	> \frac{4}{e^b + e^{-b}} = 2\sech(b)$, we have
\begin{talign}
\label{eq:sech_bound_kappa}
 \kappa_{1/a}(\omega) \leq 2^{-d} \cdot  \kappa_{1/(2a)}(\omega)^2.
\end{talign}
Putting the pieces together, we further have
\begin{talign}
 \hatkappa(\omega) 
 \seq{\cref{eq:sech_fourier}} a^{-d} \cdot \kappa_{1/a}(\omega)
 \sless{\cref{eq:sech_bound_kappa}} 
 (2a)^{-d} \cdot  \kappa_{1/(2a)}(\omega)^2
 \seq{\cref{eq:sech_fourier_sqrt}} (2a)^d\cdot  \hatkappasqrt(\omega)^2.
\end{talign}
Since $\kappa_{1/(2a)}\in\ltwo$, we have verified the condition~\cref{eq:fourier_G_relation}. 

\paragraph*{Wendland}
When $\kappa(\z) = \phi_{d, \wendparam}(\twonorm{\z})$ 
and 
$\kappasqrtdom(\z) = \phi_{d, \wendparam'}(\twonorm{\z})$, 
\citet[Thm.~10.35]{wendland2004scattered} implies that
\begin{talign}
\hatkappa(\omega) 
    \precsim_{d,\wendparam,\wendparam'}
\frac{1}{(1+\twonorm{\omega}^2)^{d+2\wendparam+1}}
    \leq
\frac{1}{(1+\twonorm{\omega}^2)^{2d+4\wendparam'+2}}
    \precsim_{d, \wendparam,\wendparam'}
\hatkappasqrt(\omega)^2
    \precsim_{d, \wendparam,\wendparam'}
\frac{1}{(1+\twonorm{\omega}^2)^{2d+4\wendparam'+2}},
\end{talign}
with $\wendparam'\in \natural_0, \wendparam' \leq \frac{(2\wendparam\!-\!1\!-d)}{4}$, thereby establishing the condition~\cref{eq:fourier_G_relation}.

\section{Online Vector Balancing in Euclidean Space}
\label{proof_of_euclidean_vector_balancing}
Using \cref{sbhw_properties}, we recover the online vector balancing result of \citet[Thm.~1.2]{alweiss2021discrepancy} with improved constants and a less conservative setting of the thresholds $\cnew$.
Note that we capture the usual obliviousness assumption by treating the sequence $(f_i)_{i=1}^n$ as a fixed, deterministic input to \cref{algo:self_balancing_walk}.
\begin{corollary}[Online vector balancing in Euclidean space]\label{euclidean_vector_balancing}
If $\rkhs = \reals^d$ equipped with the Euclidean dot product, each $\twonorm{\invec} \leq 1$, and each $\cnew = \half + \log(4n/\delta)$, then, 
with probability at least $1-\delta$,  
the self-balancing Hilbert walk (\cref{algo:self_balancing_walk}) returns $\outvec[n]$ satisfying
\begin{talign}
\label{eq:self_balancing_walk}
\infnorm{\outvec[n]}
    =
\infnorm{\sumn \eta_i\invec} 
    \leq 
    \sqrt{2\log(4d/\delta) \log(4n/\delta)}.
\end{talign}
\end{corollary}
\begin{proof}
Instantiate the notation of \cref{sbhw_properties}.
With our choice of $\cnew$, 
the signed sum representation property~\cref{item:signed_sum} implies that $\outvec[n] = \sum_{i=1}^n \eta_i \invec[i]$ with probability at least $1-\delta/2$.
Moreover, the union bound, the functional sub-Gaussianity property~\cref{item:fun_subgauss}, and 
 the sub-Gaussian Hoeffding inequality \citep[Prop.~2.5]{wainwright2019high} now imply
\begin{talign}
\Pr(\infnorm{\outvec[n]} > t \mid \vseq[n])
    \leq 
\sum_{j=1}^d
\P(|\inner{\outvec[n]}{e_j}| > t \mid \vseq[n]) & \leq
 2d\exp(-t^2/(2\sgparam[n]^2)) 
    = \delta/2 \\
    \qtext{for} 
t &\defeq \sgparam[n]\sqrt{2 \log(4d/\delta)}.
\end{talign}
Since $\cnew$ is non-decreasing in $i$,
and $\textfrac{\half +  \log(4/\delta)}{\log(4/\delta)}$ is increasing in $\delta \in (0,1]$,
\begin{talign}
\sgparam[n]^2 
    &= 
        \max_{j\leq  n}\textfrac{\cnew[j]^2}{2\cnew[j] - \hnorm{\invec[j]}^2}
    \leq 
        \max_{j\leq  n}\textfrac{\cnew[j]^2}{2\cnew[j] - 1}
    = 
        \log(4n/\delta)
        \textfrac{(\half + \log(4n/\delta))^2}{2(\log(4n/\delta))^2} \\
    &\leq 
        \log(4n/\delta)
        \textfrac{(\half + \log(4))^2}{2(\log(4))^2}
    \leq 
        \log(4n/\delta).
\end{talign}
The advertised result now follows from the union bound.
\end{proof}
    \section{$\Linf$ Coresets of \lowercase{\citet{phillips2020near}} and  \lowercase{\citet{tai2020new}}}
\label{sub:pt_coresets}
Here we provide more details on the $\Linf$ coreset construction of \citet{phillips2020near} and \citet{tai2020new} discussed in \cref{sub:prior_work_on_linf_coresets}.
Given input points $(\x_i)_{i=1}^n$, the Phillips-Tai (PT) construction forms the matrix $K = (\kernel(x_i,x_j))_{i,j=1}^n$ of pairwise kernel evaluations, finds a matrix square-root $V\in\reals^{n\times n}$ satisfying $K = VV^\top$, and  augments $V$ with a row of ones (to encourage near-halving in the next step). Then, the PT construction runs the Gram-Schmidt (GS) walk of \cite{bansal2018gram} to identify approximately half of the columns of $V$ as coreset members and \emph{rebalances} the coreset until exactly half of the input points belong to the coreset.  
The GS walk and rebalancing steps are recursively repeated $\Omega(\log(n))$ times to obtain an $(n^\half, \order_p(\sqrt{d} n^{-\frac12}\sqrt{\log n}))$-$\Linf$ coreset.
The low-dimensional Gaussian kernel construction of \citet{tai2020new} first partitions the input points into balls of radius $2\sqrt{\log n}$ and then applies the PT construction separately to each ball.
The result is %
an 
$(n^\half, \order_p(2^d n^{-\half}\sqrt{\log(d\log n)}))$-$\Linf$ coreset with an additional superexponential $\Omega(d^{5d})$ running time dependence. 

\end{appendix}
\acks{\acknowledgments}
{\bibliography{refs}}

\begin{thebibliography}{93}
\providecommand{\natexlab}[1]{#1}
\providecommand{\url}[1]{\texttt{#1}}
\expandafter\ifx\csname urlstyle\endcsname\relax
  \providecommand{\doi}[1]{doi: #1}\else
  \providecommand{\doi}{doi: \begingroup \urlstyle{rm}\Url}\fi

\bibitem[Alweiss et~al.(2021)Alweiss, Liu, and Sawhney]{alweiss2021discrepancy}
Ryan Alweiss, Yang~P Liu, and Mehtaab Sawhney.
\newblock Discrepancy minimization via a self-balancing walk.
\newblock In \emph{Proceedings of the 53rd Annual ACM SIGACT Symposium on
  Theory of Computing}, pages 14--20, 2021.

\bibitem[Augustin et~al.(2016)Augustin, Neic, Liebmann, Prassl, Niederer,
  Haase, and Plank]{augustin2016anatomically}
Christoph~M Augustin, Aurel Neic, Manfred Liebmann, Anton~J Prassl, Steven~A
  Niederer, Gundolf Haase, and Gernot Plank.
\newblock Anatomically accurate high resolution modeling of human whole heart
  electromechanics: A strongly scalable algebraic multigrid solver method for
  nonlinear deformation.
\newblock \emph{Journal of computational physics}, 305:\penalty0 622--646,
  2016.

\bibitem[Bach et~al.(2012)Bach, Lacoste-Julien, and
  Obozinski]{bach2012equivalence}
Francis Bach, Simon Lacoste-Julien, and Guillaume Obozinski.
\newblock On the equivalence between herding and conditional gradient
  algorithms.
\newblock In \emph{Proceedings of the 29th International Coference on
  International Conference on Machine Learning}, ICML’12, page 1355–1362,
  Madison, WI, USA, 2012. Omnipress.
\newblock ISBN 9781450312851.

\bibitem[Bansal et~al.(2018)Bansal, Dadush, Garg, and Lovett]{bansal2018gram}
Nikhil Bansal, Daniel Dadush, Shashwat Garg, and Shachar Lovett.
\newblock {The Gram-Schmidt walk: A cure for the Banaszczyk blues}.
\newblock In \emph{Proceedings of the 50th Annual ACM SIGACT Symposium on
  Theory of Computing}, pages 587--597, 2018.

\bibitem[Bardenet and Hardy(2020)]{bardenet2020monte}
R{\'e}mi Bardenet and Adrien Hardy.
\newblock Monte {Carlo} with determinantal point processes.
\newblock \emph{The Annals of Applied Probability}, 30\penalty0 (1):\penalty0
  368--417, 2020.

\bibitem[Batir(2017)]{batir2017bounds}
Necdet Batir.
\newblock Bounds for the gamma function.
\newblock \emph{Results in Mathematics}, 72\penalty0 (1):\penalty0 865--874,
  2017.
\newblock \doi{10.1007/s00025-017-0698-0}.
\newblock URL \url{https://doi.org/10.1007/s00025-017-0698-0}.

\bibitem[Beckner(1975)]{beckner1975inequalities}
William Beckner.
\newblock Inequalities in {Fourier} analysis.
\newblock \emph{Annals of Mathematics}, pages 159--182, 1975.

\bibitem[Belhadji et~al.(2019)Belhadji, Bardenet, and
  Chainais]{belhadji2019kernel}
Ayoub Belhadji, R\'{e}mi Bardenet, and Pierre Chainais.
\newblock Kernel quadrature with {DPP}s.
\newblock In \emph{Advances in Neural Information Processing Systems},
  volume~32, pages 12927--12937. Curran Associates, Inc., 2019.

\bibitem[Belhadji et~al.(2020)Belhadji, Bardenet, and
  Chainais]{belhadji2020kernel}
Ayoub Belhadji, R{\'e}mi Bardenet, and Pierre Chainais.
\newblock Kernel interpolation with continuous volume sampling.
\newblock In \emph{International Conference on Machine Learning}, pages
  725--735. PMLR, 2020.

\bibitem[Bochner(1933)]{bochner1933monotone}
Salomon Bochner.
\newblock Monotone funktionen, stieltjessche integrale und harmonische analyse.
\newblock \emph{Mathematische Annalen}, 108\penalty0 (1):\penalty0 378--410,
  1933.

\bibitem[Borodachov et~al.(2014)Borodachov, Hardin, and
  Saff]{borodachov2014low}
Sergiy~V Borodachov, Douglas~P Hardin, and Edward~B Saff.
\newblock Low complexity methods for discretizing manifolds via {R}iesz energy
  minimization.
\newblock \emph{Foundations of Computational Mathematics}, 14\penalty0
  (6):\penalty0 1173--1208, 2014.

\bibitem[Borwein and Chan(2009)]{borwein2009uniform}
Jonathan~M Borwein and O-Yeat Chan.
\newblock Uniform bounds for the complementary incomplete gamma function.
\newblock \emph{Mathematical Inequalities and Applications}, 12:\penalty0
  115--121, 2009.

\bibitem[Briol et~al.(2015)Briol, Oates, Girolami, and Osborne]{briol2015frank}
Fran\c{c}ois-Xavier Briol, Chris~J. Oates, Mark Girolami, and Michael~A.
  Osborne.
\newblock Frank-{W}olfe {B}ayesian {Q}uadrature: Probabilistic integration with
  theoretical guarantees.
\newblock In \emph{Advances in Neural Information Processing Systems}, pages
  1162--1170, 2015.

\bibitem[Brooks et~al.(2011)Brooks, Gelman, Jones, and
  Meng]{brooks2011handbook}
Steve Brooks, Andrew Gelman, Galin Jones, and Xiao-Li Meng.
\newblock \emph{{Handbook of Markov chain Monte Carlo}}.
\newblock CRC press, 2011.

\bibitem[Buldygin and Kozachenko(1980)]{buldygin1980subgaussian}
V.~V. Buldygin and Yu.~V. Kozachenko.
\newblock Sub-gaussian random variables.
\newblock \emph{Ukrainian Mathematical Journal}, 32\penalty0 (6):\penalty0
  483--489, 1980.
\newblock \doi{10.1007/BF01087176}.
\newblock URL \url{https://doi.org/10.1007/BF01087176}.

\bibitem[Campbell and Broderick(2019)]{campbell2019automated}
Trevor Campbell and Tamara Broderick.
\newblock Automated scalable {B}ayesian inference via {H}ilbert coresets.
\newblock \emph{The Journal of Machine Learning Research}, 20\penalty0
  (1):\penalty0 551--588, 2019.

\bibitem[Chen and Skriganov(2002)]{chen2002explicit}
W~WL Chen and MM~Skriganov.
\newblock Explicit constructions in the classical mean squares problem in
  irregularities of point distribution.
\newblock \emph{Journal fur die Reine und Angewandte Mathematik}, \penalty0
  (545):\penalty0 67--95, 2002.

\bibitem[Chen et~al.(2018)Chen, Mackey, Gorham, Briol, and
  Oates]{Chen2018SteinPoints}
W.~Y. Chen, L.~Mackey, J.~Gorham, F-X. Briol, and C.~J. Oates.
\newblock {Stein points}.
\newblock In \emph{Proceedings of the 35th International Conference on Machine
  Learning}, 2018.

\bibitem[Chen et~al.(2019)Chen, Barp, Briol, Gorham, Girolami, Mackey, and
  Oates]{chen2019stein}
Wilson~Ye Chen, Alessandro Barp, Fran{\c{c}}ois-Xavier Briol, Jackson Gorham,
  Mark Girolami, Lester Mackey, and Chris Oates.
\newblock Stein point {Markov chain Monte Carlo}.
\newblock In \emph{International Conference on Machine Learning}, pages
  1011--1021. PMLR, 2019.

\bibitem[Chen et~al.(2010)Chen, Welling, and Smola]{chen2012super}
Yutian Chen, Max Welling, and Alex Smola.
\newblock Super-samples from kernel herding.
\newblock In \emph{Proceedings of the Twenty-Sixth Conference on Uncertainty in
  Artificial Intelligence}, UAI’10, page 109–116, Arlington, Virginia, USA,
  2010. AUAI Press.
\newblock ISBN 9780974903965.

\bibitem[De~Marchi et~al.(2005)De~Marchi, Schaback, and Wendland]{de2005near}
Stefano De~Marchi, Robert Schaback, and Holger Wendland.
\newblock Near-optimal data-independent point locations for radial basis
  function interpolation.
\newblock \emph{Advances in Computational Mathematics}, 23\penalty0
  (3):\penalty0 317--330, 2005.

\bibitem[{\relax DLMF}()]{NIST:DLMF}
{\relax DLMF}.
\newblock {\it NIST Digital Library of Mathematical Functions}.
\newblock http://dlmf.nist.gov/, Release 1.1.1 of 2021-03-15.
\newblock F.~W.~J. Olver, A.~B. {Olde Daalhuis}, D.~W. Lozier, B.~I. Schneider,
  R.~F. Boisvert, C.~W. Clark, B.~R. Miller, B.~V. Saunders, H.~S. Cohl, and
  M.~A. McClain, eds.

\bibitem[Douc et~al.(2018)Douc, Moulines, Priouret, and
  Soulier]{douc2018markov}
Randal Douc, Eric Moulines, Pierre Priouret, and Philippe Soulier.
\newblock \emph{Markov chains}.
\newblock Springer, 2018.

\bibitem[Durrett(2019)]{durrett2019probability}
Rick Durrett.
\newblock \emph{{Probability: Theory and Examples}}.
\newblock Cambridge Series in Statistical and Probabilistic Mathematics.
  Cambridge University Press, 5 edition, 2019.
\newblock \doi{10.1017/9781108591034}.

\bibitem[Dwivedi and Mackey(2022)]{dwivedi2022generalized}
Raaz Dwivedi and Lester Mackey.
\newblock Generalized kernel thinning.
\newblock In \emph{International Conference on Learning Representations}, 2022.

\bibitem[Dwivedi et~al.(2019)Dwivedi, Feldheim, Gurel-Gurevich, and
  Ramdas]{dwivedi2019power}
Raaz Dwivedi, Ohad~N Feldheim, Ori Gurel-Gurevich, and Aaditya Ramdas.
\newblock The power of online thinning in reducing discrepancy.
\newblock \emph{Probability Theory and Related Fields}, 174\penalty0
  (1):\penalty0 103--131, 2019.

\bibitem[El~Karoui(2010)]{el2010spectrum}
Noureddine El~Karoui.
\newblock The spectrum of kernel random matrices.
\newblock \emph{The Annals of Statistics}, 38\penalty0 (1):\penalty0 1--50,
  2010.

\bibitem[Gallegos-Herrada et~al.(2023)Gallegos-Herrada, Ledvinka, and
  Rosenthal]{gallegosherrada2023equivalences}
M.~A. Gallegos-Herrada, D.~Ledvinka, and J.~S. Rosenthal.
\newblock Equivalences of geometric ergodicity of markov chains, 2023.

\bibitem[Garreau et~al.(2017)Garreau, Jitkrittum, and
  Kanagawa]{garreau2017large}
Damien Garreau, Wittawat Jitkrittum, and Motonobu Kanagawa.
\newblock Large sample analysis of the median heuristic.
\newblock \emph{arXiv preprint arXiv:1707.07269}, 2017.

\bibitem[Girolami and Calderhead(2011)]{girolami2011riemann}
Mark Girolami and Ben Calderhead.
\newblock Riemann manifold {Langevin and Hamiltonian Monte Carlo} methods.
\newblock \emph{Journal of the Royal Statistical Society: Series B (Statistical
  Methodology)}, 73\penalty0 (2):\penalty0 123--214, 2011.

\bibitem[Goodwin(1965)]{goodwin1965oscillatory}
Brian~C Goodwin.
\newblock Oscillatory behavior in enzymatic control process.
\newblock \emph{Advances in Enzyme Regulation}, 3:\penalty0 318--356, 1965.

\bibitem[Graham et~al.(1994)Graham, Knuth, and Patashnik]{graham94}
Ronald~L. Graham, Donald~Ervin Knuth, and Oren Patashnik.
\newblock \emph{Concrete Mathematics: A Foundation for Computer Science}.
\newblock Addison-Wesley, Reading, MA, {Second} edition, 1994.
\newblock URL
  \url{https://www.csie.ntu.edu.tw/\~r97002/temp/Concrete\%20Mathematics\%202e.pdf}.

\bibitem[Gretton et~al.(2012)Gretton, Borgwardt, Rasch, Sch{{\"o}}lkopf, and
  Smola]{JMLR:v13:gretton12a}
Arthur Gretton, Karsten~M. Borgwardt, Malte~J. Rasch, Bernhard Sch{{\"o}}lkopf,
  and Alexander Smola.
\newblock A kernel two-sample test.
\newblock \emph{Journal of Machine Learning Research}, 13\penalty0
  (25):\penalty0 723--773, 2012.

\bibitem[Gurumoorthy et~al.(2019)Gurumoorthy, Dhurandhar, Cecchi, and
  Aggarwal]{gurumoorthy2019efficient}
Karthik~S Gurumoorthy, Amit Dhurandhar, Guillermo Cecchi, and Charu Aggarwal.
\newblock Efficient data representation by selecting prototypes with importance
  weights.
\newblock In \emph{2019 IEEE International Conference on Data Mining (ICDM)},
  pages 260--269. IEEE, 2019.

\bibitem[Haario et~al.(1999)Haario, Saksman, and Tamminen]{haario1999adaptive}
Heikki Haario, Eero Saksman, and Johanna Tamminen.
\newblock Adaptive proposal distribution for random walk {Metropolis}
  algorithm.
\newblock \emph{Computational Statistics}, 14\penalty0 (3):\penalty0 375--395,
  1999.

\bibitem[Harvey and Samadi(2014)]{harvey2014near}
Nick Harvey and Samira Samadi.
\newblock Near-optimal herding.
\newblock In \emph{Conference on Learning Theory}, pages 1165--1182, 2014.

\bibitem[Havet et~al.(2020)Havet, Lerasle, Moulines, and
  Vernet]{havet2020quantitative}
Antoine Havet, Matthieu Lerasle, Eric Moulines, and Elodie Vernet.
\newblock {A quantitative McDiarmid’s inequality for geometrically ergodic
  Markov chains}.
\newblock \emph{Electronic Communications in Probability}, 25\penalty0
  (none):\penalty0 1 -- 11, 2020.
\newblock \doi{10.1214/20-ECP286}.
\newblock URL \url{https://doi.org/10.1214/20-ECP286}.

\bibitem[Hickernell(1998)]{hickernell1998generalized}
Fred Hickernell.
\newblock A generalized discrepancy and quadrature error bound.
\newblock \emph{Mathematics of computation}, 67\penalty0 (221):\penalty0
  299--322, 1998.

\bibitem[Hinch et~al.(2004)Hinch, Greenstein, Tanskanen, Xu, and
  Winslow]{hinch2004simplified}
Robert Hinch, JL~Greenstein, AJ~Tanskanen, L~Xu, and RL~Winslow.
\newblock A simplified local control model of calcium-induced calcium release
  in cardiac ventricular myocytes.
\newblock \emph{Biophysical journal}, 87\penalty0 (6):\penalty0 3723--3736,
  2004.

\bibitem[Hoeffding(1963)]{hoeffding1963probability}
Wassily Hoeffding.
\newblock Probability inequalities for sums of bounded random variables.
\newblock \emph{Journal of the American Statistical Association}, 58\penalty0
  (301):\penalty0 13--30, 1963.

\bibitem[Huggins and Mackey(2018)]{huggins2018random}
Jonathan Huggins and Lester Mackey.
\newblock Random feature {Stein} discrepancies.
\newblock In \emph{Advances in Neural Information Processing Systems}, pages
  1899--1909, 2018.

\bibitem[Husz{\'a}r and Duvenaud(2012)]{Huszr2012OptimallyWeightedHI}
Ferenc Husz{\'a}r and David~Kristjanson Duvenaud.
\newblock Optimally-weighted herding is bayesian quadrature.
\newblock \emph{ArXiv}, abs/1408.2049, 2012.

\bibitem[Joseph et~al.(2015)Joseph, Dasgupta, Tuo, and
  Wu]{joseph2015sequential}
V~Roshan Joseph, Tirthankar Dasgupta, Rui Tuo, and CF~Jeff Wu.
\newblock Sequential exploration of complex surfaces using minimum energy
  designs.
\newblock \emph{Technometrics}, 57\penalty0 (1):\penalty0 64--74, 2015.

\bibitem[Joseph et~al.(2019)Joseph, Wang, Gu, Lyu, and
  Tuo]{joseph2019deterministic}
V~Roshan Joseph, Dianpeng Wang, Li~Gu, Shiji Lyu, and Rui Tuo.
\newblock Deterministic sampling of expensive posteriors using minimum energy
  designs.
\newblock \emph{Technometrics}, 2019.

\bibitem[Joshi et~al.(2011)Joshi, Kommaraji, Phillips, and
  Venkatasubramanian]{joshi2011comparing}
Sarang Joshi, Raj~Varma Kommaraji, Jeff~M Phillips, and Suresh
  Venkatasubramanian.
\newblock Comparing distributions and shapes using the kernel distance.
\newblock In \emph{Proceedings of the twenty-seventh annual symposium on
  Computational geometry}, pages 47--56, 2011.

\bibitem[Karnin and Liberty(2019)]{karnin2019discrepancy}
Zohar Karnin and Edo Liberty.
\newblock Discrepancy, coresets, and sketches in machine learning.
\newblock In \emph{Conference on Learning Theory}, pages 1975--1993. PMLR,
  2019.

\bibitem[Karvonen et~al.(2018)Karvonen, Oates, and Sarkka]{karvonen2018bayes}
Toni Karvonen, Chris~J Oates, and Simo Sarkka.
\newblock A {B}ayes-{S}ard cubature method.
\newblock In \emph{Advances in Neural Information Processing Systems}, pages
  5882--5893, 2018.

\bibitem[Karvonen et~al.(2019)Karvonen, Kanagawa, and
  S{\"a}rkk{\"a}]{karvonen2019positivity}
Toni Karvonen, Motonobu Kanagawa, and Simo S{\"a}rkk{\"a}.
\newblock On the positivity and magnitudes of {B}ayesian quadrature weights.
\newblock \emph{Statistics and Computing}, 29\penalty0 (6):\penalty0
  1317--1333, 2019.

\bibitem[Khanna and Mahoney(2019)]{khanna2019linear}
Rajiv Khanna and Michael~W Mahoney.
\newblock On linear convergence of weighted kernel herding.
\newblock \emph{arXiv preprint arXiv:1907.08410}, 2019.

\bibitem[Kim et~al.(2016)Kim, Khanna, and Koyejo]{kim2016examples}
Been Kim, Rajiv Khanna, and Oluwasanmi~O Koyejo.
\newblock Examples are not enough, learn to criticize! {C}riticism for
  interpretability.
\newblock \emph{Advances in neural information processing systems}, 29, 2016.

\bibitem[Koltchinskii and Gin{\'e}(2000)]{koltchinskii2000random}
Vladimir Koltchinskii and Evarist Gin{\'e}.
\newblock Random matrix approximation of spectra of integral operators.
\newblock \emph{Bernoulli}, pages 113--167, 2000.

\bibitem[Lacoste-Julien et~al.(2015)Lacoste-Julien, Lindsten, and
  Bach]{lacoste2015sequential}
Simon Lacoste-Julien, Fredrik Lindsten, and Francis Bach.
\newblock Sequential kernel herding: {Frank-Wolfe} optimization for particle
  filtering.
\newblock In \emph{Artificial Intelligence and Statistics}, pages 544--552.
  PMLR, 2015.

\bibitem[Laurent and Massart(2000)]{laurent2000adaptive}
Beatrice Laurent and Pascal Massart.
\newblock Adaptive estimation of a quadratic functional by model selection.
\newblock \emph{Annals of Statistics}, pages 1302--1338, 2000.

\bibitem[Liu and Lee(2017)]{liu2016black}
Qiang Liu and Jason Lee.
\newblock Black-box importance sampling.
\newblock In \emph{Artificial Intelligence and Statistics}, pages 952--961.
  PMLR, 2017.

\bibitem[Lotka(1925)]{lotka1925elements}
Alfred~James Lotka.
\newblock \emph{Elements of physical biology}.
\newblock Williams \& Wilkins, 1925.

\bibitem[Mak and Joseph(2018)]{mak2018support}
Simon Mak and V~Roshan Joseph.
\newblock Support points.
\newblock \emph{The Annals of Statistics}, 46\penalty0 (6A):\penalty0
  2562--2592, 2018.

\bibitem[Meyn and Tweedie(2012)]{meyn2012markov}
Sean~P Meyn and Richard~L Tweedie.
\newblock \emph{Markov chains and stochastic stability}.
\newblock Springer Science \& Business Media, 2012.

\bibitem[Minh(2010)]{minh2010some}
Ha~Quang Minh.
\newblock Some properties of gaussian reproducing kernel hilbert spaces and
  their implications for function approximation and learning theory.
\newblock \emph{Constructive Approximation}, 32\penalty0 (2):\penalty0
  307--338, 2010.

\bibitem[Mukherjea(1972)]{mukherjea1972remark}
Arunava Mukherjea.
\newblock A remark on tonelli’s theorem on integration in product spaces.
\newblock \emph{Pacific Journal of Mathematics}, 42\penalty0 (1):\penalty0
  177--185, 1972.

\bibitem[Niederer et~al.(2011)Niederer, Mitchell, Smith, and
  Plank]{niederer2011simulating}
Steven~A Niederer, Lawrence Mitchell, Nicolas Smith, and Gernot Plank.
\newblock Simulating human cardiac electrophysiology on clinical time-scales.
\newblock \emph{Frontiers in Physiology}, 2:\penalty0 14, 2011.

\bibitem[Novak and Wozniakowski(2010)]{novak2010tractability}
E~Novak and H~Wozniakowski.
\newblock Tractability of multivariate problems, volume ii: Standard
  information for functionals, european math.
\newblock \emph{Soc. Publ. House, Z{\"u}rich}, 3, 2010.

\bibitem[Oates et~al.(2017)Oates, Girolami, and Chopin]{oates2017control}
Chris~J Oates, Mark Girolami, and Nicolas Chopin.
\newblock Control functionals for {Monte Carlo} integration.
\newblock \emph{Journal of the Royal Statistical Society: Series B (Statistical
  Methodology)}, 79\penalty0 (3):\penalty0 695--718, 2017.

\bibitem[Oates et~al.(2019)Oates, Cockayne, Briol, and
  Girolami]{oates2019convergence}
Chris~J Oates, Jon Cockayne, Fran{\c{c}}ois-Xavier Briol, and Mark Girolami.
\newblock Convergence rates for a class of estimators based on {S}tein's
  method.
\newblock \emph{Bernoulli}, 25\penalty0 (2):\penalty0 1141--1159, 2019.

\bibitem[O'Hagan(1991)]{o1991bayes}
Anthony O'Hagan.
\newblock Bayes--{H}ermite quadrature.
\newblock \emph{Journal of statistical planning and inference}, 29\penalty0
  (3):\penalty0 245--260, 1991.

\bibitem[Owen(2017)]{owen2017statistically}
Art~B Owen.
\newblock Statistically efficient thinning of a {Markov} chain sampler.
\newblock \emph{Journal of Computational and Graphical Statistics}, 26\penalty0
  (3):\penalty0 738--744, 2017.

\bibitem[Paige et~al.(2016)Paige, Sejdinovic, and Wood]{paige2016super}
Brooks Paige, Dino Sejdinovic, and Frank Wood.
\newblock Super-sampling with a reservoir.
\newblock In \emph{Proceedings of the Thirty-Second Conference on Uncertainty
  in Artificial Intelligence}, pages 567--576, 2016.

\bibitem[Phillips(2013)]{phillips2013varepsilon}
Jeff~M Phillips.
\newblock $\varepsilon$-samples for kernels.
\newblock In \emph{Proceedings of the twenty-fourth annual ACM-SIAM symposium
  on Discrete algorithms}, pages 1622--1632. SIAM, 2013.

\bibitem[Phillips and Tai(2018)]{phillips2018improved}
Jeff~M Phillips and Wai~Ming Tai.
\newblock Improved coresets for kernel density estimates.
\newblock In \emph{Proceedings of the Twenty-Ninth Annual ACM-SIAM Symposium on
  Discrete Algorithms}, pages 2718--2727. SIAM, 2018.

\bibitem[Phillips and Tai(2020)]{phillips2020near}
Jeff~M Phillips and Wai~Ming Tai.
\newblock Near-optimal coresets of kernel density estimates.
\newblock \emph{Discrete \& Computational Geometry}, 63\penalty0 (4):\penalty0
  867--887, 2020.

\bibitem[Rezaei and Gharan(2019)]{rezaei2019polynomial}
Alireza Rezaei and Shayan~Oveis Gharan.
\newblock A polynomial time mcmc method for sampling from continuous
  determinantal point processes.
\newblock In \emph{International Conference on Machine Learning}, pages
  5438--5447. PMLR, 2019.

\bibitem[Riabiz et~al.(2020)Riabiz, Chen, Cockayne, Swietach, Niederer, Mackey,
  and Oates]{DVN/MDKNWM_2020}
Marina Riabiz, Wilson~Ye Chen, Jon Cockayne, Pawel Swietach, Steven~A.
  Niederer, Lester Mackey, and Chris~J. Oates.
\newblock {Replication Data for: Optimal Thinning of MCMC Output}, 2020.
\newblock URL \url{https://doi.org/10.7910/DVN/MDKNWM}.
\newblock Accessed on Mar 23, 2021.

\bibitem[Riabiz et~al.(2021)Riabiz, Chen, Cockayne, Swietach, Niederer, Mackey,
  and Oates]{riabiz2021optimal}
Marina Riabiz, Wilson Chen, Jon Cockayne, Pawel Swietach, Steven~A Niederer,
  Lester Mackey, and Chris Oates.
\newblock Optimal thinning of {MCMC} output.
\newblock \emph{To appear: Journal of the Royal Statistical Society: Series B
  (Statistical Methodology)}, 2021.

\bibitem[Roberts and Rosenthal(2004)]{roberts2004general}
Gareth~O Roberts and Jeffrey~S Rosenthal.
\newblock General state space {Markov chains and MCMC algorithms}.
\newblock \emph{Probability surveys}, 1:\penalty0 20--71, 2004.

\bibitem[Roberts and Tweedie(1996)]{roberts1996exponential}
Gareth~O Roberts and Richard~L Tweedie.
\newblock Exponential convergence of {Langevin} distributions and their
  discrete approximations.
\newblock \emph{Bernoulli}, 2\penalty0 (4):\penalty0 341--363, 1996.

\bibitem[Saitoh(1999)]{saitoh1999applications}
Saburou Saitoh.
\newblock Applications of the general theory of reproducing kernels.
\newblock In \emph{Reproducing Kernels and Their Applications}, pages 165--188.
  Springer, 1999.

\bibitem[Santin and Haasdonk(2017)]{santin2017convergence}
Gabriele Santin and Bernard Haasdonk.
\newblock Convergence rate of the data-independent p-greedy algorithm in
  kernel-based approximation.
\newblock \emph{Dolomites Research Notes on Approximation}, 10\penalty0
  (Special\_Issue), 2017.

\bibitem[Schumaker(2007)]{schumaker2007spline}
Larry Schumaker.
\newblock \emph{Spline functions: basic theory}.
\newblock Cambridge University Press, 2007.

\bibitem[Shetty et~al.(2022)Shetty, Dwivedi, and
  Mackey]{shetty2022distribution}
Abhishek Shetty, Raaz Dwivedi, and Lester Mackey.
\newblock Distribution compression in near-linear time.
\newblock In \emph{International Conference on Learning Representations}, 2022.

\bibitem[Spencer(1977)]{spencer1977balancing}
Joel Spencer.
\newblock Balancing games.
\newblock \emph{Journal of Combinatorial Theory, Series B}, 23\penalty0
  (1):\penalty0 68--74, 1977.

\bibitem[Sriperumbudur et~al.(2010)Sriperumbudur, Gretton, Fukumizu,
  Sch{\"o}lkopf, and Lanckriet]{sriperumbudur2010hilbert}
Bharath~K Sriperumbudur, Arthur Gretton, Kenji Fukumizu, Bernhard
  Sch{\"o}lkopf, and Gert~RG Lanckriet.
\newblock Hilbert space embeddings and metrics on probability measures.
\newblock \emph{Journal of Machine Learning Research}, 11\penalty0
  (Apr):\penalty0 1517--1561, 2010.

\bibitem[Steinwart and Christmann(2008)]{steinwart2008support}
Ingo Steinwart and Andreas Christmann.
\newblock \emph{Support vector machines}.
\newblock Springer Science \& Business Media, 2008.

\bibitem[Steinwart and Scovel(2012)]{steinwart2012mercer}
Ingo Steinwart and Clint Scovel.
\newblock Mercer’s theorem on general domains: On the interaction between
  measures, kernels, and {RKHSs}.
\newblock \emph{Constructive Approximation}, 35\penalty0 (3):\penalty0
  363--417, 2012.

\bibitem[Strocchi et~al.(2020)Strocchi, Gsell, Augustin, Razeghi, Roney,
  Prassl, Vigmond, Behar, Gould, Rinaldi, Bishop, Plank, and
  Niederer]{strocchi2020simulating}
Marina Strocchi, Matthias~AF Gsell, Christoph~M Augustin, Orod Razeghi,
  Caroline~H Roney, Anton~J Prassl, Edward~J Vigmond, Jonathan~M Behar,
  Justin~S Gould, Christopher~A Rinaldi, Martin~J Bishop, Gernot Plank, and
  Steven~A Niederer.
\newblock Simulating ventricular systolic motion in a four-chamber heart model
  with spatially varying robin boundary conditions to model the effect of the
  pericardium.
\newblock \emph{Journal of Biomechanics}, 101:\penalty0 109645, 2020.

\bibitem[Sun(1993)]{sun1993conditionally}
Xingping Sun.
\newblock Conditionally positive definite functions and their application to
  multivariate interpolations.
\newblock \emph{Journal of approximation theory}, 74\penalty0 (2):\penalty0
  159--180, 1993.

\bibitem[Tai(2020)]{tai2020new}
Wai~Ming Tai.
\newblock New nearly-optimal coreset for kernel density estimation.
\newblock \emph{arXiv preprint arXiv:2007.08031}, 2020.

\bibitem[Tolstikhin et~al.(2017)Tolstikhin, Sriperumbudur, and
  Muandet]{tolstikhin2017minimax}
Ilya Tolstikhin, Bharath~K Sriperumbudur, and Krikamol Muandet.
\newblock Minimax estimation of kernel mean embeddings.
\newblock \emph{The Journal of Machine Learning Research}, 18\penalty0
  (1):\penalty0 3002--3048, 2017.

\bibitem[Turner et~al.(2021)Turner, Liu, and Rigollet]{turner2021statistical}
Paxton Turner, Jingbo Liu, and Philippe Rigollet.
\newblock A statistical perspective on coreset density estimation.
\newblock In \emph{International Conference on Artificial Intelligence and
  Statistics}, pages 2512--2520. PMLR, 2021.

\bibitem[{Virtanen} et~al.(2020){Virtanen}, {Gommers}, {Oliphant}, {Haberland},
  {Reddy}, et~al.]{2020SciPy-NMeth}
Pauli {Virtanen}, Ralf {Gommers}, Travis~E. {Oliphant}, Matt {Haberland}, Tyler
  {Reddy}, and SciPy 1.~0 others.
\newblock {SciPy 1.0: Fundamental Algorithms for Scientific Computing in
  Python}.
\newblock \emph{Nature Methods}, 2020.
\newblock \doi{https://doi.org/10.1038/s41592-019-0686-2}.

\bibitem[Volterra(1926)]{volterra1926variazioni}
Vito Volterra.
\newblock Variazioni e fluttuazioni del numero d'individui in specie animali
  conviventi.
\newblock 1926.

\bibitem[Wainwright(2019)]{wainwright2019high}
Martin~J Wainwright.
\newblock \emph{High-dimensional statistics: A non-asymptotic viewpoint},
  volume~48.
\newblock Cambridge University Press, 2019.

\bibitem[Wendland(2004)]{wendland2004scattered}
Holger Wendland.
\newblock \emph{Scattered data approximation}, volume~17.
\newblock Cambridge university press, 2004.

\bibitem[Zhang and Zhao(2013)]{zhang2013inclusion}
Haizhang Zhang and Liang Zhao.
\newblock On the inclusion relation of reproducing kernel hilbert spaces.
\newblock \emph{Analysis and Applications}, 11\penalty0 (02):\penalty0 1350014,
  2013.

\bibitem[Zhou(2002)]{zhou2002covering}
Ding-Xuan Zhou.
\newblock The covering number in learning theory.
\newblock \emph{Journal of Complexity}, 18\penalty0 (3):\penalty0 739--767,
  2002.

\end{thebibliography}
\end{document}